%% file: floral.tex
\documentclass{article} %
\usepackage{iclr2025_conference,times}

\input{math_commands.tex}

\input{config/packages}

\input{config/definitions}

\title{Adversarial Training for Defense Against \\ Label Poisoning Attacks}

\author{Melis Ilayda~Bal\textsuperscript{\textmd{1}}\thanks{Corrresponding author. Code is available at \url{https://github.com/melisilaydabal/floral}. }  , Volkan Cevher\textsuperscript{\textmd{2,3}}, Michael Muehlebach\textsuperscript{\textmd{1}}\\
    \textsuperscript{\textmd{1}}Max Planck Institute for Intelligent Systems, T\"{u}bingen, Germany \\
    \textsuperscript{\textmd{2}}LIONS, EPFL \quad
    \textsuperscript{\textmd{3}}AGI Foundations, Amazon\\ 
    \texttt{\{mbal, michaelm\}@tuebingen.mpg.de,} \\
    \texttt{volkan.cevher@epfl.ch, volkcevh@amazon.de} \\
}

\iclrfinalcopy %
\begin{document}
\doparttoc %
\faketableofcontents %

\maketitle

\begin{abstract}
As machine learning models grow in complexity and increasingly rely on publicly sourced data, such as the human-annotated labels used in training large language models, they become more vulnerable to label poisoning attacks.
These attacks, in which adversaries subtly alter the labels within a training dataset, can severely degrade model performance, posing significant risks in critical applications.
In this paper, we propose \textsc{Floral}, a novel adversarial training defense strategy based on support vector machines (SVMs) to counter these threats. 
Utilizing a bilevel optimization framework, we cast the training process as a non-zero-sum Stackelberg game between an \textit{attacker}, who strategically poisons critical training labels, and the \textit{model}, which seeks to recover from such attacks. 
Our approach accommodates various model architectures and employs a projected gradient descent algorithm with kernel SVMs for adversarial training. 
We provide a theoretical analysis of our algorithm’s convergence properties and empirically evaluate \textsc{Floral}'s effectiveness across diverse classification tasks.
Compared to robust baselines and foundation models such as RoBERTa, \textsc{Floral} consistently achieves higher robust accuracy under increasing attacker budgets.
These results underscore the potential of \textsc{Floral} to enhance the resilience of machine learning models against label poisoning threats, 
thereby ensuring robust classification in adversarial settings.
\looseness -1
\end{abstract}

\section{Introduction}
\label{sec:intro}
\looseness -1 
The susceptibility of machine learning models to the integrity of their training data is a growing concern, particularly as these models become more complex and reliant on large volumes of publicly sourced data, such as the human-annotated labels used in training large language models \citep{kumar2020adversarial, cheng2020learning, wang-etal-2023-noise}.
Any compromise in training data can severely undermine a model’s performance and reliability \citep{adversarial-classification, szegedy2013intriguing}--- leading to catastrophic outcomes in security-critical applications, such as fraud detection \citep{fraud-detection}, medical diagnosis \citep{adv-medical-ml}, and autonomous driving \citep{adv-autonomous-driving}. \looseness -1

One of the most insidious forms of threat is the \textit{data poisoning} (\textit{causative}) attacks \citep{security-ml}, where adversaries subtly manipulate a subset of the training data, causing the model to learn erroneous input-output associations.
These attacks can involve either feature or label perturbations. 
Unlike feature poisoning, which alters the input data itself, (\textit{triggerless}) label poisoning is particularly challenging to detect because only the labels are modified, leaving the input data unchanged, as illustrated in Figure~\ref{fig:label-attacks-sketch}.
Deep learning models are inherently vulnerable to random label noise \citep{zhang2021understanding}, and this susceptibility is magnified when the noise is adversarially crafted to be more damaging.
Figure~\ref{fig:roberta-motivation} illustrates this vulnerability: The RoBERTa model \citep{roberta} fine-tuned for sentiment analysis suffers substantial performance degradation under label poisoning attacks \citep{bert-noise}, with severity growing as the attacker’s budget increases.
In contrast, Figure~\ref{fig:roberta-motivation-floral} highlights \textsc{Floral}’s effectiveness in mitigating these attacks.
Here, the adversarially labelled dataset is generated by poisoning the labels of the most influential training points 
(see Appendix~\ref{app:sec:roberta-exp-details} for details). \looseness -1

A line of work has addressed label poisoning through designing triggerless attacks against SVMs \citep{biggio2012poisoning, adversarial-flip-svm, xiao2015support}, backdoor attacks in vision contexts \citep{chen2022clean, label-poisoning} or combining label poisoning with adversarial attacks \citep{fowl2021adversarial, geiping2021doesn}. 
Defense mechanisms 
typically focus on filtering (sanitization) techniques \citep{curie, label-sanitization}, kernel correction \citep{svm-adversarial-noise}, intrinsic dimensionality-based sample weighting \citep{defending-svms} or robust learning \citep{steinhardt2017certified}. 
Adversarial training (AT) \citep{goodfellow-2014, madry2017towards} is a widely adopted empirical defense against data poisoning—particularly for feature perturbations—framing the interaction as a zero-sum game and training models on adversarially perturbed data \citep{huang2015learning, kurakin2018adversarial}. 
However, as shown in our experiments (Section~\ref{sec:experiment-results}), conventional AT does not adequately defend against label poisoning attacks, and its direct application to label poisoning remains largely unexplored.
\looseness -1

In this paper, we address robust classification under label poisoning attacks and introduce \textsc{Floral} (\textbf{F}lipping \textbf{L}abels f\textbf{or} \textbf{A}dversarial \textbf{L}earning), an SVM-based adversarial training defense that can be seamlessly adapted to other model architectures.
We formulate our defense strategy as a bilevel optimization problem \citep{ at-nonzero-game}, enabling a computationally \textit{efficient} generation of optimal label attacks, and forming a non-zero-sum Stackelberg game between an \textit{attacker} (or \textit{adversary}), targeting critical training labels, and the \textit{model}, recovering from such attacks. 
We propose a projected gradient descent algorithm tailored for kernel SVMs to solve the bilevel optimization problem.
Our experiments on various classification tasks demonstrate that \textsc{Floral} improves robustness in the face of adversarially manipulated labels by effectively leveraging the inherent robustness of SVMs combined with the strengths of adversarial training, achieving enhanced model resilience against label poisoning while maintaining a balance with classification accuracy. \looseness -1

 \begin{figure}[t]
    {\centering  
   \begin{subfigure}[t]
{0.44\textwidth}\centering{\includegraphics[width=1\linewidth,trim=3 305 385 5,clip, valign=t]{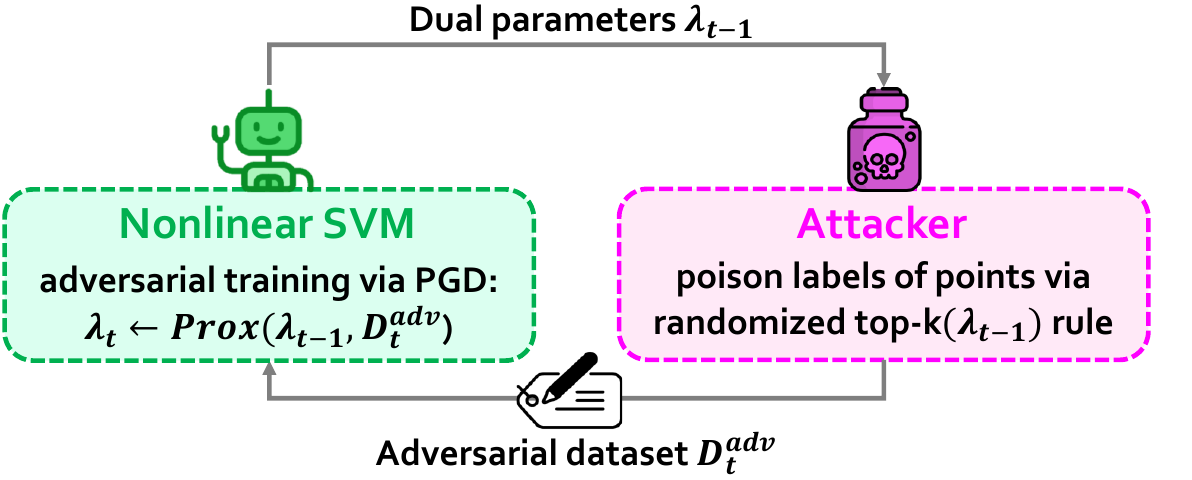}}
  \caption{\textsc{Floral} illustration.}
  \label{fig:robust_label_alg_sketch}
  \end{subfigure}
  \begin{subfigure}[t]
{0.27\textwidth}\centering{\includegraphics[width=1\linewidth,trim=0 10 0 70,clip, valign=t]{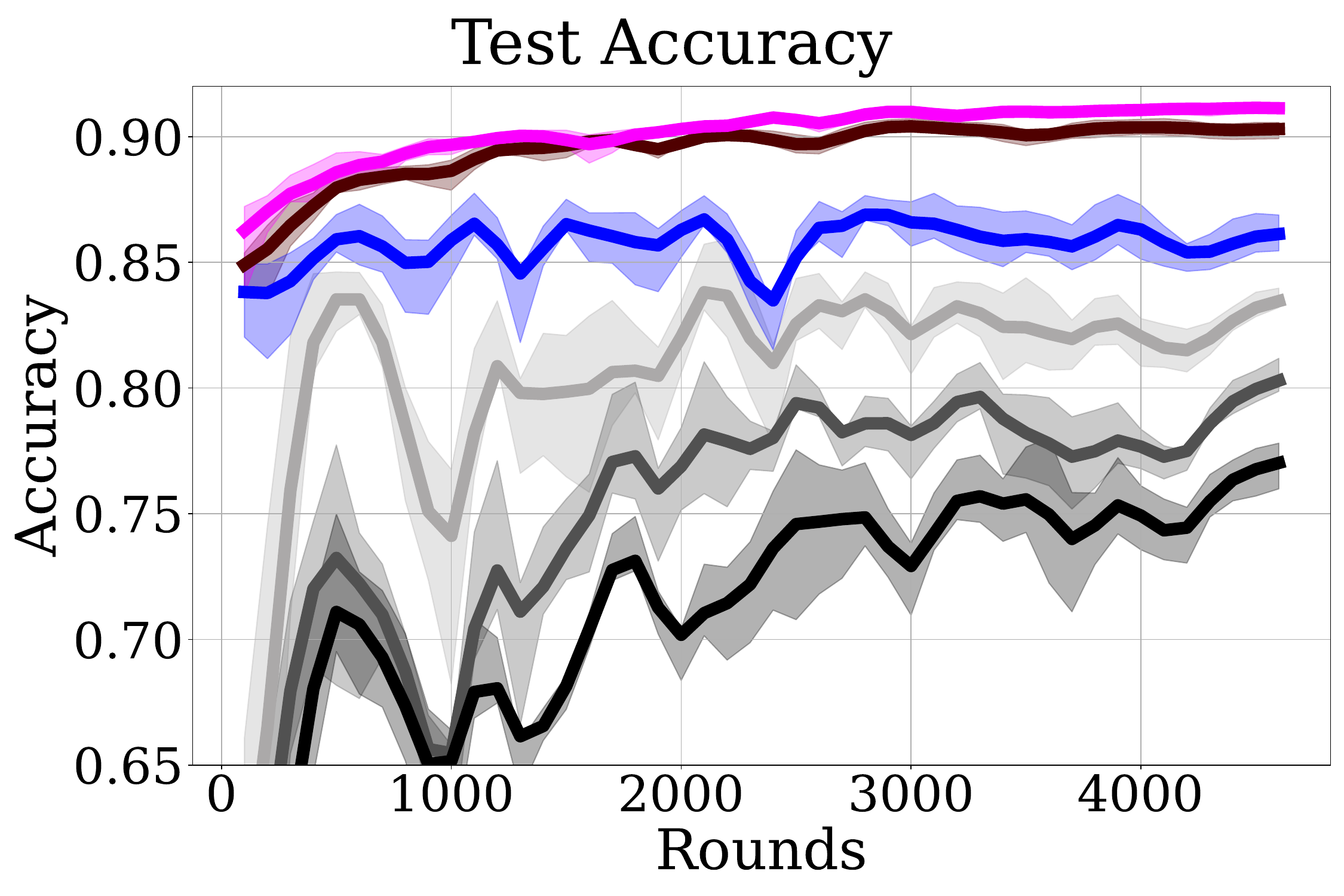}}
  \caption{RoBERTa.}
  \label{fig:roberta-motivation}
  \end{subfigure}
   \begin{subfigure}[t]
{0.27\textwidth}\centering{\includegraphics[width=1\linewidth,trim=0 10 0 70,clip, valign=t]{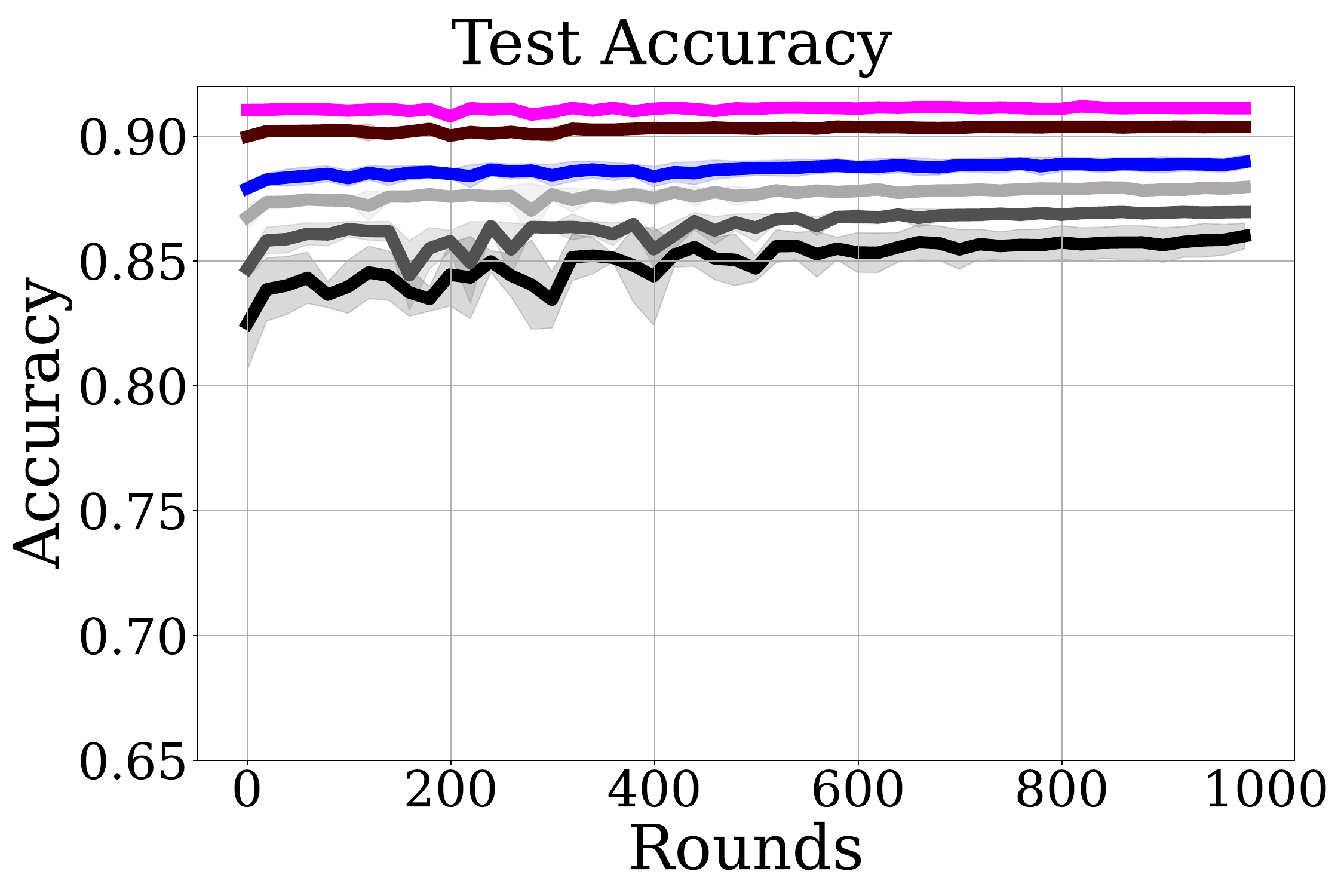}}
  \caption{\textsc{Floral}.}
  \label{fig:roberta-motivation-floral}
  \end{subfigure}
}
\\
\hspace*{\fill}
\begin{subfigure}[t]
{0.54\textwidth}\centering{\includegraphics[width=01\linewidth,trim=150 10 120 20,clip]{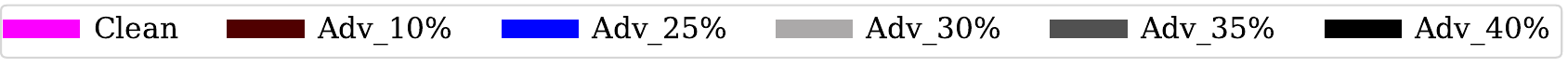}}
\end{subfigure}
  \setlength{\belowcaptionskip}{-10pt}
  \caption{(a): The illustration of \textsc{Floral} defense, adversarial training under label poisoning attacks. (b): The test accuracy degradation of RoBERTa fine-tuned on the IMDB dataset with adversarial labels, showing its vulnerability to such attacks. (c): \textsc{Floral} effectively mitigates the impact of label poisoning in (b), achieving significantly higher robust accuracy.
  }
  \label{fig:floral-motivation}
\end{figure}
\vspace{-0.2cm}
\paragraph{Contributions.} 
Our main contributions are the following.
\begin{itemize}[left=0.2cm,topsep=0pt]
\setlength\itemsep{-0.05em}
    \item We propose \textsc{Floral}, a support vector machine-based adversarial training strategy that defends against label poisoning attacks. To the best of our knowledge, this is the first work to introduce adversarial training as a defense specifically for \textit{label poisoning attacks}. We consider kernel SVMs in our formulation, however, as we show in our experiments, the method can be easily integrated with other models such as neural networks. \looseness -1
    
    \item We utilize a bilevel optimization formulation for the robust learning problem, leading to
    a non-zero-sum Stackelberg game between an \textit{attacker} who poisons the labels of influential training points and the \textit{model} trying to recover from such attacks. We provide a projected gradient descent (PGD)--based algorithm to solve the game efficiently.
    
    \item We theoretically analyze the local asymptotic stability of our algorithm by proving that its iterative updates remain bounded and 
    characterizing its convergence to the Stackelberg equilibrium. \looseness -1
    
    \item We empirically analyze \textsc{Floral}'s effectiveness through experiments on various classification tasks against robust baselines as well as foundation models such as RoBERTa. Our results demonstrate that as the attacker's budget increases, \textsc{Floral} maintains higher robust accuracy compared to baselines trained on adversarial data. \looseness -1

    \item Finally, we show the generalizability of \textsc{Floral} against attacks from the literature, \texttt{alfa}, \texttt{alfa-tilt} \citep{xiao2015support} and \texttt{LFA} \citep{label-sanitization}, which aim to maximize the difference in empirical risk between classifiers trained on tainted and untainted label sets. \looseness -1
\end{itemize}
\looseness -1 \vspace{-0.2cm}
\begin{figure}[t!]
    \centering  
{\includegraphics[width=0.6\linewidth,trim=0 140 0 139,clip]{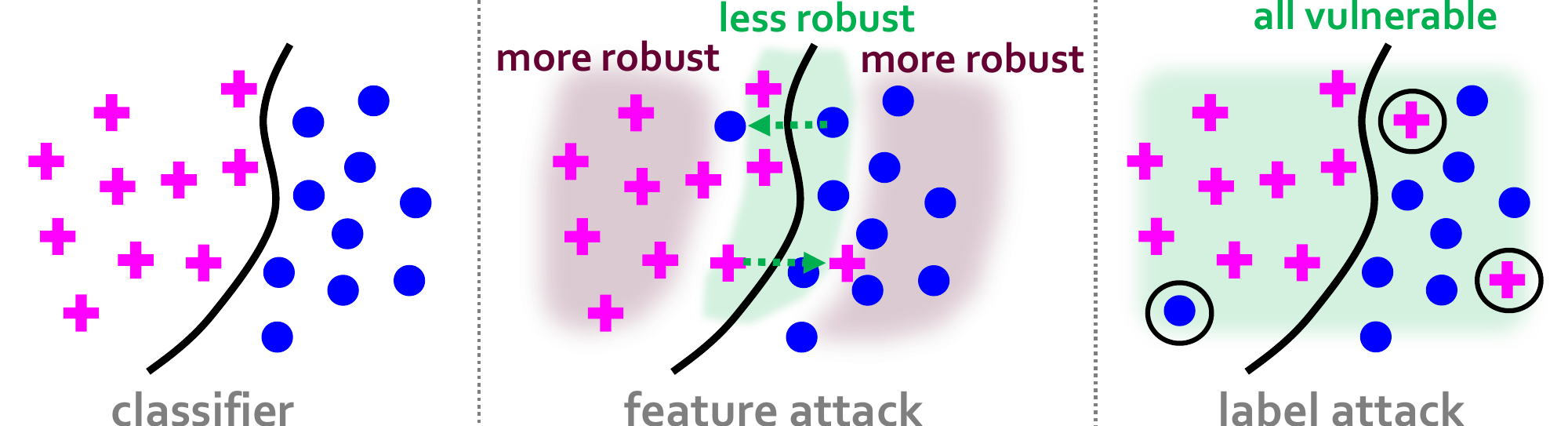}}
    \setlength{\belowcaptionskip}{-5pt}
     \caption{
     \textbf{Sensitivity of the decision boundary to label poisoning attacks.}
     The vulnerability of data points differs between feature perturbation and label poisoning attacks. 
    Given a perfect classifier, points near the decision boundary are less robust to feature attacks \citep{gairat, explore-exploit-db-dynamics}, leading to localized shifts in classification regions when the attack is performed.
    In contrast, the decision boundary has a broader sensitivity with respect to label poisoning attacks which can affect both near-boundary and distant points.
    By injecting incorrect labels, these attacks can create more widespread disruption and an overall degradation in classifier performance across the input space.
     \looseness -1
}
\label{fig:label-attacks-sketch}
\end{figure} 
\section{Problem Statement and Background}
\label{sec:problem-statement-background}
\looseness -1 \vspace{-0.2cm}
We tackle the problem of robust binary classification in the presence of label poisoning attacks 
(see Section~\ref{sec:proposed-approach} for an extension to multi-class classification).
Given a training dataset $\mathcal{D}=\{(x_i, y_i) \in (\mathcal{X},\mathcal{Y}) \}_{i=1}^{n}$,  where $\mathcal{X} \subseteq \mathbb{R}^{d}$ are the input features and $\mathcal{Y} = \{ \pm 1\}$ are the binary labels (potentially involving adversarial labels), we consider a kernel SVM classifier 
$f_\lambda (x):= \mathrm{sign}(\sum_{j} \lambda_j y_j k(x, x_j) + b)$,
parametrized by $\lambda \in \mathbb{R}^n$ and bias $b \in \mathbb{R}$, which assigns a label to each data point and is derived from the following quadratic program (dual formulation) \citep{boser1992training, hearst1998support}: \looseness -1 
\begin{alignat}{3}
D(f_{\lambda};\mathcal{D})&:  \min\limits_{\lambda \in \mathbb{R}^{n}} \quad \dfrac{1}{2} \lambda^\mathrm{T} && Q \lambda - \mathbbm{1}^\mathrm{T} \lambda  \label{soft-svm-dual-kernel-compact-obj} \\
& \st && y^\mathrm{T} \lambda = 0, \quad 0 \leq \lambda \leq C,
\label{soft-svm-dual-kernel-compact-const2}
\end{alignat}
where $Q \in \mathbb{R}^{n \times n}$ is a positive semi-definite matrix, with elements $Q_{ij}=y_i y_j K_{ij}$ and $\mathbbm{1}$ is the $n$-dimensional vector of all ones. Here, $K$ is the Gram matrix with entries $K_{ij}=k(x_i, x_j), \forall i,j \in [n] := \{1,\dots,n\}$, derived from a kernel function $k$.
A common kernel choice is the radial basis function (RBF), given as $k(x_i, x_j)=\exp (-\gamma\left\|x_i-x_j\right\|^2)$, with width parameter $\gamma$.
The parameter $C \geq 0$ is a regularization term, balancing the trade-off between maximizing the margin and minimizing classification errors. 
In this formulation, each dual variable $\lambda_i, i \in [n]$ corresponds to the Lagrange multiplier associated with the misclassification constraint for the training point $x_i$. 
\looseness -1

\vspace{-0.2cm}
\section{The \textsc{Floral} Approach}
\label{sec:proposed-approach}
\looseness -1 \vspace{-0.2cm}
In the context of label poisoning attacks, the attacker’s objective is to maximize the model’s test classification error by subtly altering the labels in the training dataset to an optimal adversarial configuration. 
Adversarial training \citep{goodfellow-2014, madry2017towards} can be extended to counter these attacks and minimize model sensitivity to disruptive labels by actively optimizing for robustness under worst-case scenarios.
In this setting, the attacker generates the optimal label attack within a budget of 
$k$ flips to maximize the model’s loss, while the model seeks parameters that minimize this worst-case loss.
A straightforward, yet naive \citep{at-nonzero-game}, way to implement this approach would be to use the following minimax formulation:
\begin{equation}
\min _{\lambda \in \mathbb{R}^n} \frac{1}{n} \sum\limits_{i=1}^{n} 
\left\{ 
    \max_{\substack{
    \sum_{i \in [n]} \mathbf{1}\{y_i \neq \Tilde{y}_i\}=k \\
    \Tilde{y}_i \in \mathcal{Y}, i \in [n] \\ 
    }} 
    \mathcal{L} \left(f_\lambda(x_i), \Tilde{y}_i\right) 
\right\},
\label{eq:zero-sum-game-surrogate}
\end{equation}
where $\mathcal{L}$ denotes a loss function,
which in the case of the kernel SVM is related to the hinge loss \citep{smola1998learning}, 
and $\Tilde{y}$ represents the adversarial label set. 
This formulation is problematic for multiple reasons: \looseness -1
\vspace{-0.2cm}
\begin{enumerate}[left=0.2cm]
  \setlength\itemsep{-0.05em}
    \item \textit{Misaligned objectives}: The loss is only a surrogate for the test accuracy, which is the actual quantity of interest to both the learner and the attacker. 
    However, from an optimization perspective, maximizing an upper bound (such as the hinge loss in SVMs) on the classification error as in (\ref{eq:zero-sum-game-surrogate}) is not meaningful as such a bound does not represent the true objective of the attacker. Hence, a non-zero-sum formulation would allow for a more nuanced representation of the attacker's objectives \citep{yasodharan2019nonzero}. \looseness -1

    \item \textit{Ineffective defense against critical points}: In the case of an SVM-based classifier, the minimax formulation would only safeguard against attacks targeting data points with the largest hinge loss, i.e., those farthest from the decision boundary. These attacks are easily distinguishable \citep{adversarial-flip-svm} as, e.g., soft margin SVMs are shown to be robust to outliers \citep{smola1998learning}.
    In contrast, attacks targeting the critical points that define the decision boundary (\textit{support vectors}) would be more effective in degrading the classifier's performance.
    \looseness -1
    \item \textit{Combinatorial explosion}: Even if a bilevel formulation is employed where the attacker minimizes the margin, the problem remains computationally challenging.
    Ordering data points by their margin and then searching for the best adversarial label set within a budget constraint results in a vast combinatorial space.
    \looseness -1
\end{enumerate}
\begin{algorithm*}[t]
\setlength{\textfloatsep}{0pt}
\caption{\textsc{Floral}}
\label{alg:robust-svm-game}
\begin{algorithmic}[1]
\STATE {\bfseries Input:} Initial kernel SVM model $f_{\lambda_{0}} $, training dataset $\mathcal{D}_{0}=\{(x_i, y_i)\}_{i=1}^{n}, x_i \in \mathbb{R}^{d}, y_i \in \{ \pm 1\}$, attacker budget $B \in \{0,\dots,n\}$, parameter $k$, where $k \ll B$, learning rate $\eta > 0$.
\FOR{round $t = 1, \dots, T$}
\STATE $\Tilde{y}^t \gets$ Solve (\ref{eq:inner-problem-objective}-\ref{eq:inner-problem-const2}) via \mbox{\colorbox{magenta!15}{randomized top-$k$}}: randomly selecting $k$ points from top $B$ w.r.t. $\lambda_{t-1}$. 
\STATE $\mathcal{D}_{t} \gets \{(x_i, \Tilde{y}_i^t)\}_{i=1}^{n}$ \hfill {\color{caribbeangreen} */ Adversarial dataset with selected $k$ poisoned labels}
\STATE Compute gradient of the objective (\ref{eq:outer-problem-objective}), $\nabla_{\lambda} D(f_{\lambda_{t-1}}; \mathcal{D}_{t})$, based on $\lambda_{t-1}, \mathcal{D}_{t}$ as given in (\ref{eq:svm-dual-gradient}).
\STATE Take a PGD step $\lambda_t \gets$ \mbox{\colorbox{magenta!15}{\textsc{Prox}}}$_{\mathcal{S}(\Tilde{y}^t)} (\lambda_{t-1} - \eta \nabla_{\lambda} D(f_{\lambda_{t-1}};\mathcal{D}_{t}))$, based on (\ref{eq:svm-projection-operation}-\ref{proj-ineq-const}). \hfill {\color{caribbeangreen} */ AT}
\ENDFOR
\STATE \textbf{return} $f_{\lambda_{T}}$
\end{algorithmic}
\setlength{\textfloatsep}{0pt}
\end{algorithm*}
\looseness -1 
As a result of these, we formulate our adversarial training routine as a non-zero-sum Stackelberg game \citep{von2010market, conitzer2006computing} and propose \textsc{Floral} defense using the bilevel optimization formulation \citep{bilevel-optimization}:
\begin{mdframed}[roundcorner=5pt, backgroundcolor=caribbeangreen!5, linecolor=sbbluedeep, linewidth=1pt, splittopskip = 0pt,  skipabove = 5pt,
skipbelow=5pt]
\begin{minipage}[t]{.43\textwidth}
\begin{alignat}{3} \label{eq:outer-problem-objective}
D(f_{\lambda};&\mathcal{D}): \min\limits_{\lambda \in \mathbb{R}^{n}} \quad  \dfrac{1}{2} && \lambda^\mathrm{T}  \Tilde{Q} \lambda - \mathbbm{1}^\mathrm{T}\lambda \\
    & \st && \Tilde{y}(\lambda)^\mathrm{T}  \lambda = 0 \label{eq:outer-problem-const1} \\
    &&& 0 \leq \lambda \leq C \label{eq:outer-problem-const2}
\end{alignat}
\end{minipage}
\hfill
\vline
\hfill
\begin{minipage}[t]{.54\textwidth}
\begin{alignat}{2} \label{eq:inner-problem-objective}
&\text{where }%
 \Tilde{y}(\lambda) \in  \arg && \max\limits_{\substack{ y^\prime \in \mathcal{Y}^n,  
u \in \{0,1\}^n}} \lambda^\mathrm{T} u \\
    & \qquad \st && y^\prime_i = y_i (1-2u_i), \forall i \in [n] \label{eq:inner-problem-const1} \\
    &&&\sum\limits_{i \in [n]} \mathbf{1}\{y_i \neq y^\prime_i \} = k. \label{eq:inner-problem-const2}
\end{alignat}
\end{minipage}
\end{mdframed}
\looseness -1 \vspace{-0.2cm}
In the outer (model's) problem, defined by (\ref{eq:outer-problem-objective}-\ref{eq:outer-problem-const2}), the SVM classifier is derived under an adversarial label set. The key difference from the formulation in Section~\ref{sec:problem-statement-background} is that the elements of $\Tilde{Q}$ are defined as $\Tilde{Q}_{ij}=\Tilde{y}_i \Tilde{y}_j K_{ij}$.
Meanwhile, the inner (attacker's) problem, given by (\ref{eq:inner-problem-objective}-\ref{eq:inner-problem-const2}) identifies the top-$k$ most \textit{influential} data points affecting the model's decision boundary. 
The intuition behind this approach is similar to identifying the most responsible training points for the model's prediction as in \citep{understanding-influence}.
However, rather than relying on influence functions \citep{hampel1974influence}, the attacker leverages the dual variables $\lambda$, which provides \textit{direct} access to such influential points.
These points correspond to the support vectors, 
and the higher the value of a dual variable, the more critical that data point is in determining the model's decision boundary. \looseness -1

We address the bilevel optimization problem in (\ref{eq:outer-problem-objective}-\ref{eq:inner-problem-const2}) as a non-zero-sum Stackelberg game \citep{von2010market} between the learning \textit{model}, and the \textit{attacker} acting as the leader and follower, respectively, as shown in Figure~\ref{fig:robust_label_alg_sketch}. 
The game begins with an initial kernel SVM model $f_{\lambda_0}$ and a training dataset $\mathcal{D}_0$, and proceeds iteratively.
In each round $t$, the model shares its dual parameters with the attacker, who then generates an adversarially labelled dataset $\mathcal{D}_{t}$ using a \textit{randomized top-$k$} rule. 
That is, the attacker identifies the top-$B$ data points based on their $\lambda_{t-1}$ values, (constrained by the budget $B$) and flips the labels of $k$ randomly chosen points among them.
We incorporate randomization to account for the attacker's budget and to reduce the risk of settling in local optima.
Adversarial training is performed via a projected gradient descent step using $\lambda_{t-1}$ and $\mathcal{D}_{t}$, after which the updated parameters, $\lambda_t$, are shared with the attacker. 
This iterative interplay between the attacker and defender model forms a soft-margin kernel SVM robust to adversarial label poisoning.
Our overall approach is detailed in Algorithm~\ref{alg:robust-svm-game}. \looseness -1
\vspace{-0.2cm}
\paragraph{\textsc{Floral}'s effectiveness.}
\textsc{Floral} iteratively exposes the model to learn adversarial configurations of the decision boundary. Hence, when the training data is clean, the training process \textit{proactively} adjusts the model to be less sensitive to the influence of individual poisoned labels. 
In cases where poisoned labels are already present in the initial training data, \textsc{Floral} effectively neutralizes their impact by \textit{implicitly sanitizing} the corrupted labels. This behavior is evaluated empirically and detailed in Section~\ref{sec:experiment-results} and Appendix~\ref{app:defense-analysis}.
\looseness -1
\vspace{-0.2cm}
\paragraph{The attacker's capability.} The attacker solves (\ref{eq:inner-problem-objective}-\ref{eq:inner-problem-const2}) with respect to the shared model parameters $\lambda$, generating label attacks by targeting the most influential support vectors.
This white-box attack \citep{survey-wu2023attacks} assumes that the attacker can access model parameters.
To reflect practical constraints, we limit the attacker’s budget to at most $B$ label poisons per round, from which $k$ points are randomly selected.
While this scenario may still seem to give the attacker significant power, notably, $(i)$ relying on secrecy for security is generally considered poor practice \citep{biggio2013security}, 
and $(ii)$ our method is designed to defend against the strongest possible attacker.
Even in 
black-box attack scenarios, where the attacker lacks parameter access,
\textsc{Floral} remains effective for generating transferable attacks \citep{survey-zheng2023blackboxbench}. In such cases, the attacker could fit a kernel SVM on the available data and apply a similar selection rule to craft adversarial labels.
\vspace{-0.2cm}
\paragraph{Gradient of the objective (\ref{eq:outer-problem-objective}).} 
In each round, the adversarial training PGD step requires computing the gradient $\nabla_{\lambda} D(f_{\lambda};\mathcal{D})$ of the objective (\ref{eq:outer-problem-objective}) based on $\lambda_{t-1}$ and $\mathcal{D}_{t}$, which is defined as
\begin{equation}
\nabla_{\lambda} D(f_{\lambda_{t-1}};\mathcal{D}_t)= \Tilde{Q}\lambda_{t-1} - \mathbbm{1},
\label{eq:svm-dual-gradient}
\end{equation}
where 
$\Tilde{Q}$ is the matrix with entries $\Tilde{Q}_{ij}=\Tilde{y}^t_i \Tilde{y}^t_j K_{ij}, \forall i,j \in [n]$, detailed in Appendix~\ref{app:gradient-svm-dual}.
\looseness -1  \vspace{-0.2cm}
\paragraph{Projection.}
The feasible set $\mathcal{S}$ changes in each round $t$ depending on the adversarial label set $\Tilde{y}^t$ (see (\ref{eq:outer-problem-const2})).
We introduce the variable $z_{t} := \lambda_{t-1} - \eta \nabla_{\lambda} D(f_{\lambda_{t-1}};\mathcal{D}_{t})$ and define the projection operator \textsc{Prox}$_{\mathcal{S}(\Tilde{y}^t)}(z_{t}): \mathbb{R}^n \to \mathbb{R}^n$ as follows: \looseness -1  \vspace{-0.2cm}
\begin{alignat}{2}
\textsc{Prox}_{\mathcal{S}(\Tilde{y}^t)}(z_{t}): \lambda_t \in \arg & \min _{\lambda \in \mathbb{R}^{n}} \quad  \frac{1}{2} \|  \lambda - z_{t} \|^{2} \label{eq:svm-projection-operation} \\
\st  & \Tilde{y}^{{t}^\mathrm{T}}  \lambda = 0, \quad 0 \leq \lambda \leq C. \label{proj-ineq-const}
\end{alignat}
However, solving this quadratic program for large-scale instances is computationally challenging unless the specific problem structure is exploited. 
Therefore, we provide a scalable and efficient implementation of Algorithm~\ref{alg:robust-svm-game} that relies on a fixed point iteration strategy as detailed in Section~\ref{sec:large-scale-implementation}.
\looseness -1  \vspace{-0.2cm}
\paragraph{A form of geometry-aware AT.}
\textsc{Floral} aligns with geometry-aware AT principles \citep{gairat}. 
Support vectors with large Lagrange multipliers ($\lambda$)
play a critical role in defining the decision boundary \citep{hearst1998support}. 
In \textsc{Floral}, the attacker strategically identifies these points
using a randomized top-$k$ rule.
This method inherently integrates the geometric proximity to the decision boundary into the label attack, targeting points that significantly impact the hinge loss.
\looseness -1 \vspace{-0.2cm}
\paragraph{Robust multi-class classification.}
\looseness -1 
We extend our algorithm to multi-class classification tasks, as detailed in Algorithm~\ref{alg:robust-svm-game-multiclass} in Appendix~\ref{app:multi-class-classification}.
The primary modification involves adopting a one-vs-all approach and considering multiple attackers, each corresponding to a different class.
\looseness -1 \vspace{-0.2cm}
\subsection{Stability Analysis}
\label{sec:stability-analysis}
\looseness -1 \vspace{-0.2cm}
We theoretically analyze the stability of \textsc{Floral} (Algorithm~\ref{alg:robust-svm-game}) by $(i)$ demonstrating that its iterative updates are bounded and $(ii)$ characterizing its convergence to the Stackelberg equilibrium.
For simplicity of notation, let us define the update rule at round $t$ as $\lambda_t := \textsc{Prox}_{\mathcal{S}(y_t)}(z_{t}) = \textsc{Prox}_{\mathcal{S}(y_t)} (\lambda_{t-1} - \eta \nabla_{\lambda} f (\lambda_{t-1}, y_t))$, where $\textsc{Prox}$ is defined in (\ref{eq:svm-projection-operation}-\ref{proj-ineq-const}).
We use the operator $\textsc{LFlip}:\mathcal{X} \times \mathcal{Y} \to \mathcal{Y}$ to define label poisoning attack formulated in (\ref{eq:inner-problem-objective}-\ref{eq:inner-problem-const2}).
\begin{lemma}
\label{thm:convergence-bound-lambda-steps}
Let $(\Hat{\lambda}, \Hat{y}(\hat{\lambda}))$ denote a Stackelberg equilibrium, i.e., $\Hat{y}(\Hat{\lambda}) := \textsc{LFlip}(\Hat{\lambda})$ and $\Hat{\lambda} := \textsc{Prox}_{\mathcal{S}(\Hat{y}(\Hat{\lambda}))}(\Hat{z}) = \textsc{Prox}_{\mathcal{S}(\Hat{y}(\Hat{\lambda}))}(\Hat{\lambda} - \eta \nabla_{\lambda} f({\Hat{\lambda}}, \Hat{y}(\Hat{\lambda})))$ and $\{\lambda_t\}_{t=0}^T$ be the sequence of iterates generated by \textsc{Floral} (Algorithm~\ref{alg:robust-svm-game}). 
The following bound holds for the iterates:
\begin{align}
\| \lambda_{t} - \Hat{\lambda} \|_{\infty} & \leq  \| z_{t} - \Hat{z} \|_{\infty} + \kappa_{y} \| y_{t} -\Hat{y}(\Hat{\lambda}) \|_{\infty}
\label{eq:bound-on-lambda}
\end{align}
where $\kappa_{y}$ is a constant defined by the \textsc{Prox} operator and index set corresponding to $\lambda_t \in (0,C)$, as detailed in Appendix~\ref{app:proof-of-lemma-1}, and $\| \cdot \|_{\infty}$ denotes the infinity norm. \looseness -1
\end{lemma} \vspace{-0.2cm}
\begin{proof}[Proof] \renewcommand{\qedsymbol}{}
See Appendix~\ref{app:proof-of-lemma-1} for the proof.
\end{proof} \vspace{-0.2cm}
\begin{lemma}\label{thm:convergence-bound-x-steps}
Let $(\Hat{\lambda},\Hat{y}(\Hat{\lambda}))$ denote the Stackelberg equilibrium as before. The following bound holds for the non-projected iterates $\{z_t\}_{t=0}^T$ of \textsc{Floral} (Algorithm~\ref{alg:robust-svm-game}):
\begin{align}
    \| z_{t} - \Hat{z} \|_{\infty} & \leq \kappa_{\lambda} \| \lambda_{t-1} - \Hat{\lambda} \|_{\infty} + \kappa'_y \| y_{t} - \Hat{y}(\hat{\lambda}) \|_{\infty}      
\label{eq:bound-on-x}
\end{align}
where $\kappa_{\lambda}$ and $\kappa'_y$ are kernel dependent constants that are below $1$ for small enough $\eta$. \looseness -1
\end{lemma} \vspace{-0.2cm}
\vspace{-0.2cm}
\begin{proof}[Proof] \renewcommand{\qedsymbol}{}
See Appendix~\ref{app:proof-of-lemma-2} for the proof.
\end{proof}  \vspace{-0.2cm}
\begin{theorem}[$\varepsilon$-local asymptotic stability]
\label{thm:stability}
The Stackelberg equilibrium $(\Hat{\lambda},\Hat{y}(\Hat{\lambda}))$ defined as before, is $\varepsilon$-locally asymptotically stable for the Stackelberg game solved via Algorithm~\ref{alg:robust-svm-game} for a small enough step size $\eta$.
This implies that for every $\varepsilon > 0$, there exists $\delta > 0$ such that 
\looseness -1
\begin{align}
\| \lambda_0 - \Hat{\lambda} \|_{\infty} < \delta \Rightarrow \| & \lambda_t - \Hat{\lambda} \|_{\infty} < \varepsilon, \forall t > 0 \text{ and } \lambda_t \to \Hat{\lambda}.
\end{align}
\end{theorem} \vspace{-0.2cm}
\begin{proof}[Proof (sketch)]\renewcommand{\qedsymbol}{} 
The proof relies on characterizing the distance between the update $\lambda_t$ at round $t$ and the equilibrium $\Hat{\lambda}$ using Lemma~\ref{thm:convergence-bound-lambda-steps} and Lemma~\ref{thm:convergence-bound-x-steps}, then leveraging the fact that the label flipping operator (\textsc{LFlip}) formulated in (\ref{eq:inner-problem-objective}-\ref{eq:inner-problem-const2}) returns the same adversarial label set when $\lambda_t$ is within an $\varepsilon$ distance from the equilibrium. The complete proof is given in Appendix~\ref{app:proof-local-stability}, with the global convergence result discussed in Appendix~\ref{app:global-convergence-result}. \looseness -1
\end{proof}
\looseness -1 \vspace{-0.2cm}
\subsection{Large-Scale Implementation}
\label{sec:large-scale-implementation}
\looseness -1 
We scale our algorithm for large problem instances by approximating the projection operation (step $6$ in Algorithm~\ref{alg:robust-svm-game}) via a fixed-point iteration method, as outlined in Algorithm~\ref{alg:projection-fpi}.
The key idea leverages the optimal $\lambda^\star$ expression from Appendix~\ref{app:proof-of-lemma-1} and involves an iterative splitting of variables based on non-projected $\lambda$ values within the range $[0, C]$.
In each iteration, the variable $\mu$ is updated using the expression in Appendix~\ref{app:proof-of-lemma-1}
until convergence to a specified error $\epsilon$ is achieved. \looseness -1
\vspace{-0.2cm}
\begin{algorithm*}[ht]
\caption{\textsc{ProjectionViaFixedPointIteration}}
\label{alg:projection-fpi}
\begin{algorithmic}[1]
\STATE {\bfseries Input:} Non-projected $\lambda_{0}$, adversarial label set $\Tilde{y} = \{\Tilde{y}_i\}_{i=1}^{n}, y_i \in \{ \pm 1\}$, parameters $\{C,\epsilon\} > 0$. \looseness=-1
\STATE Initialize $\mu_0 = 0$.
\FOR{round $t = 1,\dots, T_{\text{proj}}$}
\STATE $\lambda_{t} = \textsc{Clip}_{[0,C]}(\lambda_{0} - \mu_{t-1} \Tilde{y})$ \hfill {\color{caribbeangreen} */ Clip to satisfy constraint in (\ref{proj-ineq-const})} 
\IF{$\lambda_{t} \Tilde{y} = 0$}
\STATE \textbf{return} $\lambda_{t}$
\ENDIF
\STATE $\mathcal{I}_C, \mathcal{I}_z \gets$ indices of $\lambda_{t} \geq C$, $\lambda_{t} \in (0,C)$ \hfill {\color{caribbeangreen} */ Variable splitting} 
\STATE $\eta \gets \max (\mid \mathcal{I}_z \mid, 1)$ \hfill {\color{caribbeangreen} */ To avoid empty $\mathcal{I}_z$ case} 
\STATE $\mu_{t} \gets \frac{\eta - \mid \mathcal{I}_z \mid}{\eta} \mu_{t-1} + \frac{1}{\eta} (\sum_{i \in \mathcal{I}_C} C \Tilde{y}_i + \sum_{i \in \mathcal{I}_z} \lambda_{t}^i \Tilde{y}_i )$
\IF{$\mid \mu_{t} - \mu_{t-1} \mid \leq \epsilon$}
\STATE \textbf{return} $\textsc{Clip}_{[0,C]}(\lambda_{0} - \mu_{t} \Tilde{y})$
\ENDIF
\ENDFOR
\end{algorithmic}
\end{algorithm*}
\vspace{-0.2cm}
\subsection{Related Work}
\label{sec:related-work}
\looseness -1 

\paragraph{Label poisoning.} 
\citet{biggio2012poisoning} first analyzed label poisoning attacks, showing that flipping a small number of training labels severely degrades SVM performance. 
\citet{adversarial-flip-svm} later formalized optimal label flip attacks under budget constraints as a bilevel optimization problem, which then expanded to transferable attacks on black-box models \citep{label-contamination-linear}, considering arbitrary attacker objectives.
Beyond SVMs, recent works have explored label poisoning in backdoor attack scenarios, where adversaries inject triggers or alter triggerless data with poisoned labels in multi-label settings \citep{label-poisoning, chen2022clean}.
In contrast, our approach focuses on triggerless poisoning attacks.
\looseness -1

Defenses against these attacks include heuristic-based kernel correction \citep{svm-adversarial-noise}, which uses expectation for $Q$ in (\ref{eq:outer-problem-objective}),
though assuming independent label flipping with equal probability--a condition not guaranteed in practice.
Other defenses such as clustering-based filtering \citep{curie,tavallali2022adversarial}, data complexity analysis \citep{chan2018data}, re-labeling \citep{label-sanitization} and label smoothing \citep{rosenfeld2020certified} offer straightforward solutions, however, they do not scale well to high-dimensional or large datasets.
Sample weighting based on local intrinsic dimensionality (LID) \citep{defending-svms, ma2018characterizing} shows promise, but relies on accurate and computationally expensive LID estimation.
Our approach, however, avoids strong assumptions about the data distribution or the attacker, preserves feasibility, and scales effectively to large-scale problem instances as demonstrated in Section~\ref{sec:experiments}. \looseness -1
Additionally, while learning under noisy labels \citep{frenay2013classification, learning-w-noisy-labels,hallaji2023label, zhang2024effective} may seem relevant,
our work focuses specifically on \textit{adversarial} label noise \citep{svm-adversarial-noise}, where the adversary \textit{intentionally} crafts the most damaging label perturbations. 
\looseness -1 
\vspace{-0.2cm}
\paragraph{Adversarial training (AT).} 
Adversarial examples, introduced by \citet{szegedy2013intriguing}, revealed how small perturbations cause misclassification in deep neural networks (DNNs).
Building on this, AT \citep{goodfellow-2014} emerged as a prominent defense, training models on both original and adversarially perturbed data.
Defenses have utilized adversarial examples generated by methods such as the Fast Gradient Sign Method (FGSM) \citep{goodfellow-2014}, PGD \citep{madry2017towards}, Carlini \& Wagner attack \citep{carlini2017towards}, among others \citep{chen2017zoo, moosavi2016deepfool}.
For SVMs, \citet{zhou2012adversarial} formulated convex AT for linear SVMs, later extended to kernel SVMs by \citet{fast-scalable-adv-svm} via doubly stochastic gradients under feature perturbations.
Despite this progress, AT for label poisoning remains underexplored.
\textsc{Floral} fills this gap, by leveraging AT specifically for label poisoning scenarios, using PGD to train models on poisoned datasets rather than generating adversarial examples.
\looseness -1

In parallel, game-theoretical approaches have modeled adversarial interactions 
as simultaneous games, where classifiers and adversaries select strategies independently
\citep{adversarial-classification}, or as Stackelberg games with a leader-follower dynamic \citep{bruckner2011stackelberg, zhou2019survey, chivukula2020game}. 
AT has further linked these concepts, particularly in simultaneous zero-sum games \citep{hsieh2019finding, pinot2020randomization, pal2020game} to non-zero-sum formulations \citep{at-nonzero-game}.
We adopt a sequential setup, using the Stackelberg framework where the leader commits to a strategy and the follower responds accordingly.
\looseness -1 
\section{Experiments}
\label{sec:experiments}
\looseness -1 
In this section, we showcase the effectiveness of \textsc{Floral} across various robust classification tasks, utilizing the following datasets:
\vspace{-0.2cm}
\begin{itemize}[left=0.2cm]
\setlength\itemsep{-0.1em}
    \item \textbf{Moon} \citep{pedregosa2011scikit}: We employed a synthetic benchmark dataset, $\mathcal{D}=\{(x_i, y_i)\}_{i=1}^{2000}$ where $x_i \in \mathbb{R}^{2}$ and $y_i \in \{ \pm 1\}$.
    Adversarial versions are generated
    by flipping the labels of points farther from the decision boundary of a linear classifier trained on the clean dataset, using label poisoning levels ($\%$) of $\{5, 10, 15, 20, 25\}$.
    The details on the adversarial datasets are given in Appendix~\ref{app:datasets}.
    
    \item \textbf{IMDB} \citep{maas-EtAl:2011:ACL-HLT2011}: A benchmark review sentiment analysis dataset with $\mathcal{D}=\{(x_i, y_i)\}_{i=1}^{50000}$ where $x_i \in \mathbb{R}^{768}$ and $y_i \in \{ \pm 1\}$. For SVM training, we extracted $768$-dimensional embeddings from the fine-tuned RoBERTa \citep{roberta}. 
    We created adversarial datasets by fine-tuning the RoBERTa-base model on the clean dataset to identify influential training points based on the gradient with respect to the inputs, then flipping their labels at 
    poisoning levels ($\%$) of $\{10, 25, 30, 35, 40\}$. 
\item \textbf{MNIST} \citep{deng2012mnist}: In Appendix~\ref{app:mnist1vs7-experiments}, we provide the additional experiments with the MNIST dataset in detail. \looseness -1
\end{itemize}
\vspace{-0.2cm}
\paragraph{Experimental setup.}
For all SVM-based methods, we used RBF kernel, exploring various values of $C$ and $\gamma$. 
We conducted five replications with different train/test splits, including the corresponding adversarial datasets for each dataset. 
In all \textsc{Floral} experiments, we constrain the attacker's capability with a limited budget. 
That is, the attacker identifies the most influential \textit{candidate} points, with $B = 2k$, from the training set and randomly selects $k \in \{1, 2, 5, 10, 25\}$  to poison, where $k$ represents the \% of points relative to the training set size.
Detailed experimental configurations are provided in Appendix~\ref{app:experiment-details} (see Table~\ref{tab:params-and-hyperparams}).

\vspace{-0.2cm}
\paragraph{Baselines.}
We benchmark \textsc{Floral} against the following baselines:
\vspace{-0.2cm}
\begin{enumerate}[left=0.2cm]
  \setlength\itemsep{-0.1em}
    \item (Vanilla) SVM with an RBF kernel, which serves as a basic benchmark \citep{hearst1998support}.
    \item LN-SVM \citep{svm-adversarial-noise} applies a heuristic-based kernel matrix correction. 
    \item Curie \citep{curie}, utilizes the DBSCAN clustering \citep{ester1996density} to identify and filter-out poisoned data points.
    \item LS-SVM \citep{label-sanitization} applies label sanitization 
    based on k-NN \citep{cover1967nearest}. 
    \item K-LID \citep{defending-svms}, a weighted SVM based on kernel local intrinsic dimensionality.
    \item NN: A DNN trained using the SGD optimizer with momentum,
    serving as a non-linear baseline.
    \item NN-PGD: A DNN trained with PGD-AT \citep{madry2017towards}, 
    to evaluate a robust model designed to withstand feature perturbation attacks under label poisoning.
    \item RoBERTa \citep{roberta}, used in experiments with IMDB dataset to assess a fine-tuned transformer-based language model's robustness under label poisoning.
    \looseness -1
\vspace{-0.1cm}
\end{enumerate}  
Appendix~\ref{app:comp-additional-baselines} includes further comparisons with a least squares classifier using randomized smoothing \citep{rosenfeld2020certified} and a filtering-out defense based on regularized synthetic reduced nearest neighbor \citep{tavallali2022adversarial} on the Moon and MNIST datasets.
\looseness -1 \vspace{-0.2cm}
\paragraph{Performance metrics.} We assess our method using two key metrics: robust and clean accuracy, tracked over a test set with \textit{clean labels} during training.
Unlike feature perturbation studies,
where robust accuracy is gauged on adversarially perturbed test examples 
\citep{yang2020closer}, in our study, robust accuracy reflects model performance tested on clean labels using adversarially labelled training data, thereby indicating the models' resilience and generalization capabilities under poisoning. Conversely, clean accuracy measures the performance of models trained and tested on clean-labelled data, offering a benchmark for comparison under both adversarial and non-adversarial conditions.
We additionally report hinge loss on the clean-labelled test data (see Appendix~\ref{sec:imdb-results-appendix}) in experiments with the IMDB dataset. \looseness -1 
\subsection{Experiment Results}
\label{sec:experiment-results}
In this section, we report the performance of \textsc{Floral} against the baseline methods on the Moon dataset, followed by the results of its integration with RoBERTa on the IMDB dataset.
\looseness -1 \vspace{-0.2cm}
\begin{table*}[t]
\caption{Test accuracies of methods trained on the Moon dataset, averaged over five runs. 
Highlighted values indicate the best performance in the  \mbox{\colorbox{tablegreen}{"\textbf{Best}"}} (peak accuracy during training) and \mbox{\colorbox{tablegreen}{"Last"}} (final accuracy after training) columns. This notation is consistently applied in the subsequent tables.
\textsc{Floral} outperforms baselines in most of the settings, providing a particularly robust defense in highly adversarial scenarios.
See Appendix~\ref{sec:moon-results-appendix} (Table~\ref{tab:test-accuracy-comp-moon-appendix}) for the results of other settings.
\looseness -1
}
\label{tab:test-accuracy-comp-moon} 
\centering
    \resizebox{1\linewidth}{!}{%
    \begin{tabular}{ll|cccccccccccccccc} \specialrule{1.5pt}{1pt}{1pt}
         \multicolumn{2}{c}{\multirow{2}{*}{\makecell{\\ \\ \textbf{Setting}}}} & \multicolumn{16}{c}{\textbf{Method}} \\ \cmidrule(lr){3-18}
        & & \multicolumn{2}{c}{\textsc{Floral}} & \multicolumn{2}{c}{SVM} & \multicolumn{2}{c}{NN} & \multicolumn{2}{c}{NN-PGD} & \multicolumn{2}{c}{LN-SVM} & \multicolumn{2}{c}{Curie} & \multicolumn{2}{c}{LS-SVM} &  \multicolumn{2}{c}{K-LID} \\ 
        \cmidrule(lr){3-4} \cmidrule(lr){5-6} \cmidrule(lr){7-8} \cmidrule(lr){9-10} \cmidrule(lr){11-12} \cmidrule(lr){13-14} \cmidrule(lr){15-16} \cmidrule(lr){17-18}
        & &  Best & Last & Best & Last & Best & Last & Best & Last & Best & Last & Best & Last & Best & Last & Best & Last \\ 
        \specialrule{1.5pt}{1pt}{1pt}
         \text{Clean} & $C=10, \gamma=1$    & \cellcolor{tablegreen} \textbf{0.968} &  0.966 & 0.968 & \cellcolor{tablegreen} 0.968 & 0.960 & 0.960 & 0.966 & 0.964 & 0.940 & 0.940 & 0.941 & 0.941 & 0.881 & 0.881 & 0.966 & 0.966  \\
         $D^{\text{adv}}=5\%$ & $C=10, \gamma=1$   & \cellcolor{tablegreen} \textbf{0.966} &  \cellcolor{tablegreen} 0.966 & 0.965 & 0.957 & 0.926 & 0.926 & 0.964 & 0.937 & 0.940 & 0.940 &  0.903 &  0.903 & 0.881 & 0.881 & 0.964 & 0.964   \\
         $D^{\text{adv}}=10\%$ & $C=10, \gamma=1$   & 0.924  & \cellcolor{tablegreen} 0.907 & 0.912 & 0.900 & 0.859 & 0.855 & \cellcolor{tablegreen} \textbf{0.927} & 0.853 & 0.869 & 0.868 & 0.907 & 0.907 & 0.894 & 0.894 & 0.908 & 0.907  \\
         $D^{\text{adv}}=15\%$ & $C=10, \gamma=1$  & \cellcolor{tablegreen} \textbf{0.924} &  \cellcolor{tablegreen} 0.917 & 0.892 & 0.823 & 0.826 & 0.826 & 0.871 & 0.871 & 0.893 & 0.829 & 0.906 & 0.858 & 0.892 & 0.823 & 0.883 & 0.826\\
         $D^{\text{adv}}=20\%$ & $C=10, \gamma=1$   & \cellcolor{tablegreen} \textbf{0.875} & \cellcolor{tablegreen} 0.865 & 0.840 & 0.771 & 0.788 & 0.787 & 0.854 & 0.853 & 0.839 & 0.758 & 0.859 & 0.763 & 0.840 & 0.771 & 0.830 & 0.755 \\
        $D^{\text{adv}}=25\%$ & $C=10, \gamma=1$   & \cellcolor{tablegreen} \textbf{0.801} & \cellcolor{tablegreen} 0.768 & 0.753 & 0.717 & 0.693 & 0.647 & 0.740 & 0.655 & 0.754 & 0.693 & 0.779 & 0.697 & 0.766 & 0.721 & 0.747 & 0.690  \\ 
         \specialrule{1.5pt}{1pt}{1pt}
    \end{tabular}
    }
\end{table*}
\begin{table*}[t]
\centering
    \caption{Test accuracies of methods trained on the IMDB dataset, averaged over five replications. 
\textsc{Floral} demonstrates superior robustness compared to baselines, particularly in more adversarial scenarios. See also Figures~\ref{fig:roberta-motivation}-\ref{fig:roberta-motivation-floral}.}
    \label{tab:test-accuracy-comp-imdb-additional} 
    \resizebox{1\linewidth}{!}{%
    \begin{tabular}{l|cccccccccccccccc} \specialrule{1.5pt}{1pt}{1pt}
         \multicolumn{1}{c}{\multirow{2}{*}{\makecell{\\ \\ \textbf{Setting}}}} & \multicolumn{16}{c}{\textbf{Method}} \\ \cmidrule(lr){2-16}
        & \multicolumn{2}{c}{\textsc{Floral}}  & \multicolumn{2}{c}{RoBERTa} & \multicolumn{2}{c}{SVM} & \multicolumn{2}{c}{LN-SVM} & \multicolumn{2}{c}{Curie} & \multicolumn{2}{c}{LS-SVM} &  \multicolumn{2}{c}{K-LID} \\ 
        \cmidrule(lr){2-3} \cmidrule(lr){4-5} \cmidrule(lr){6-7} \cmidrule(lr){8-9} \cmidrule(lr){10-11} \cmidrule(lr){12-13} \cmidrule(lr){14-15} \cmidrule(lr){16-17}
        &  Best & Last & Best & Last & Best & Last & Best & Last & Best & Last & Best & Last & Best & Last & \\ \specialrule{1.5pt}{1pt}{1pt}
         \text{Clean}  & 0.9113 & 0.9113 & \cellcolor{tablegreen} \textbf{0.9119} & 0.9110 & 0.9113 &  0.9113 & 0.9113 & 0.9113 & 0.9116 & 0.9113 & 0.9108 & 0.9108 & 0.9116 & \cellcolor{tablegreen} 0.9115 & \\
         $D^{\text{adv}}=10\%$  & 0.9039 & \cellcolor{tablegreen} 0.9039 & \cellcolor{tablegreen}  \textbf{0.9048} & 0.9031 & 0.9039  &  0.9039 &  0.9029 &  0.9028 & 0.9039 & 0.9039 & 0.9010 & 0.9010 & 0.9039 & 0.9039 & \\
         $D^{\text{adv}}=25\%$  & \cellcolor{tablegreen} \textbf{0.8896} & \cellcolor{tablegreen} 0.8896 & 0.8827 & 0.8612 & 0.8887 & 0.8886 &  0.8860 &  0.8860 & 0.8889 & 0.8888 & 0.8771 & 0.8769 & 0.8885 & 0.8883 & \\
         $D^{\text{adv}}=30\%$ & \cellcolor{tablegreen} \textbf{0.8801} & \cellcolor{tablegreen} 0.8801 & 0.8675 & 0.8357 & 0.8792 & 0.8792 &  0.8771 &  0.8771 & 0.8797 & 0.8797 & 0.8325 & 0.8324 & 0.8795 & 0.8795 & \\
         $D^{\text{adv}}=35\%$ & \cellcolor{tablegreen} \textbf{0.8713} & \cellcolor{tablegreen} 0.8713 & 0.8270 & 0.8053 & 0.8660 & 0.8660 & 0.8646 & 0.8646 & 0.8695 & 0.8695 & 0.7667 & 0.7667 & 0.8700 & 0.8700 & \\ 
         $D^{\text{adv}}=40\%$ & \cellcolor{tablegreen} \textbf{0.8636} & \cellcolor{tablegreen} 0.8636 & 0.7792 & 0.7717 & 0.8574 & 0.8584 & 0.8515 & 0.8515 & 0.8589 & 0.8589 & 0.7060 & 0.7060 & 0.8594 & 0.8594 & \\
         \specialrule{1.5pt}{1pt}{1pt}
    \end{tabular}
    }
\end{table*}
\paragraph{Moon.} 
As reported in Table~\ref{tab:test-accuracy-comp-moon} and Figure~\ref{fig:moon-exp-results-C10-gamma0.5}, \textsc{Floral} achieves higher robust accuracy across almost all settings compared to baseline methods. Notably, in scenarios with more severe poisoning levels, \textsc{Floral} significantly outperforms all baselines, which experience a marked drop in their accuracy.
We report results under various kernel hyperparameters in Appendix~\ref{sec:moon-results-appendix} (see Table~\ref{tab:test-accuracy-comp-moon-appendix} and Figure~\ref{fig:moon-exp-results-plots-appendix}).
We additionally visualize the decision boundaries of trained methods on the test dataset in Figure~\ref{fig:moon-decision-boundaries-C10-gamma0.5}, which shows that \textsc{Floral} produces a smoother decision boundary compared to the baselines, promoting generalization (see Figure~\ref{fig:moon-decision-boundaries-app-C10-gamma0.5} in Appendix~\ref{app:additional-experiment-results} for the complete results). \looseness -1

When the initial training data is clean, \textsc{Floral} provides a \textit{proactive defense} by introducing adversarial labels during training, thereby effectively reducing the model's sensitivity to potential label attacks. 
Notably, \textsc{Floral} matches performance on par with vanilla SVM on clean data, demonstrating that its robust framework maintains high accuracy without compromising clean accuracy.
In scenarios with already poisoned training data, \textsc{Floral} achieves robustness through two key mechanisms: $(i)$ implicitly sanitizing corrupted labels, while $(ii)$ introducing additional adversarial labels to further reduce model sensitivity to attacks.
We analyze this sanitization effect in detail in Appendix~\ref{app:defense-analysis} and show that \textsc{Floral} sanitizes $25-35\%$ portion of the initial poisoned training labels. Furthermore, \textsc{Floral} operates on \textit{dynamically evolving} adversarial datasets during training, unlike baselines that are trained on fixed adversarially labelled datasets. 
This dynamic strategy introduces new adversarial labels in each round, further testing and enhancing the model's robustness.
These capabilities position \textsc{Floral} as a superior defense over baselines, particularly in maintaining robust accuracy under more challenging adversarial scenarios.
\looseness -1

\textsc{Floral} offers several distinct advantages over existing baselines, e.g., unlike LN-SVM, \textsc{Floral} does not rely on the strong assumption that training labels are independently flipped with equal probability.
Compared to Curie, \textsc{Floral} avoids a filter-out system that risks discarding data with valuable feature representations.
Further, in terms of scalability, Curie’s dependence on distance metrics makes it vulnerable to the curse of dimensionality, diminishing its clustering performance in high-dimensional and complex datasets.
To ensure a fair comparison, we calibrated the noise, confidence level, and threshold value parameters of LN-SVM, Curie, and LS-SVM baselines, aligning them with the poisoning level in each dataset. This ensures that the baselines are at their strongest configurations, highlighting the robustness and scalability of \textsc{Floral}.
\looseness -1
 \begin{figure*}[t]
    \centering  
    \begin{subfigure}[t]{0.2\linewidth}\centering{\includegraphics[width=1\linewidth,trim=0 0 0 0,clip]{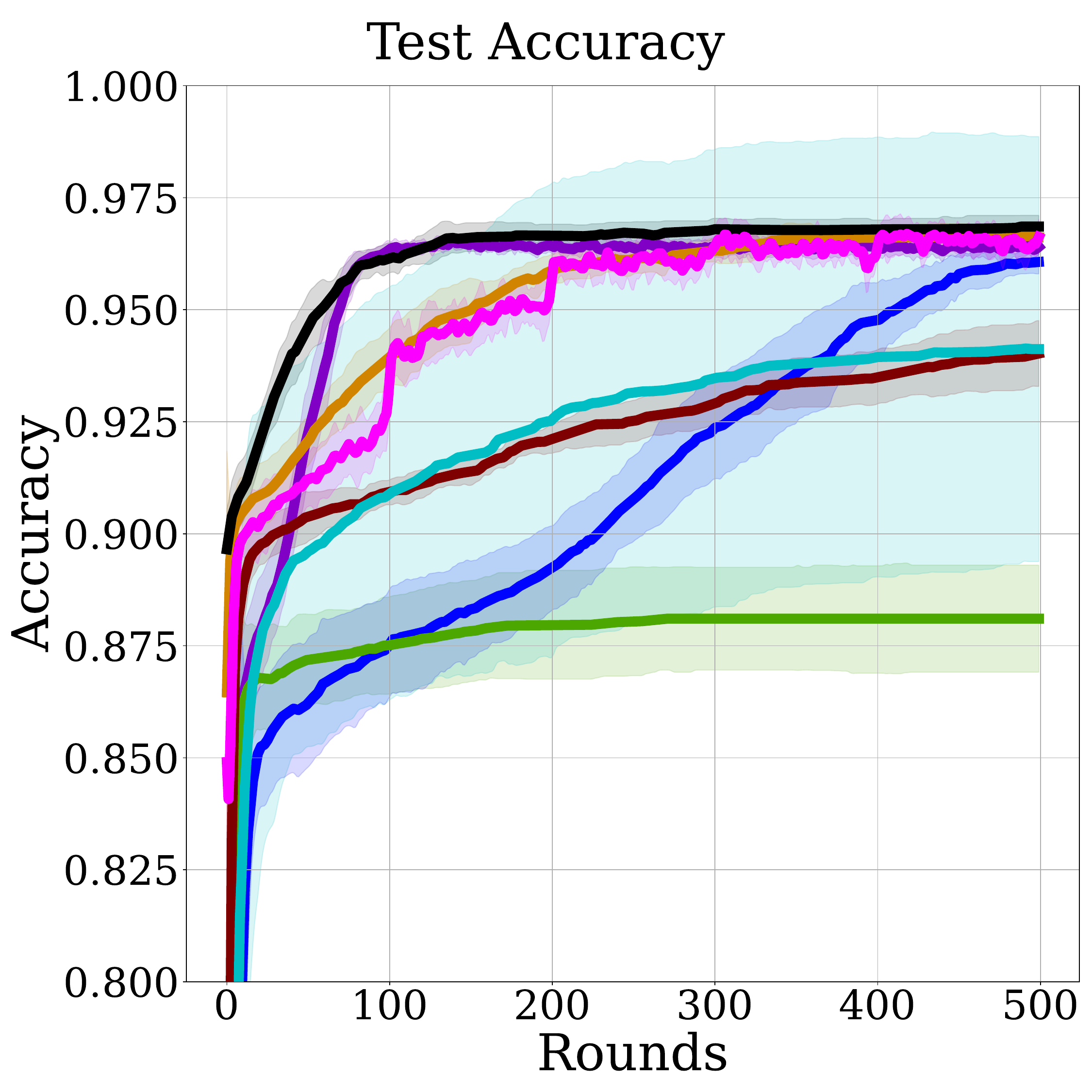}}
    \caption{Clean.}
    \end{subfigure}%
    \begin{subfigure}[t]{0.2\linewidth}\centering{\includegraphics[width=1\linewidth,trim=0 0 0 0,clip]{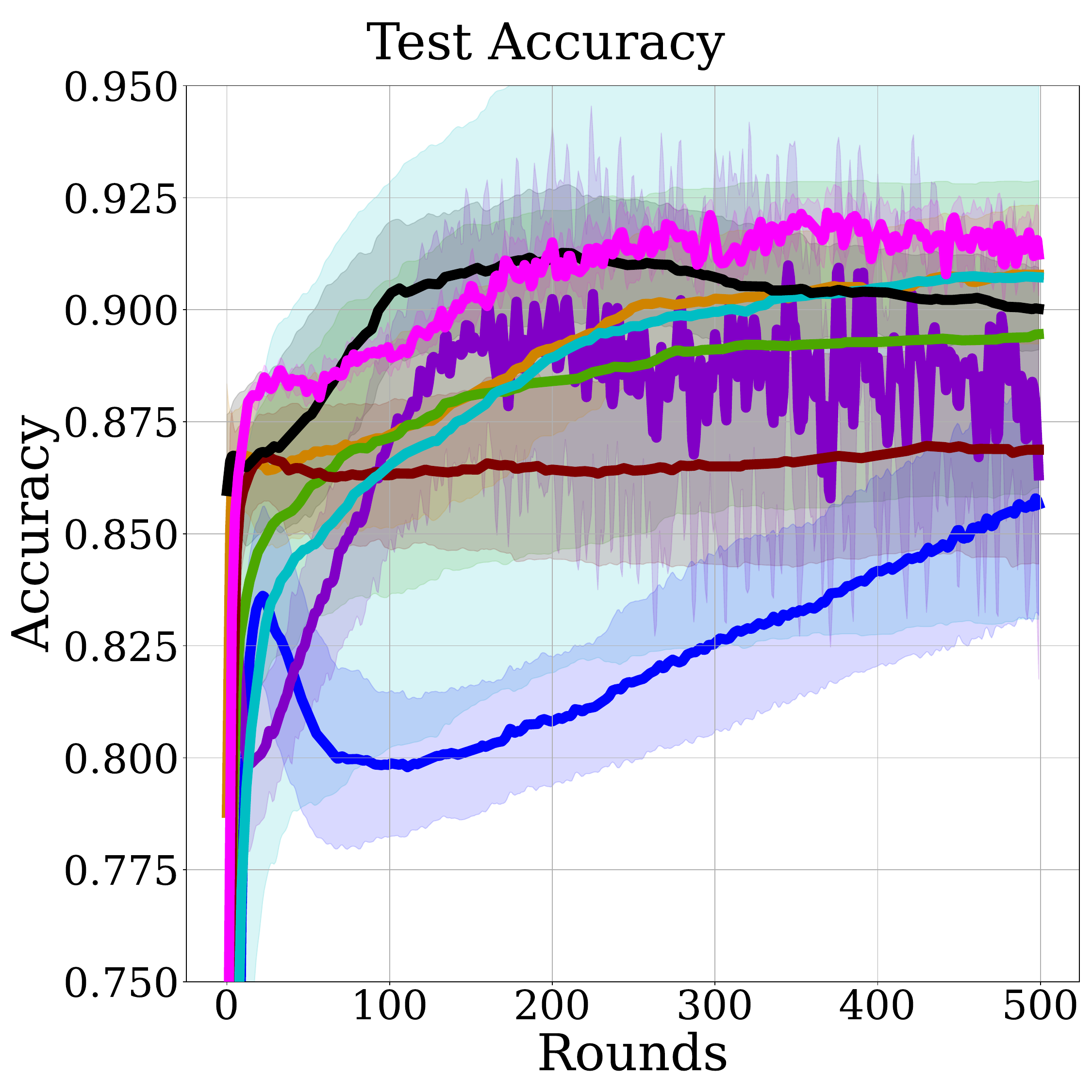}}
    \caption{$D^{\text{adv}} = 10\%$.}
    \end{subfigure}%
    \begin{subfigure}[t]{0.2\linewidth}\centering{\includegraphics[width=1\linewidth,trim=0 0 0 0,clip]{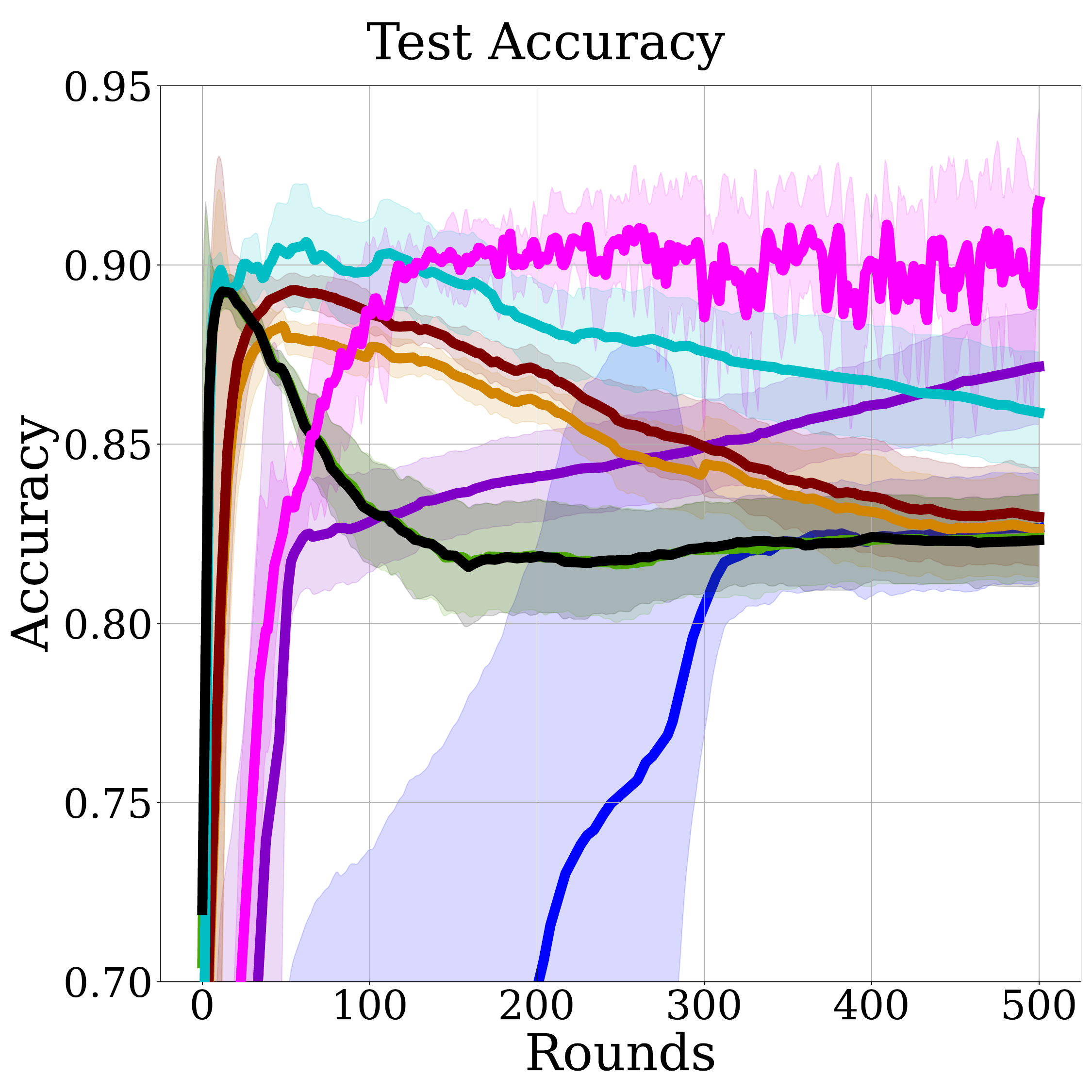}}
    \caption{$D^{\text{adv}} = 15\%$.}
    \end{subfigure}%
    \begin{subfigure}[t]{0.2\linewidth}\centering{\includegraphics[width=1\linewidth,trim=0 0 0 0,clip]{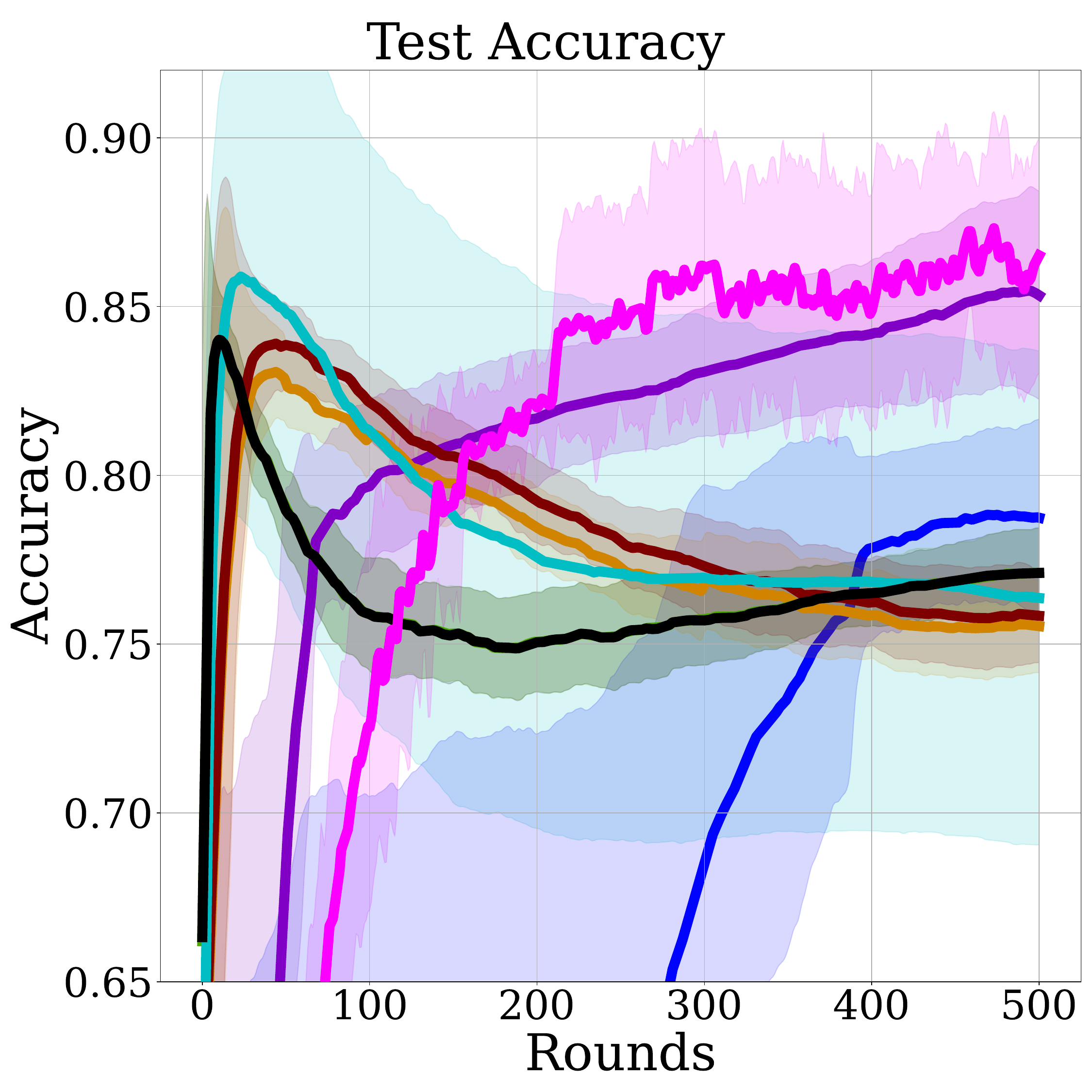}}
    \caption{$D^{\text{adv}} = 20\%$.}
    \end{subfigure}%
    \begin{subfigure}[t]
    {0.2\linewidth}\centering{\includegraphics[width=1\linewidth,trim=0 0 0 0,clip]{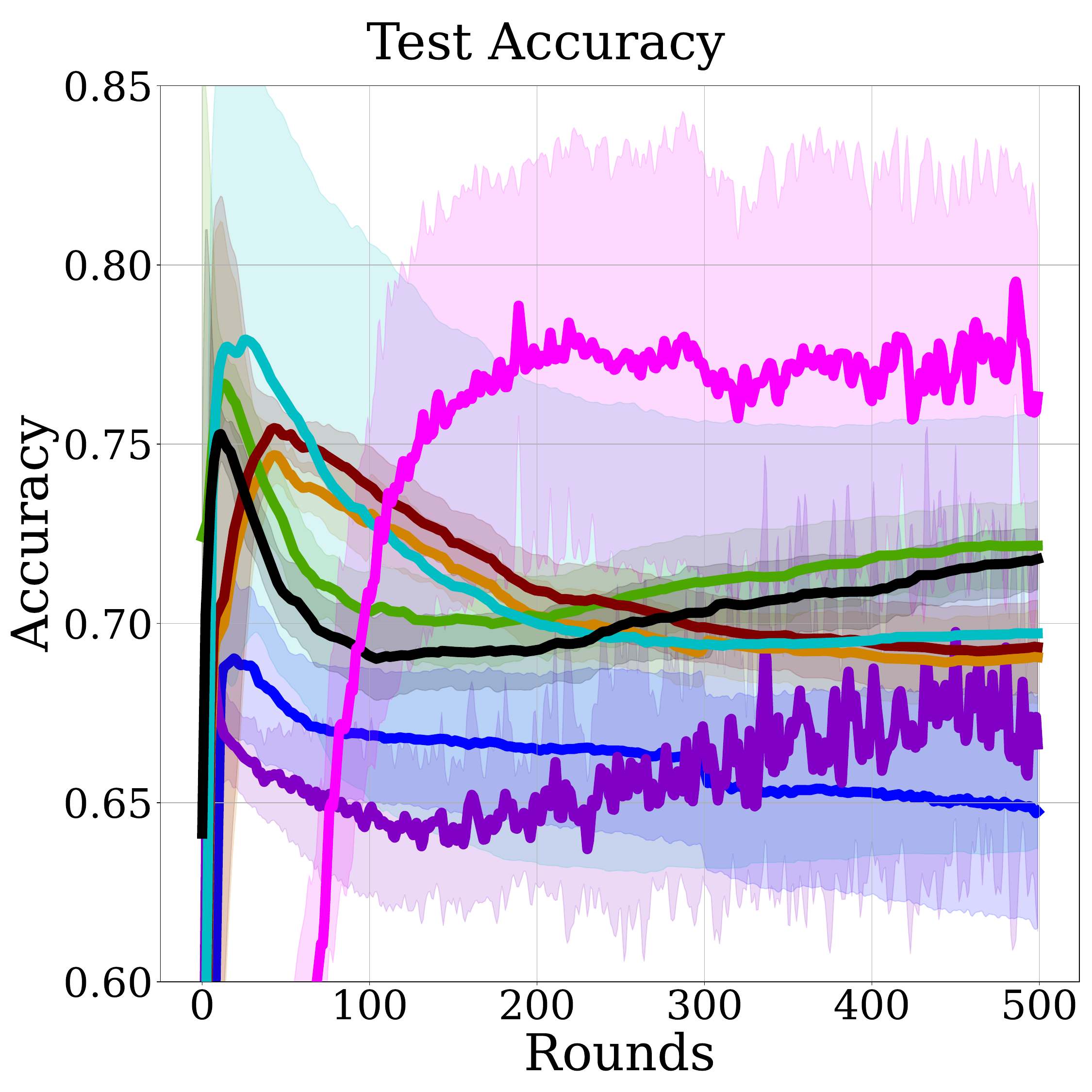}}
    \caption{$D^{\text{adv}} = 25\%$.}
    \end{subfigure}%
    \\
    \begin{subfigure}{1\linewidth}\centering{\includegraphics[width=1\linewidth,trim=40 10 10 10,clip]{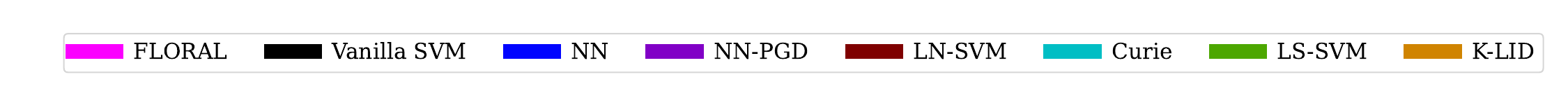}}
    \end{subfigure}%
    \caption{Test accuracy of methods on the Moon dataset under varying label poisoning levels.
    For SVM models, $C=10$, $\gamma=1$ are used. 
    See Appendix~\ref{app:additional-experiment-results} (Figure~\ref{fig:moon-exp-results-plots-appendix}) for results with other settings.
    As the label poisoning level increases, the accuracy of methods generally declines, however, \textsc{Floral} maintains higher robust accuracy across all adversarial settings, without compromising clean accuracy. \looseness -1
    }
\label{fig:moon-exp-results-C10-gamma0.5}
\vspace{-0.2cm}
\end{figure*}
\vspace{-0.2cm}
\paragraph{IMDB.} 
We integrate \textsc{Floral} as a robust classifier head for RoBERTa. 
The test performance on the IMDB dataset, shown in Table~\ref{tab:test-accuracy-comp-imdb-additional} and Figures~\ref{fig:roberta-motivation}-\ref{fig:roberta-motivation-floral}, reveals that \textsc{Floral} significantly improves robustness, outperforming fine-tuned RoBERTa along with other baselines. 
Our approach also converges faster to lower loss values, in more adversarial scenarios (see Table~\ref{tab:test-accuracy-loss-comp-imdb} and Figure~\ref{fig:imdb-result-comp-roberta-svm} in Appendix~\ref{sec:imdb-results-appendix}).
\looseness -1

In Appendix~\ref{sec:imdb-results-appendix}, we analyze the changes in influential training points—those that most affect model predictions—when applying \textsc{Floral} to RoBERTa-extracted embeddings (see Figures~\ref{fig:percentage-common-influential-points}-\ref{fig:most-important-points-roberta-vs-floral}).
The results reveal some overlap in the identified points, under clean train data scenario. 
However, as the training data becomes more adversarial, \textsc{Floral} identifies different critical points, which effectively shape the decision boundary and contribute to improved robust accuracy.
\looseness -1
\vspace{-0.2cm}
\paragraph{Adaptability.} As shown with the IMDB experiments, \textsc{Floral} integrates seamlessly with other model architectures (e.g., RoBERTa, NNs) by utilizing the last-layer embeddings for training.
We additionally demonstrate this in Appendix~\ref{app:nn-integration}, \textsc{Floral} integrated with an NN learns more robust representations and achieves higher robust accuracy on the Moon and MNIST datasets.
\looseness -1
\vspace{-0.2cm}
\paragraph{Sensitivity analysis.}
We further examined the sensitivity of our approach to the attacker's budget, and the results are detailed in Appendix~\ref{sec:sensitivity-analysis} with Figure~\ref{fig:moon-sensitivity-attacker-budget}. 
\looseness -1 
\vspace{-0.2cm}
\paragraph{Generalizability.} We demonstrate the effectiveness of \textsc{Floral} under different attacks: \texttt{alfa}, \texttt{alfa-tilt} \citep{xiao2015support} and \texttt{LFA} \citep{label-sanitization} in Appendix~\ref{app:different-attacks}.
Our experiments on the Moon and MNIST \citep{deng2012mnist} datasets 
again confirmed that \textsc{Floral} can also defend and achieve higher robust accuracy in the presence of other types of label attacks. 
\looseness -1
\vspace{-0.2cm}
\paragraph{Limitations.} Defense strategies may not be universally effective against all label poisoning attacks due to their non-adaptive nature \citep{papernot2016transferability}. 
Our defense strategy relies on a white-box attack, where the attacker can access the model. 
While we also show the performance of our approach under various label attacks from the literature, its efficacy may vary under different attack scenarios. \looseness -1
\vspace{-0.2cm}

\section{Conclusion} 
\label{sec:conclusions}
\vspace{-0.1cm}
In this paper, we address the vulnerability of machine learning models to label poisoning attacks and propose \textsc{Floral}, an adversarial training defense strategy based on kernel SVMs. We formulate the problem using bilevel optimization and frame the adversarial interaction between the learning model and the attacker as a non-zero-sum Stackelberg game. To compute the game equilibrium that solves the optimization problem, we introduce a projected gradient descent-based algorithm and analyze its local stability and convergence properties. Our approach demonstrates superior empirical robustness across various classification tasks compared to robust baseline methods. 

Future research includes exploring SVM-based transfer attacks or integrating our approach to robust fine-tuning of foundation models for supervised downstream tasks. 
Additionally, a detailed analysis of how \textsc{Floral} alters the most influential training points for model predictions, e.g. when integrated with foundation models such as RoBERTa could provide interesting insights. \looseness -1 
\vspace{-0.2cm}
 \begin{figure*}[t]
    \centering  
    \begin{subfigure}{0.25\textwidth}\centering{\includegraphics[width=1\linewidth,trim=0 20 0 55,clip]{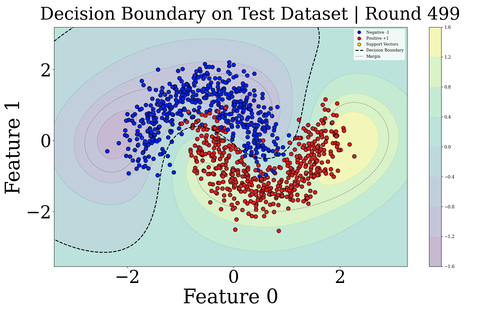}}
    \caption{\textsc{Floral} (Clean).}
    \end{subfigure}%
    \begin{subfigure}{0.25\textwidth}\centering{\includegraphics[width=1\linewidth,trim=0 20 0 55,clip]{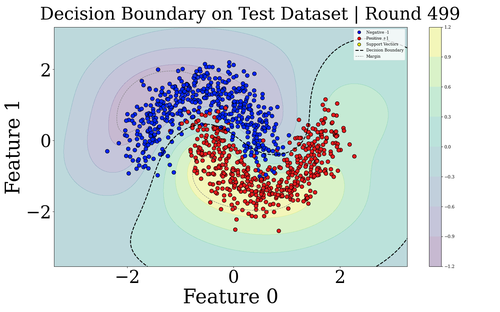}}
    \caption*{$D^{\text{adv}}= 5\%$.}
    \end{subfigure}%
    \begin{subfigure}
    {0.25\textwidth}\centering{\includegraphics[width=1\linewidth,trim=0 20 0 55,clip]{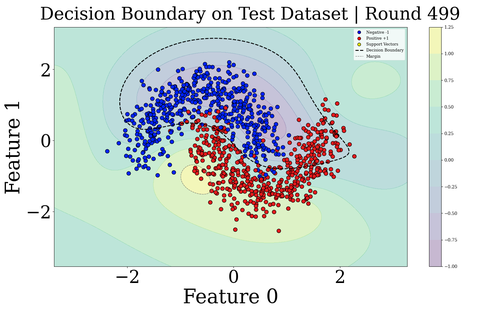}}
    \caption*{$D^{\text{adv}}= 10\%$.}
    \end{subfigure}%
    \begin{subfigure}{0.25\textwidth}\centering{\includegraphics[width=1\linewidth,trim=0 20 0 55,clip]{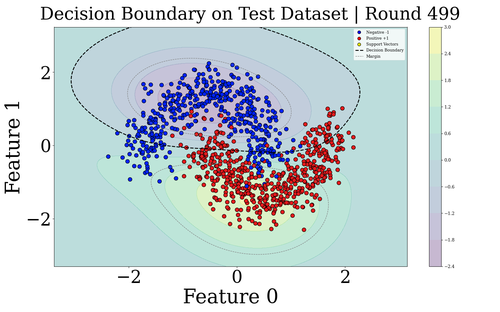}}
    \caption*{$D^{\text{adv}}= 25\%$.}
    \end{subfigure}%
    \\
    \begin{subfigure}{0.25\textwidth}\centering{\includegraphics[width=1\linewidth,trim=0 20 0 55,clip]{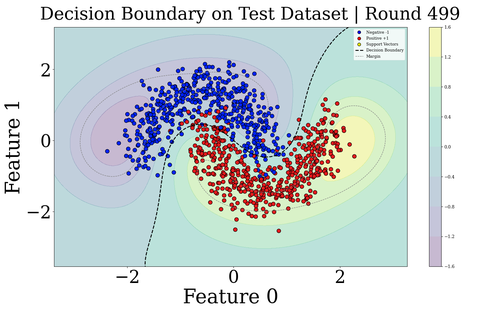}}
    \caption{SVM (Clean).}
    \end{subfigure}%
    \begin{subfigure}{0.25\textwidth}\centering{\includegraphics[width=1\linewidth,trim=0 20 0 55,clip]{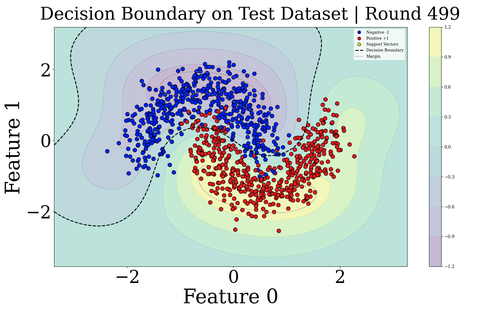}}
    \caption*{$D^{\text{adv}}= 5\%$.}
    \end{subfigure}%
    \begin{subfigure}{0.25\textwidth}\centering{\includegraphics[width=1\linewidth,trim=0 20 0 55,clip]{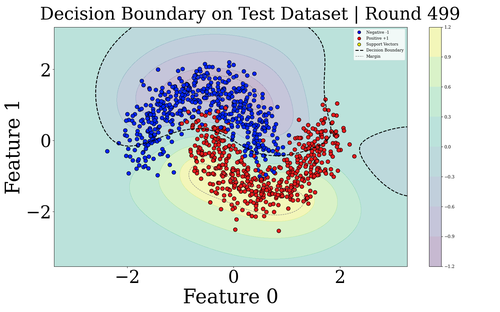}}
    \caption*{$D^{\text{adv}}= 10\%$.}
    \end{subfigure}%
    \begin{subfigure}{0.25\textwidth}\centering{\includegraphics[width=1\linewidth,trim=0 20 0 55,clip]{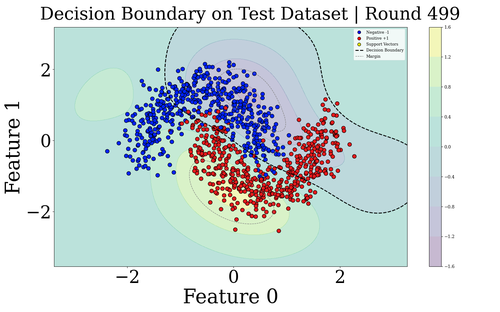}}
    \caption*{$D^{\text{adv}}= 25\%$.}
    \end{subfigure}%
    \\
    \begin{subfigure}{0.25\textwidth}\centering{\includegraphics[width=1\linewidth,trim=0 20 0 55,clip]{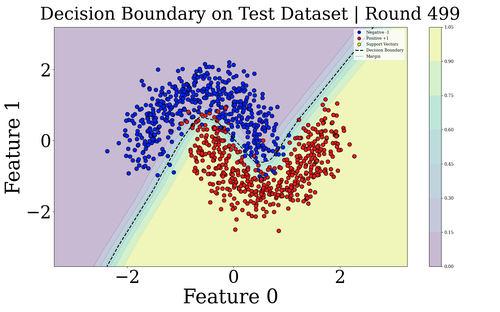}}
    \caption{NN (Clean).}
    \end{subfigure}%
    \begin{subfigure}{0.25\textwidth}\centering{\includegraphics[width=1\linewidth,trim=0 20 0 55,clip]{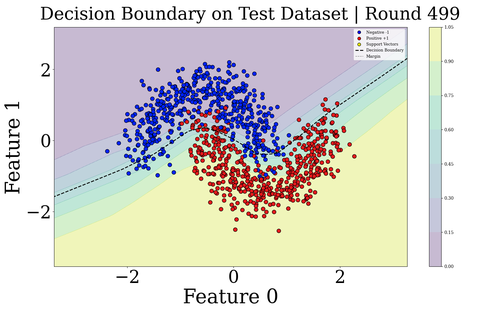}}
    \caption*{$D^{\text{adv}}= 5\%$.}
    \end{subfigure}%
    \begin{subfigure}{0.25\textwidth}\centering{\includegraphics[width=1\linewidth,trim=0 20 0 55,clip]{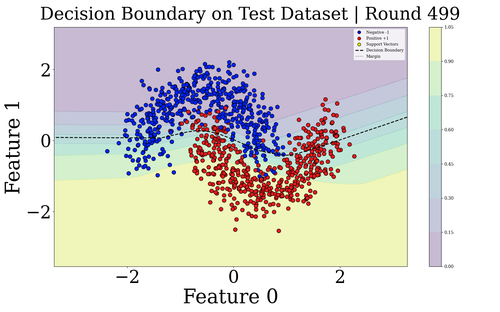}}
    \caption*{$D^{\text{adv}}= 10\%$.}
    \end{subfigure}%
    \begin{subfigure}{0.25\textwidth}\centering{\includegraphics[width=1\linewidth,trim=0 20 0 55,clip]{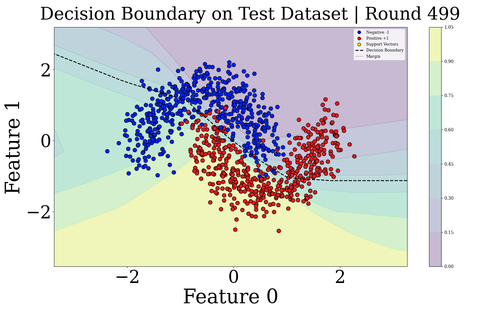}}
    \caption*{$D^{\text{adv}}= 25\%$.}
    \end{subfigure}%
    \\
    \begin{subfigure}{0.25\textwidth}\centering{\includegraphics[width=1\linewidth,trim=0 20 0 55,clip]{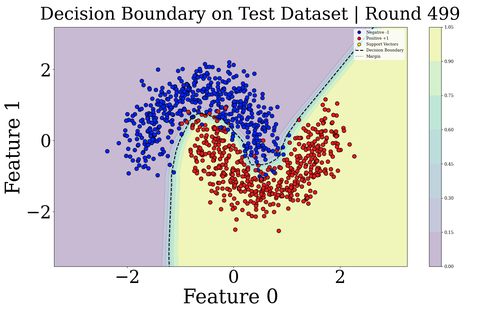}}
    \caption{NN-PGD (Clean).}
    \end{subfigure}%
    \begin{subfigure}{0.25\textwidth}\centering{\includegraphics[width=1\linewidth,trim=0 20 0 55,clip]{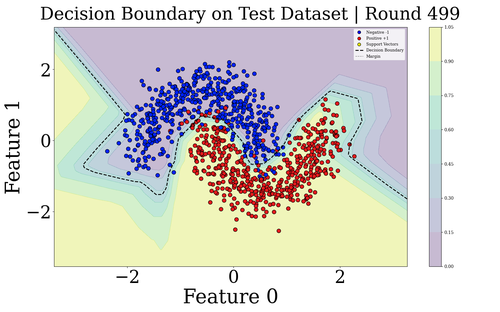}}
    \caption*{$D^{\text{adv}}= 5\%$.}
    \end{subfigure}%
    \begin{subfigure}{0.25\textwidth}\centering{\includegraphics[width=1\linewidth,trim=0 20 0 55,clip]{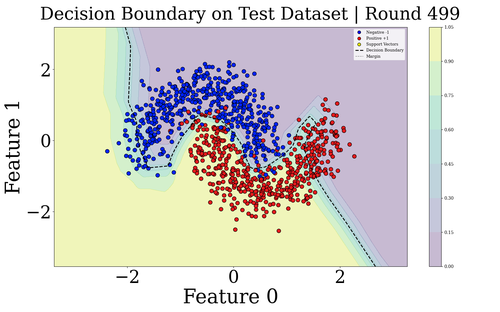}}
    \caption*{$D^{\text{adv}}= 10\%$.}
    \end{subfigure}%
    \begin{subfigure}{0.25\textwidth}\centering{\includegraphics[width=1\linewidth,trim=0 20 0 55,clip]{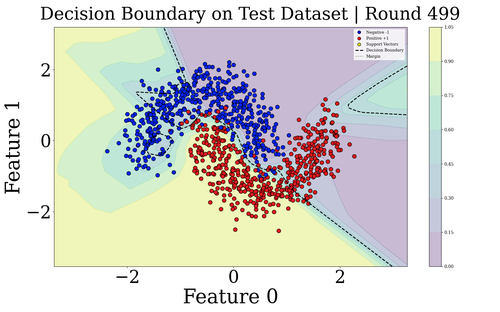}}
    \caption*{$D^{\text{adv}}= 25\%$.}
    \end{subfigure}%
    \caption{The decision boundaries on the Moon test dataset under varying label poisoning levels. 
    SVM models use an RBF kernel with $C=10$ and $\gamma=0.5$. 
    \textsc{Floral} generates a smooth decision boundary compared to baseline methods, which show drastic changes due to adversarial training label manipulations. 
    For the complete results with other baselines, see Appendix~\ref{app:additional-experiment-results} (Figure~\ref{fig:moon-decision-boundaries-app-C10-gamma0.5}).
    \looseness -1
     }
\label{fig:moon-decision-boundaries-C10-gamma0.5}
\end{figure*}

\clearpage

\section*{Acknowledgements}
\looseness -1 \vspace{-0.2cm}
This research was supported by the Max Planck \& Amazon Science Hub. We also thank the German Research Foundation for the support and Zhiyu He for the helpful comments on the manuscript. %
The work was conducted during Volkan Cevher's time at Amazon.

\bibliography{refs}
\bibliographystyle{iclr2025_conference}

\newpage
\appendix

\addcontentsline{toc}{section}{Appendix} %
\part{Appendix} %
\parttoc %
\newpage

\section{Theoretical Analysis Proofs}
\label{app:convergence-analysis-proofs}

In this section, we present the proofs for the local asymptotic stability analysis of \textsc{Floral} (Algorithm~\ref{alg:robust-svm-game}). 
We begin by proving Lemma~\ref{thm:convergence-bound-lambda-steps} in Section~\ref{app:proof-of-lemma-1}, which establishes that the distance of the updates of Algorithm~\ref{alg:robust-svm-game} from the equilibrium of the game is bounded.
In Section~\ref{app:proof-of-lemma-2}, we prove Lemma~\ref{thm:convergence-bound-x-steps}, demonstrating that the distance of the non-projected updates from the equilibrium of the game is also bounded.
Lastly, in Section~\ref{app:proof-local-stability}, we provide the proof of Theorem~\ref{thm:stability}, which shows the local asymptotic stability of our algorithm, with a derivation of a global convergence result presented in Section~\ref{app:global-convergence-result}.

\subsection{Proof of Lemma~\ref{thm:convergence-bound-lambda-steps}}
\label{app:proof-of-lemma-1}

\begin{proof}[\unskip\nopunct]
Our objective is to prove that the distance of the iterates of Algorithm~\ref{alg:robust-svm-game} from the Stackelberg equilibrium $(\Hat{\lambda}, \Hat{y}(\hat{\lambda}))$, specifically $\lambda_{t} -\Hat{\lambda}$, is bounded.
We begin by recalling the update rule at round $t$, $\lambda_t := \textsc{Prox}_{\mathcal{S}(y_t)}(z_{t}) = \textsc{Prox}_{\mathcal{S}(y_t)} (\lambda_{t-1} - \eta \nabla_{\lambda} f (\lambda_{t-1}, y_t))$, where $y_t = \Hat{y}(\lambda_{t-1})$, $\mathcal{S}(y_t)$ is the feasible region defined by constraints (\ref{proj-ineq-const}), using the labels at round $t$. The operator $\textsc{Prox}$ is defined below. 

\begin{definition}[\textsc{Prox} operator]
The operator \textsc{Prox}$_{\mathcal{S}(y_t)}(z_t): \mathbb{R}^n \to \mathbb{R}^n$ denotes the projection of $z_t \in \mathbb{R}^n$ onto the convex set $\mathcal{S}(y_t)$ at round $t$ of Algorithm~\ref{alg:robust-svm-game}. \textsc{Prox} minimizes the Euclidean distance and is defined by the following optimization problem:
\begin{alignat}{2}
\textsc{Prox}_{\mathcal{S}(y_t)}(z_t): \lambda_t \in \arg \min _{\lambda \in \mathbb{R}^{n}} \quad  \frac{1}{2} \| & \lambda - z_{t} \|^{2} \label{app:eq:svm-projection-operation} \\
\st  & y_t^{\mathrm{T}}  \lambda = 0  \label{app:eq:proj-eq-const} \\
& 0 \leq \lambda \leq C \label{app:eq:proj-ineq-const}
\end{alignat}
Equivalently, \textsc{Prox}$_{\mathcal{S}(y_t)}$ solves the following optimization problem:
\begin{align}
\min _{\substack{\lambda \in \mathbb{R}^{n} \\ 0 \leq \lambda \leq C}} \quad \sup_{\mu \in \mathbb{R}} \quad  \frac{1}{2} \| \lambda - z_{t} \|^{2} + \mu y_t^{\mathrm{T}}  \lambda.
\label{app:eq:svm-prox-operation}
\end{align}
\end{definition}

\begin{lemma}[Bounded iterates] \label{app:thm:bounded-iterates}
The sequence $\{ \lambda_t\}$ generated by the iterative update rule $\lambda_t := \textsc{Prox}_{\mathcal{S}(y_t)}(z_{t}) = \textsc{Prox}_{\mathcal{S}(y_t)} (\lambda_{t-1} - \eta \nabla_{\lambda} f (\lambda_{t-1}, y_t))$ is bounded, i.e., $\| \lambda_t \|_\infty \leq C, \forall t \geq 0$. 
\end{lemma}
\vspace{-0.2cm}
\begin{proof}
This follows immediately from the definition of $\mathcal{S}(y_t)$.
\end{proof}
\vspace{-0.2cm} 
In the following, our aim is to quantify the sensitivity of  (\ref{app:eq:svm-prox-operation}) with respect to its arguments $y_t$ and $z_t$.
Let $\lambda^{\star}$ denote the optimal solution to the projection operation. We can express this solution through the following steps.
First, we simplify the expression by omitting the index $t$ in (\ref{app:eq:svm-prox-operation}).
Then, we exploit the fact that the objective function is convex-concave with convex constraints, which allows us to interchange the order of the $\min$ and the $\sup$. This yields
\begin{align}
& \min _{\substack{\lambda \in \mathbb{R}^{n} \\ 0 \leq \lambda \leq C}} \quad \sup_{\mu \in \mathbb{R}} \quad  \frac{1}{2} \| \lambda - z \|^{2} + \mu y^{\mathrm{T}}  \lambda
\nonumber \\
&= \sup_{\mu \in \mathbb{R}} \quad \min _{\substack{\lambda \in \mathbb{R}^{n} \\ 0 \leq \lambda \leq C}} \frac{1}{2} \| \lambda - z + \mu y \|^{2} - \frac{1}{2} \mu^{2} \| y \|^{2}.
\nonumber
\end{align}
At this stage, the optimization problem over $\lambda$ reduces to the minimization of a quadratic function over box constraints. 
We can therefore express $\lambda^{\star}$ based on the optimal choice $\mu^{\star}$ for $\mu$ as follows:
\begin{equation}
\lambda^{\star}_i = \left.
\begin{cases}
0, & \text{ if } z_i - \mu^{\star}(z,y) y_i \leq 0 \\
z_i - \mu^{\star}(z,y) y_i, & \text{ if } 0 <  z_i - \mu^{\star}(z,y) y_i < C  \\
C, & \text{ if } z_i - \mu^{\star}(z,y) y_i \geq C \\
\end{cases}
\right\} \text{ and choose }  \mu^{\star}(z,y) \text{ such that } y^{\mathrm{T}} \lambda^{\star} = 0,
\nonumber
\end{equation}
$\forall i \in \{1,\dots,n\}$. We use the notation $\mu^{\star}(z,y)$ to highlight the dependency of the multiplier $\mu^{\star}$ on the variable $z$ and the label $y$.

We introduce the $\textsc{Clip}_{[0,C]}(\cdot)$ operator which clips the value of the given input to the interval $[0,C]$. 
This operator yields the following compact expression for $\lambda^{\star}$:
\begin{equation}
\lambda^{\star} = \textsc{Clip}_{[0,C]}(z - \mu^{\star}(z,y) y) \text{, where } \mu^{\star}(z,y) \text{ is chosen such that } y^{\mathrm{T}} \lambda^{\star} = 0.
\label{app:eq:opt-lambda-clip-expression}
\end{equation}
By substituting the previous expression for $\lambda^{\star}$ into the equality constraint, we obtain
\begin{equation}
y^{\mathrm{T}} \textsc{Clip}_{[0,C]}(z - \mu^{\star}(z,y)y) = 0,
\label{app:eq:opt-lambda-eqn}
\end{equation}
which provides an equation that implicitly defines $\mu^{\star}(z,y)$. We further simplify (\ref{app:eq:opt-lambda-eqn}) by indexing the components of $z- \mu^{\star}(z,y)y$ with respect to their values, which yields
\begin{align}
0 &= \sum_{i \in \mathcal{I}_C} C y_i + \sum_{i \in \mathcal{I}_z} y_i(z_i - \mu^{\star}(z,y)y_i) \nonumber \\
&= \sum_{i \in \mathcal{I}_C} C y_i + \sum_{i \in \mathcal{I}_z} z_i y_i - \sum_{i \in \mathcal{I}_z} \mu^{\star}(z,y) y_i^2 \nonumber \\
&= \sum_{i \in \mathcal{I}_C} C y_i + \sum_{i \in \mathcal{I}_z} z_i y_i - \sum_{i \in \mathcal{I}_z} \mu^{\star}(z,y). \tag{from $y_i^2=1$} 
\end{align}
where $\mathcal{I}_z:=\{i \mid \lambda_i = z_i - \mu^{\star}(z,y)y_i \in (0,C) \}$ with cardinality $\mid \mathcal{I}_z \mid$ and $\mathcal{I}_C:=\{i \mid \lambda_i = z_i - \mu^{\star}(z,y)y_i \geq C \}$.
We further solve for $\mu^{\star}$, which yields 
\begin{align}
\mu^{\star}(z,y) &= \frac{1}{\mid \mathcal{I}_z\mid} \left( \sum_{i \in \mathcal{I}_C} C y_i + \sum_{i \in \mathcal{I}_z} z_i y_i \right). \nonumber
\end{align}
This equation implicitly defines $\mu^{\star}(z,y)$, which represents the basis for the fixed point iteration introduced in Algorithm~\ref{alg:projection-fpi}.

This equation will also be the basis for computing sensitivities, i.e. quantifying how $\lambda^{\star}$ and $\mu^{\star}$ change when altering $z$ or $\lambda$.
We first compute $\frac{\partial \lambda^{\star}}{\partial z}$. For a data point $i$, the following can be stated:
\begin{align}
    \frac{\partial \lambda^{\star}_i}{\partial z} &= \left.
    \begin{cases}
    e_i^{\mathrm{T}} - \frac{\partial \mu^{\star}(z,y)}{\partial z} y_i, & \text{ if } z_i-\mu^{\star}(z,y) y_i \in (0,C)\\
    0, & \text{ else, } 
    \end{cases}
\right.
\label{app:eq:diff-lambda-i-x}
\end{align}
where $e_i$ is the $i$th standard basis vector. 
Differentiating the constraint (\ref{app:eq:proj-eq-const}) yields
\begin{align}
    0 &= \frac{\partial (y^{\mathrm{T}}\lambda^{\star})}{\partial z} = \sum_{i=1}^{n} \frac{\partial \lambda^{\star}_i}{\partial z} y_i \nonumber \\
    &= \sum_{i \in \mathcal{I}_z } \left(e_i^{\mathrm{T}} y_i - \frac{\partial \mu^{\star}(z,y)}{\partial z}  y_i^2\right).
\nonumber
\end{align}

Substituting $y_i^2=1$ into the previous equation yields
\begin{align}
    \frac{\partial \mu^{\star}(z,y)}{\partial z} = \frac{\sum\limits_{i \in \mathcal{I}_z } e_i^{\mathrm{T}} y_i }{\mid \mathcal{I}_z \mid}, 
\label{app:eq:diff-mu-x-final}
\end{align}
where $\mathcal{I}_z$ with cardinality $\mid \mathcal{I}_z \mid$ is defined previously.
From (\ref{app:eq:diff-lambda-i-x}) and (\ref{app:eq:diff-mu-x-final}), we have
\begin{align}
    \frac{\partial \lambda^{\star}_i}{\partial z} = \left.
    \begin{cases}
    e_i^{\mathrm{T}} - \frac{\sum\limits_{j \in \mathcal{I}_z } e_j^{\mathrm{T}} y_j }{\mid \mathcal{I}_z \mid} y_i, & \text{if } i \in \mathcal{I}_z\\
    0, & \text{ if } i \notin \mathcal{I}_z. 
    \end{cases}
\right.
\nonumber
\end{align}
Therefore, we conclude that
\begin{align}
    \bigg\| \frac{\partial \lambda^{\star}_i}{\partial z} \bigg\|_{\infty} \leq 1, \forall i \in [n],
\label{app:eq:diff-lambda-i-x-final-norm-bound}
\end{align}
where we have exploited the fact that $y_i \in \{ \pm 1\}$.

We further note that in the situation $\mathcal{I}_z = \emptyset$, $\lambda^{\star} \in \{ 0, C\}$, a change in $z$ or $y$ will not affect $\lambda^{\star}$ unless $z = \mu^{\star}y$ or $z = C + \mu^{\star}y$. As a result, we have for $\mathcal{I}_z = \emptyset$, $\frac{\partial \lambda^{\star}}{\partial z} = \frac{\partial \lambda^{\star}}{\partial y} = 0$ (a.e.).

We now compute $\frac{\partial \lambda^{\star}}{\partial y}$. For a data point $i$, the following holds:
\begin{align}
    \frac{\partial \lambda^{\star}_i}{\partial y} = \left.
    \begin{cases}
    - \frac{\partial \mu^{\star}(z,y)}{\partial y}y_i - e_i^{\mathrm{T}} \mu^{\star}(z,y), & \text{ if } i \in \mathcal{I}_z\\
    0, & \text{ if } i \notin \mathcal{I}_z. 
    \end{cases}
\right.
\label{app:eq:diff-lambda-i-y}
\end{align}
Differentiating the constraint (\ref{app:eq:proj-eq-const}) with respect to $y$ yields
\begin{align}
    0 &= \frac{\partial (y^{\mathrm{T}}\lambda^{\star})}{\partial y} = \lambda^{{\star}^{\mathrm{T}}} + \sum\limits_{i=1}^{n} \frac{\partial \lambda^{\star}_i}{\partial y} y_i  \nonumber \\
    &= \lambda^{{\star}^{\mathrm{T}}} + \sum\limits_{i \in \mathcal{I}_z} \left(- \frac{\partial \mu^{\star}(z,y)}{\partial y} | y_i |^2 -  y_i \mu^{\star}(z,y) e_i^{\mathrm{T}}  \right).  \nonumber
\end{align}
It follows from $| y_i |^2=1$ that
\begin{align}
\frac{\partial \mu^{\star}(z,y)}{\partial y} = \frac{\lambda^{{\star}^{\mathrm{T}}} - \mu^{\star}(z,y) \sum\limits_{i\in \mathcal{I}_z} y_i e_i^{\mathrm{T}}}{\mid \mathcal{I}_z \mid}.
\label{app:eq:diff-mu-y}
\end{align}

From (\ref{app:eq:diff-lambda-i-y}) and (\ref{app:eq:diff-mu-y}), we obtain the following.
\begin{align}
    \frac{\partial \lambda^{\star}_i}{\partial y} = \left.
    \begin{cases}
    - \frac{y_i \lambda^{{\star}^{\mathrm{T}}}}{\mid \mathcal{I}_z \mid} + \mu^{\star}(z,y) \left( \frac{y_i \sum\limits_{j \in \mathcal{I}_z} e_j^{\mathrm{T}} y_j}{\mid \mathcal{I}_z \mid} - e_i^{\mathrm{T}}\right), & \text{ if } i \in \mathcal{I}_z\\
    0, & \text{ if } i \notin \mathcal{I}_z.
    \end{cases}
\right.
\nonumber
\end{align}
As a result, we conclude using Lemma~\ref{app:thm:bounded-iterates} that the following bound holds $\forall i \in [n]$
\begin{align}
    \bigg\| \frac{\partial \lambda^{\star}_i}{\partial y} \bigg\|_{\infty} \leq \frac{\| \lambda \|_{\infty}}{\mid \mathcal{I}_z \mid } + \mid \mu^{\star}(z,y) \mid \leq \underbrace{ \frac{C}{\mid \mathcal{I}_z \mid} + \mid \mu^{\star} \mid}_{\kappa_y},
\label{app:eq:diff-lambda-i-y-final-norm-bound}
\end{align}
where $\kappa_y$ is a constant that only depends on $C$ and the features of the dataset.

From (\ref{app:eq:diff-lambda-i-x-final-norm-bound}) and (\ref{app:eq:diff-lambda-i-y-final-norm-bound}), we conclude that 
\begin{align}
    \| \lambda_{t} - \Hat{\lambda} \|_{\infty} &= \| \lambda^{\star}(z_{t}, y_{t}) - \lambda^{\star}(z_{t}, \Hat{y}(\Hat{\lambda})) + \lambda^{\star}(z_{t}, \Hat{y}(\Hat{\lambda})) - \lambda^{\star}(\Hat{z}, \Hat{y}(\Hat{\lambda})) \|_{\infty} \nonumber \\
    & \leq \kappa_y \| y_{t} - \Hat{y}(\Hat{\lambda}) \|_{\infty} +  \| z_{t} - \Hat{z} \|_{\infty}. \nonumber
\end{align}
\end{proof}
\subsection{Proof of Lemma~\ref{thm:convergence-bound-x-steps}}
\label{app:proof-of-lemma-2}
\begin{proof}[\unskip\nopunct]
Our objective is to prove that the distance of the non-projected updates of Algorithm~\ref{alg:robust-svm-game} from the Stackelberg equilibrium $(\Hat{\lambda}, \Hat{y}(\hat{\lambda}))$, specifically $z_{t} - \Hat{z}$, is bounded. 

We begin by recalling the update rule at round $t$, $\lambda_t := \textsc{Prox}_{\mathcal{S}(y_t)}(z_{t}) = \textsc{Prox}_{\mathcal{S}(y_t)} (\lambda_{t-1} - \eta \nabla_{\lambda} f (\lambda_{t-1}, y_t))$, where $y_t = \Hat{y}(\lambda_{t-1})$, $\mathcal{S}(y_t)$ is the feasible set defined by constraints (\ref{proj-ineq-const}), using the labels at round $t$.
We further recall the Stackelberg equilibrium $(\Hat{\lambda}, \Hat{y}(\hat{\lambda}))$, i.e., 
\begin{align*}
\Hat{\lambda} &:= \textsc{Prox}_{\mathcal{S}(\Hat{y}(\Hat{\lambda}))}(\Hat{z}) = \textsc{Prox}_{\mathcal{S}(\Hat{y}(\Hat{\lambda}))}(\Hat{\lambda} - \eta \nabla_{\lambda} f({\Hat{\lambda}}, \Hat{y}(\Hat{\lambda})))\\
\Hat{y}(\Hat{\lambda}) &:= \textsc{LFlip}(\Hat{\lambda}),
\end{align*} 
where the operator $\textsc{LFlip}:\mathcal{X} \times \mathcal{Y} \to \mathcal{Y}$ defines the label poisoning attack formulated in (\ref{eq:inner-problem-objective}-\ref{eq:inner-problem-const2}).
We conclude the following:
\begin{align}
    z_{t} &= \lambda_{t-1} - \eta \nabla_\lambda f(\lambda_{t-1}, y_{t}) \nonumber \\
    \Hat{z} &= \Hat{\lambda} - \eta \nabla_{\lambda} f({\Hat{\lambda}}, \Hat{y}(\Hat{\lambda})) \nonumber \\
    z_{t} - \Hat{z} &= \lambda_{t-1} - \Hat{\lambda} - \eta \left( \nabla_\lambda f(\lambda_{t-1}, y_{t}) - \nabla_{\lambda} f({\Hat{\lambda}}, \Hat{y}(\Hat{\lambda})) \right).
\nonumber
\end{align}
We apply the mean value theorem for functions with multiple variables to the previous expression which allows us to rewrite $z_{t} - \Hat{z}$ as
\begin{align}
    &= \lambda_{t-1} - \Hat{\lambda} - \eta \left( \nabla_{\lambda} f(\lambda_{t-1}, y_{t}) - \nabla_{\lambda} f({\Hat{\lambda}}, y_{t}) + \nabla_{\lambda} f({\Hat{\lambda}}, y_{t})
    - \nabla_{\lambda} f({\Hat{\lambda}}, \Hat{y}(\Hat{\lambda})) \right) \nonumber \\
    &= \lambda_{t-1} -\Hat{\lambda} - \eta \left( \nabla_{\lambda}^2 f(\xi_\lambda, y_{t}) (\lambda_{t-1} - {\Hat{\lambda}}) + \nabla_{\lambda y}^2 f({\Hat{\lambda}}, \xi_y) (y_{t} - \Hat{y}(\Hat{\lambda})) \right),
    \nonumber
\end{align}
where $\xi_\lambda \in ({\Hat{\lambda}}, \lambda_{t-1})$ and $\xi_y \in (\Hat{y}(\Hat{\lambda}), y_{t})$. The last equation can be restated as:
\begin{align}
    z_{t} -\Hat{z} &= (I - \eta \nabla_{\lambda}^2 f(\xi_\lambda, y_{t}) )(\lambda_{t-1} - {\Hat{\lambda}}) - \eta \nabla_{\lambda y}^2 f({\Hat{\lambda}}, \xi_y) (y_{t} - \Hat{y}(\Hat{\lambda})),
\label{app:eq:x-difference-mvt}
\end{align}
where $I$ denotes the identity matrix.
We have defined the gradient of the objective in (\ref{eq:svm-dual-gradient}) as $\nabla_{\lambda} f(\lambda, y)= \Tilde{Q}\lambda - \mathbbm{1}$, where $\Tilde{Q}$ is the matrix with entries $\Tilde{Q}_{ij}=y_i y_j K_{ij}, \forall i,j \in [n]$, using the simplified notation.
We express the second-order partial derivatives as:
\begin{align}
    \nabla_{\lambda}^2 f({\lambda};y) &=   K \odot y y^{\mathrm{T}}, 
    \label{eq:eq:svm-dual-gradient-wrt-point-1} \\
    \nabla_{\lambda y}^2 f({\lambda};y) &= K \odot y \lambda^{\mathrm{T}} + I \odot (K(\lambda \odot y) \mathbbm{1}^{\mathrm{T}}),
\label{eq:eq:svm-dual-gradient-wrt-point-2}
\end{align}
where $K$ is the Gram matrix, $I$ is the $n \times n$ identity matrix, $\mathbbm{1}$ is the all-one vector and $\odot$ denotes the Hadamard product.
From (\ref{app:eq:x-difference-mvt}), (\ref{eq:eq:svm-dual-gradient-wrt-point-1}) and (\ref{eq:eq:svm-dual-gradient-wrt-point-2}), we obtain
\begin{align}
 z_{t} - \Hat{z}  &=  (I - \eta \left((K \odot y_t y_t^{\mathrm{T}} ) \right)(\lambda_{t-1} - {\Hat{\lambda}}) 
    \nonumber \\
    & \quad - \eta \left( K \odot \xi_y \Hat{\lambda}^{\mathrm{T}} + I \odot (K(\Hat{\lambda} \odot \xi_y) \mathbbm{1}^{\mathrm{T}}) \right)(y_{t} - \Hat{y}(\Hat{\lambda})).
    \nonumber
\end{align}
We take the infinity norm and conclude: 
\begin{align}
    \| z_{t} - \Hat{z} \|_{\infty} &= \| (I - \eta \left(K \odot y_t y_t^{\mathrm{T}} \right))(\lambda_{t-1} - {\Hat{\lambda}}) \nonumber \\
    & \quad - \eta \left( K \odot \xi_y \Hat{\lambda}^{\mathrm{T}} + I \odot (K(\Hat{\lambda} \odot \xi_y) \mathbbm{1}^{\mathrm{T}}) \right)(y_{t} - \Hat{y}(\Hat{\lambda})) \|_{\infty}  \nonumber \\
    & \leq \| (I - \eta \left(K \odot y_t y_t^{\mathrm{T}} \right))(\lambda_{t-1} - {\Hat{\lambda}}) \|_{\infty} \nonumber \\
    & \quad + \| - \eta \left( K \odot \xi_y \Hat{\lambda}^{\mathrm{T}} + I \odot (K(\Hat{\lambda} \odot \xi_y) \mathbbm{1}^{\mathrm{T}}) \right)(y_{t} - \Hat{y}(\Hat{\lambda})) \|_{\infty} \tag{triangle inequality} \nonumber \\
    & = \| (I - \eta \left(K \odot y_t y_t^{\mathrm{T}} \right))(\lambda_{t-1} - {\Hat{\lambda}}) \|_{\infty} \nonumber \\
    & \quad + \| \eta \left( K \odot \xi_y \Hat{\lambda}^{\mathrm{T}} + I \odot (K(\Hat{\lambda} \odot \xi_y) \mathbbm{1}^{\mathrm{T}}) \right)(y_{t} - \Hat{y}(\Hat{\lambda})) \|_{\infty} \tag{homogeneity}  \nonumber \\
    & \leq \underbrace{ \| (I - \eta \left(K \odot y_t y_t^{\mathrm{T}} \right)) \|_{\infty}}_{\kappa_{\lambda}}  \| \lambda_{t-1} - {\Hat{\lambda}} \|_{\infty}
    \nonumber \\
    & \quad + \underbrace{ \| \eta \left( K \odot \xi_y \Hat{\lambda}^{\mathrm{T}} + I \odot (K(\Hat{\lambda} \odot \xi_y) \mathbbm{1}^{\mathrm{T}}) \right)\|_{\infty} }_{\kappa'_y} \| y_{t} - \Hat{y}(\Hat{\lambda}) \|_{\infty}. \tag{homogeneity} 
\label{app:eq:coeff-of-terms} 
\end{align}
This implies that
\begin{align}
    \| z_{t} - \Hat{z} \|_{\infty} & \leq \kappa_{\lambda} \| \lambda_{t-1} - \Hat{\lambda} \|_{\infty} + \kappa'_y \| y_{t} - \Hat{y}(\Hat{\lambda}) \|_{\infty}.
    \nonumber
\end{align}
\end{proof}
\vspace{-0.2cm}
We note that $\kappa_\lambda \leq 1$ if the learning rate $\eta$ is chosen small enough.

\subsection{Proof of Theorem~\ref{thm:stability}}
\label{app:proof-local-stability}
Let $(\Hat{\lambda}, \Hat{y}(\hat{\lambda}))$ denote the Stackelberg equilibrium, i.e., 
\begin{align*}
\Hat{\lambda} &:= \textsc{Prox}_{\mathcal{S}(\Hat{y}(\Hat{\lambda}))}(\Hat{z}) = \textsc{Prox}_{\mathcal{S}(\Hat{y}(\Hat{\lambda}))}(\Hat{\lambda} - \eta \nabla_{\lambda} f({\Hat{\lambda}}, \Hat{y}(\Hat{\lambda})))\\
\Hat{y}(\Hat{\lambda}) &:= \textsc{LFlip}(\Hat{\lambda}),
\end{align*} 
where the operator $\textsc{LFlip}:\mathcal{X} \times \mathcal{Y} \to \mathcal{Y}$ defines the label poisoning attack formulated in (\ref{eq:inner-problem-objective}-\ref{eq:inner-problem-const2}). We further assume that the $\textsc{LFlip}$ operator returns a unique set of adversarial labels at the Stackelberg equilibrium $(\Hat{\lambda}, \Hat{y}(\hat{\lambda}))$, which implies that there are no ties with respect to $\Hat{\lambda}$ values. As a result, there exists a small enough constant $\delta'>0$ such that for any $\lambda_0$ with $\| \lambda_0 - \Hat{\lambda} \|_{\infty} < \delta'$, the corresponding $\Hat{y}(\lambda_0)$ satisfies $\Hat{y}(\lambda_0)=\Hat{y}(\Hat{\lambda})$. (Indeed, as long as $\delta'$ is small enough, such that the top-k entries between $\Hat{\lambda}$ and $\lambda_0$ agree, $\Hat{y}(\lambda_0)=\Hat{y}(\Hat{\lambda})$ will be satisfied.)

By combining Lemma~\ref{thm:convergence-bound-lambda-steps} and Lemma~\ref{thm:convergence-bound-x-steps} we conclude 
\begin{align*}
\| \lambda_{1} - \Hat{\lambda} \|_{\infty} &\leq \kappa_y \| \Hat{y}(\lambda_0) - \Hat{y}(\Hat{\lambda}) \|_{\infty} +  \| z_{1} - \hat{z} \|_{\infty}\leq  \| z_{1} - \hat{z} \|_{\infty} \leq \kappa_\lambda \|\lambda_0-\Hat{\lambda}\|_\infty < \kappa_\lambda \delta',
\end{align*}
where we used the fact that $\Hat{y}(\lambda_0)=\Hat{y}(\Hat{\lambda})$. The learning rate $\eta$ is chosen small enough, such that $\kappa_\lambda < 1$ and therefore $\|\lambda_1-\Hat{\lambda}\|_\infty< \kappa_\lambda \delta' < \delta'$. We therefore conclude by induction on $t$ that $\|\lambda_t-\hat{\lambda}\|_\infty<\kappa_\lambda^t \delta'$ for all $t>0$. This readily implies $\lambda_t \rightarrow \Hat{\lambda}$. Moreover, choosing $\delta=\min\{ \epsilon, \delta'\}$ concludes $\|\lambda_t-\Hat{\lambda}\|_\infty < \epsilon$ and concludes the proof.
\qed

\subsection{Global convergence result}
\label{app:global-convergence-result}
The previous section provides the proof of Theorem~\ref{thm:stability}, which provides a local stability and convergence result. Under additional assumptions on the constants $\kappa_y$ and $\kappa_y'$ that capture the sensitivity of the iterates $\lambda_t$ with respect to changes in the labels, one can derive a global convergence result, as summarized by the following proposition:

\begin{proposition}
Let $(\Hat{\lambda},\Hat{y}(\Hat{\lambda}))$ denote the Stackelberg equilibrium as before and let $\delta'=(\Hat{\lambda}_{\{k\}}-\Hat{\lambda}_{\{k+1\}})/2>0$, where $\Hat{\lambda}_{\{1\}}$ denotes the largest entry of $\Hat{\lambda}$, $\Hat{\lambda}_{\{2\}}$ the second larges entry of $\Hat{\lambda}$, etc. Provided that 
\begin{equation*}
\frac{2(\kappa_y+\kappa_y')k}{1-\kappa_\lambda} < \delta'
\end{equation*}
holds and that the step-size $\eta$ is chosen to be small enough, the iterates $\{\lambda_t\}$ of \textsc{Floral} are guaranteed to converge to $\Hat{\lambda}$ from any initial condition $\lambda_0$.
\end{proposition}
\begin{proof}
    As a result of  Lemma~\ref{thm:convergence-bound-lambda-steps} and Lemma~\ref{thm:convergence-bound-x-steps} we conclude that
\begin{align*}
\| \lambda_t-\hat{\lambda} \|_\infty &\leq \kappa_y \| \hat{y}(\lambda_{t-1}) - \hat{y}(\hat{\lambda}) \|_\infty + \| z_{t-1} - \hat{z} \|_\infty \\
&\leq (\kappa_y+\kappa_y') \| \hat{y}(\lambda_{t-1}) - \hat{y}(\hat{\lambda}) \|_\infty + \kappa_\lambda \| \lambda_{t-1} - \lambda_0 \|_\infty. 
\end{align*}
We further take advantage of the fact that $\|\Hat{y}(\lambda) - \Hat{y}(\Hat{\lambda})\|_\infty \leq 2k$ for any $\lambda$ (at most $k$ labels are flipped), which implies
\begin{align*}
\|\lambda_t-\Hat{\lambda}\|_\infty &\leq \kappa_\lambda \|\lambda_{t-1} - \lambda_0\|_\infty + 2(\kappa_y+\kappa_y') k. \nonumber
\end{align*}
The previous inequality is satisfied for all $t$, and can be used to conclude that
\begin{equation}
\| \lambda_{t} - \Hat{\lambda} \|_{\infty} \leq \kappa_\lambda^t \| \lambda_{0} - \Hat{\lambda} \|_{\infty} + \frac{2(\kappa_y + \kappa'_y)k}{1 - \kappa_\lambda} 
\label{eq:global-conv}
\end{equation}
holds for all $t$ (this can be verified by an induction argument). As a result, there exists an integer $t'>0$ such that $\| \lambda_{t} - \Hat{\lambda} \|_{\infty} < \delta'$ for all $t>t'$. This implies, due to the choice of $\delta'$, that $\Hat{y}(\lambda_t)=\Hat{y}(\Hat{\lambda})$ for all $t>t'$. We therefore conclude that for all $t>t'+1$
\begin{align}
\|\lambda_t-\Hat{\lambda}\|_\infty &\leq \kappa_\lambda \|\lambda_{t-1}-\Hat{\lambda}\|_\infty. \nonumber
\end{align}
This readily implies $\lambda_t\rightarrow \Hat{\lambda}$, due to the fact that $\kappa_\lambda <1$, and implies the desired result.
\end{proof}

\section{The Gradient of the Objective (\ref{eq:outer-problem-objective})}
\label{app:gradient-svm-dual}

We begin by recalling the kernel SVM dual formulation \citep{boser1992training, hearst1998support}:
\begin{subequations}
\begin{align}
D(f_{\lambda};\mathcal{D}): \min \limits_{\lambda \in \mathbb{R}^{n}} \quad  & \dfrac{1}{2} \sum\limits_{i=1}^{n} \sum\limits_{j=1}^{n} \lambda_{i} \lambda_{j} y_i y_{j} K_{ij} - \sum\limits_{i=1}^{n} \lambda_{i} \nonumber \label{eq:svm-dual-kernel-objective}  \\
 \text { subject to }  & \quad \sum\limits_{i=1}^{n} \lambda_{i} y_{i} = 0 \nonumber \\
&  \quad 0 \leq \lambda_{i} \leq C, \forall i \in [n], \nonumber
\end{align}
\end{subequations}
where $K$ represents the Gram matrix with entries $K_{ij}=k(x_i, x_j), \forall i,j \in [n] := \{1,\dots,n\}$, derived from a kernel function $k$.
We consider the $p^{\text{th}}$ data point and apply differentiation of a double summation to the objective, which yields
\begin{align}
\frac{\partial \left( D(f_{\lambda};\mathcal{D}) \right) }{\partial \lambda_p} &= \frac{1}{2} \left( \sum\limits_{i=1}^{n} \lambda_i y_i y_p K_{ip} + \sum\limits_{j=1}^{n} \lambda_j y_p y_j K_{pj} \right) - 1 \nonumber \\
&= y_p \sum\limits_{i=1}^{n} \lambda_i y_i K_{ip} - 1. \tag{from the symmetry of the kernel function}\nonumber 
\end{align}
In compact form, we obtain the following.
\begin{align}
\nabla_{\lambda} D(f_{\lambda};\mathcal{D}) = Q\lambda - \mathbbm{1}, \nonumber
\end{align}
where $Q$ is the matrix with entries $Q_{ij}=y_i y_j K_{ij}, \forall i,j \in [n]$ and $\mathbbm{1}$ is the vector of all ones.

\section{Experiment Details}
\label{app:experiment-details}
For our experiments, we set the hyperparameter values as given in Table~\ref{tab:params-and-hyperparams}. We provide the experiment details as follows.
\begin{itemize}[left=0cm,topsep=0pt]
\setlength\itemsep{-0.2em}
    \item We initialize the model $f_{\lambda_{0}}$ with parameters set to $0$. In \textsc{Floral}, however, the attacker uses a randomized top-$k$ rule to identify the $B$ most influential support vectors based on the $\lambda$ values. Due to the $0$ initialization of $\lambda$, a warm-up period is required, which we set to $1$ round for all SVM-related methods. 

    \item To train kernel SVM classifiers for all SVM-related methods other than \textsc{Floral}, we use our PGD-based Algorithm~\ref{alg:robust-svm-game} with a \textit{dummy} attack, that is, we eliminate the adversarial dataset generation step and employ vanilla PGD training.
    
    \item For large datasets such as IMDB, we implement projection via fixed point iteration as given in Algorithm~\ref{alg:projection-fpi} in Section~\ref{sec:large-scale-implementation} instead of constructing a quadratic program as defined in (\ref{eq:svm-projection-operation}-\ref{proj-ineq-const}). 
    
\end{itemize}

\begin{table}[ht]
\caption{Hyperparameter values.}
\centering
\resizebox{1\linewidth}{!}{
\begin{tabular}{c|l|l}
\specialrule{1.5pt}{1pt}{1pt} 
\textbf{Symbol} & \textbf{Hyperparameter}     & \textbf{Value} \\ 
\specialrule{1.5pt}{1pt}{1pt} 
$n$ & The size of the training dataset & Moon: $500$, IMDB: $20000$ \\
$T$  & The number of training rounds  & Moon: $500$,  IMDB: $1000$   \\ 
$T_{\text{proj}}$  & The number of projection via fixed point iteration rounds  & $1000$  \\ 
$B$  & The attacker budget  &  Moon: $\{10, 20, 50, 100, 250\}$, IMDB: $\{500, 2500, 5000, 12500\}$ \\ 
$k$  & The number of labels to poison  & Moon: $\{5, 10, 25, 50, 125\}$, IMDB $\{250, 1250, 2500, 6250\}$ \\ 
$C$ & Regularization parameter for soft-margin SVM & Moon: $\{10, 100\}$, IMDB: $10$ \\
$\gamma$ & RBF kernel parameter & Moon: $\{0.5, 1, 10\}$, IMDB: $0.005$ \\
$\epsilon$ & Error rate  for projection via fixed point iteration & $1e-21$\\
$\eta$ & Learning rate & Optimized over the set $\{0.0001, 0.0003, 0.0005, 0.0007\}$, for RoBERTa: $2e-05$ \\
& Learning rate scheduler & Moon: a decay rate of $0.1$ at every $\{100, 200\}$ rounds (optimized), for RoBERTa: linear scheduler \\
& The model architecture for NN and NN-PGD & Fully connected MLP with $2$ hidden layers with  $32$ units each \\
& Batch size & $32$\\
& NN-PGD perturbation amount & $8/255$\\
& NN-PGD step size & $2/255$\\
& SGD optimizer momentum value & $0.9$ \\
\specialrule{1.5pt}{1pt}{1pt}
\end{tabular}
}
\label{tab:params-and-hyperparams}
\end{table}
\vspace{-0.2cm}
\subsection{Datasets}
\label{app:datasets}
\begin{itemize}[left=0cm,topsep=0pt]
\setlength\itemsep{-0.2em}
    \item Moon is a benchmark dataset for binary classification tasks, generated directly using the \texttt{scikit-learn} library \citep{pedregosa2011scikit}. It contains two-dimensional input examples with each feature taking value in the range $[-2.5, 2.5]$.
    We generate its adversarial versions by flipping the labels of farthest points from the decision boundary of a linear classifier trained on the clean dataset, using label poisoning levels ($\%$) of $\{5, 10, 15, 20, 25\}$.
    We provide the visualizations of the Moon training dataset with clean and adversarial labels in Figure~\ref{fig:moon-datasets-clean-and-adv}.
    
    \item IMDB review sentiment analysis benchmark dataset \citep{maas-EtAl:2011:ACL-HLT2011} contains train and test datasets, each containing $25,000$ examples. We used randomly selected $20,000$ points from the training set as training examples, and the rest as validation examples.
    We fine-tuned the RoBERTa-base model \footnote{\label{roberta-base-link}The pre-trained RoBERTa-base model can be found in \url{https://huggingface.co/FacebookAI/roberta-base}.} \citep{roberta} on this dataset and extracted features ($768$-dimensional embeddings) to train SVM-related models on this dataset.
    We generated adversarially labelled datasets using the fine-tuned RoBERTa-base model on the clean dataset. Specifically, we identified the most influential training points based on the gradient of loss with respect to the inputs and flipped their labels under various poisoning levels ($\%$) of $\{10, 25, 30, 35, 40\}$. 
\end{itemize}

 \begin{figure*}[ht]
    \centering  
    \begin{subfigure}{0.25\textwidth}\centering{\includegraphics[width=1\linewidth,trim=0 20 0 55,clip]{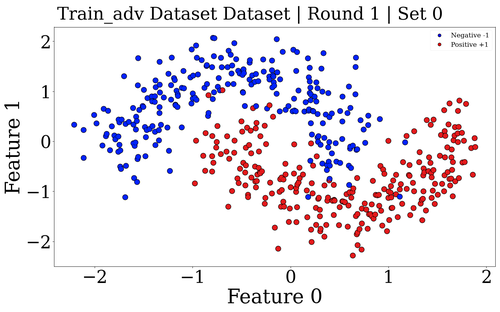}}
    \caption{Clean.}
    \end{subfigure}%
    \begin{subfigure}{0.25\textwidth}\centering{\includegraphics[width=1\linewidth,trim=0 20 0 55,clip]{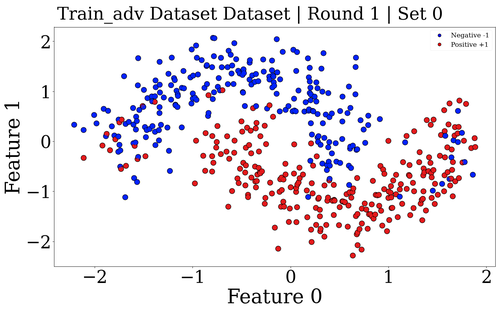}}
    \caption{$D^{\text{adv}} = 5\%$.}
    \end{subfigure}%
    \begin{subfigure}{0.25\textwidth}\centering{\includegraphics[width=1\linewidth,trim=0 20 0 55,clip]{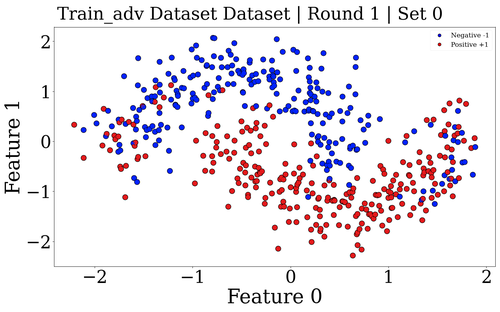}}
    \caption{$D^{\text{adv}} = 10\%$.}
    \end{subfigure}%
    \begin{subfigure}{0.25\textwidth}\centering{\includegraphics[width=1\linewidth,trim=0 20 0 55,clip]{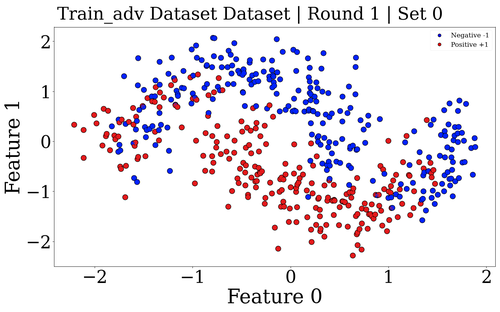}}
    \caption{$D^{\text{adv}} = 25\%$.}
    \end{subfigure}%
    \caption{Illustrations of the Moon training sets from an example replication, using clean and adversarial labels with poisoning levels: $5\%$, $10\%$, $25\%$. }
\label{fig:moon-datasets-clean-and-adv}
\end{figure*}
\subsection{Baselines}
\label{app:baselines}
In our main experiments, we compared \textsc{Floral} against the baseline methods by carefully selecting their hyperparameters using the domain knowledge, which we detail below.
\begin{itemize}[left=0cm,topsep=0pt]
    \item LN-SVM \citep{svm-adversarial-noise} applies a heuristic-based kernel matrix correction by assuming that every label in the training set is independently flipped with the same probability. It requires a predefined noise parameter $\mu$, which we set to $\mu \in \{ 0.05, 0.1, 0.15, 0.2, 0.25\}$ by leveraging the domain label poisoning knowledge, i.e. using the poisoning levels of the adversarial datasets. 

    \item For Curie \citep{curie}, we set the confidence parameter to $\{0.95, 0.90, 0.85, 0.8, 0.75\}$. To compute the average distance, we considered $k=20$ neighbors in the same cluster for the Moon dataset and $k=1000$ neighbors for the IMDB dataset experiments.

    \item For LS-SVM \citep{label-sanitization}, we use the relabeling confidence threshold from $\{0.95, 0.90, 0.85, 0.8, 0.75\}$, again aligning with the poisoning level of the adversarial datasets. For its $k$-NN step, we considered $k=20$ and $k=1000$ neighbors for the Moon and IMDB datasets, respectively. 

    \item NN baseline is a fully connected multi-layer perceptron with two hidden layers with $32$ units each, trained using the SGD optimizer with $0.9$ momentum and binary cross-entropy loss. For additional experiments on the MNIST dataset, a similar architecture with two hidden layers having $\{32,10\}$ units is employed.

    \item NN-PGD is based on the same NN architecture as above, trained with PGD-AT \citep{madry2017towards} using a standard perturbation budget of $8/255$ and a step size of $2/255$.
    \looseness -1
    
\end{itemize}

\subsection{RoBERTa Experiment Details}
\label{app:sec:roberta-exp-details}
We fine-tune the RoBERTa-base model\textsuperscript{\ref{roberta-base-link}} on the IMDB review sentiment analysis dataset\footnote{The IMDB review sentiment dataset can be found in \url{https://huggingface.co/datasets/stanfordnlp/imdb}.}.
We fine-tune the model for three epochs with no warm-up steps, using the AdamW optimizer, weight decay $0.01$, batch size $16$, and learning rate $2\mathrm{e}{-05}$ with a linear scheduler, using a single NVIDIA A100 40GB GPU. We extract the last layer embeddings of the trained model for experiments with \textsc{Floral} integration.

\section{Effectiveness Analysis of \textsc{Floral} Defense}
\label{app:defense-analysis}

As explained in Section~\ref{sec:proposed-approach}, \textsc{Floral} takes a \textit{proactive} defense when the initial training data is clean, iteratively adjusting the model to reduce sensitivity to potential label poisoning attacks by exposing it to adversarial decision boundary configurations through adversarial training. Conversely, when the training data is already contaminated with adversarial labels, \textsc{Floral} mitigates their effect by \textit{implicitly sanitizing} the corrupted labels.

To demonstrate how \textsc{Floral} defenses under already poisoned training data, we further analyze the efficacy of \textsc{Floral} by measuring its "recovery" rate of poisoned labels. That is, we quantify \textsc{Floral}'s rate of disrupting the initial attack ($\%$) on the adversarially labelled training sets, averaged over replications. 

As reported in Figure~\ref{fig:floral-recovery-rate} on the adversarial Moon datasets, \textsc{Floral} is able to disrupt the initial label attack (already inherited in the training set), at a $25-35\%$ rate. This contributes to the success of the \textsc{Floral} in achieving higher robust accuracy in training with adversarial datasets.
Moreover, we provide example illustrations (Figures~\ref{fig:floral-recovery-rate-Dadv5}-\ref{fig:floral-recovery-rate-Dadv25}) that show which poisoned data points are recovered by \textsc{Floral} under the randomized top-$k$ attack.

\begin{figure}[h]
    \centering  
    \begin{subfigure}[t]{0.5\textwidth}\centering{\includegraphics[width=1\linewidth,trim=0 0 0 0,clip]{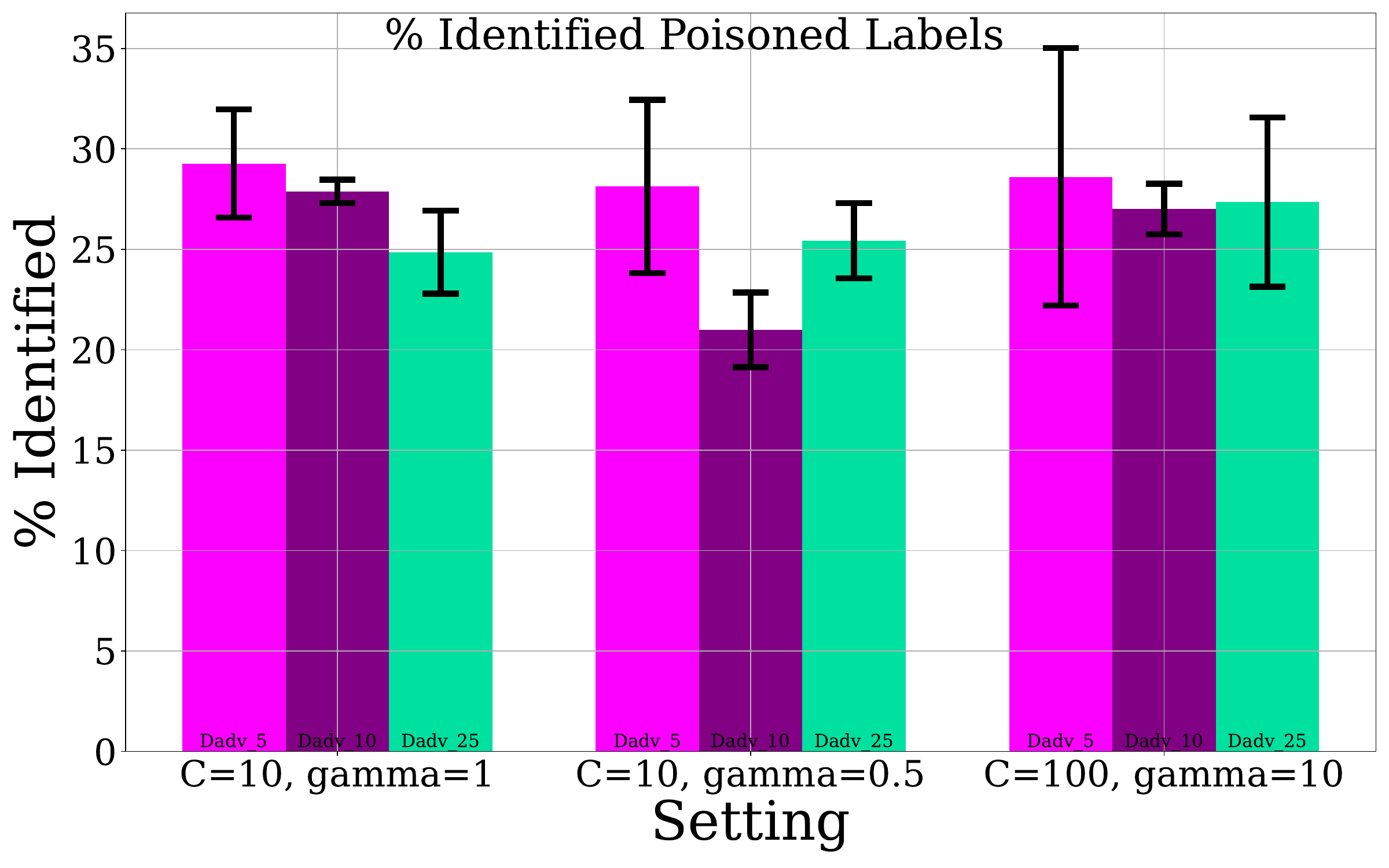}}
    \caption{ Average ($\%$) poisoned labels recovered by \textsc{Floral}.}
    \label{fig:floral-recovery-rate}
    \end{subfigure}%
    \begin{subfigure}[t]{0.5\textwidth}\centering{\includegraphics[width=1\linewidth,trim=0 0 0 0,clip]{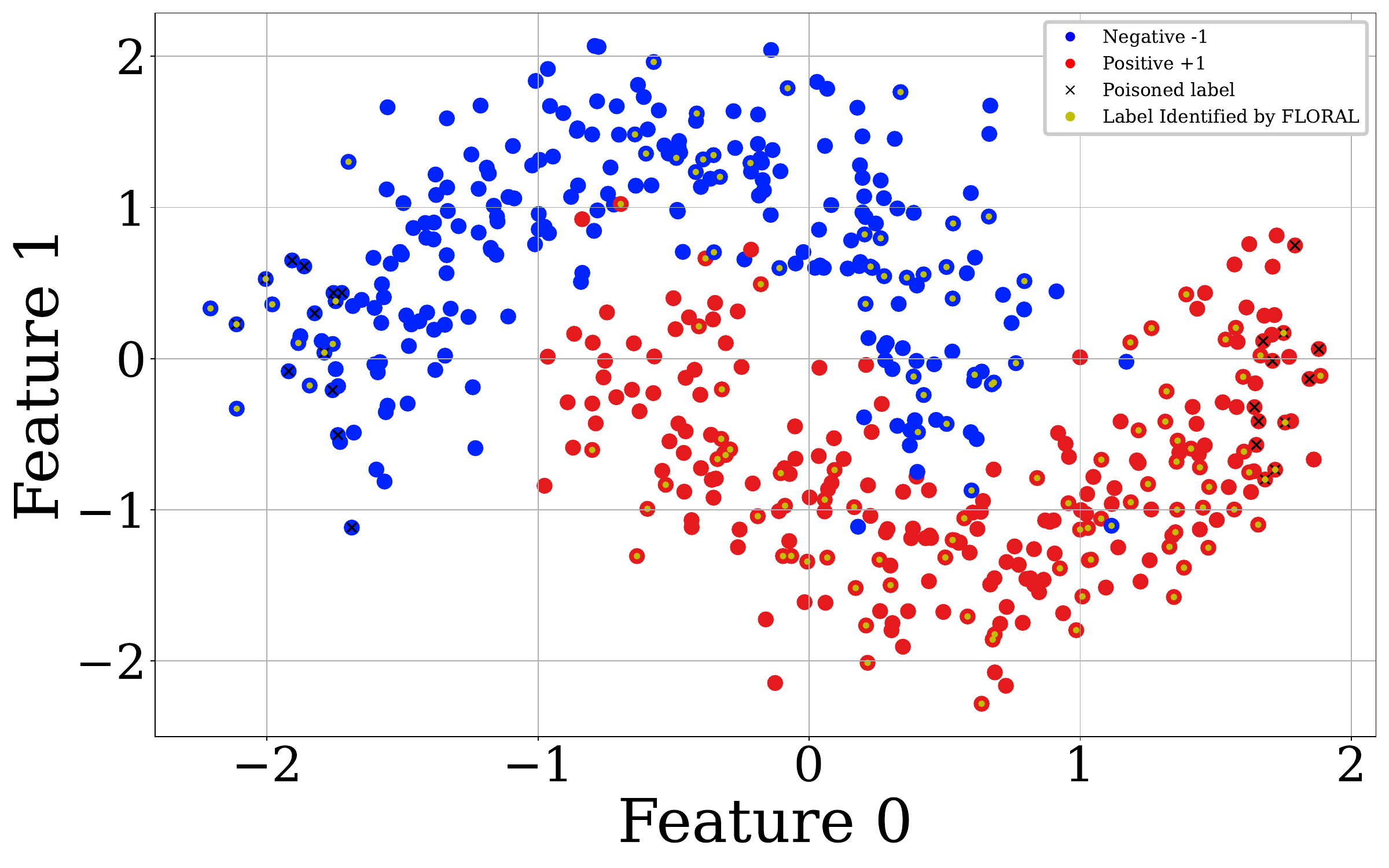}}
    \caption{ A trace for recovered points on $D^{\text{adv}}=5\%$.}
    \label{fig:floral-recovery-rate-Dadv5}
    \end{subfigure}%
    \\
    \begin{subfigure}[t]{0.5\textwidth}\centering{\includegraphics[width=1\linewidth,trim=0 0 0 0,clip]{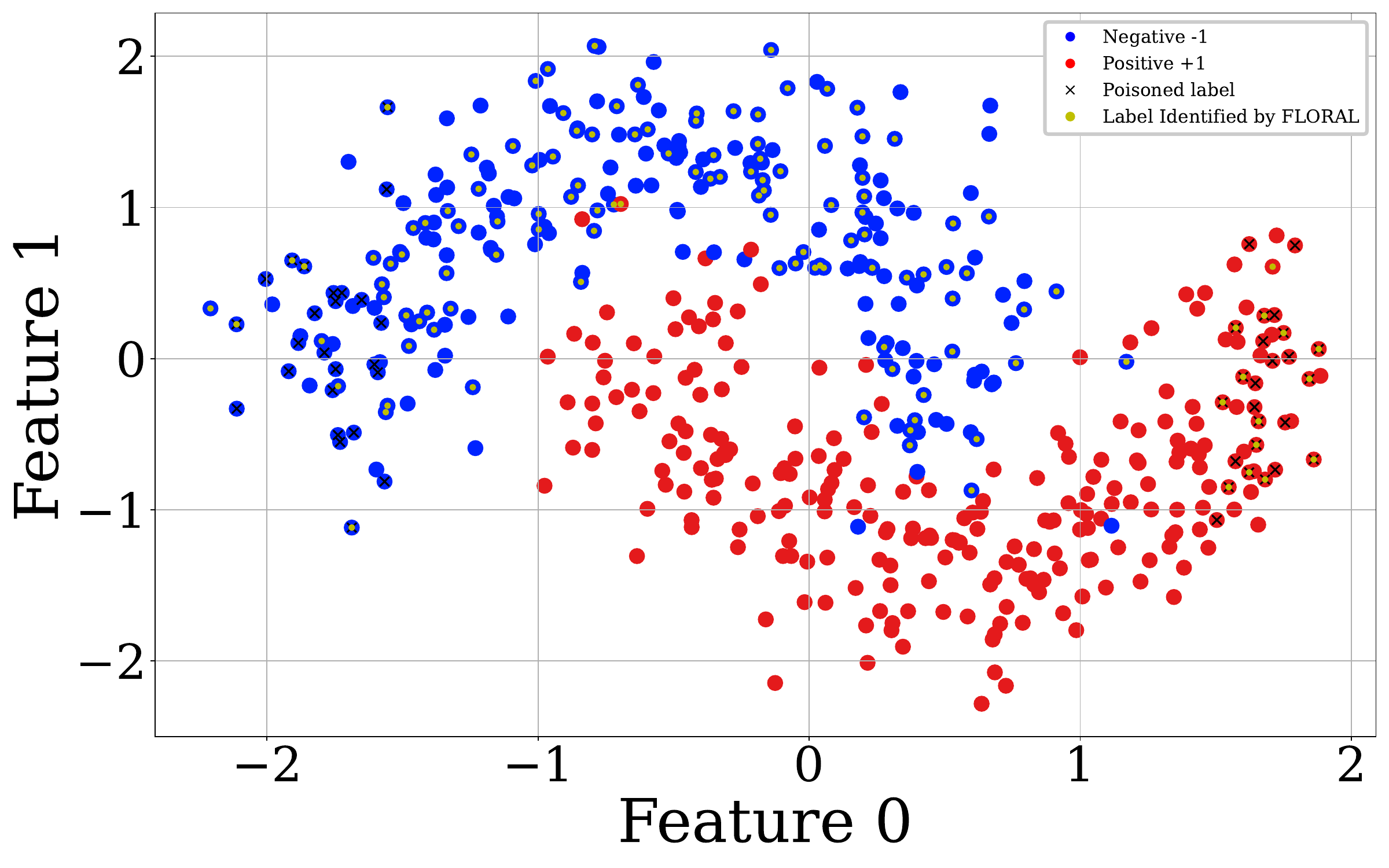}}
    \caption{ A trace for recovered points on $D^{\text{adv}}=10\%$.}
    \label{fig:floral-recovery-rate-Dadv10}
    \end{subfigure}%
    \begin{subfigure}[t]{0.5\textwidth}\centering{\includegraphics[width=1\linewidth,trim=0 0 0 0,clip]{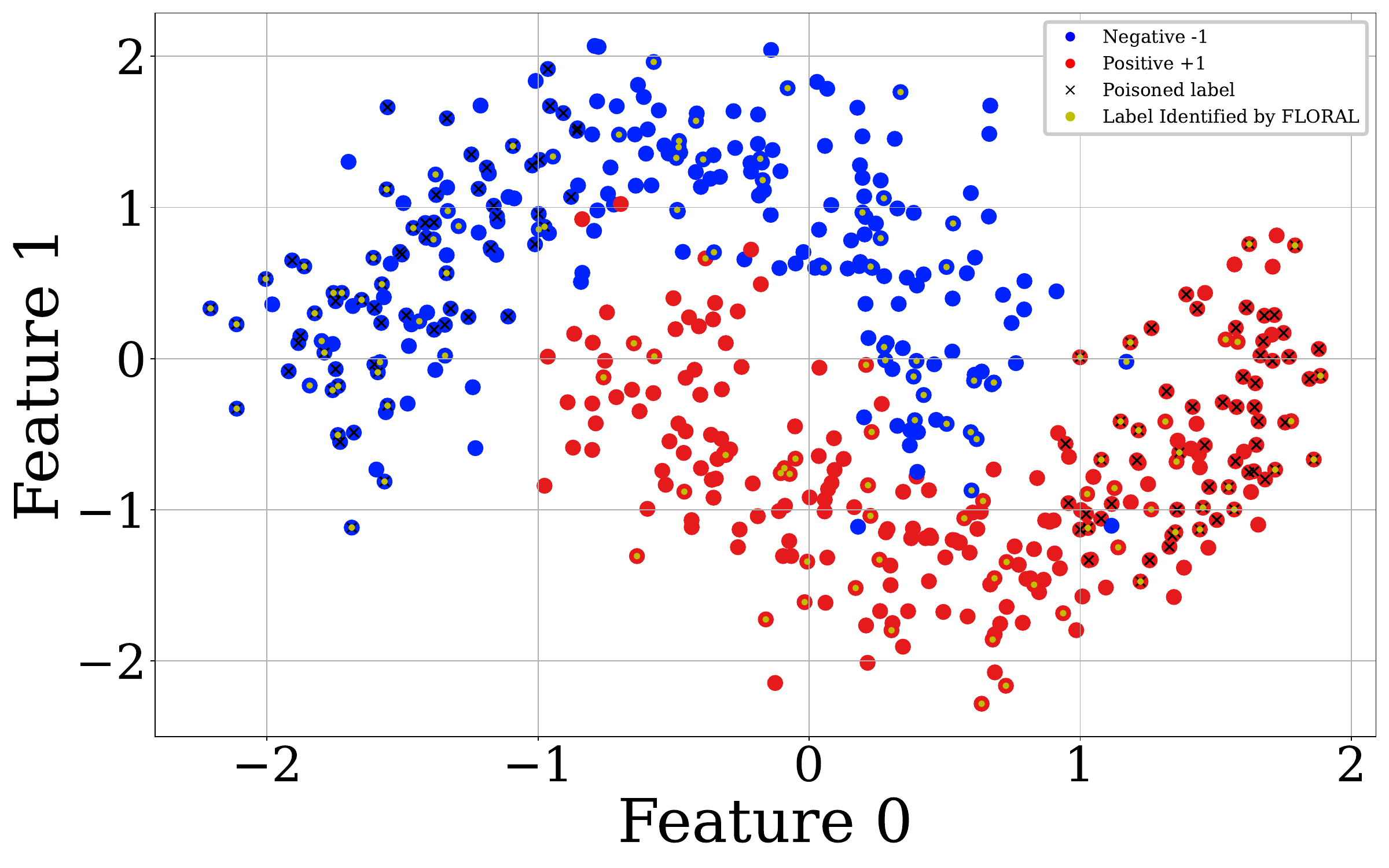}}
    \caption{A trace for recovered points on $D^{\text{adv}}=25\%$.}
    \label{fig:floral-recovery-rate-Dadv25}
    \end{subfigure}%
     \caption{ The average percentage of "\textbf{recovered}" poisoned labels by \textsc{Floral} over the adversarial Moon datasets containing $\{5, 10, 25\}$ ($\%$) poisoned labels. As shown in \textbf{(a)}, \textsc{Floral} is able to recover, on average, $25-35 \%$ of the poisoned labels. The plots \textbf{(b)-(d)} illustrate example traces, showing which poisoned data points are recovered by \textsc{Floral}.  \looseness -1}
\label{fig:recovery-of-floral}
\end{figure}

\clearpage 
\section{Additional Experimental Results}
\label{app:additional-experiment-results}
We provide additional experimental results under various hyperparameter settings for the Moon dataset in Appendix~\ref{sec:moon-results-appendix}. In Appendix~\ref{sec:imdb-results-appendix}, we first report a comprehensive comparison of \textsc{Floral} against other baselines on the IMDB dataset, followed by an analysis of how \textsc{Floral} shifts the most influential training points for RoBERTa's predictions on the IMDB dataset. 
In Appendix~\ref{app:mnist1vs7-experiments}, we provide experiments on the MNIST \citep{deng2012mnist} dataset.
Finally, we present a sensitivity analysis with respect to the attacker's budget in Appendix~\ref{sec:sensitivity-analysis}.
\looseness -1
\subsection{Moon}
\label{sec:moon-results-appendix}
We report the clean and robust test accuracy of methods under different (non-optimal) kernel hyperparameter choices and considering label poisoning levels $\{5,10,25\} (\%)$ in Figure~\ref{fig:moon-exp-results-plots-appendix} and Table~\ref{tab:test-accuracy-comp-moon-appendix}. 

When the kernel hyperparameters are not optimally chosen, NN-PGD shows superior performance in less adversarial scenarios compared to SVM-based methods. However, it also demonstrates significant sensitivity to label attacks in $25\%$ adversarial settings, against all other baselines. 
\textsc{Floral} particularly advances by maintaining a higher robust accuracy in more adversarial settings.

 \begin{figure*}[ht]
    \centering  
    \begin{subfigure}[t]{0.25\textwidth}\centering{\includegraphics[width=1\linewidth,trim=0 0 0 0,clip]{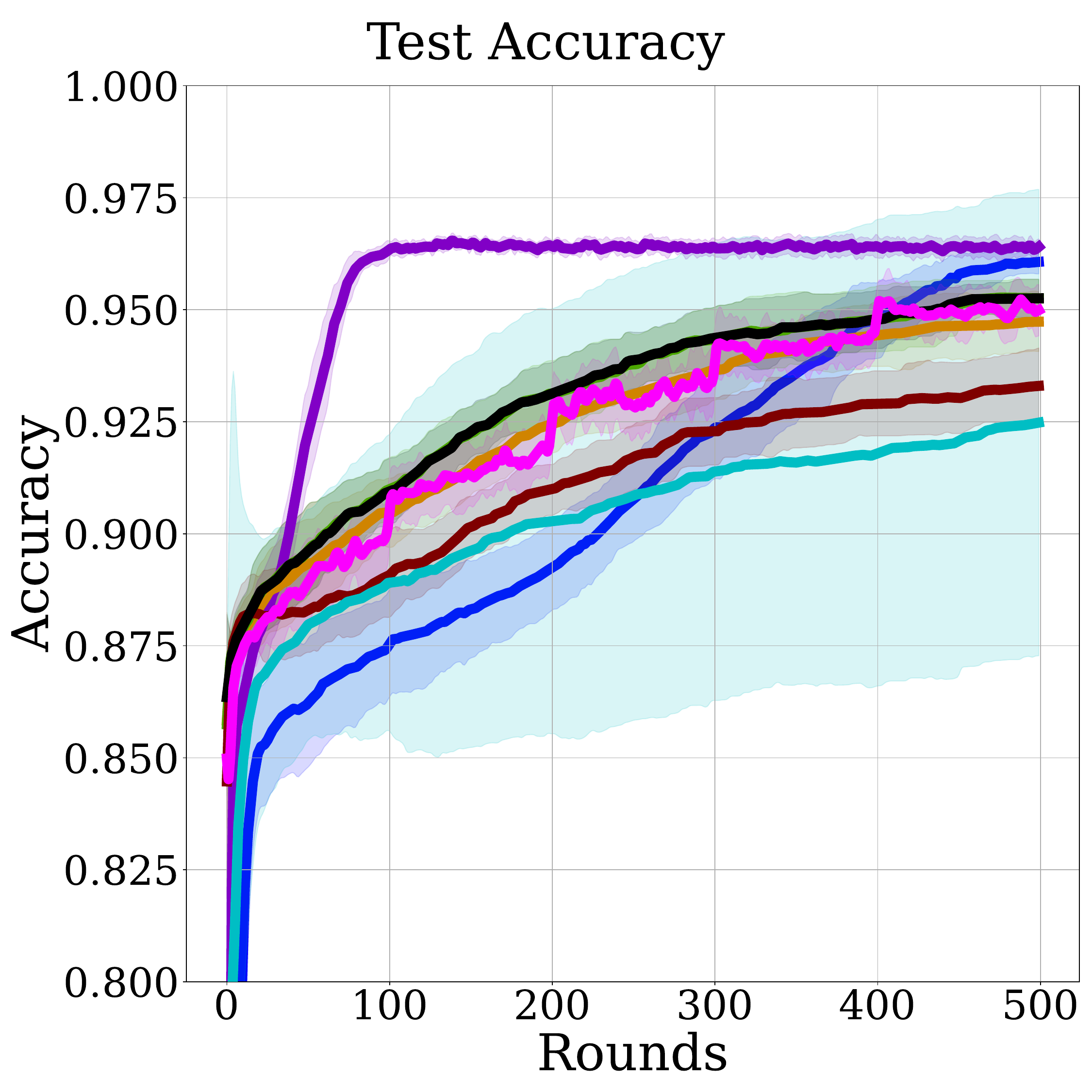}}
    \caption{Clean.}
    \end{subfigure}%
    \begin{subfigure}[t]{0.25\textwidth}\centering{\includegraphics[width=1\linewidth,trim=0 0 0 0,clip]{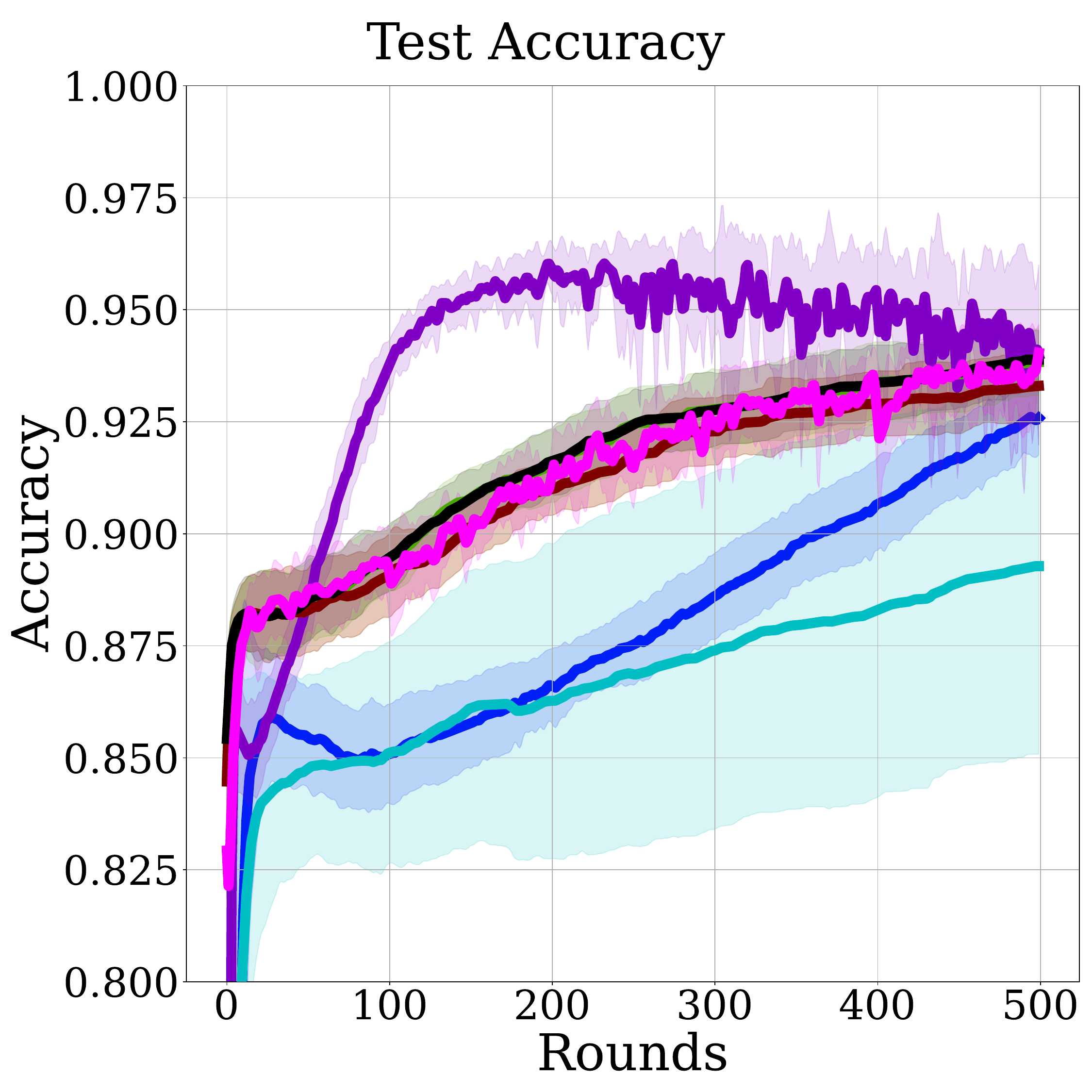}}
    \caption{$D^{\text{adv}} = 5\%$.}
    \end{subfigure}%
    \begin{subfigure}[t]{0.25\textwidth}\centering{\includegraphics[width=1\linewidth,trim=0 0 0 0,clip]{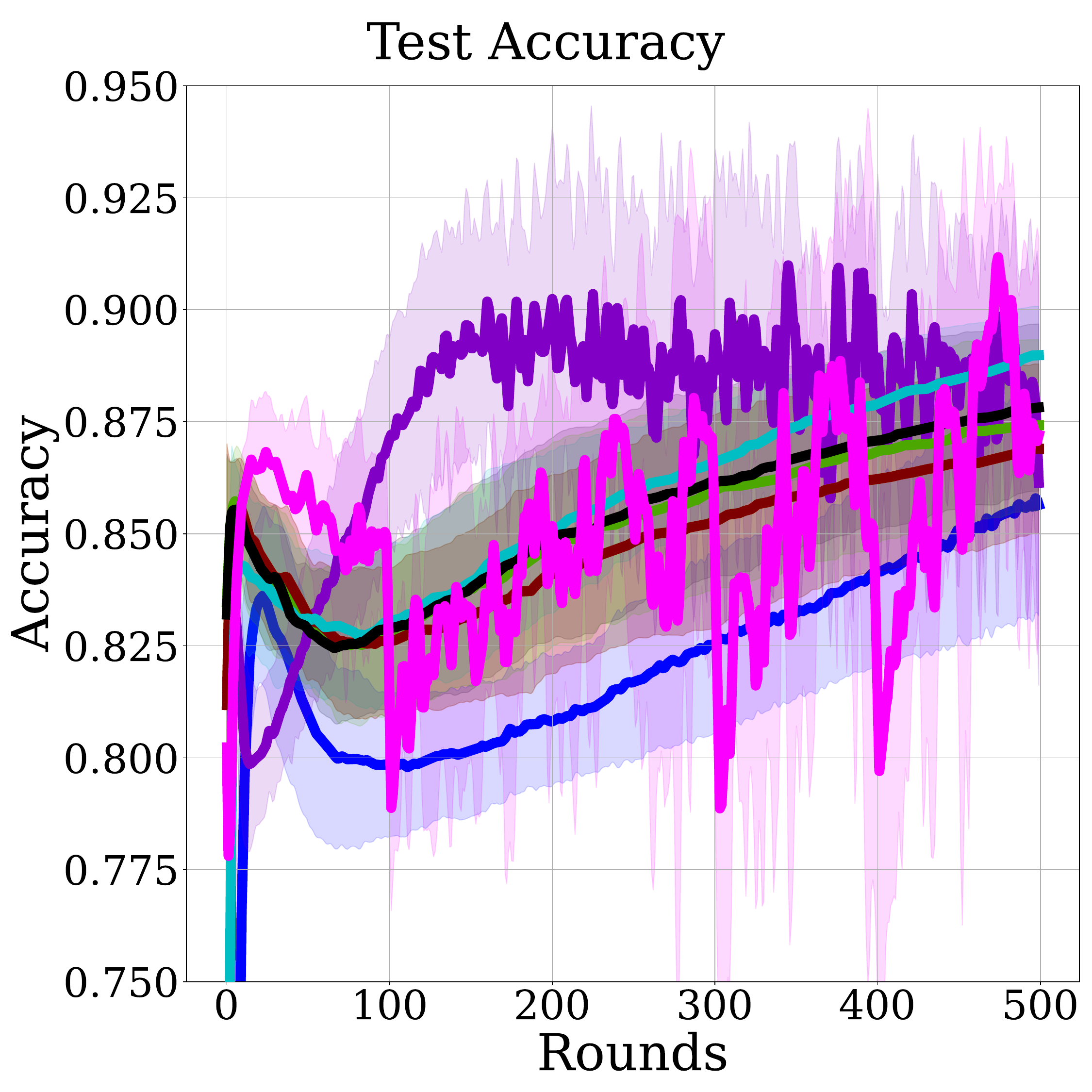}}
    \caption{$D^{\text{adv}} = 10\%$.}
    \end{subfigure}%
    \begin{subfigure}[t]{0.25\textwidth}\centering{\includegraphics[width=1\linewidth,trim=0 0 0 0,clip]{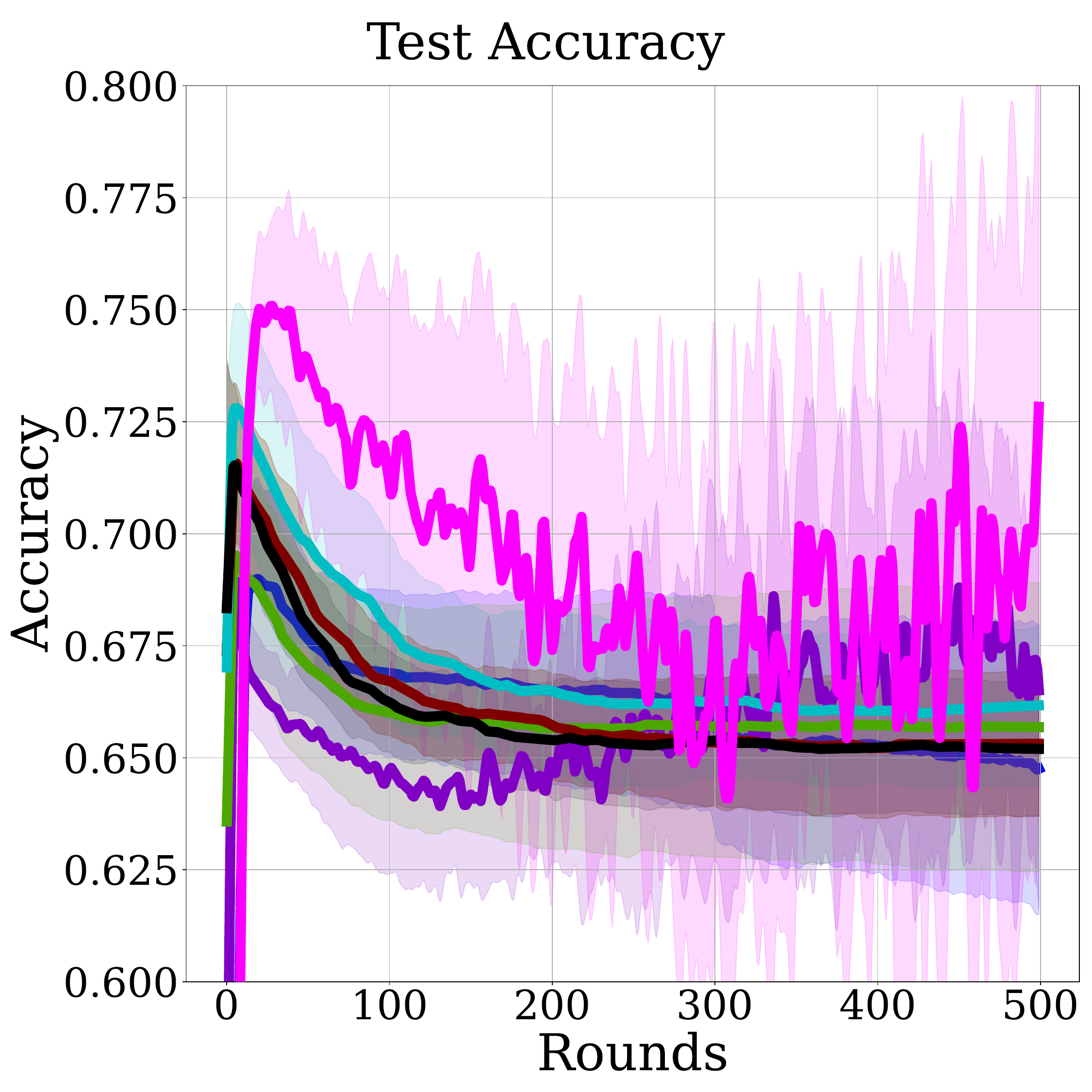}}
    \caption{$D^{\text{adv}} = 25\%$.}
    \end{subfigure}%
    \\
    \begin{subfigure}[t]{0.25\textwidth}\centering{\includegraphics[width=1\linewidth,trim=0 0 0 0,clip]{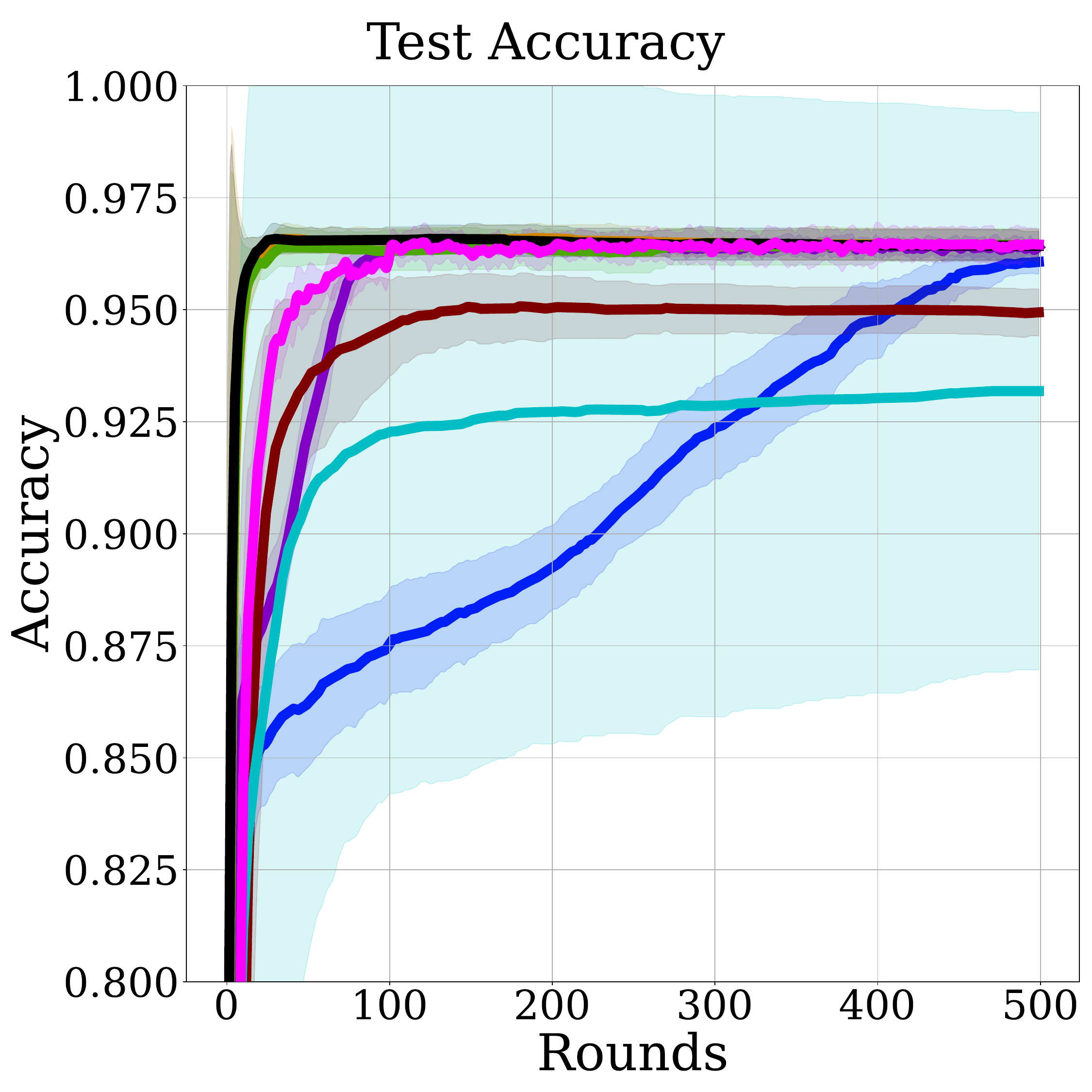}}
    \caption{Clean.}
    
    \end{subfigure}%
    \begin{subfigure}[t]{0.25\textwidth}\centering{\includegraphics[width=1\linewidth,trim=0 0 0 0,clip]{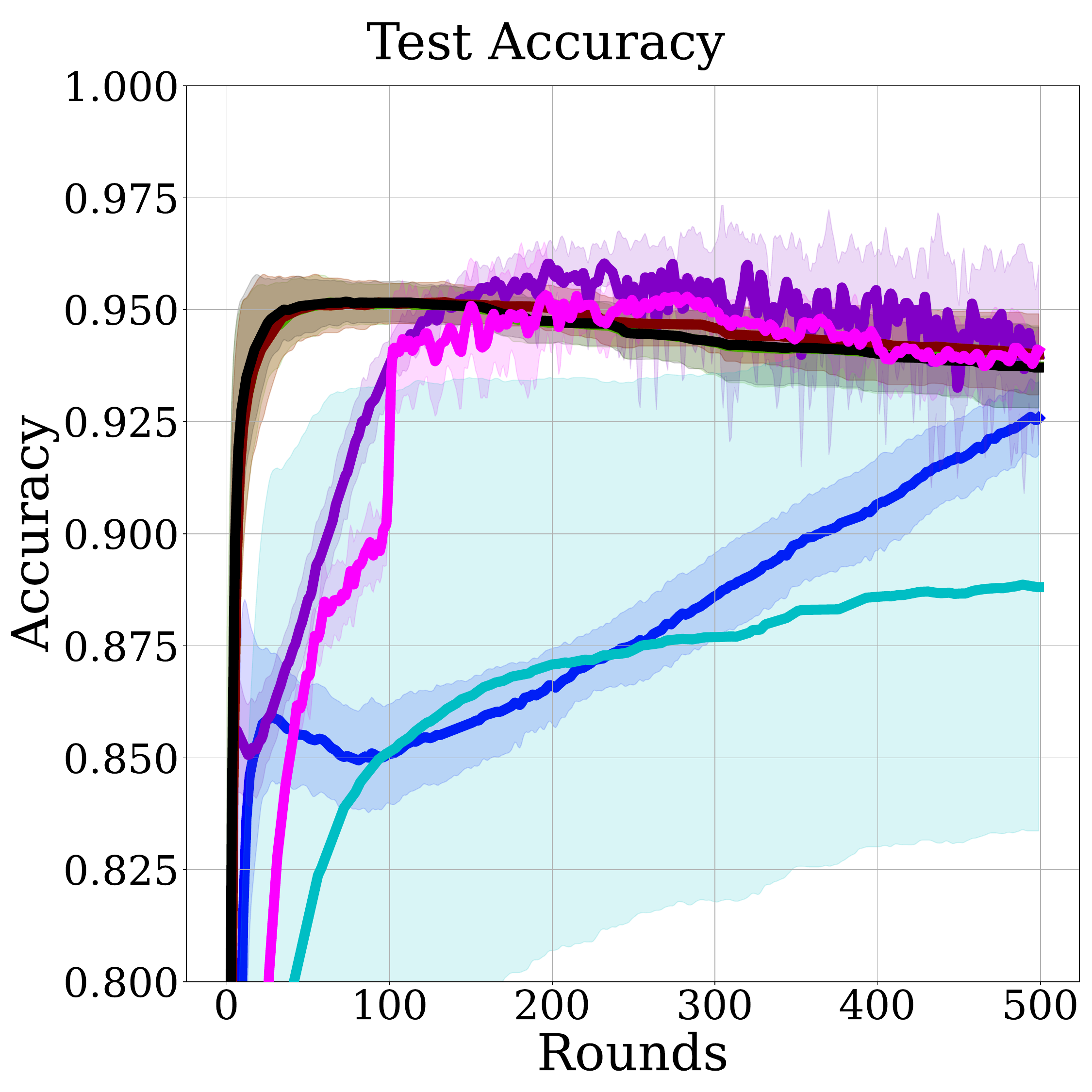}}
    \caption{$D^{\text{adv}} = 5\%$.}
    
    \end{subfigure}%
    \begin{subfigure}[t]{0.25\textwidth}\centering{\includegraphics[width=1\linewidth,trim=0 0 0 0,clip]{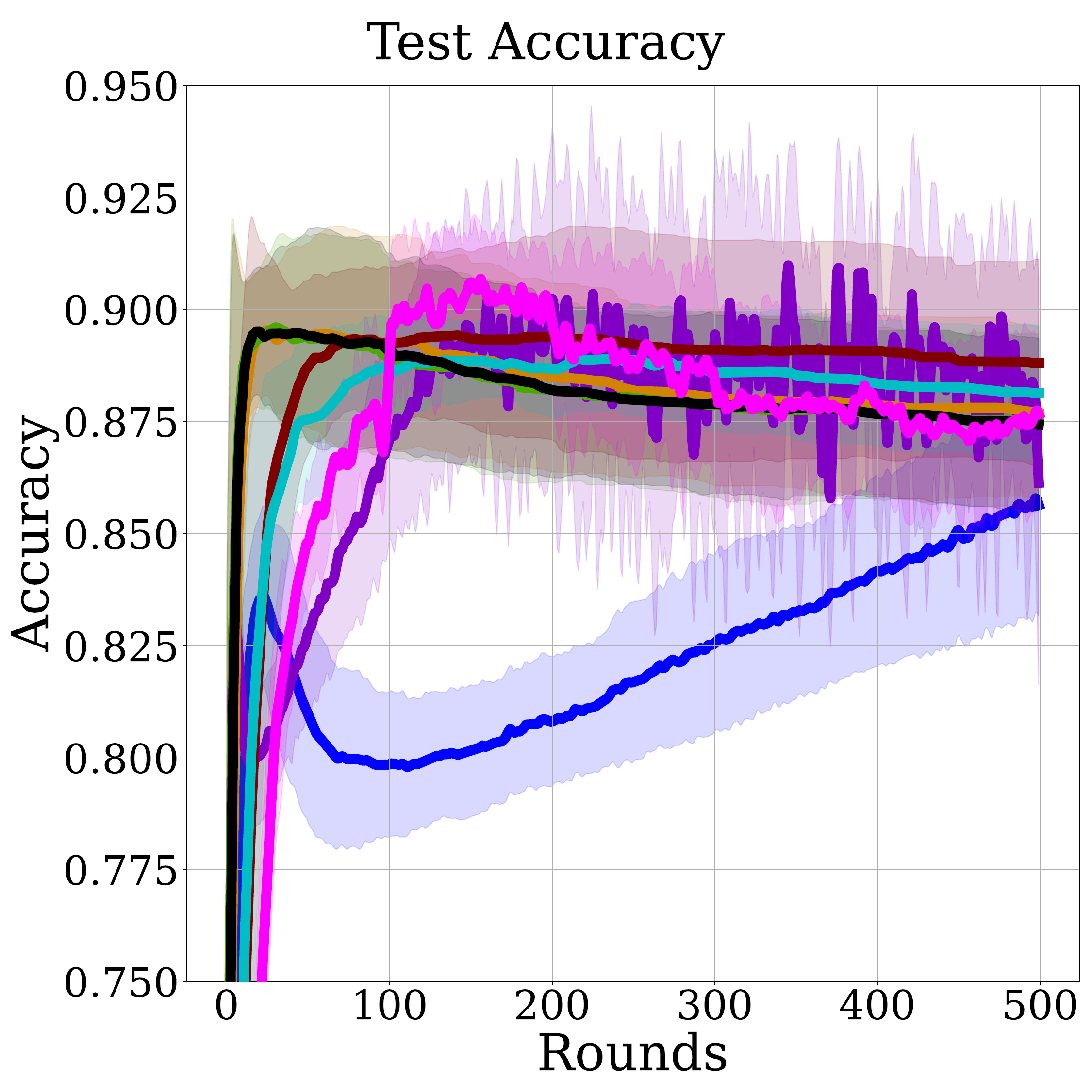}}
    \caption{$D^{\text{adv}} = 10\%$.}
    
    \end{subfigure}%
    \begin{subfigure}[t]{0.25\textwidth}\centering{\includegraphics[width=1\linewidth,trim=0 0 0 0,clip]{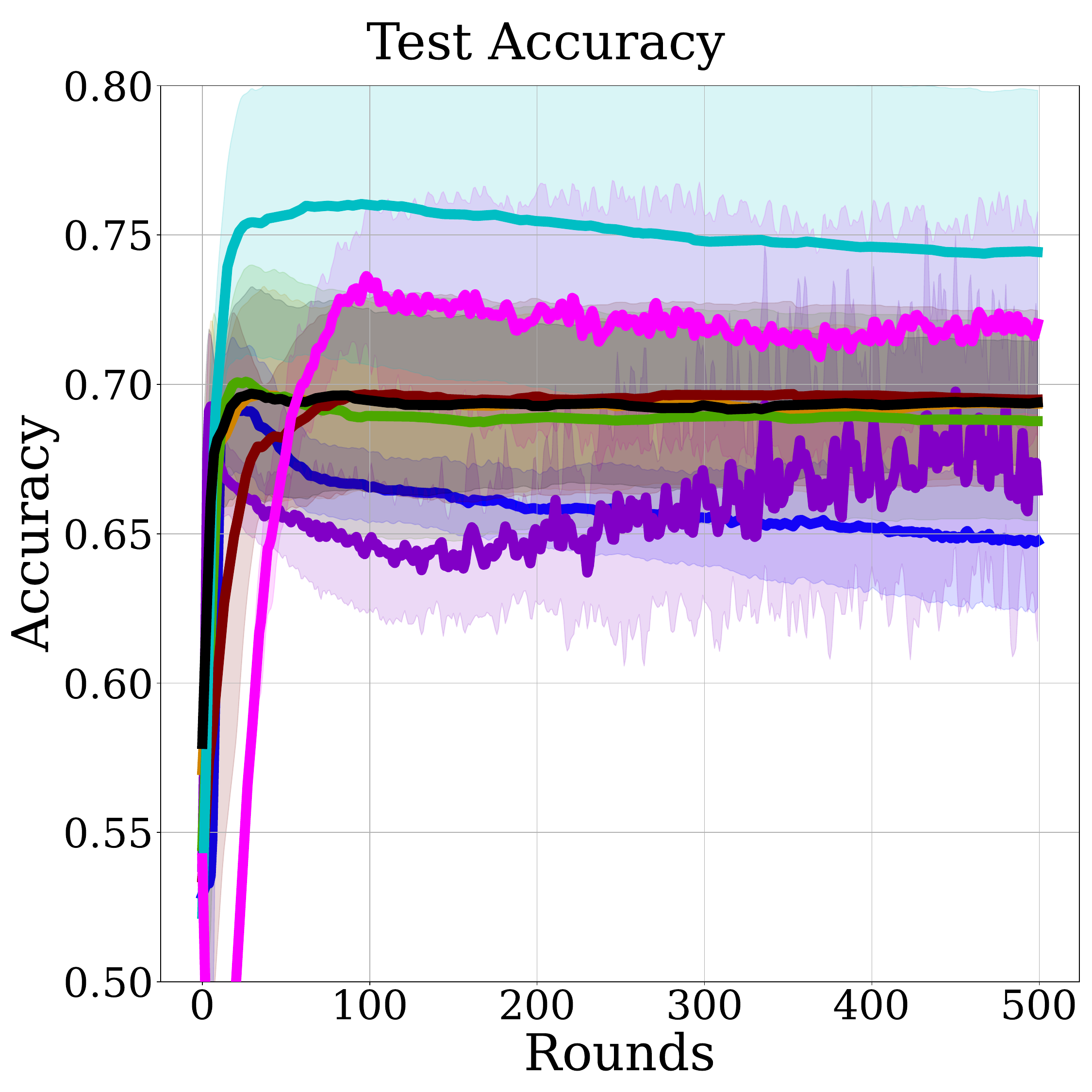}}
    \caption{$D^{\text{adv}} = 25\%$.}
    
    \end{subfigure}%
    \\
    \begin{subfigure}
{1\textwidth}\centering{\includegraphics[width=1\linewidth,trim=40 10 10 10,clip]{figures/moon_legend_8cols_new.pdf}}
    
    \end{subfigure}%
    \caption{Comparison of clean and robust test accuracy of methods trained on the Moon dataset under different kernel hyperparameter choices. 
    For all SVM-related models, the first row corresponds to the setting ($C=10$, $\gamma=0.5$), whereas the second row shows the setting ($C=100$, $\gamma=10$).
    As the level of label poisoning increases, the accuracy of models trained on adversarial datasets generally declines. When the kernel parameters are not optimally chosen, \textsc{Floral} demonstrates superior performance, particularly under the $25\%$ attack.
    }
\label{fig:moon-exp-results-plots-appendix}
\end{figure*}

\begin{table}[h]
\centering
    \caption{Test accuracies of methods trained on the Moon dataset. Each entry shows an average of five runs.
    Highlighted values indicate the best performance in the  \mbox{\colorbox{tablegreen}{"\textbf{Best}"}} (peak accuracy during training) and \mbox{\colorbox{tablegreen}{"Last"}} (final accuracy after training) columns.
    \looseness -1
    }
    \label{tab:test-accuracy-comp-moon-appendix} 
    \resizebox{1\columnwidth}{!}{%
    \begin{tabular}{ll|cccccccccccccccc} \specialrule{1.5pt}{1pt}{1pt}
         \multicolumn{2}{c}{\multirow{2}{*}{\makecell{\\ \\ \textbf{Setting}}}} & \multicolumn{16}{c}{\textbf{Method}} \\ \cmidrule(lr){3-18}
        & & \multicolumn{2}{c}{\textsc{Floral}} & \multicolumn{2}{c}{SVM} & \multicolumn{2}{c}{NN} & \multicolumn{2}{c}{NN-PGD} & \multicolumn{2}{c}{LN-SVM} & \multicolumn{2}{c}{Curie} & \multicolumn{2}{c}{LS-SVM} &  \multicolumn{2}{c}{K-LID} \\ 
        \cmidrule(lr){3-4} \cmidrule(lr){5-6} \cmidrule(lr){7-8} \cmidrule(lr){9-10} \cmidrule(lr){11-12} \cmidrule(lr){13-14} \cmidrule(lr){15-16} \cmidrule(lr){17-18}
        & &  Best & Last & Best & Last & Best & Last & Best & Last & Best & Last & Best & Last & Best & Last & Best & Last \\ \specialrule{1.5pt}{1pt}{1pt}
         \text{Clean} & $C=10, \gamma=0.5$  & 0.954 & 0.950 & 0.952 & 0.952 & 0.960 & 0.960 & \cellcolor{tablegreen} \textbf{0.966} & \cellcolor{tablegreen} 0.964 & 0.933 & 0.933 & 0.924 & 0.924 & 0.952 & 0.952 & 0.947 & 0.947 \\
         $D^{\text{adv}}=5\%$ & $C=10, \gamma=0.5$   & 0.941 & \cellcolor{tablegreen} 0.941 & 0.938 & 0.938 & 0.926 & 0.926 & \cellcolor{tablegreen} \textbf{0.964} & 0.937 & 0.933 & 0.933 & 0.892 & 0.892 & 0.938 & 0.938 & 0.933 & 0.933 \\ 
         $D^{\text{adv}}=10\%$ & $C=10, \gamma=0.5$  & 0.915 & 0.874 & 0.878 & 0.878 & 0.859 & 0.855 & \cellcolor{tablegreen} \textbf{0.927} & 0.853 & 0.868 & 0.868 & 0.889 & \cellcolor{tablegreen} 0.889 & 0.874 & 0.874 & 0.868 & 0.868  \\
         $D^{\text{adv}}=25\%$ & $C=10, \gamma=0.5$   & \cellcolor{tablegreen} \textbf{0.769} & \cellcolor{tablegreen} 0.738 & 0.717 & 0.651 & 0.693 & 0.647 & 0.740 & 0.655 & 0.717 & 0.653 & 0.731 & 0.661 & 0.696 & 0.656 & 0.717 & 0.653 \\ 
         \midrule
         \text{Clean} & $C=10, \gamma=1$    & \cellcolor{tablegreen} \textbf{0.968} &  0.966 & 0.968 & \cellcolor{tablegreen} 0.968 & 0.960 & 0.960 & 0.966 & 0.964 & 0.940 & 0.940 & 0.941 & 0.941 & 0.881 & 0.881 & 0.966 & 0.966  \\
         $D^{\text{adv}}=5\%$ & $C=10, \gamma=1$   & \cellcolor{tablegreen} \textbf{0.966} &  \cellcolor{tablegreen} 0.966 & 0.965 & 0.957 & 0.926 & 0.926 & 0.964 & 0.937 & 0.940 & 0.940 &  0.903 &  0.903 & 0.881 & 0.881 & 0.964 & 0.964   \\
         $D^{\text{adv}}=10\%$ & $C=10, \gamma=1$   & 0.924  & \cellcolor{tablegreen} 0.907 & 0.912 & 0.900 & 0.859 & 0.855 & \cellcolor{tablegreen} \textbf{0.927} & 0.853 & 0.869 & 0.868 & 0.907 & 0.907 & 0.894 & 0.894 & 0.908 & 0.907  \\
         $D^{\text{adv}}=25\%$ & $C=10, \gamma=1$   & \cellcolor{tablegreen} \textbf{0.801} & \cellcolor{tablegreen} 0.768 & 0.753 & 0.717 & 0.693 & 0.647 & 0.740 & 0.655 & 0.754 & 0.693 & 0.779 & 0.697 & 0.766 & 0.721 & 0.747 & 0.690  \\ 
         \midrule
         \text{Clean} & $C=100, \gamma=10$  & 0.965  & \cellcolor{tablegreen} 0.964 & \cellcolor{tablegreen} \textbf{0.966} & 0.964 & 0.960 & 0.960 & 0.966 & 0.964 & 0.950  & 0.949 & 0.932 & 0.931 & 0.964 & 0.964 & 0.966 & 0.964 \\ 
         $D^{\text{adv}}=5\%$ & $C=100, \gamma=10$  & 0.955 & 0.940 & 0.951 & 0.937 & 0.926 & 0.926 & \cellcolor{tablegreen} \textbf{0.964} & 0.937 & 0.951 & 0.940 & 0.888 & 0.888 & 0.947 & \cellcolor{tablegreen} 0.945 & 0.951 & 0.940  \\ 
         $D^{\text{adv}}=10\%$ & $C=100, \gamma=10$   & 0.910 & 0.877 & 0.895 & 0.874 & 0.859 & 0.855 & \cellcolor{tablegreen} \textbf{0.927} & 0.853 & 0.894 & \cellcolor{tablegreen} 0.888 & 0.889 & 0.881 & 0.896 & 0.875 & 0.895 & 0.876  \\ 
         $D^{\text{adv}}=25\%$ & $C=100, \gamma=10$   & 0.740 & 0.720 & 0.697 & 0.693 & 0.693 & 0.647 & 0.740 & 0.655 & 0.697 & 0.694 & \cellcolor{tablegreen} \textbf{0.760} & \cellcolor{tablegreen} 0.744 & 0.701 & 0.687 & 0.697 & 0.693 \\  \specialrule{1.5pt}{1pt}{1pt}
    \end{tabular}
    }
\end{table}
\vspace{-0.1cm}
\subsection{IMDB}
\label{sec:imdb-results-appendix}
\vspace{-0.2cm}
We report the test accuracy and loss performance of \textsc{Floral} against RoBERTa on the IMDB dataset in Figure~\ref{fig:imdb-result-comp-roberta-svm} and Table~\ref{tab:test-accuracy-comp-imdb-additional}.
As demonstrated, \textsc{Floral} consistently exhibits superior accuracy and a smaller loss in more adversarial problem instances, without sacrificing the clean performance. 
This shows the effectiveness of \textsc{Floral} in achieving robust classifiers when integrated with foundation models such as RoBERTa.
\vspace{-0.2cm}
\begin{table}[h]
\centering
    \caption{Test accuracy and loss of methods trained on the IMDB dataset. Each entry shows an average of five replications.
    Highlighted values indicate the best performance in the  \mbox{\colorbox{tablegreen}{"\textbf{Best}"}} (peak accuracy during training) and \mbox{\colorbox{tablegreen}{"Last"}} (final accuracy after training) columns.
    \textsc{Floral} demonstrates superior robust accuracy and lower test loss compared to RoBERTa, particularly in more adversarial scenarios. 
    }
    \label{tab:test-accuracy-loss-comp-imdb} 
    \resizebox{0.6\columnwidth}{!}{%
    \begin{tabular}{l|cccc|cccc} \specialrule{1.5pt}{1pt}{1pt}
         \multicolumn{1}{c}{\multirow{2}{*}{\makecell{\\ \\ \textbf{Setting}}}} & \multicolumn{4}{c}{\textbf{Accuracy}} & \multicolumn{4}{c}{\textbf{Loss}} \\ \cmidrule(lr){2-9}
        & \multicolumn{2}{c}{\textsc{Floral}}  & \multicolumn{2}{c}{RoBERTa} \vrule & \multicolumn{2}{c}{\textsc{Floral}} & \multicolumn{2}{c}{RoBERTa}  \\ 
        \cmidrule(lr){2-3} \cmidrule(lr){4-5} \cmidrule(lr){6-7} \cmidrule(lr){8-9} 
        &  Best & Last & Best & Last & Best & Last & Best & Last \\ \specialrule{1.5pt}{1pt}{1pt}
         \text{Clean}  & 0.911 & \cellcolor{tablegreen} 0.911 & \cellcolor{tablegreen} \textbf{0.911} & 0.911 & \cellcolor{tablegreen} \textbf{0.196} &  \cellcolor{tablegreen} 0.216 & 0.229 & 0.282   \\
         $D^{\text{adv}}=10\%$  & 0.903 & \cellcolor{tablegreen} 0.903 & \cellcolor{tablegreen} \textbf{0.904} & 0.903 & 0.234 & 0.259 & \cellcolor{tablegreen} \textbf{0.227} & \cellcolor{tablegreen} 0.231    \\
         $D^{\text{adv}}=25\%$  & \cellcolor{tablegreen} \textbf{0.889} & \cellcolor{tablegreen} 0.889 & 0.882 & 0.861 & \cellcolor{tablegreen} \textbf{0.310} & \cellcolor{tablegreen} 0.333 &   0.337 & 0.365   \\
         $D^{\text{adv}}=30\%$ & \cellcolor{tablegreen} \textbf{0.880} & \cellcolor{tablegreen} 0.880 & 0.867 & 0.835 & \cellcolor{tablegreen} \textbf{0.353} & \cellcolor{tablegreen} 0.366 &  0.428 & 0.428   \\
         $D^{\text{adv}}=35\%$ & \cellcolor{tablegreen} \textbf{0.871} & \cellcolor{tablegreen} 0.871 & 0.827 & 0.805 & \cellcolor{tablegreen} \textbf{0.381} & \cellcolor{tablegreen} 0.395 &  0.496 & 0.496    \\ 
         $D^{\text{adv}}=40\%$ & \cellcolor{tablegreen} \textbf{0.863} & \cellcolor{tablegreen} 0.863 & 0.779 & 0.771 & \cellcolor{tablegreen} \textbf{0.428} & \cellcolor{tablegreen} 0.439 & 0.551 & 0.551 \\
         \specialrule{1.5pt}{1pt}{1pt}
    \end{tabular}
    }\vspace{-0.2cm}
\end{table}
\begin{figure*}[t]
    \centering 
    \begin{subfigure}{0.24\textwidth}
    \centering
    {\includegraphics[width=1\linewidth,trim=0 0 0 0,clip]{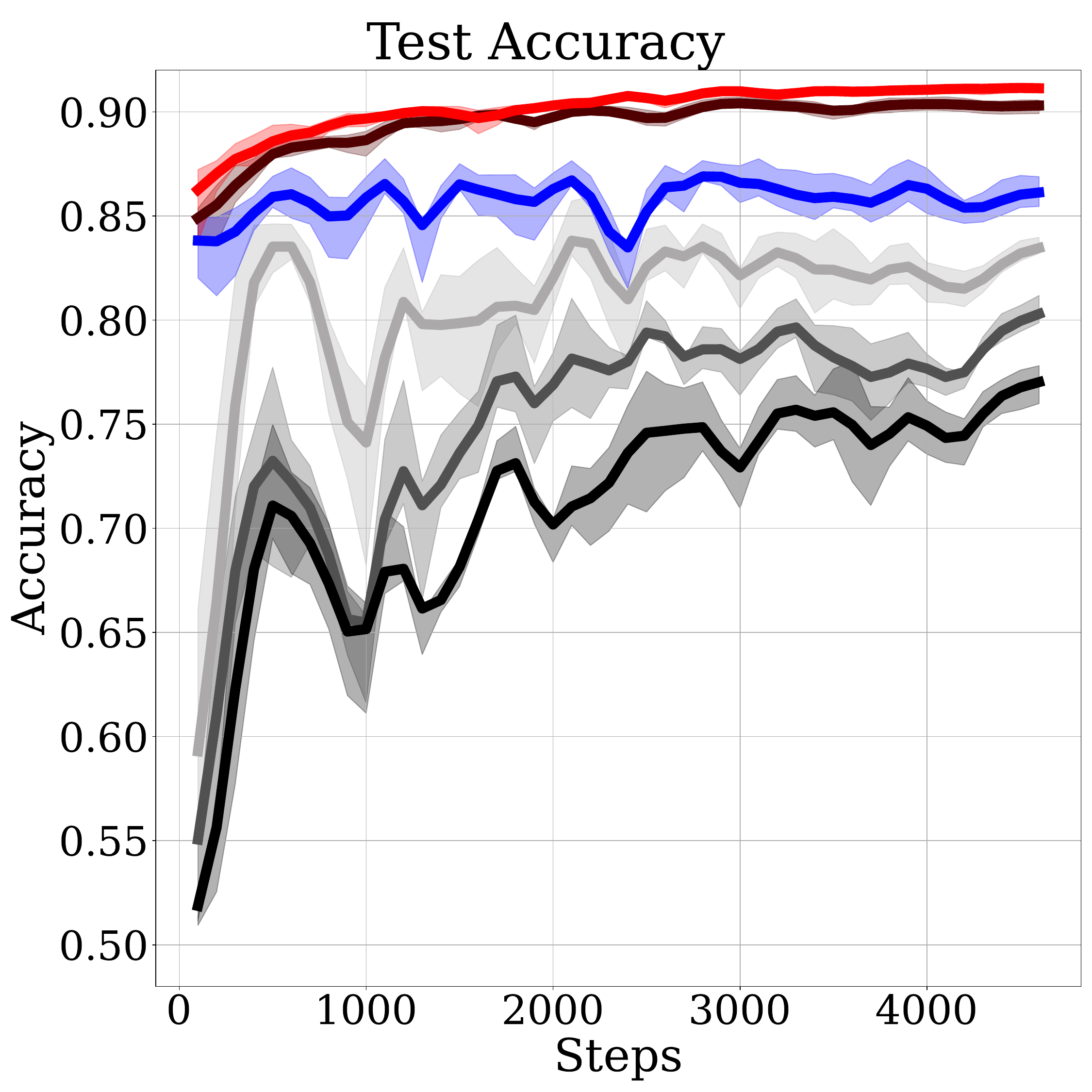}}
    \caption{RoBERTa.}
    
    \end{subfigure} 
    \begin{subfigure}{0.24\textwidth}
    \centering
    {\includegraphics[width=1\linewidth,trim=0 0 0 0,clip]{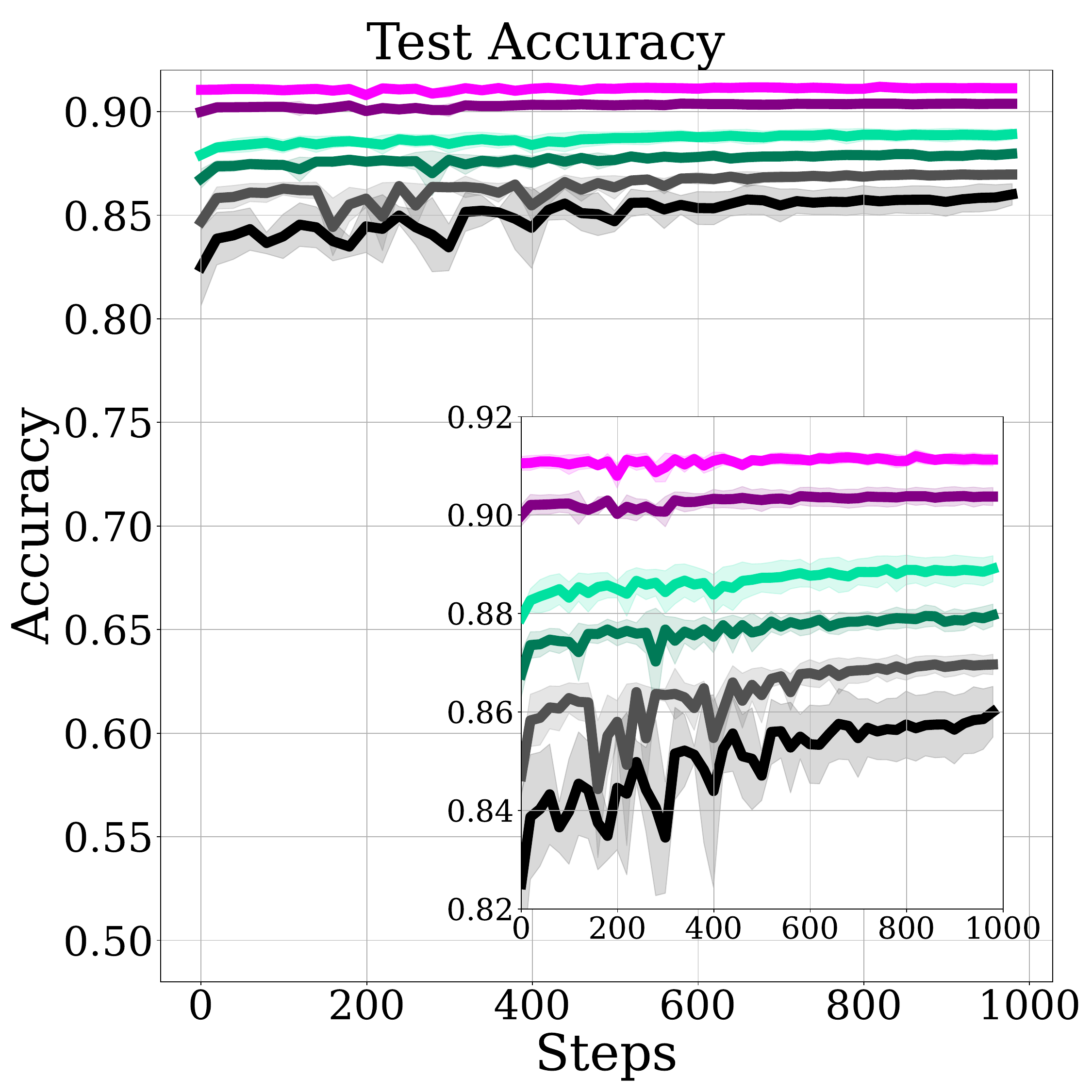}}
    \caption{\textsc{Floral}.}
    
    \end{subfigure} 
    \begin{subfigure}{0.24\textwidth}
    \centering
    {\includegraphics[width=1\linewidth,trim=0 0 0 0,clip]{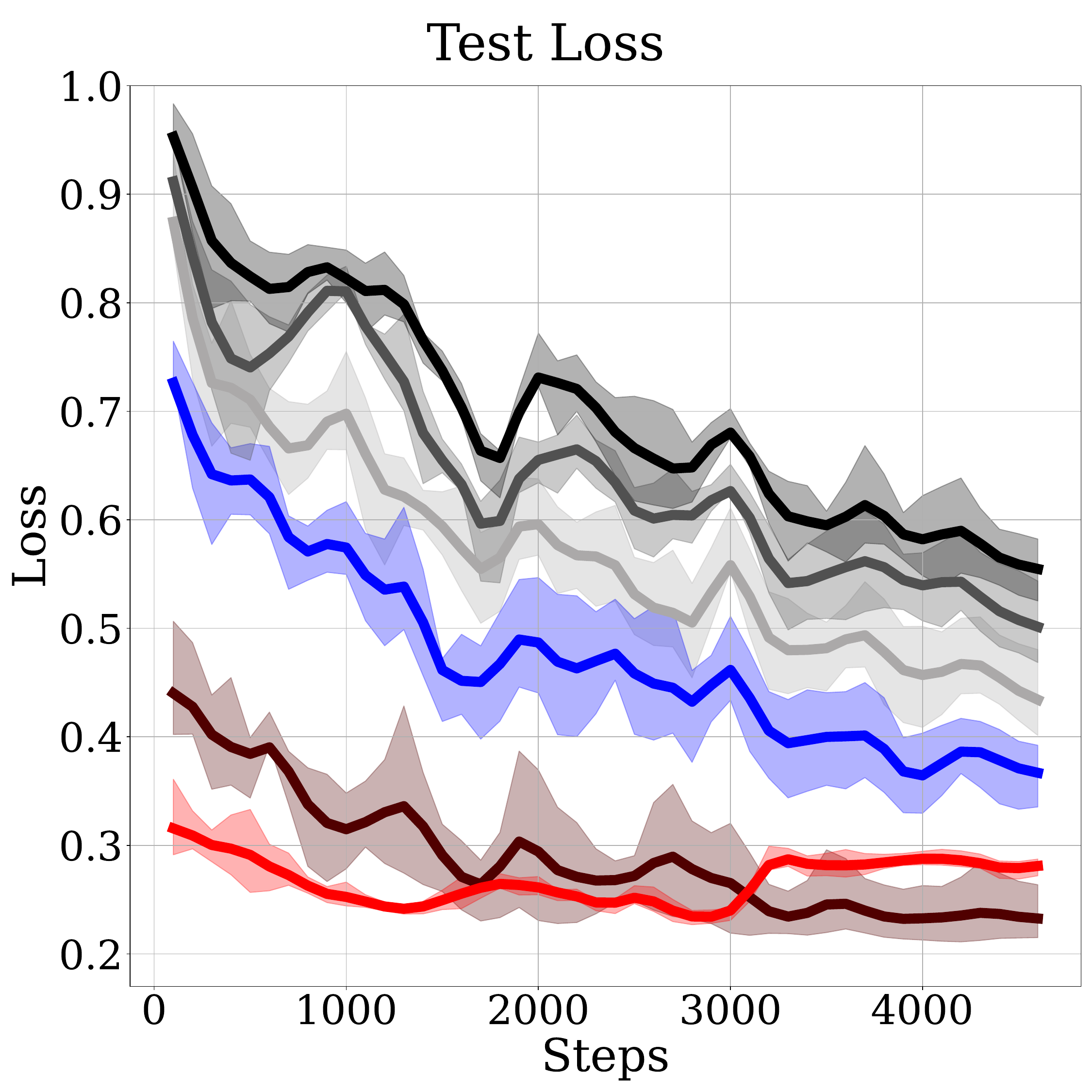}}
    \caption{RoBERTa.}
    
    \end{subfigure} 
    \begin{subfigure}{0.24\textwidth}
    \centering
    {\includegraphics[width=1\linewidth,trim=0 0 0 0,clip]{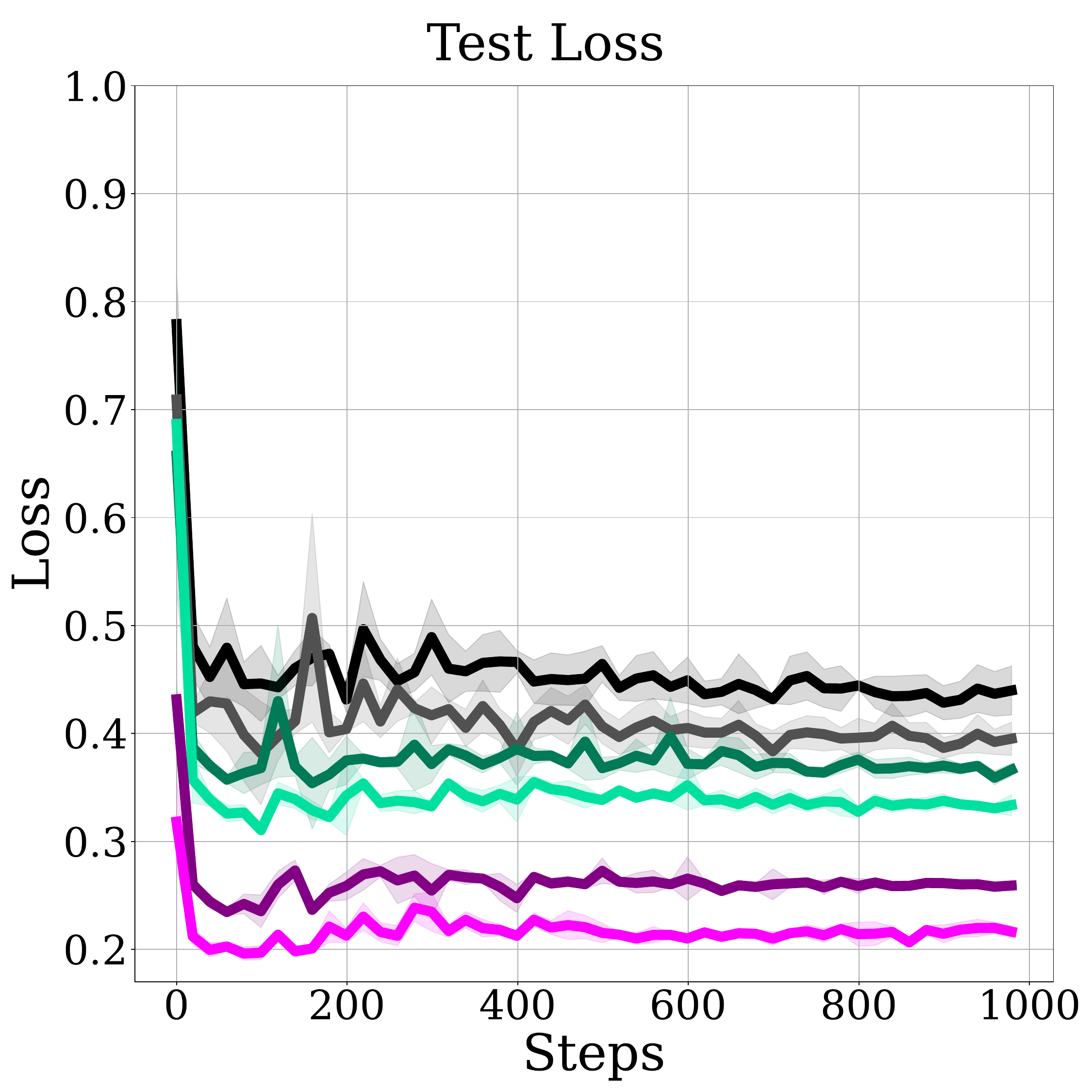}}
    \caption{\textsc{Floral}.}
    
    \end{subfigure} 
    \\
    \begin{subfigure}{0.9\textwidth}{\includegraphics[width=1\linewidth,trim=50 20 20 20,clip]{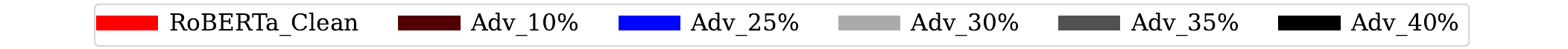}}
    \end{subfigure}
    \begin{subfigure}{0.91\textwidth}
    {\includegraphics[width=1\linewidth,trim=50 20 20 20,clip]{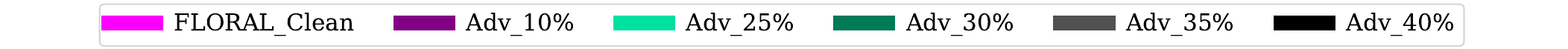}}
    \end{subfigure}
    \\
\setlength{\belowcaptionskip}{-15pt}
  \caption{Test accuracy (\textbf{(a)-(b)}) and test loss (\textbf{(c)-(d)}) of methods on the IMDB dataset. 
  \textsc{Floral} integration outperforms fine-tuned RoBERTa in maintaining better test accuracy and converging faster to lower loss, even when trained on extracted embeddings with heavily adversarial labels. \looseness -1
  }
\label{fig:imdb-result-comp-roberta-svm}
\end{figure*}

\vspace{1cm}
\clearpage
\paragraph{Analysis on influential training points.}
We further analyze how the influential training points (affecting the model's predictions) identified by \textsc{Floral} and RoBERTa change.

To illustrate, Figure~\ref{fig:most-important-points-roberta-vs-floral} shows an example from a replication where both models are trained on a dataset with 40\% adversarially labelled examples. We also provide the result for RoBERTa fine-tuned on the clean dataset.
For \textsc{Floral}, the most influential points are selected from the most important support vectors, while for RoBERTa, these are the points yielding the largest loss gradient with respect to the input. 
The example clearly demonstrates that \textsc{Floral}, implemented on RoBERTa-extracted embeddings, shifts the most important training point for the model's decision boundary. \textsc{Floral} identified a more descriptive point compared to others as given in Figure~\ref{fig:most-important-points-roberta-vs-floral}, however, further analysis is required to determine whether \textsc{Floral} consistently identifies such training points across all cases.

Additionally, we investigate the overlap in influential training points between the two methods. 
To this end, for each method, we extract the $25\%$ most influential training points (for the model predictions) among the training dataset, and measure how much overlap between these two sets. 
In Figure~\ref{fig:percentage-common-influential-points}, we report the percentage of "common" influential points identified from the IMDB dataset, averaged over replications, with error bars denoting the standard deviation. 
The left figure shows the percentage overlap between \textsc{Floral} trained on the IMDB dataset with different poison levels and RoBERTa fine-tuned on clean labels. Whereas, the right plot shows the overlap between both models trained on the dataset with different poison levels. 
On the clean dataset, although there are some differences, both methods almost identify the same set of influential points. 
However, as adversarial labels increase, the overlap decreases. This shows that \textsc{Floral} extracts more critical points that enhance the model's robustness in adversarial settings, as supported by its superior robust accuracy, already shown in Figure~\ref{fig:imdb-result-comp-roberta-svm} in Section~\ref{sec:experiment-results}.
\begin{figure}[ht]
    \centering  
    \begin{subfigure}[t]{0.3\textwidth}\centering{\includegraphics[width=1\linewidth,trim=0 0 0 0,clip]{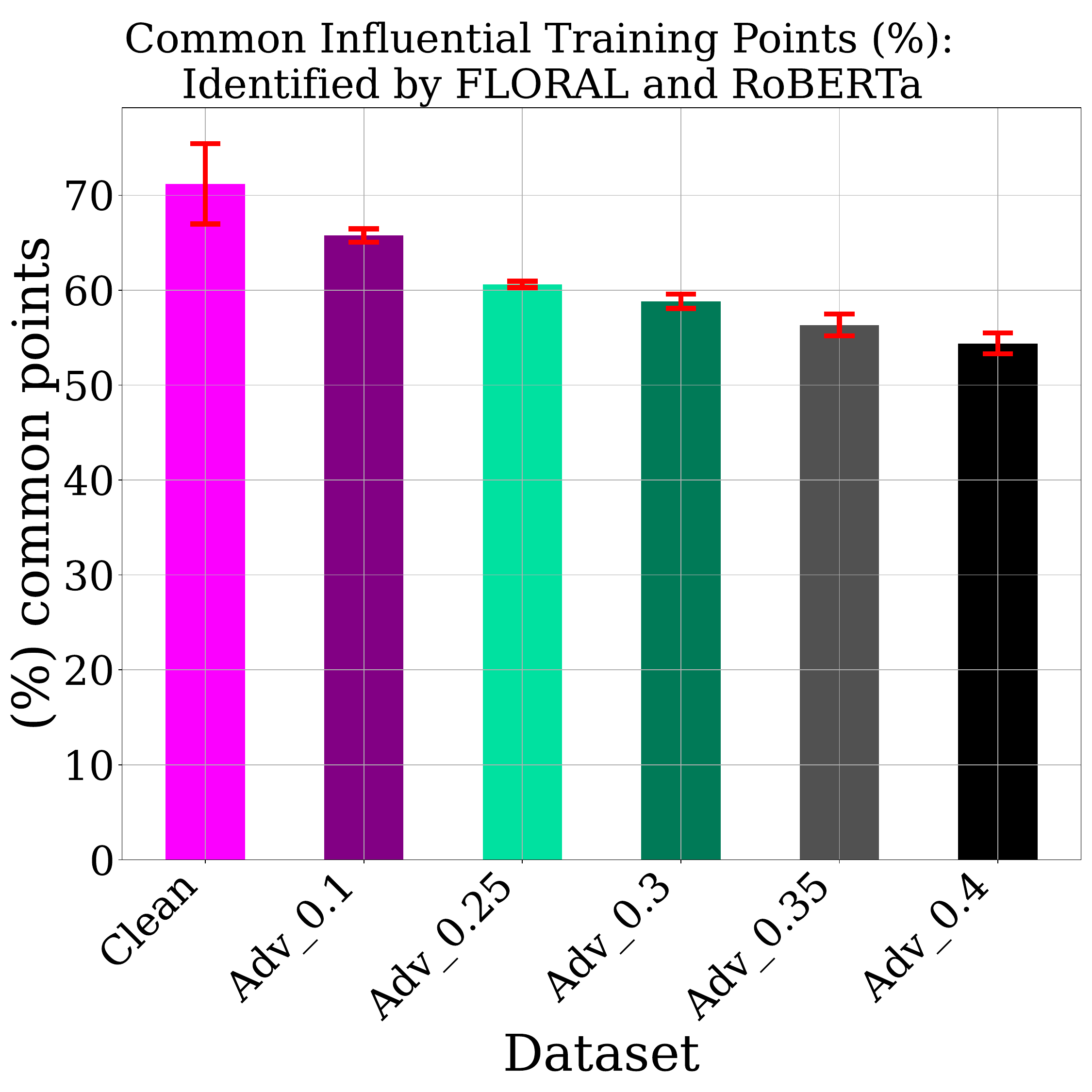}}
    \caption{w.r.t. RoBERTa (clean).}
    
    \end{subfigure}%
    \begin{subfigure}[t]{0.3\textwidth}\centering{\includegraphics[width=1\linewidth,trim=0 0 0 0,clip]{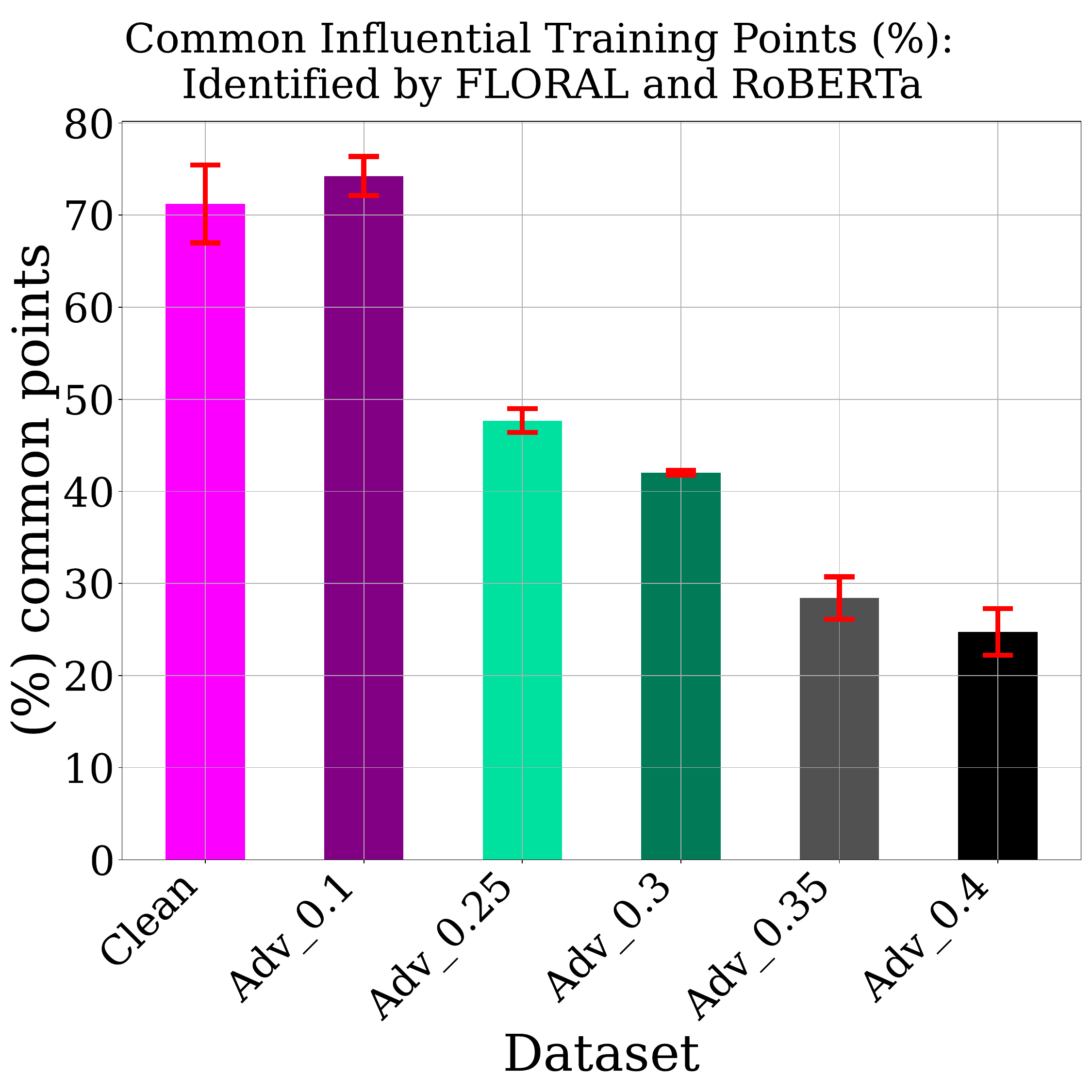}}
    \caption{w.r.t. RoBERTa (adversarial).}
    
    \end{subfigure}%
     \caption{The percentage of "common" influential points identified by \textsc{Floral} and RoBERTa from the IMDB dataset, averaged over replications, with error bars denoting the standard deviation. "Clean" shows the dataset with clean labels, whereas adversarial datasets contain $\{10, 25, 30, 35, 40\} (\%)$ poisoned labels. 
     Even when both methods are fine-tuned on the clean dataset, slight differences emerge in the identified influential training points. The discrepancy increases as the dataset becomes more adversarial, highlighting that \textsc{Floral} adjusts the influential training points affecting the model's predictions. \looseness -1}
\label{fig:percentage-common-influential-points}
\end{figure}

\begin{figure}[htp]\vspace{-0.2cm}
    \centering  
    {\includegraphics[width=1\linewidth,trim=0 90 0 0,clip]{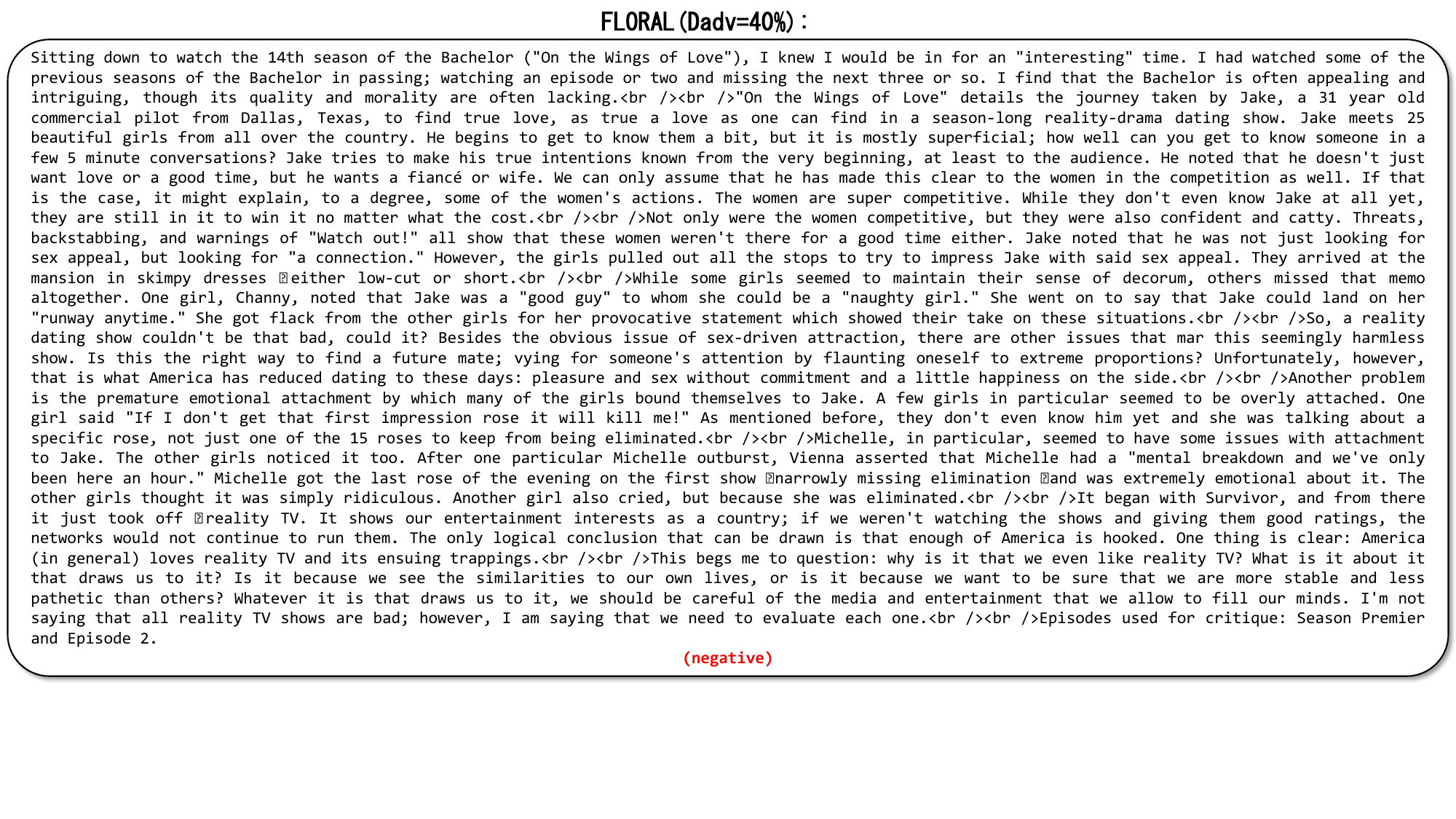}} \\
    {\includegraphics[width=1\linewidth,trim=0 230 0 0,clip]{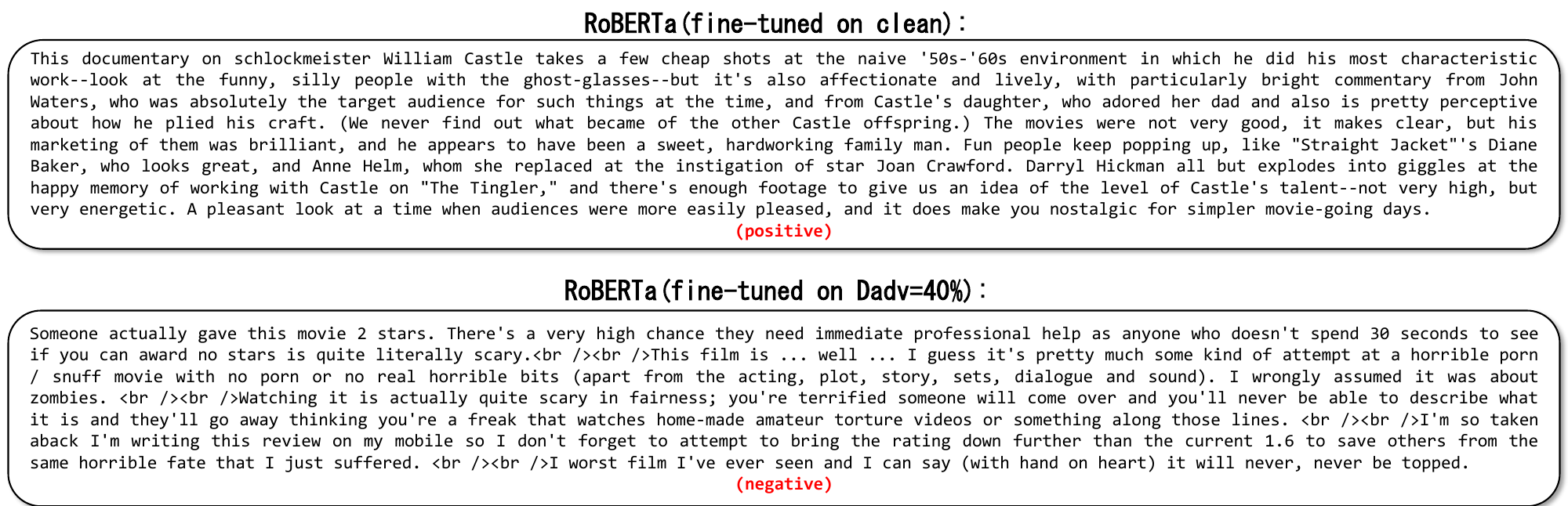}}
     \caption{{\color{red}\textit{Note: Figure might contain offensive content.}} The most influential training point for the model's predictions, identified by \textsc{Floral} and RoBERTa from the IMDB dataset. \textsc{Floral} implemented on RoBERTa extracted-embeddings changes the most important training point for the model's decision boundary. }
\label{fig:most-important-points-roberta-vs-floral}
\end{figure}

\clearpage 
\subsection{MNIST}
\label{app:mnist1vs7-experiments}

To demonstrate the generalizability of \textsc{Floral} across diverse datasets, we provide additional experiments on the MNIST dataset \citep{deng2012mnist}. Similar to \citep{rosenfeld2020certified}, we consider classes $1$ and $7$ which leads to a dataset of $\mathcal{D}=\{(x_i, y_i)\}_{i=1}^{13007}$ where $x_i \in \mathbb{R}^{784}$ and $y_i \in \{ \pm 1\}$, with $784$ pixel values for each image.

We perform the randomized top-$k$ label poisoning attack described in Section~\ref{sec:proposed-approach} and report the clean and robust test accuracy performance of methods in Figure~\ref{fig:mnist1vs7-exp-results-plots-appendix} and Table~\ref{tab:test-accuracy-comp-mnist1vs7}. The results show that \textsc{Floral} maintains a higher robust accuracy compared to most of the baselines. While Curie behaves almost on par, \textsc{Floral} achieves higher robust accuracy. Although NN baselines perform better on clean and $5\%$ adversarially labelled datasets, they show a significant accuracy decrease when the training dataset gets more adversarial. \looseness -1

 \begin{figure*}[h]
    \centering  
    \begin{subfigure}[t]
    {0.25\textwidth}\centering{\includegraphics[width=1\linewidth,trim=0 0 0 0,clip]{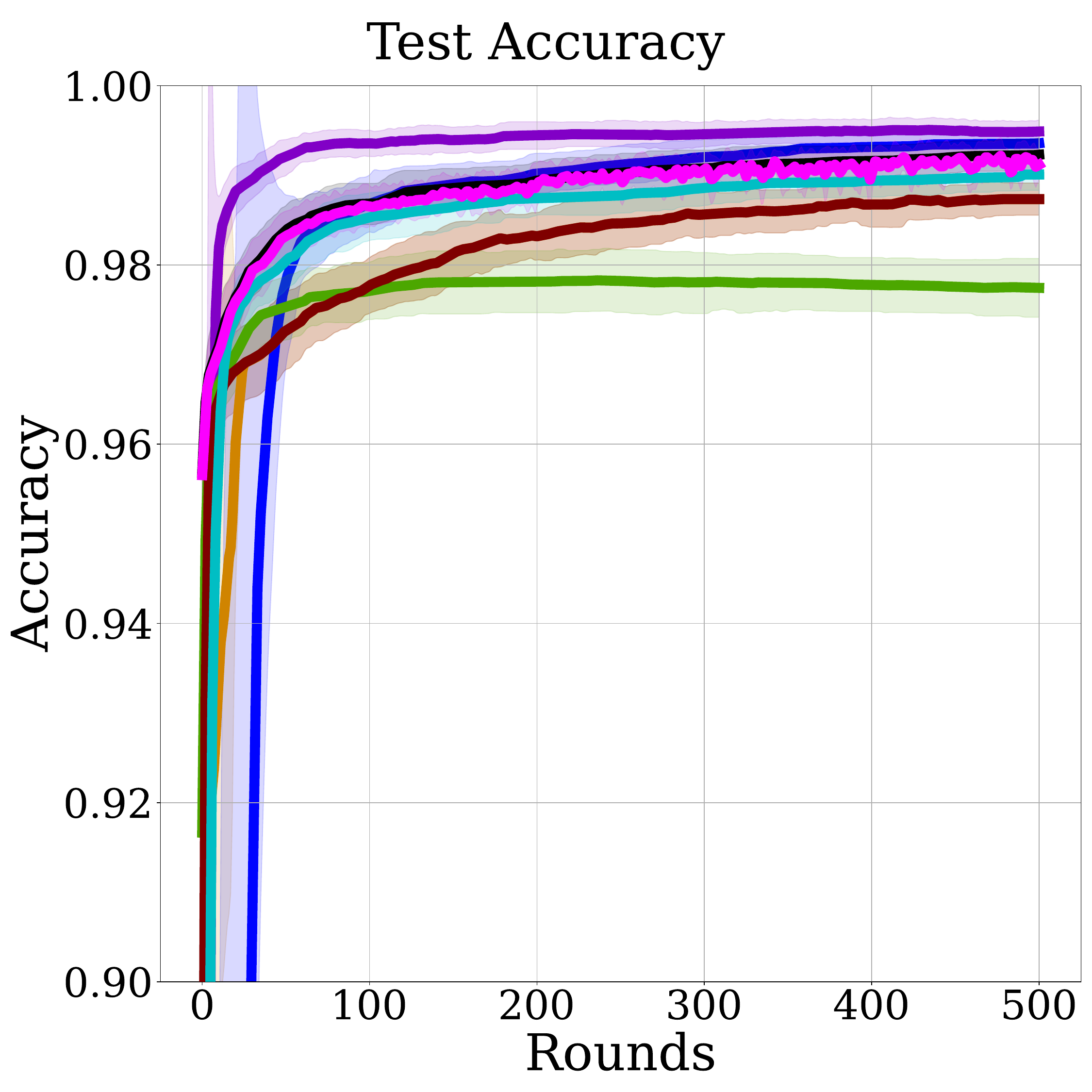}}
    \caption{Clean.}
    
    \end{subfigure}%
    \begin{subfigure}[t]{0.25\textwidth}\centering{\includegraphics[width=1\linewidth,trim=0 0 0 0,clip]{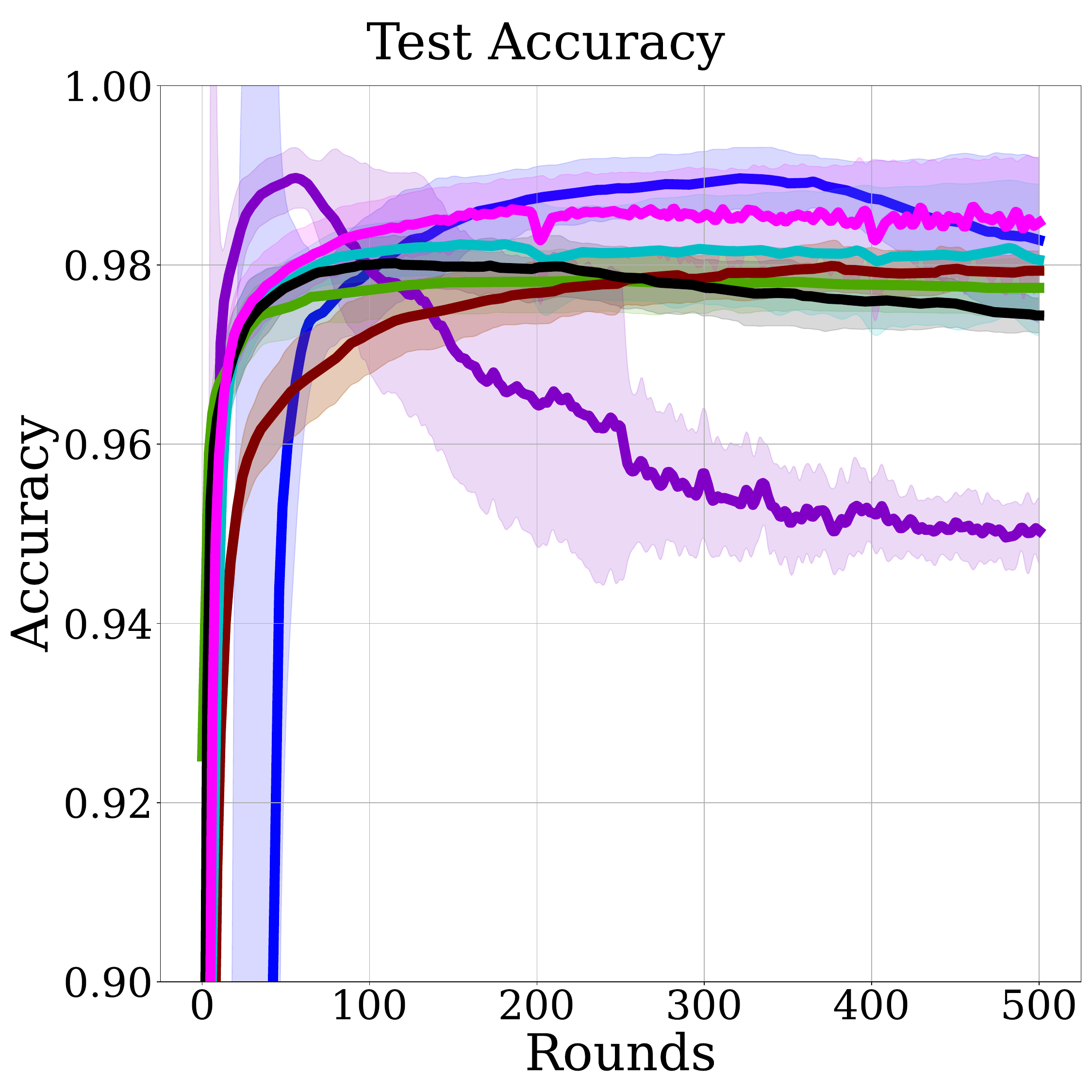}}
    \caption{$D^{\text{adv}} = 5\%$.}
    
    \end{subfigure}%
    \begin{subfigure}[t]{0.25\textwidth}\centering{\includegraphics[width=1\linewidth,trim=0 0 0 0,clip]{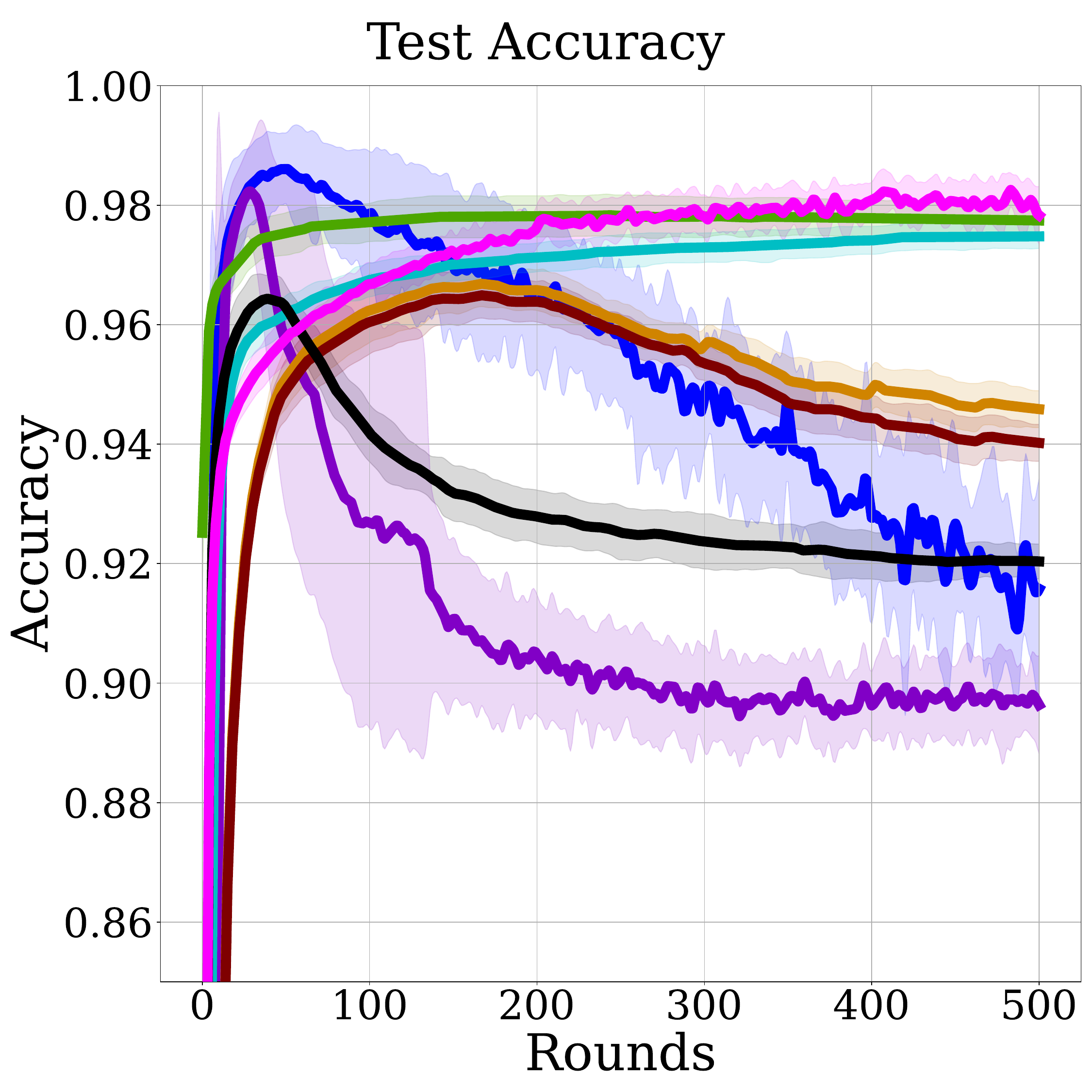}}
    \caption{$D^{\text{adv}} = 10\%$.}
    
    \end{subfigure}%
    \begin{subfigure}[t]{0.25\textwidth}\centering{\includegraphics[width=1\linewidth,trim=0 0 0 0,clip]{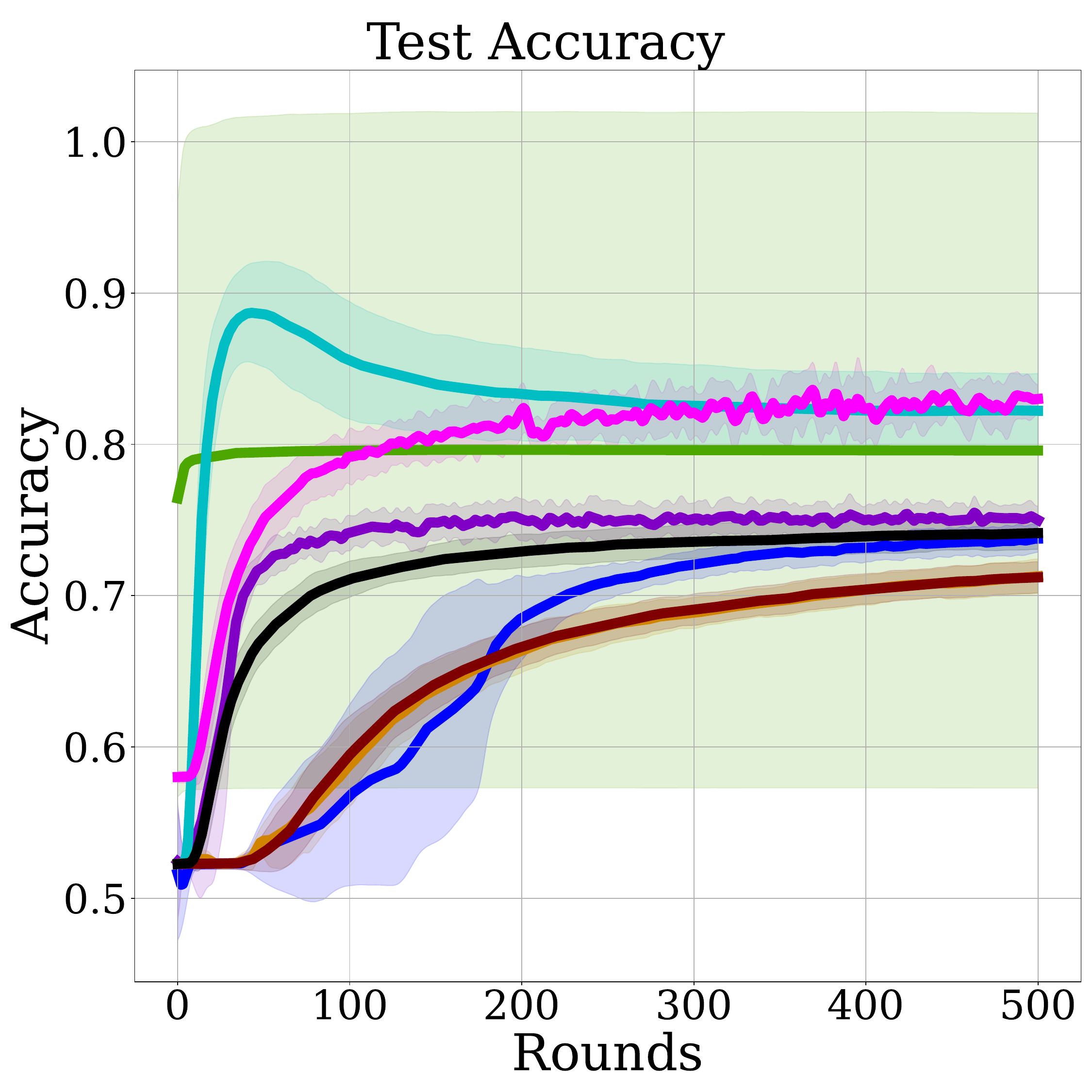}}
    \caption{$D^{\text{adv}} = 25\%$.}
    
    \end{subfigure}%
    \\
    \begin{subfigure}
{1\textwidth}\centering{\includegraphics[width=1\linewidth,trim=40 10 10 10,clip]{figures/moon_legend_8cols_new.pdf}}
    
    \end{subfigure}%
    \caption{Clean and robust test accuracy of methods trained on the MNIST-1vs7 dataset. "Clean" refers to the dataset with clean labels, while the adversarial datasets contain $\{5, 10, 25\}$ ($\%$) poisoned labels.
    For all SVM-related models, the setting $C=5$, $\gamma=0.005$ is used.
    As the level of label poisoning increases, models trained on adversarial datasets generally demonstrate a decline in accuracy.
    However, \textsc{Floral} maintains a higher robust accuracy level compared to most of the baselines and behaving on par with the Curie method.
    }
\label{fig:mnist1vs7-exp-results-plots-appendix}
\end{figure*}

\begin{table}[h]
\centering
    \caption{Test accuracies of methods trained over the MNIST-1vs7 dataset. Each entry shows the average of five replications with different train/test splits. Highlighted values indicate the best performance in the  \mbox{\colorbox{tablegreen}{"\textbf{Best}"}} (peak accuracy during training) and \mbox{\colorbox{tablegreen}{"Last"}} (final accuracy after training) columns.
    }
    \label{tab:test-accuracy-comp-mnist1vs7} 
    \resizebox{1\columnwidth}{!}{%
    \begin{tabular}{ll|cccccccccccccccc} \specialrule{1.5pt}{1pt}{1pt}
         \multicolumn{2}{c}{\multirow{2}{*}{\makecell{\\ \\ \textbf{Setting}}}} & \multicolumn{16}{c}{\textbf{Method}} \\ \cmidrule(lr){3-18}
        & & \multicolumn{2}{c}{\textsc{Floral}} & \multicolumn{2}{c}{SVM} & \multicolumn{2}{c}{NN} & \multicolumn{2}{c}{NN-PGD} & \multicolumn{2}{c}{LN-SVM} & \multicolumn{2}{c}{Curie} & \multicolumn{2}{c}{LS-SVM} &  \multicolumn{2}{c}{K-LID} \\ 
        \cmidrule(lr){3-4} \cmidrule(lr){5-6} \cmidrule(lr){7-8} \cmidrule(lr){9-10} \cmidrule(lr){11-12} \cmidrule(lr){13-14} \cmidrule(lr){15-16} \cmidrule(lr){17-18}
        & &  Best & Last & Best & Last & Best & Last & Best & Last & Best & Last & Best & Last & Best & Last & Best & Last \\ \specialrule{1.5pt}{1pt}{1pt}
         \text{Clean} & $C=5, \gamma=0.005$ & 0.992 & 0.991 & 0.992 & 0.992 & 0.993 & 0.993 & \cellcolor{tablegreen} \textbf{0.995} & \cellcolor{tablegreen} 0.994 & 0.987 & 0.987 & 0.990 & 0.990 & 0.978 & 0.977 & 0.987 & 0.987  \\
         $D^{\text{adv}}=5\%$ & $C=5, \gamma=0.005$ & 0.988 & \cellcolor{tablegreen} 0.984 & 0.980 & 0.974 & \cellcolor{tablegreen} 0.989 & 0.982 & 0.989 & 0.949 & 0.979 & 0.979 & 0.984 & 0.979 & 0.978 & 0.977 & 0.979 & 0.979 \\
         $D^{\text{adv}}=10\%$ & $C=5, \gamma=0.005$ &  \cellcolor{tablegreen} \textbf{0.984} &  \cellcolor{tablegreen} 0.978 & 0.964 & 0.920 & 0.982 & 0.930 & 0.982 & 0.894 & 0.965 & 0.940 & 0.974 & 0.974 & 0.978 & 0.977 & 0.966 & 0.945 \\
         $D^{\text{adv}}=25\%$ & $C=5, \gamma=0.005$ & 0.853 & \cellcolor{tablegreen} 0.830 & 0.741 & 0.741 & 0.738 & 0.738 & 0.763 & 0.750 & 0.712 & 0.712 & \cellcolor{tablegreen} \textbf{0.887} & 0.822 & 0.796 & 0.795 & 0.712 & 0.712 \\  \specialrule{1.5pt}{1pt}{1pt}
    \end{tabular}
    }
\end{table}

\subsection{Sensitivity Analysis}
\label{sec:sensitivity-analysis}
In our experiments with the Moon dataset under varying label poisoning levels, we consider attacker budgets $B=2k$ under varying $k$ values, and report the best performing setting in Figure~\ref{fig:moon-exp-results-C10-gamma0.5} in Section~\ref{sec:experiment-results}.

However, we further investigate the sensitivity of \textsc{Floral} to the attacker's budget $B$, by considering levels $B \in \{5, 10, 25, 50, 125 \}$, with results presented in Figure~\ref{fig:moon-sensitivity-attacker-budget}.
As demonstrated, \textsc{Floral} shows superior performance under a constrained attacker budget in the clean label scenario, as expected, since an increasing number of adversarially labelled examples during training degrades clean test accuracy.
In contrast, baseline methods operate on a fixed dataset. However, as the dataset gets more adversarial, \textsc{Floral} outperforms under higher attacker budgets.

 \begin{figure*}[ht]
    \centering  
    \begin{subfigure}[t]{0.25\textwidth}\centering{\includegraphics[width=1\linewidth,trim=0 0 0 0,clip]{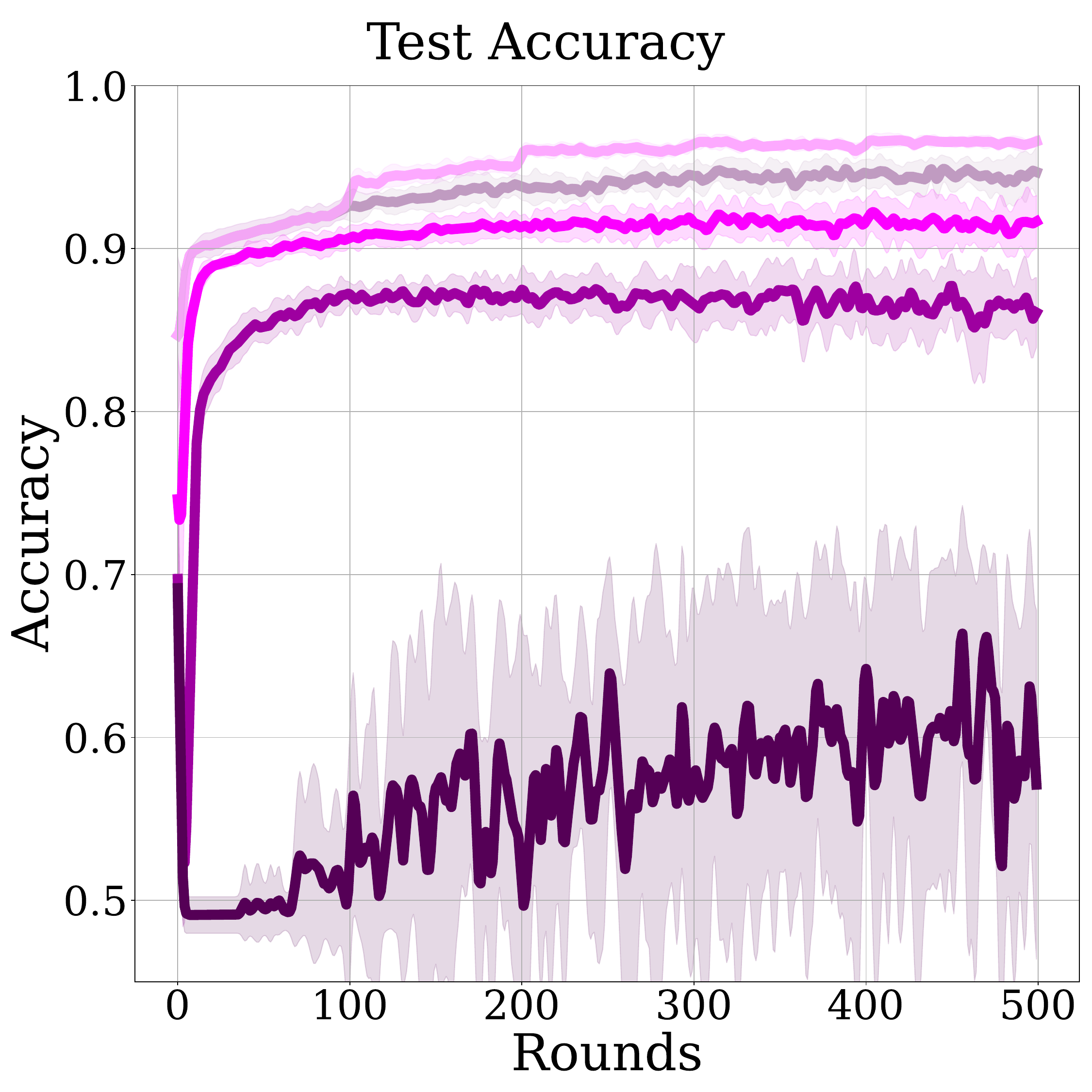}}
    \caption{Clean.}
    \end{subfigure}%
    \begin{subfigure}[t]{0.25\textwidth}\centering{\includegraphics[width=1\linewidth,trim=0 0 0 0,clip]{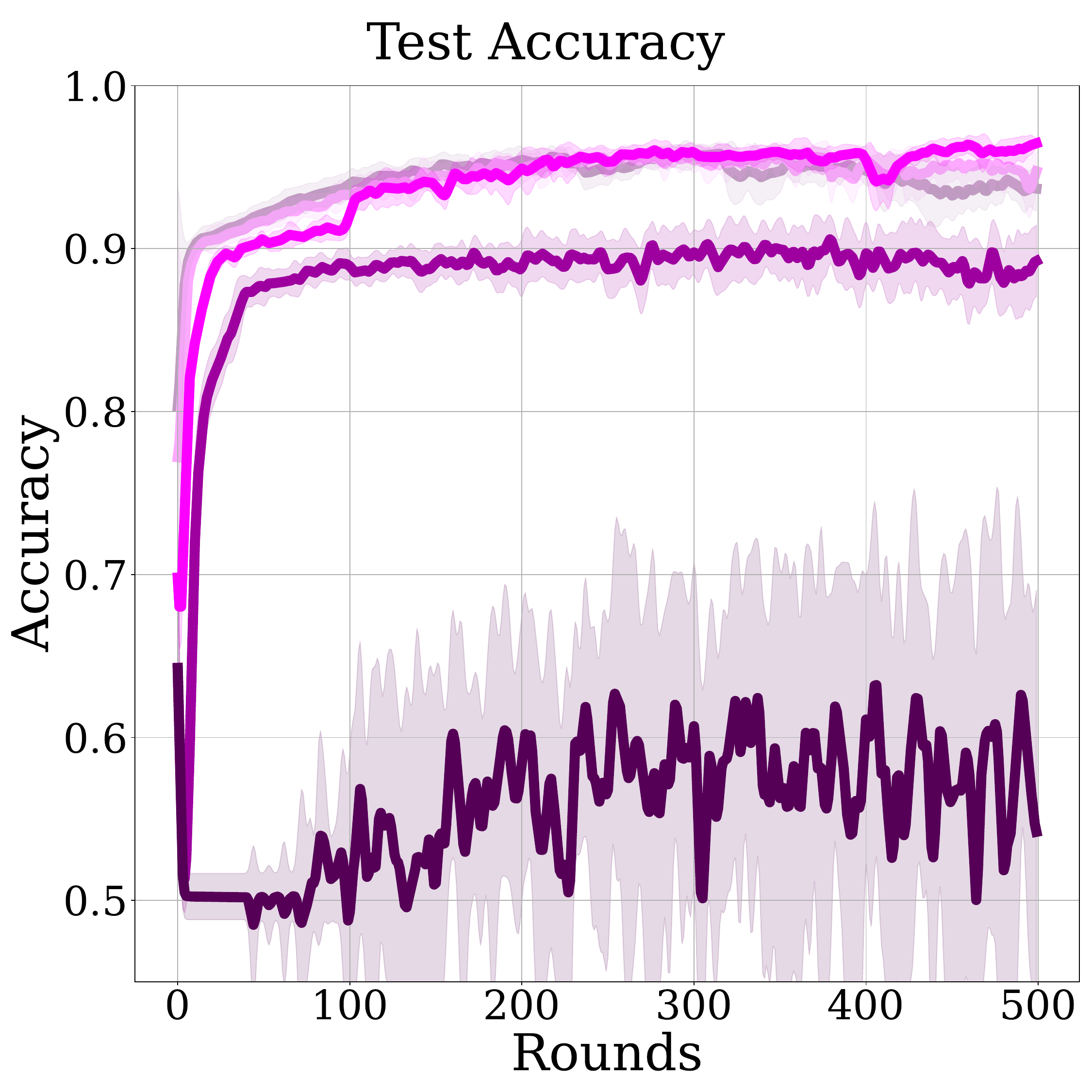}}
    \caption{$D^{\text{adv}} = 5\%$.}
    \end{subfigure}%
    \begin{subfigure}[t]{0.25\textwidth}\centering{\includegraphics[width=1\linewidth,trim=0 0 0 0,clip]{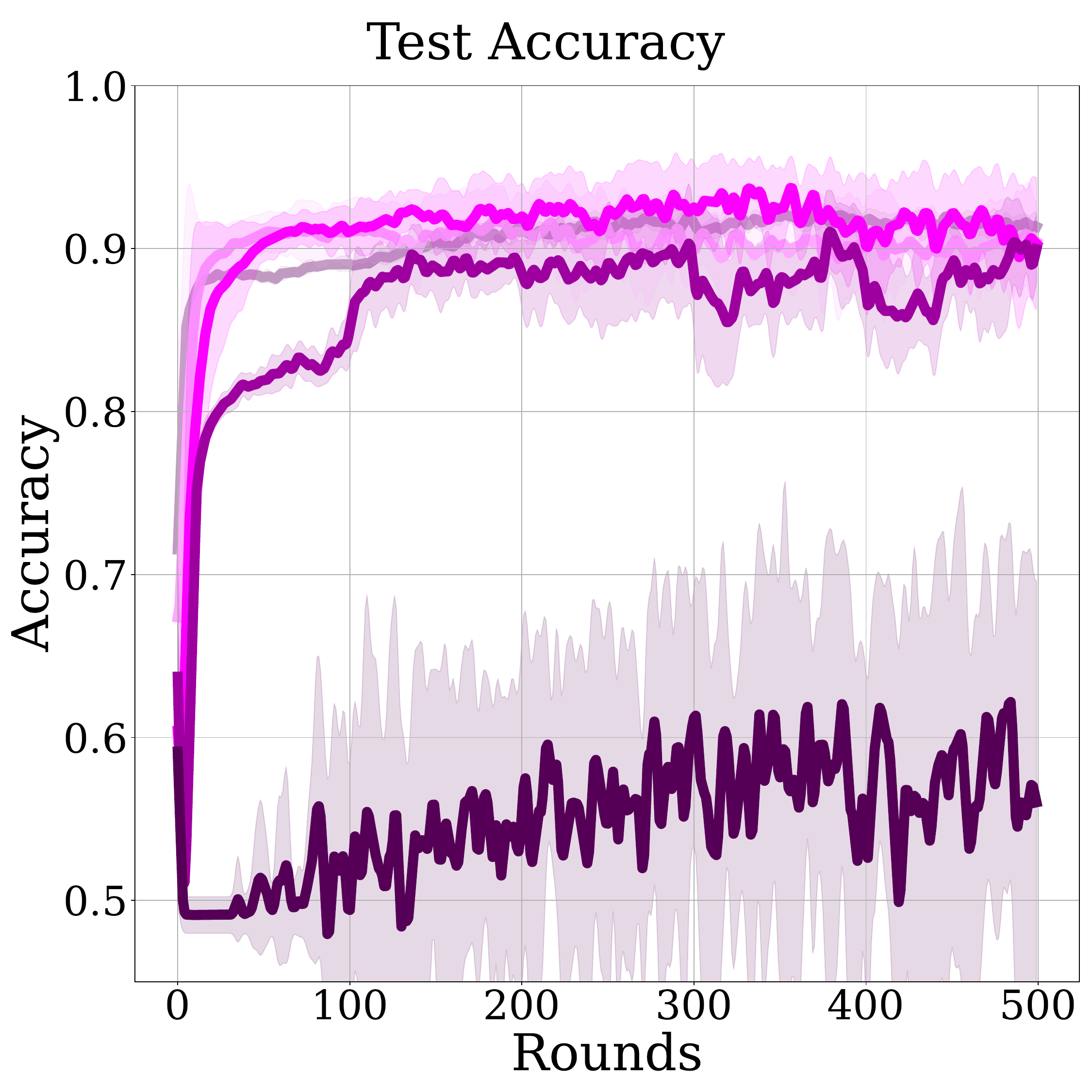}}
    \caption{$D^{\text{adv}} = 10\%$.}
    \end{subfigure}%
    \begin{subfigure}[t]
    {0.25\textwidth}\centering{\includegraphics[width=1\linewidth,trim=0 0 0 0,clip]{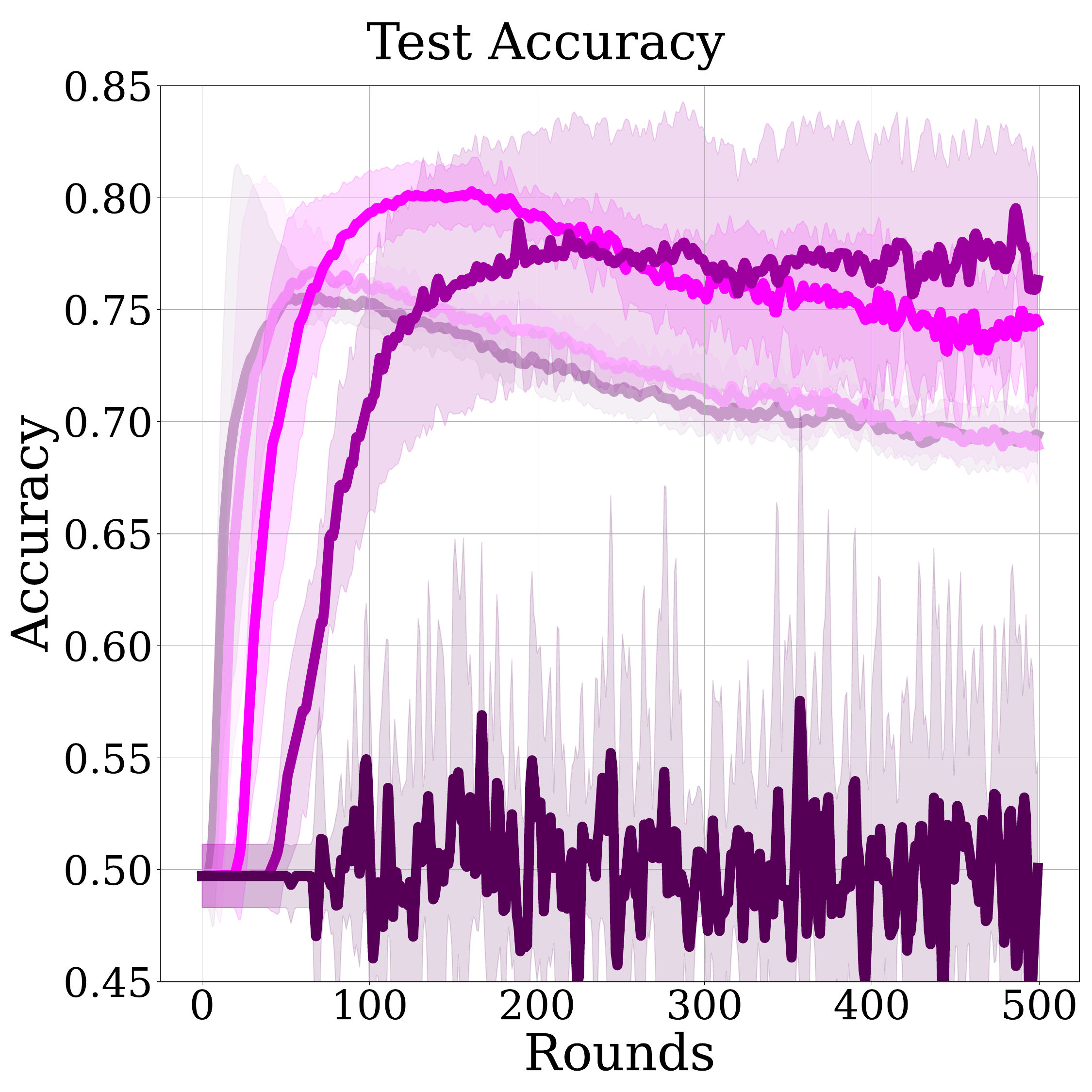}}
    \caption{$D^{\text{adv}} = 25\%$.}
    \end{subfigure}%
    \\
    \begin{subfigure}
{1\textwidth}\centering{\includegraphics[width=1\linewidth,trim=70 10 50 10,clip]{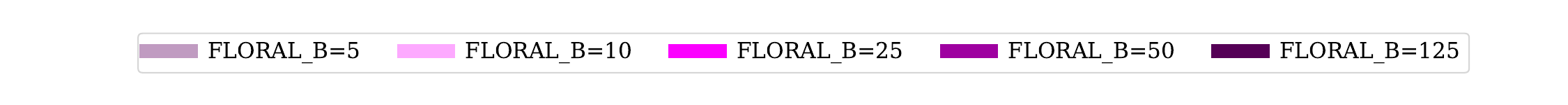}}
    
    \end{subfigure}%
    \caption{The sensitivity of \textsc{Floral} to the attacker's budget $B$. 
    "Clean" refers to the dataset with clean labels, while the adversarial datasets contain $\{5, 10, 25 \} (\%)$ poisoned labels.
    The performance under setting ($C=10$, $\gamma=1$) is presented.
    As the level of label poisoning increases, \textsc{Floral} performs better under higher attacker budget settings.
    }
\label{fig:moon-sensitivity-attacker-budget}
\end{figure*}

\vspace{-0.2cm}
\section{Extension to Multi-class Classification}
\label{app:multi-class-classification}
We extend our algorithm to multi-class classification tasks, as detailed in Algorithm~\ref{alg:robust-svm-game-multiclass}.
The primary modification involves adopting a one-vs-all approach \citep{hsu2002comparison} by employing kernel SVM model $f_{\lambda_0}^m$ for each class $m \in \mathcal{M}$ and associating multiple attackers $a_m, m \in \mathcal{M}$ for the corresponding classifiers.
In each round $t$, the attackers identify the $B_m$ most influential data points with respect to $\lambda_{t}^m$ values of the corresponding models under their constrained budgets $B_m$, and gather them into a set $\mathcal{B}_t$. Among the points in $\mathcal{B}_t$, the labels of top-$k$ influential data points are poisoned according to a predefined label poisoning distribution $q$. The dataset with adversarial labels is then shared with each kernel SVM model and local training is applied via PGD training step.

\begin{algorithm*}[ht]
\setlength{\textfloatsep}{0pt}
\caption{\textsc{Floral}-MultiClass}
\label{alg:robust-svm-game-multiclass}
\begin{algorithmic}[1]
\STATE {\bfseries Input:} Initial kernel SVM models $f_{\lambda_{0}}^m$ for each class $m \in \mathcal{M}$, training dataset $\mathcal{D}_{0}=\{(x_i, y_i)\}_{i=1}^{n}, x_i \in \mathbb{R}^{d}, y_i \in \{ \pm 1\}$, attackers' budgets $B_m$, parameter $k$, where $k \ll \min\{B_m\}_{m \in \mathcal{M}}$, learning rate $\eta$, a pre-defined label flip distribution $q$. \looseness -1
\FOR{round $t = 1, \dots, T$}
\STATE Initialize $\mathcal{B}_t \gets \emptyset$.
\FOR{$m \in \mathcal{M}$} 
\STATE $\mathcal{B}_t^m \gets $ Identify top-$B_m$ support vectors w.r.t. $\lambda_{t-1}^m$ values by solving (\ref{eq:inner-problem-objective}-\ref{eq:inner-problem-const2}).
\ENDFOR
\STATE $\mathcal{B}_t \gets \bigcup_{m \in \mathcal{M}} \mathcal{B}_t^m$.
\STATE $\Tilde{y}^t \gets$ \mbox{\colorbox{magenta!15}{randomized top-$k$}}: Randomly select $k$ points among $\mathcal{B}_t$ and poison their labels w.r.t. $q$. \looseness=-1
\STATE $\mathcal{D}_{t} \gets \{(x_i, \Tilde{y}_i^t)\}_{i=1}^{n}$ \hfill {\color{caribbeangreen} */ Adversarial dataset with selected $k$ poisoned labels}
\FOR{$m \in \mathcal{M}$} 
\STATE Compute gradient of the objective (\ref{eq:outer-problem-objective}), $\nabla_{\lambda} D(f_{\lambda_{t-1}}^m; \mathcal{D}_t)$, based on $\lambda_{t-1}^m, \mathcal{D}_{t}$ as given in (\ref{eq:svm-dual-gradient}).
\STATE Take a PGD step $\lambda_t^m \gets$ \mbox{\colorbox{magenta!15}{\textsc{Prox}}}$_{\mathcal{S}(\Tilde{y}^t)} (\lambda_{t-1}^m - \eta \nabla_{\lambda} D(f_{\lambda_{t-1}^m};\mathcal{D}_{t}))$. \hfill {\color{caribbeangreen} */ Adversarial training}
\ENDFOR
\ENDFOR
\STATE \textbf{return} $\{f_{\lambda_{T}}^m\}_{m \in \mathcal{M}}$
\end{algorithmic}
\end{algorithm*}

\vspace{-0.2cm}
\section{Comparison Against Additional Methods}
\label{app:comp-additional-baselines}

We additionally compare \textsc{Floral} against least squares classifier using randomized smoothing (RS) \citep{rosenfeld2020certified}, and regularized synthetic reduced nearest neighbor (RSRNN) \citep{tavallali2022adversarial} methods on the Moon and MNIST-1vs7 datasets. 
RS provides a robustness certification for a linear classifier under label-flipping attacks. Whereas, RSRNN, conceptually similar to Curie \citep{curie}, provides a filtering-out defense based on clustering. 

We evaluated the performance under different noise ($q$) and $l_2$ regularization ($\lambda$) hyperparameter values for the RS method suggested in \citep{rosenfeld2020certified}, whereas we considered varying number of centroids $K$, penalty coefficient $\lambda$, cost complexity coefficient $\alpha$, for the RSRNN method, using referenced values in the study \citep{tavallali2022adversarial} for the MNIST dataset.  
\looseness -1

The results presented in Tables~\ref{tab:test-accuracy-comp-moon-rs-comp}-\ref{tab:test-accuracy-comp-mnist-rsrnn-comp} demonstrate that \textsc{Floral} consistently outperforms both RS and RSRNN across all datasets and experimental settings. While RSRNN achieves comparable performance on the MNIST dataset, it still falls short of \textsc{Floral}. The performance of the RS method, which employs a linear classifier with a pointwise robustness certificate, aligns with expectations, as its simpler classifier limits its ability to capture complex patterns. In contrast, \textsc{Floral} utilizes kernel SVMs, enabling it to effectively model intricate patterns within the data and achieve superior results.

\begin{table}[h]
\centering
    \caption{Test accuracies of \textsc{Floral} against randomized smoothing (RS) method \citep{rosenfeld2020certified} on the Moon dataset. Each entry shows an average of five runs. 
    We evaluated the performance under different noise ($q$) values for RS. 
    }
    \label{tab:test-accuracy-comp-moon-rs-comp} 
    \resizebox{0.75\columnwidth}{!}{%
    \begin{tabular}{ll|cccc} \specialrule{1.5pt}{1pt}{1pt}
         \multicolumn{2}{c}{\multirow{2}{*}{\makecell{\\ \textbf{Setting}}}} & \multicolumn{4}{c}{\textbf{Method}} \\ \cmidrule(lr){3-6}
        & & \textsc{Floral} & RS ($q=0.1, \lambda=0.01$) & RS ($q=0.3, \lambda=0.01$) & RS ($q=0.4, \lambda=0.01$) 
        \\ \specialrule{1.5pt}{1pt}{1pt}
         \text{Clean} & $C=10, \gamma=1$  & \cellcolor{tablegreen} \textbf{0.968} & 0.557  & 0.509 & 0.509      \\
         $D^{\text{adv}}=5\%$ & $C=10, \gamma=1$  & \cellcolor{tablegreen} \textbf{0.963} & 0.552 & 0.509 & 0.509    \\
         $D^{\text{adv}}=10\%$ & $C=10, \gamma=1$   & \cellcolor{tablegreen} \textbf{0.954} & 0.540  & 0.509 &  0.509    \\
         $D^{\text{adv}}=25\%$ & $C=10, \gamma=1$   & \cellcolor{tablegreen} \textbf{0.776} & 0.520 & 0.505 & 0.505    \\  \specialrule{1.5pt}{1pt}{1pt}
    \end{tabular}
    }
\end{table}

\begin{table}[h]
\centering
    \caption{Test accuracies of \textsc{Floral} against randomized smoothing (RS) method \citep{rosenfeld2020certified} on the MNIST-$1$vs$7$ dataset. Each entry shows an average of five runs. 
    We evaluated the performance under different noise ($q$) and $l_{2}$ regularization ($\lambda$) hyperparameter values for RS, as suggested in \citep{rosenfeld2020certified}. 
    }
    \label{tab:test-accuracy-comp-mnist1vs7-rs-comp} 
    \resizebox{1\columnwidth}{!}{%
    \begin{tabular}{ll|ccccccc} \specialrule{1.5pt}{1pt}{1pt}
         \multicolumn{2}{c}{\multirow{2}{*}{\makecell{\\ \textbf{Setting}}}} & \multicolumn{7}{c}{\textbf{Method}} \\ \cmidrule(lr){3-9}
        & & \textsc{Floral} & RS ($q=0.1, \lambda=0.01$) & RS ($q=0.3, \lambda=0.01$) & RS ($q=0.4, \lambda=0.01$) & RS ($q=0.1, \lambda=12291$) & RS ($q=0.3, \lambda=12291$) & RS ($q=0.4, \lambda=13237$) 
        \\ \specialrule{1.5pt}{1pt}{1pt}
         \text{Clean} & $C=5, \gamma=0.005$  & \cellcolor{tablegreen} \textbf{0.991} & 0.973  & 0.921 & 0.836 & 0.940 & 0.846 & 0.732    \\
         $D^{\text{adv}}=5\%$ & $C=5, \gamma=0.005$  & \cellcolor{tablegreen} \textbf{0.984} & 0.921 & 0.876 & 0.800  & 0.895 & 0.802 & 0.701    \\
         $D^{\text{adv}}=10\%$ & $C=5, \gamma=0.005$   & \cellcolor{tablegreen} \textbf{0.978} & 0.868  & 0.831 & 0.768  & 0.830 & 0.745 & 0.673   \\
         $D^{\text{adv}}=25\%$ & $C=5, \gamma=0.005$  & \cellcolor{tablegreen} \textbf{0.830} & 0.706 & 0.693 & 0.669 & 0.548 & 0.594 & 0.595   \\  \specialrule{1.5pt}{1pt}{1pt}
    \end{tabular}
    }
\end{table}

\begin{table}[h]
\centering
    \caption{Test accuracies of \textsc{Floral} against regularized synthetic reduced nearest neighbor (RSRNN) \citep{tavallali2022adversarial} trained on the Moon dataset. Each entry shows an average of five runs. 
    We evaluated the performance under different hyperparameter values (number of centroids $K$, penalty coefficient $\lambda$, cost complexity coefficient $\alpha$) for the RSRNN method. 
    }
    \label{tab:test-accuracy-comp-moon-rsrnn-comp} 
    \resizebox{1\columnwidth}{!}{%
    \begin{tabular}{ll|ccccc} \specialrule{1.5pt}{1pt}{1pt}
         \multicolumn{2}{c}{\multirow{2}{*}{\makecell{\\ \textbf{Setting}}}} & \multicolumn{5}{c}{\textbf{Method}} \\ \cmidrule(lr){3-7}
        & & \textsc{Floral} & RSRNN ($K=2, \alpha=0.01, \lambda=0.1$) & RSRNN ($K=10, \alpha=0.01, \lambda=0.1$) & RSRNN ($K=10, \alpha=0.1, \lambda=1$) & RSRNN ($K=20, \alpha=0.01, \lambda=0.1$)  
        \\ \specialrule{1.5pt}{1pt}{1pt}
         \text{Clean} & $C=10, \gamma=1$  & \cellcolor{tablegreen} \textbf{0.968} & 0.505 & 0.629 & 0.688 & 0.617    \\
         $D^{\text{adv}}=5\%$ & $C=10, \gamma=1$  & \cellcolor{tablegreen} \textbf{0.963} & 0.502 & 0.547 &  0.603 &  0.512   \\
         $D^{\text{adv}}=10\%$ & $C=10, \gamma=1$   & \cellcolor{tablegreen} \textbf{0.954} & 0.502  & 0.532 & 0.566  &  0.482  \\
         $D^{\text{adv}}=25\%$ & $C=10, \gamma=1$   & \cellcolor{tablegreen} \textbf{0.776} & 0.494 & 0.434 & 0.476 & 0.439    \\  \specialrule{1.5pt}{1pt}{1pt}
    \end{tabular}
    }
\end{table}
\vspace{-0.2cm}
\begin{table}[H]
\centering
    \caption{Test accuracies of \textsc{Floral} against regularized synthetic reduced nearest neighbor (RSRNN) \citep{tavallali2022adversarial} trained on the MNIST-$1$vs$7$ dataset. Each entry shows an average of five runs. 
    We evaluated the performance under different cost complexity coefficient ($\alpha$) values for the RSRNN method. 
    }
    \label{tab:test-accuracy-comp-mnist-rsrnn-comp} 
    \resizebox{0.75\columnwidth}{!}{%
    \begin{tabular}{ll|ccc} \specialrule{1.5pt}{1pt}{1pt}
         \multicolumn{2}{c}{\multirow{2}{*}{\makecell{\\ \textbf{Setting}}}} & \multicolumn{3}{c}{\textbf{Method}} \\ \cmidrule(lr){3-5}
        & & \textsc{Floral} & RSRNN ($K=10, \alpha=0.1, \lambda=1.0$) & RSRNN ($K=10, \alpha=1.0, \lambda=1.0$)
        \\ \specialrule{1.5pt}{1pt}{1pt}
         \text{Clean} & $C=5, \gamma=0.005$  & \cellcolor{tablegreen} \textbf{0.991} &  0.619 & 0.692   \\
         $D^{\text{adv}}=5\%$ & $C=5, \gamma=0.005$ & \cellcolor{tablegreen} \textbf{0.984} & 0.599 & 0.441  \\
         $D^{\text{adv}}=10\%$ & $C=5, \gamma=0.005$   & \cellcolor{tablegreen} \textbf{0.978} &  0.432 & 0.408     \\
         $D^{\text{adv}}=25\%$ & $C=5, \gamma=0.005$   & \cellcolor{tablegreen} \textbf{0.830} & 0.403 & 0.408    \\  \specialrule{1.5pt}{1pt}{1pt}
    \end{tabular}
    }
\end{table}

\section{Experiments Under Different Label Attacks}
\label{app:different-attacks}
To show the generalizability of our approach in the presence of otherlabel poisoning attack types, we further compare \textsc{Floral} against baselines on adversarial datasets generated using the \texttt{alfa}, \texttt{alfa-tilt} \citep{xiao2015support} and \texttt{LFA} attacks \citep{label-sanitization}.

\subsection{Experiments with the $\texttt{alfa-tilt}$ attack}
\label{app:alfa-tilt-attack-experiments}

We further evaluate \textsc{Floral}'s performance in the presence of $\texttt{alfa-tilt}$ attack \citep{xiao2015support} on the Moon and MNIST-$1$vs$7$ datasets.
We report the results on the Moon datasets in Table~\ref{tab:test-accuracy-comp-moon-alfa-tilt}, whereas we present the results for MNIST-$1$vs$7$ dataset in Figure~\ref{fig:mnist1vs7-exp-results-plots-appendix-alfa-tilt} and Table~\ref{tab:test-accuracy-comp-mnist1vs7-alfatilt}.

As shown with the results on the Moon dataset, \textsc{Floral} is able to achieve a higher "Best" robust accuracy level throughout the training process. 
Furthermore, \textsc{Floral}'s effectiveness under $\texttt{alfa-tilt}$ attack is best shown on the MNIST dataset. As reported in Figure~\ref{fig:mnist1vs7-exp-results-plots-appendix-alfa-tilt} and Table~\ref{tab:test-accuracy-comp-mnist1vs7-alfatilt}, \textsc{Floral} achieves an outperforming robust accuracy level compared to baseline methods on all adversarial settings. This demonstrates the potential of \textsc{Floral} defense against other label poisoning attacks. 
\looseness -1

\begin{table}[h]
\centering
    \caption{Test accuracies of methods trained over the Moon dataset with adversarial labels generated by the \texttt{alfa-tilt} \citep{xiao2015support} attack. Each entry shows the average of five replications.
    Highlighted values indicate the best performance in the  \mbox{\colorbox{tablegreen}{"\textbf{Best}"}} (peak accuracy during training) and \mbox{\colorbox{tablegreen}{"Last"}} (final accuracy after training) columns.
    }
    \label{tab:test-accuracy-comp-moon-alfa-tilt} 
    \resizebox{1\columnwidth}{!}{%
    \begin{tabular}{ll|cccccccccccccccc} \specialrule{1.5pt}{1pt}{1pt}
         \multicolumn{2}{c}{\multirow{2}{*}{\makecell{\\ \\ \textbf{Setting}}}} & \multicolumn{16}{c}{\textbf{Method}} \\ \cmidrule(lr){3-18}
        & & \multicolumn{2}{c}{\textsc{Floral}} & \multicolumn{2}{c}{SVM} & \multicolumn{2}{c}{NN} & \multicolumn{2}{c}{NN-PGD} & \multicolumn{2}{c}{LN-SVM} & \multicolumn{2}{c}{Curie} & \multicolumn{2}{c}{LS-SVM} &  \multicolumn{2}{c}{K-LID} \\ 
        \cmidrule(lr){3-4} \cmidrule(lr){5-6} \cmidrule(lr){7-8} \cmidrule(lr){9-10} \cmidrule(lr){11-12} \cmidrule(lr){13-14} \cmidrule(lr){15-16} \cmidrule(lr){17-18}
        & &  Best & Last & Best & Last & Best & Last & Best & Last & Best & Last & Best & Last & Best & Last & Best & Last \\ \specialrule{1.5pt}{1pt}{1pt}
        \text{Clean} & $C=10, \gamma=1$ & \cellcolor{tablegreen} \textbf{0.968} &  0.966 & 0.968 & \cellcolor{tablegreen} 0.968 & 0.960 & 0.960 & 0.966 & 0.964 & 0.940 & 0.940 & 0.941 & 0.941 & 0.881 & 0.881 & 0.966 & 0.966 \\
         $D^{\text{adv}}=5\%$ & $C=10, \gamma=1$   & \cellcolor{tablegreen} \textbf{0.972} & \cellcolor{tablegreen} 0.957 & 0.944 & 0.939 & 0.948 & 0.948 & 0.962 & 0.943 & 0.956 & 0.956 & 0.940 & 0.939 & 0.898 & 0.896 & 0.937 & 0.936 \\
         $D^{\text{adv}}=10\%$ & $C=10, \gamma=1$   & \cellcolor{tablegreen} \textbf{0.971} & 0.928 & 0.910 & 0.886 & 0.915 & 0.914 & 0.940 & 0.906 & 0.930 & \cellcolor{tablegreen} 0.930 & 0.920 & 0.902 & 0.898 & 0.896 & 0.926 & 0.926 \\  
         $D^{\text{adv}}=25\%$ & $C=10, \gamma=1$ & \cellcolor{tablegreen} \textbf{0.893} & \cellcolor{tablegreen} 0.824 & 0.787 & 0.722 & 0.837 & 0.750 & 0.837 & 0.720 & 0.786 & 0.723 & 0.792 & 0.759 & 0.792 & 0.791 & 0.770 & 0.708
          \\ \specialrule{1.5pt}{1pt}{1pt}
    \end{tabular}
    }
\end{table}

\begin{table}[h]
\centering
    \caption{Test accuracies of methods trained on the MNIST-$1$vs$7$ dataset under \texttt{alfa-tilt} poisoning attack \citep{xiao2015support}. Each entry shows the average of five replications with different train/test splits.
    Highlighted values indicate the best performance in the  \mbox{\colorbox{tablegreen}{"\textbf{Best}"}} (peak accuracy during training) and \mbox{\colorbox{tablegreen}{"Last"}} (final accuracy after training) columns.
    }
    \label{tab:test-accuracy-comp-mnist1vs7-alfatilt} 
    \resizebox{1\columnwidth}{!}{%
    \begin{tabular}{ll|cccccccccccccccc} \specialrule{1.5pt}{1pt}{1pt}
         \multicolumn{2}{c}{\multirow{2}{*}{\makecell{\\ \\ \textbf{Setting}}}} & \multicolumn{16}{c}{\textbf{Method}} \\ \cmidrule(lr){3-18}
        & & \multicolumn{2}{c}{\textsc{Floral}} & \multicolumn{2}{c}{SVM} & \multicolumn{2}{c}{NN} & \multicolumn{2}{c}{NN-PGD} & \multicolumn{2}{c}{LN-SVM} & \multicolumn{2}{c}{Curie} & \multicolumn{2}{c}{LS-SVM} &  \multicolumn{2}{c}{K-LID} \\ 
        \cmidrule(lr){3-4} \cmidrule(lr){5-6} \cmidrule(lr){7-8} \cmidrule(lr){9-10} \cmidrule(lr){11-12} \cmidrule(lr){13-14} \cmidrule(lr){15-16} \cmidrule(lr){17-18}
        & &  Best & Last & Best & Last & Best & Last & Best & Last & Best & Last & Best & Last & Best & Last & Best & Last \\ \specialrule{1.5pt}{1pt}{1pt}
         \text{Clean} & $C=5, \gamma=0.005$ & 0.992 & 0.991 & 0.992 & 0.992 & 0.993 & 0.993 & \cellcolor{tablegreen} \textbf{0.995} & \cellcolor{tablegreen} 0.994 & 0.987 & 0.987 & 0.990 & 0.990 & 0.978 & 0.977 & 0.987 & 0.987 \\
         $D^{\text{adv}}=5\%$ & $C=5, \gamma=0.005$     & \cellcolor{tablegreen} \textbf{0.991} & \cellcolor{tablegreen} 0.990  & 0.980 & 0.980 & 0.991 & 0.958 & 0.988 & 0.955 & 0.979 & 0.979  & 0.987 & 0.987 & 0.980 & 0.979 & 0.978 & 0.978 \\
         $D^{\text{adv}}=10\%$ & $C=5, \gamma=0.005$   & 0.984 & \cellcolor{tablegreen} 0.982  &  0.970 & 0.970 & 0.986 & 0.917 & \cellcolor{tablegreen} \textbf{0.988} & 0.909 & 0.966 & 0.966 & 0.974 & 0.974 & 0.979 & 0.978  & 0.965 & 0.965 \\
         $D^{\text{adv}}=25\%$ & $C=5, \gamma=0.005$   & 0.811 & \cellcolor{tablegreen} 0.788 & 0.713  & 0.713 & 0.795 & 0.739 & \cellcolor{tablegreen} \textbf{0.824} & 0.754 & 0.703 & 0.701  & 0.734 & 0.734 & 0.526 & 0.526  & 0.707 & 0.705   \\  \specialrule{1.5pt}{1pt}{1pt}
    \end{tabular}
    }
\end{table}

 \begin{figure*}[h]
    \centering  
    \begin{subfigure}[t]
    {0.25\textwidth}\centering{\includegraphics[width=1\linewidth,trim=0 0 0 0,clip]{figures/mnist1vs7_accuracy_Dclean_C10_gamma05.pdf}}
    \caption{Clean.}
    
    \end{subfigure}%
    \begin{subfigure}[t]{0.25\textwidth}\centering{\includegraphics[width=1\linewidth,trim=0 0 0 0,clip]{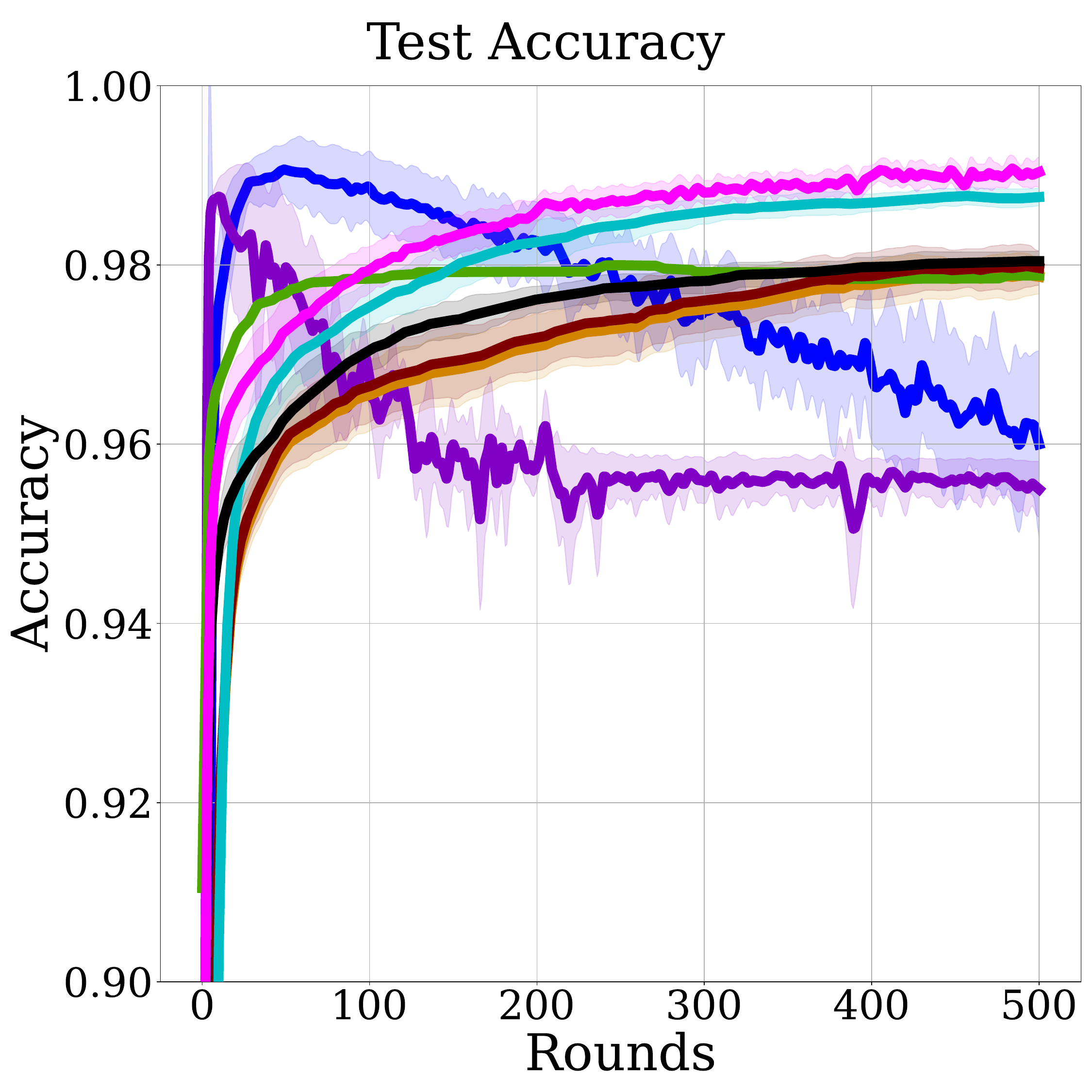}}
    \caption{$D^{\text{adv}} = 5\%$.}
    
    \end{subfigure}%
    \begin{subfigure}[t]{0.25\textwidth}\centering{\includegraphics[width=1\linewidth,trim=0 0 0 0,clip]{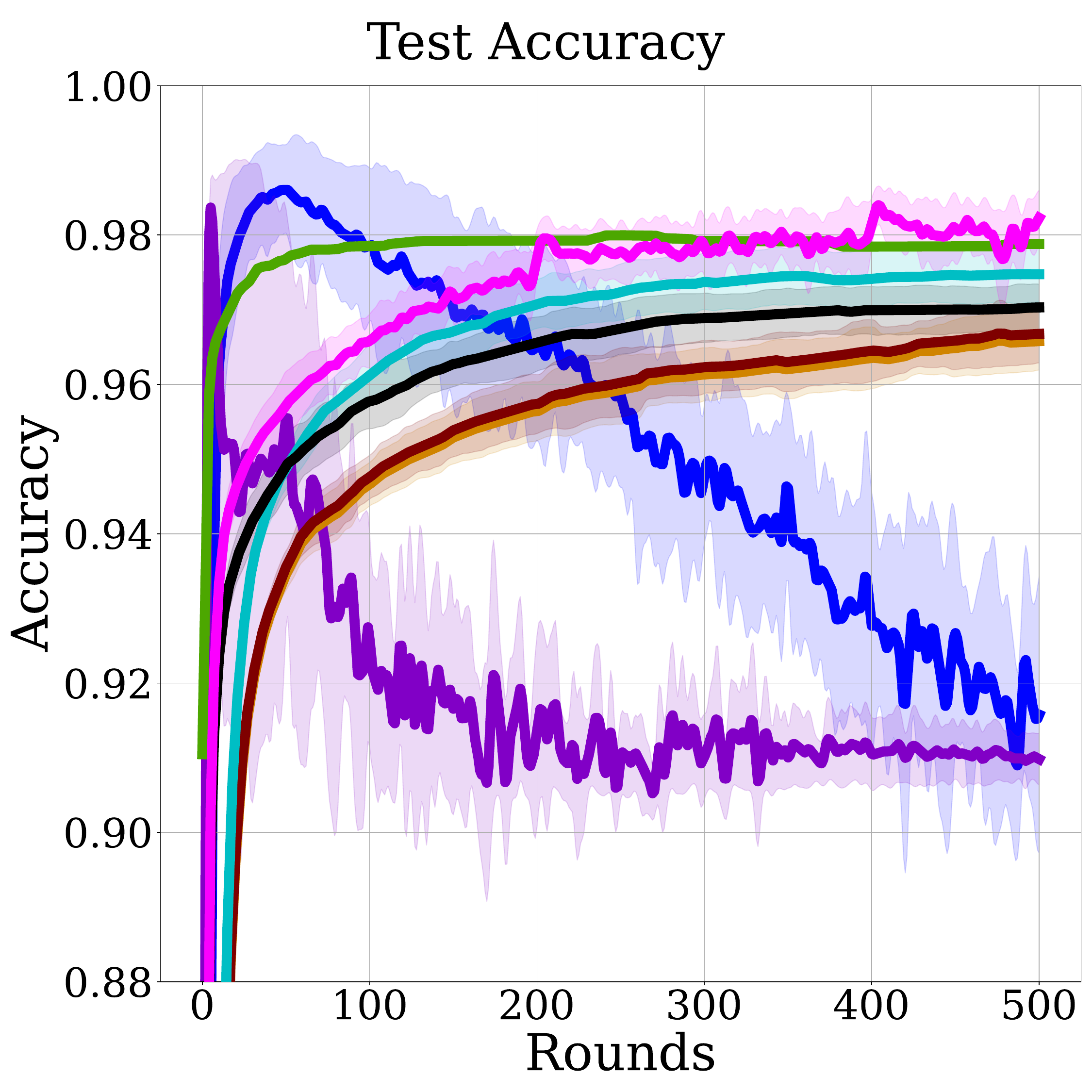}}
    \caption{$D^{\text{adv}} = 10\%$.}
    
    \end{subfigure}%
    \begin{subfigure}[t]{0.25\textwidth}\centering{\includegraphics[width=1\linewidth,trim=0 0 0 0,clip]{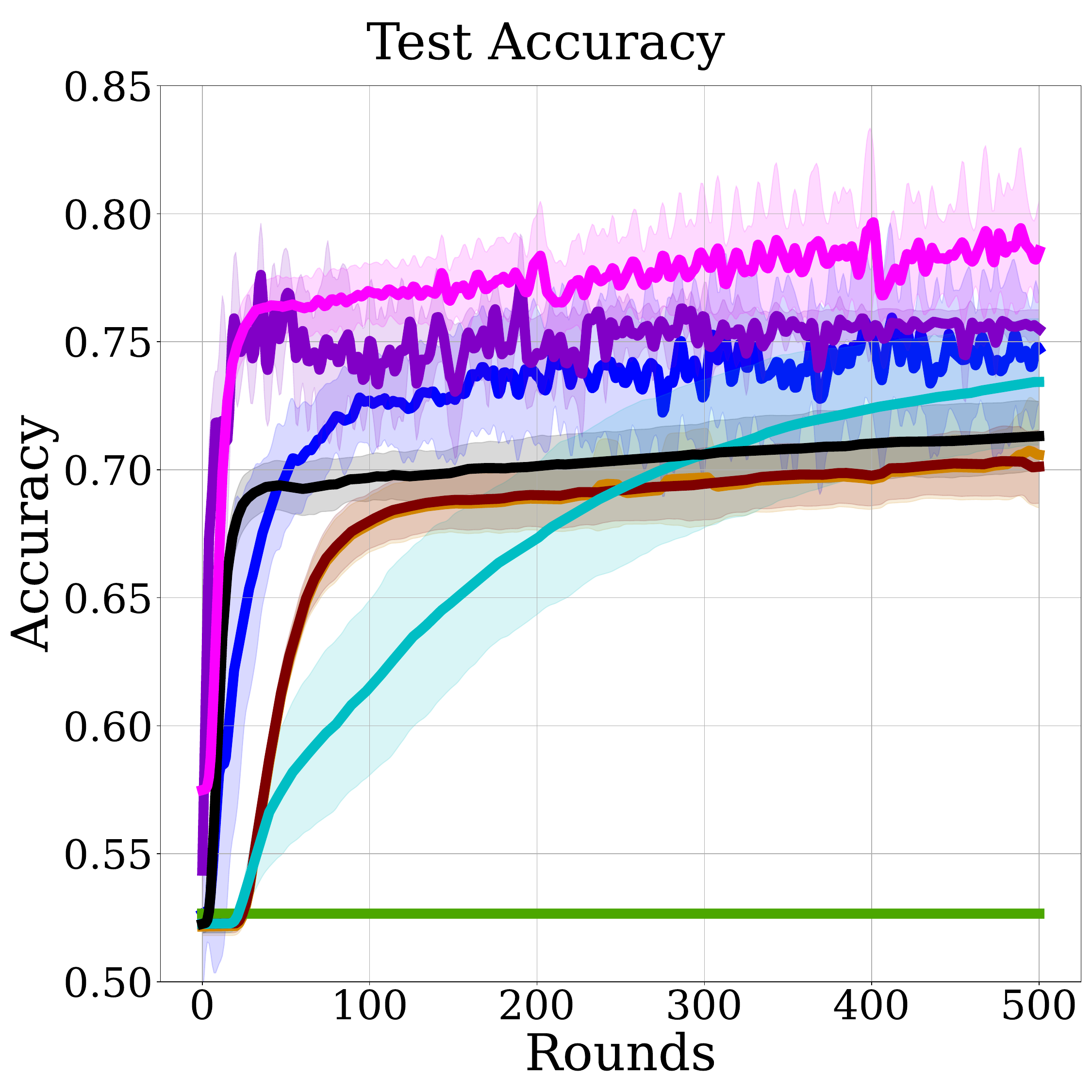}}
    \caption{$D^{\text{adv}} = 25\%$.}
    
    \end{subfigure}%
    \\
    \begin{subfigure}
{1\textwidth}\centering{\includegraphics[width=1\linewidth,trim=40 10 10 10,clip]{figures/moon_legend_8cols_new.pdf}}
    
    \end{subfigure}%
    \caption{Clean and robust test accuracy of methods trained on the MNIST-$1$vs$7$ dataset under \texttt{alfa-tilt} poisoning attack \citep{xiao2015support}. "Clean" refers to the dataset with clean labels, while the adversarial datasets contain $\{5, 10, 25\}$ $(\%)$ poisoned labels.
    For all SVM-related models, the setting $C=5$, $\gamma=0.005$ is used.
    \textsc{Floral} achieves outperforming robust accuracy level compared to baseline methods on all adversarial settings, clearly demonstrating the potential of \textsc{Floral} as a defense against other types of label poisoning attacks.
    }
\label{fig:mnist1vs7-exp-results-plots-appendix-alfa-tilt}
\end{figure*}

\subsection{Experiments with the $\texttt{alfa}$ attack}
\label{app:alfa-attack-experiments}

The \texttt{alfa} attack is generated under the assumption that the attacker can maliciously alter the training labels to maximize the empirical loss of
the original classifier on the tainted dataset. From this, the attacker's objective is formulated as maximizing the difference between the empirical risk for classifiers under tainted and untainted label sets.

We experimented on the Moon dataset and considered label poisoning levels ($\%$) of $\{5, 10, 25\}$.
The illustrations of the Moon dataset with clean and \texttt{alfa}-attacked adversarial labels are given Figure~\ref{fig:moon-datasets-clean-and-adv-alfa}.
We report the results under \texttt{alfa} attack in Figure~\ref{fig:moon-exp-results-plots-appendix-alfa} and Table~\ref{tab:test-accuracy-comp-moon-alfa}. As shown, \textsc{Floral} is especially effective against vanilla SVM in maintaining higher robust accuracy in adversarial settings. NN-based methods again fail drastically as the dataset becomes more adversarial under \texttt{alfa} attack. Although LS-SVM shows premise in moderate adversarial scenarios, it fails to be effective considering clean and most adversarial performance. 
Furthermore, \textsc{Floral} demonstrates superior performance compared to all baselines in the most adversarial setting (with $25\%$ poisoned labels).

 \begin{figure*}[ht]
    \centering  
    \begin{subfigure}{0.25\textwidth}\centering{\includegraphics[width=1\linewidth,trim=0 20 0 55,clip]{figures/moon_clean.png}}
    \caption{Clean.}
    
    \end{subfigure}%
    \begin{subfigure}{0.25\textwidth}\centering{\includegraphics[width=1\linewidth,trim=0 20 0 55,clip]{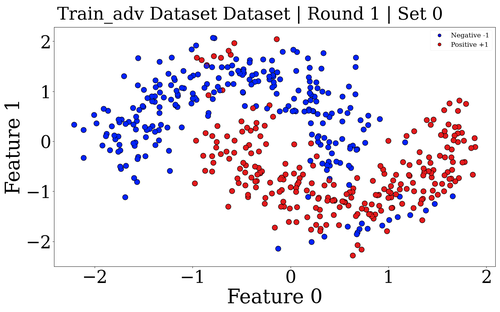}}
    \caption{$D^{\text{adv}} = 5\%$.}
    
    \end{subfigure}%
    \begin{subfigure}{0.25\textwidth}\centering{\includegraphics[width=1\linewidth,trim=0 20 0 55,clip]{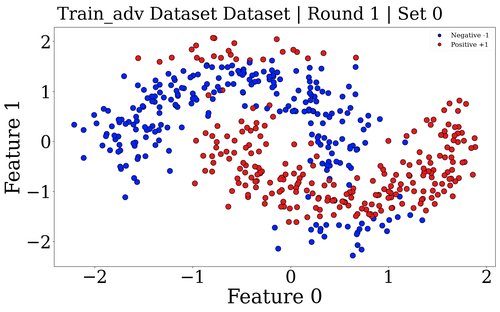}}
    \caption{$D^{\text{adv}} = 10\%$.}
    
    \end{subfigure}%
    \begin{subfigure}{0.25\textwidth}\centering{\includegraphics[width=1\linewidth,trim=0 20 0 55,clip]{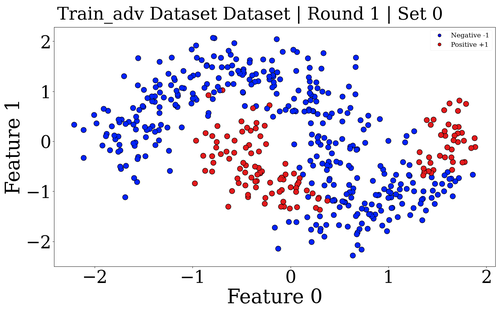}}
    \caption{$D^{\text{adv}} = 25\%$.}
    
    \end{subfigure}%
    \caption{Illustrations of the Moon training sets from an example replication, using clean and \texttt{alfa}-attacked adversarial labels with poisoning levels: $5\%$, $10\%$, $25\%$. }
\label{fig:moon-datasets-clean-and-adv-alfa}
\end{figure*}

\begin{table}[h]
\centering
    \caption{Test accuracies of methods trained over the Moon dataset with adversarial labels generated by the \texttt{alfa} attack. Each entry shows the average of five replications with different train/test splits.
    Highlighted values indicate the best performance in the  \mbox{\colorbox{tablegreen}{"\textbf{Best}"}} (peak accuracy during training) and \mbox{\colorbox{tablegreen}{"Last"}} (final accuracy after training) columns.
    }
    \label{tab:test-accuracy-comp-moon-alfa} 
    \resizebox{1\columnwidth}{!}{%
    \begin{tabular}{ll|cccccccccccccccc} \specialrule{1.5pt}{1pt}{1pt}
         \multicolumn{2}{c}{\multirow{2}{*}{\makecell{\\ \\ \textbf{Setting}}}} & \multicolumn{16}{c}{\textbf{Method}} \\ \cmidrule(lr){3-18}
        & & \multicolumn{2}{c}{\textsc{Floral}} & \multicolumn{2}{c}{SVM} & \multicolumn{2}{c}{NN} & \multicolumn{2}{c}{NN-PGD} & \multicolumn{2}{c}{LN-SVM} & \multicolumn{2}{c}{Curie} & \multicolumn{2}{c}{LS-SVM} &  \multicolumn{2}{c}{K-LID} \\ 
        \cmidrule(lr){3-4} \cmidrule(lr){5-6} \cmidrule(lr){7-8} \cmidrule(lr){9-10} \cmidrule(lr){11-12} \cmidrule(lr){13-14} \cmidrule(lr){15-16} \cmidrule(lr){17-18}
        & &  Best & Last & Best & Last & Best & Last & Best & Last & Best & Last & Best & Last & Best & Last & Best & Last \\ \specialrule{1.5pt}{1pt}{1pt}
         \text{Clean} & $C=10, \gamma=1$    & \cellcolor{tablegreen} \textbf{0.968} &  0.966 & 0.968 & \cellcolor{tablegreen} 0.968 & 0.960 & 0.960 & 0.966 & 0.964 & 0.940 & 0.940 & 0.941 & 0.941 & 0.881 & 0.881 & 0.966 & 0.966  \\
         $D^{\text{adv}}=5\%$ & $C=10, \gamma=1$     & 0.963 & 0.950 & 0.954 & 0.946 & 0.875 & 0.875 & 0.963 & 0.958 & 0.942 & 0.942 & 0.934 & 0.933 & \cellcolor{tablegreen} \textbf{0.964} & \cellcolor{tablegreen} 0.964 & 0.942 & 0.942 \\
         $D^{\text{adv}}=10\%$ & $C=10, \gamma=1$   & 0.954 & 0.902 & 0.914 & 0.893 & 0.836 & 0.816 & 0.918 & 0.895 & 0.914 & 0.907 & 0.915 & 0.914 & \cellcolor{tablegreen} \textbf{0.955} & \cellcolor{tablegreen} 0.954 & 0.914 & 0.907 \\
         $D^{\text{adv}}=25\%$ & $C=10, \gamma=1$   & \cellcolor{tablegreen} \textbf{0.776} & \cellcolor{tablegreen} 0.763 & 0.750 & 0.750 & 0.693 & 0.658 & 0.693 & 0.645 & 0.729 & 0.729 & 0.741 & 0.741 & 0.740 & 0.740 & 0.729 & 0.729 
         \\ \specialrule{1.5pt}{1pt}{1pt}
    \end{tabular}
    }
\end{table}

 \begin{figure*}[ht]
    \centering

    \begin{subfigure}[t]
    {0.25\textwidth}\centering{\includegraphics[width=1\linewidth,trim=0 0 0 0,clip]{figures/moon_accuracy_Dclean_C10_gamma1.pdf}}
    \caption{Clean.}
    
    \end{subfigure}%
    \begin{subfigure}[t]{0.25\textwidth}\centering{\includegraphics[width=1\linewidth,trim=0 0 0 0,clip]{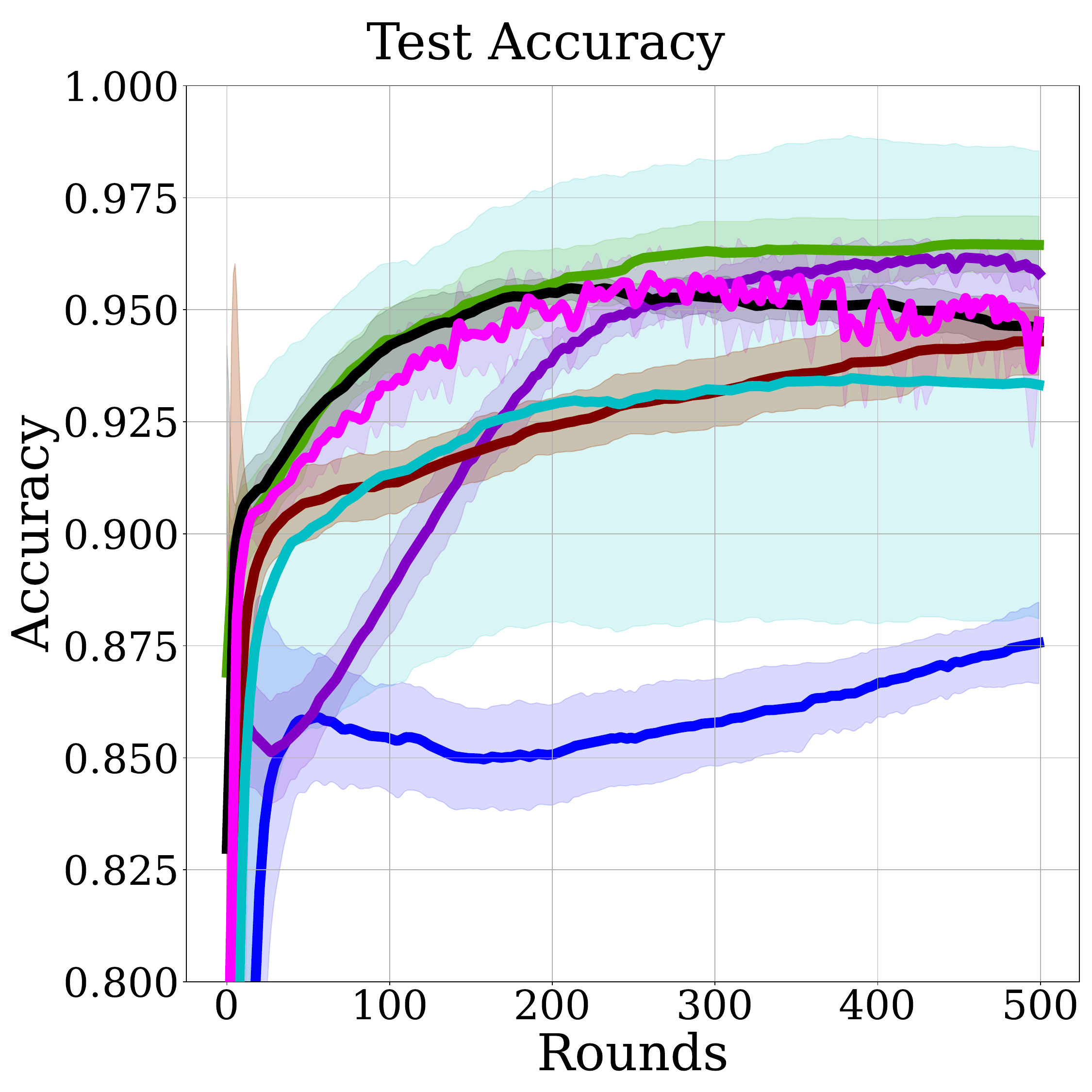}}
    \caption{$D^{\text{adv}} = 5\%$.}
    
    \end{subfigure}%
    \begin{subfigure}[t]{0.25\textwidth}\centering{\includegraphics[width=1\linewidth,trim=0 0 0 0,clip]{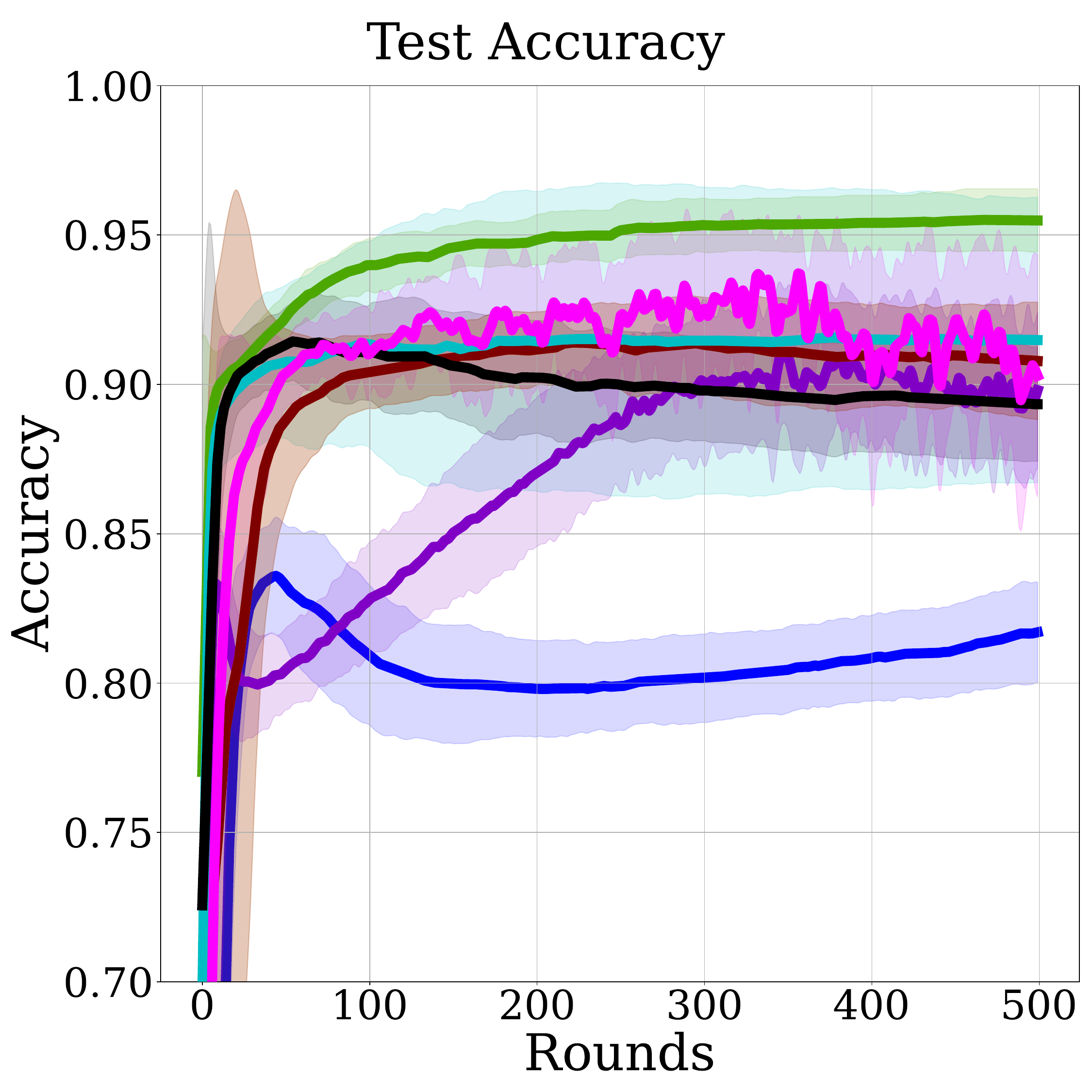}}
    \caption{$D^{\text{adv}} = 10\%$.}
    
    \end{subfigure}%
    \begin{subfigure}[t]{0.25\textwidth}\centering{\includegraphics[width=1\linewidth,trim=0 0 0 0,clip]{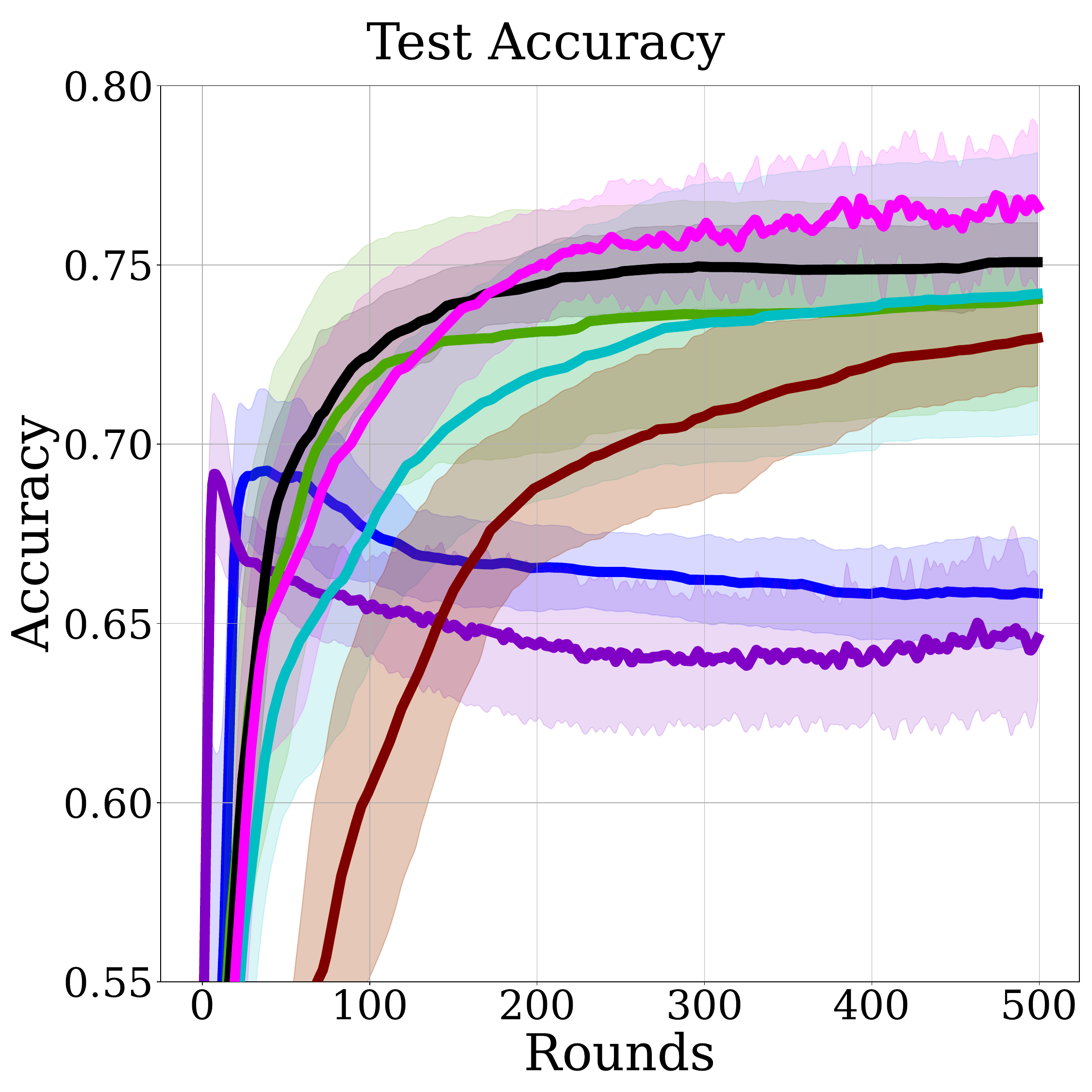}}
    \caption{$D^{\text{adv}} = 25\%$.}
    \end{subfigure}%
    \\
    \begin{subfigure}
{1\textwidth}\centering{\includegraphics[width=1\linewidth,trim=40 10 10 10,clip]{figures/moon_legend_8cols_new.pdf}}
    
    \end{subfigure}%
    \caption{Clean and robust test accuracy of methods trained on the Moon dataset under \texttt{alfa} poisoning attack. 
    "Clean" refers to the dataset with clean labels, while the adversarial datasets contain $\{5, 10, 25\} (\%)$ poisoned labels.
    As the level of label poisoning increases, models trained on adversarial datasets generally demonstrate a decline in accuracy.
    However, \textsc{Floral} demonstrates a gradually improving robust accuracy performance, particularly when the attack intensity increases to $25\%$.
    }
\label{fig:moon-exp-results-plots-appendix-alfa}
\end{figure*}

\subsection{Experiments with the $\texttt{LFA}$ attack}
\label{app:lfa-attack-experiments}
We additionally evaluate \textsc{Floral}'s effectiveness compared to baselines in the presence of $\texttt{LFA}$ attack \citep{label-sanitization} on the Moon dataset. 
As results are shown in Figure~\ref{fig:moon-exp-results-plots-appendix-lfa} and Table~\ref{tab:test-accuracy-comp-moon-lfa}, \textsc{Floral} demonstrates significant performance when the label poisoning attack level is high, i.e., $10\%$ or $25\%$. However, under those settings, LS-SVM \citep{label-sanitization} baseline shows faster convergence, which is expected as the LS-SVM \citep{label-sanitization} method is specifically crafted against the $\texttt{LFA}$ attack. 
Considering that LS-SVM fails in clean test performance, it is clear that \textsc{Floral} provides a generalizable defense.

\begin{figure*}[t]
\centering  
    \begin{subfigure}[t]
{0.25\textwidth}\centering{\includegraphics[width=1\linewidth,trim=0 0 0 0,clip]{figures/moon_accuracy_Dclean_C10_gamma1.pdf}}
    \caption{Clean.}
    \end{subfigure}%
    \begin{subfigure}[t]{0.25\textwidth}\centering{\includegraphics[width=1\linewidth,trim=0 0 0 0,clip]{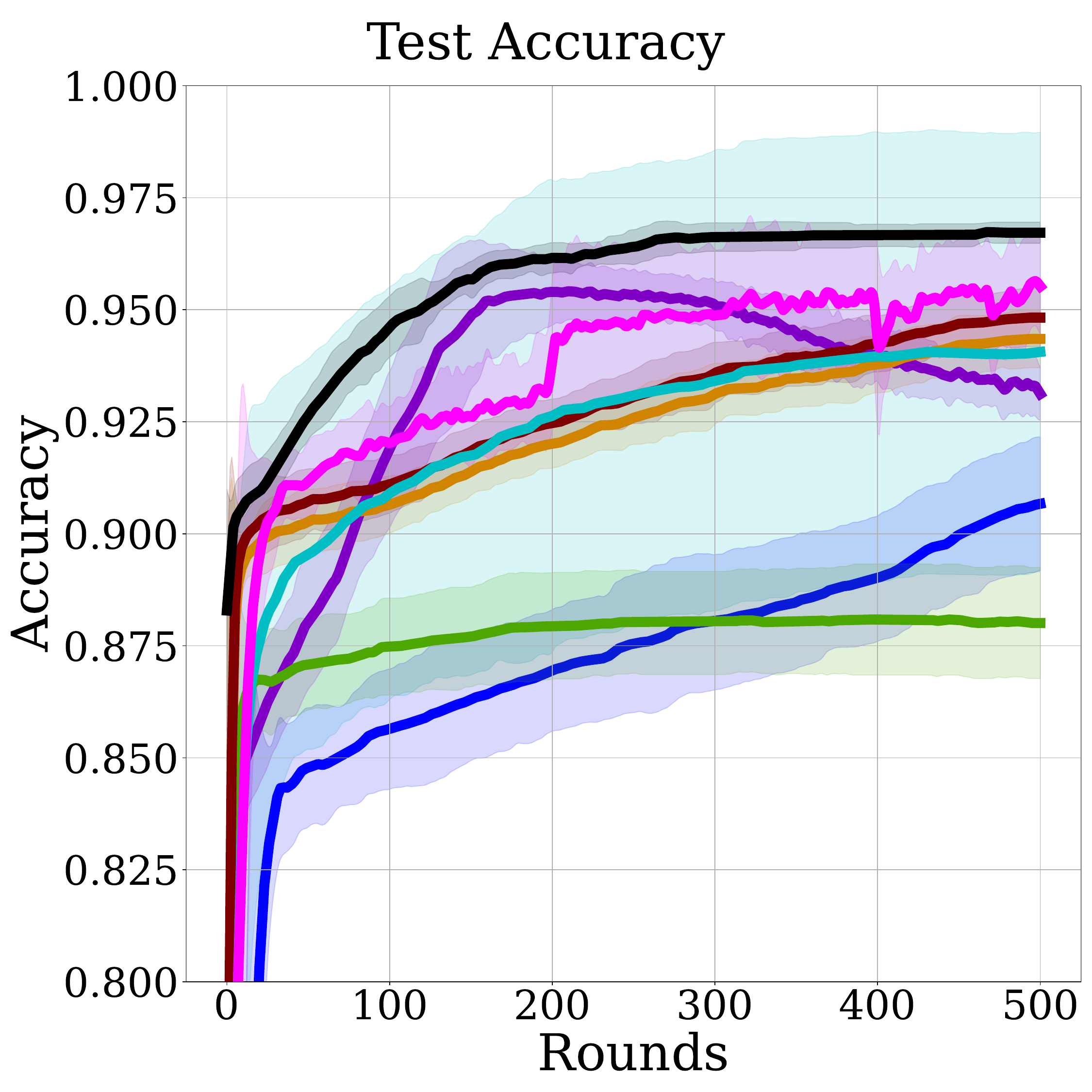}}
    \caption{$D^{\text{adv}} = 5\%$.}
    \end{subfigure}%
    \begin{subfigure}[t]{0.25\textwidth}\centering{\includegraphics[width=1\linewidth,trim=0 0 0 0,clip]{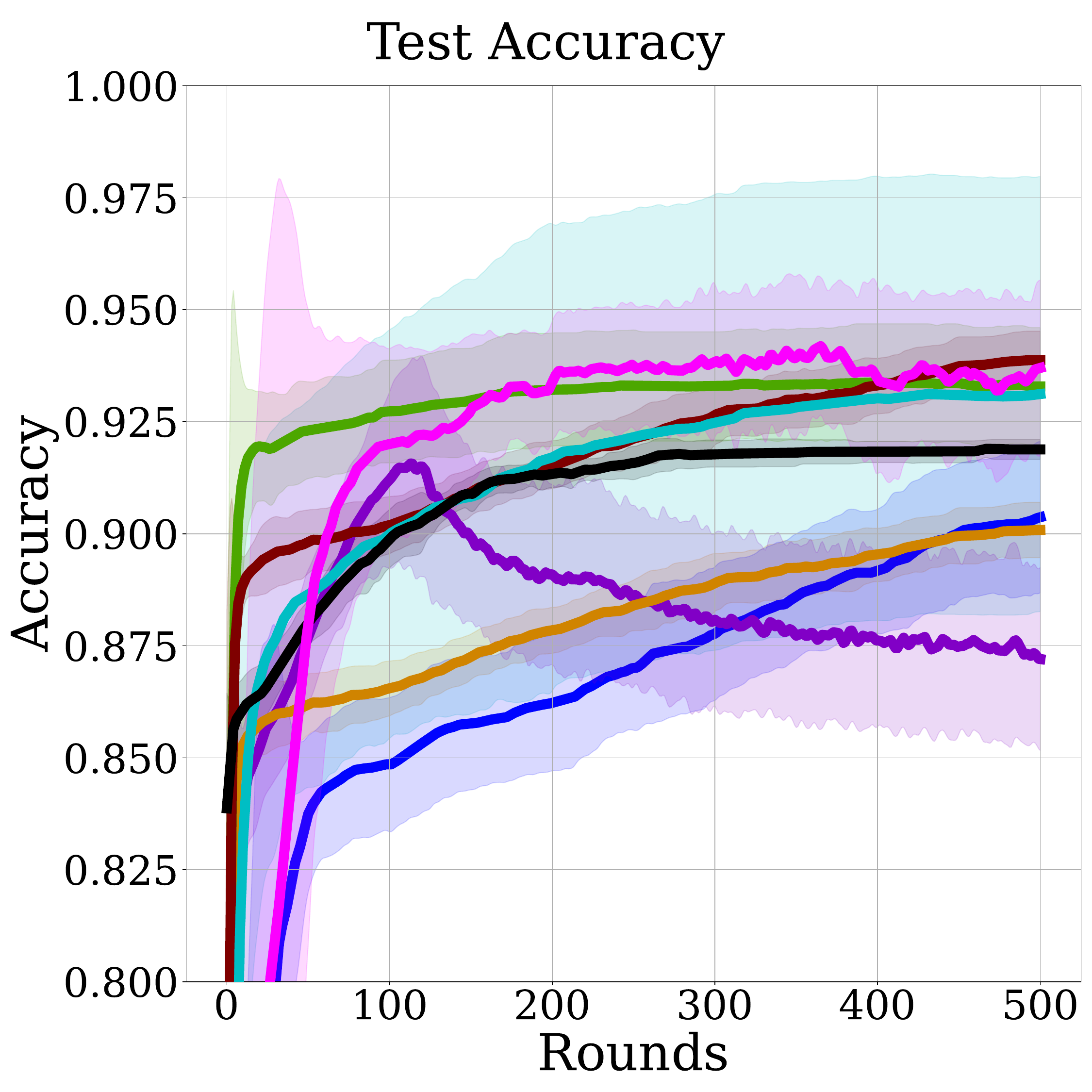}}
    \caption{$D^{\text{adv}} = 10\%$.}
    \end{subfigure}%
    \begin{subfigure}[t]{0.25\textwidth}\centering{\includegraphics[width=1\linewidth,trim=0 0 0 0,clip]{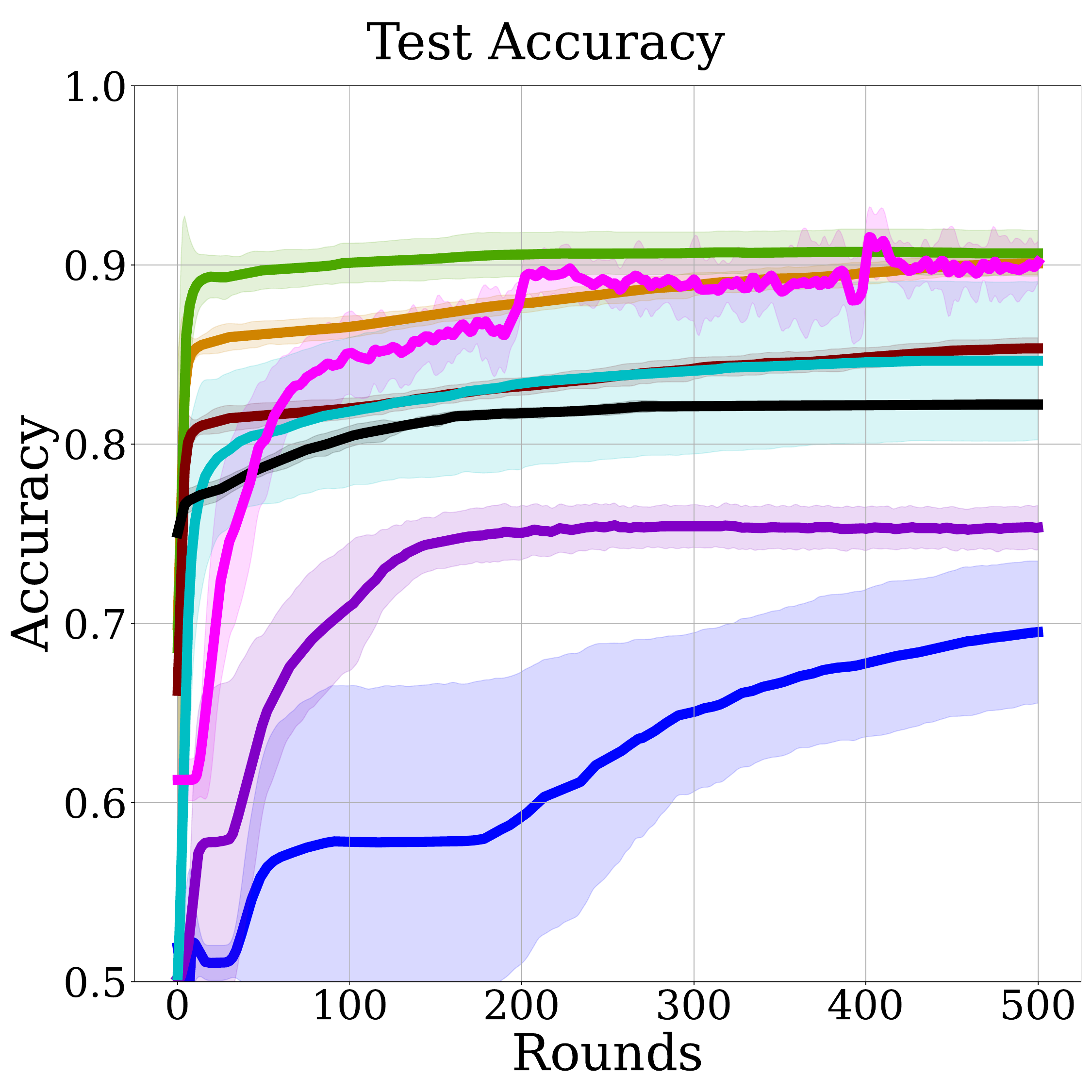}}
    \caption{$D^{\text{adv}} = 25\%$.}
    \end{subfigure}%
    \\
    \begin{subfigure}
{1\textwidth}\centering{\includegraphics[width=1\linewidth,trim=40 10 10 10,clip]{figures/moon_legend_8cols_new.pdf}}
    \end{subfigure}%
    \caption{Clean and robust test accuracy of methods trained on the Moon dataset under \texttt{LFA} poisoning attack \citep{label-sanitization}. "Clean" refers to the dataset with clean labels, while the adversarial datasets contain $\{5, 10, 25\}$ $(\%)$ poisoned labels.
    For all SVM-related models, the setting $C=1$, $\gamma=1.0$ is used.
    As the level of label poisoning increases, models trained on adversarial datasets generally demonstrate a decline in accuracy.
    However, \textsc{Floral} demonstrates a gradually improving robust accuracy performance, particularly when the attack level is $10\%$ or $25\%$. \looseness -1
    }
\label{fig:moon-exp-results-plots-appendix-lfa}
\end{figure*}

\begin{table}[t]
\centering
    \caption{Test accuracies of methods trained over the Moon dataset with adversarial labels generated by the \texttt{LFA} \citep{label-sanitization} attack. Each entry shows the average of five replications.
    Highlighted values indicate the best performance in the  \mbox{\colorbox{tablegreen}{"\textbf{Best}"}} (peak accuracy during training) and \mbox{\colorbox{tablegreen}{"Last"}} (final accuracy after training) columns.
    }
    \label{tab:test-accuracy-comp-moon-lfa} 
    \resizebox{1\columnwidth}{!}{%
    \begin{tabular}{ll|cccccccccccccccc} \specialrule{1.5pt}{1pt}{1pt}
         \multicolumn{2}{c}{\multirow{2}{*}{\makecell{\\ \\ \textbf{Setting}}}} & \multicolumn{16}{c}{\textbf{Method}} \\ \cmidrule(lr){3-18}
        & & \multicolumn{2}{c}{\textsc{Floral}} & \multicolumn{2}{c}{SVM} & \multicolumn{2}{c}{NN} & \multicolumn{2}{c}{NN-PGD} & \multicolumn{2}{c}{LN-SVM} & \multicolumn{2}{c}{Curie} & \multicolumn{2}{c}{LS-SVM} &  \multicolumn{2}{c}{K-LID} \\ 
        \cmidrule(lr){3-4} \cmidrule(lr){5-6} \cmidrule(lr){7-8} \cmidrule(lr){9-10} \cmidrule(lr){11-12} \cmidrule(lr){13-14} \cmidrule(lr){15-16} \cmidrule(lr){17-18}
        & &  Best & Last & Best & Last & Best & Last & Best & Last & Best & Last & Best & Last & Best & Last & Best & Last \\ \specialrule{1.5pt}{1pt}{1pt}
        \text{Clean} & $C=10, \gamma=1$ & \cellcolor{tablegreen} \textbf{0.968} &  0.966 & 0.968 & \cellcolor{tablegreen} 0.968 & 0.960 & 0.960 & 0.966 & 0.964 & 0.940 & 0.940 & 0.941 & 0.941 & 0.881 & 0.881 & 0.966 & 0.966 \\
         $D^{\text{adv}}=5\%$ & $C=10, \gamma=1$   & 0.957 & 0.954 & \cellcolor{tablegreen} \textbf{0.967} & \cellcolor{tablegreen} 0.967 & 0.906 & 0.906 & 0.955 & 0.930 &  0.948 &  0.948 &  0.940 &  0.940 & 0.880 & 0.880 &  0.943 &  0.943 \\
         $D^{\text{adv}}=10\%$ & $C=10, \gamma=1$   &  \cellcolor{tablegreen} \textbf{0.943} &  \cellcolor{tablegreen} 0.937 & 0.919 & 0.918 & 0.903 & 0.903 & 0.917 & 0.872 & 0.938 & 0.938 & 0.931 & 0.931 & 0.933 & 0.932 &  0.900 &  0.900 \\  
         $D^{\text{adv}}=25\%$ & $C=10, \gamma=1$ & \cellcolor{tablegreen} \textbf{0.922} & 0.903 & 0.822 & 0.822 & 0.695 & 0.695 & 0.757 & 0.753 &  0.853 &  0.853 & 0.892 & 0.846 &  0.907 & \cellcolor{tablegreen} 0.906 & 0.900 & 0.900
          \\ \specialrule{1.5pt}{1pt}{1pt}
    \end{tabular}
    }
\end{table}

\clearpage

\section{Integration with Neural Networks}
\label{app:nn-integration}

As demonstrated with the IMDB experiments in Section~\ref{sec:experiment-results}, \textsc{Floral} can be integrated with complex model architectures, e.g. a transformer-based language model such as RoBERTa, serving as a robust classifier head that enhances model robustness on classification tasks. 

Similarly, \textsc{Floral} can be directly incorporated into neural networks by utilizing the last-layer embeddings (the $x_i$’s in Algorithm~\ref{alg:robust-svm-game}) as inputs. These extracted representations can then be trained using \textsc{Floral}, resulting in more robust feature representations. Notably, our theoretical analysis remains valid under this integration, ensuring the approach's soundness.

To demonstrate this further, we performed additional experiments on the Moon and MNIST-$1$vs$7$ \citep{deng2012mnist} datasets, by integrating \textsc{Floral} with a neural network.

From Figure~\ref{fig:integration-nn-floral}, we can conclude that \textsc{Floral} integration achieves a higher robust accuracy level compared to plain neural network training.

\begin{figure}[ht]
    \centering  
{\color{red}
    \begin{subfigure}[t]
    {0.25\textwidth}\centering{\includegraphics[width=1\linewidth,trim=0 0 0 0,clip]{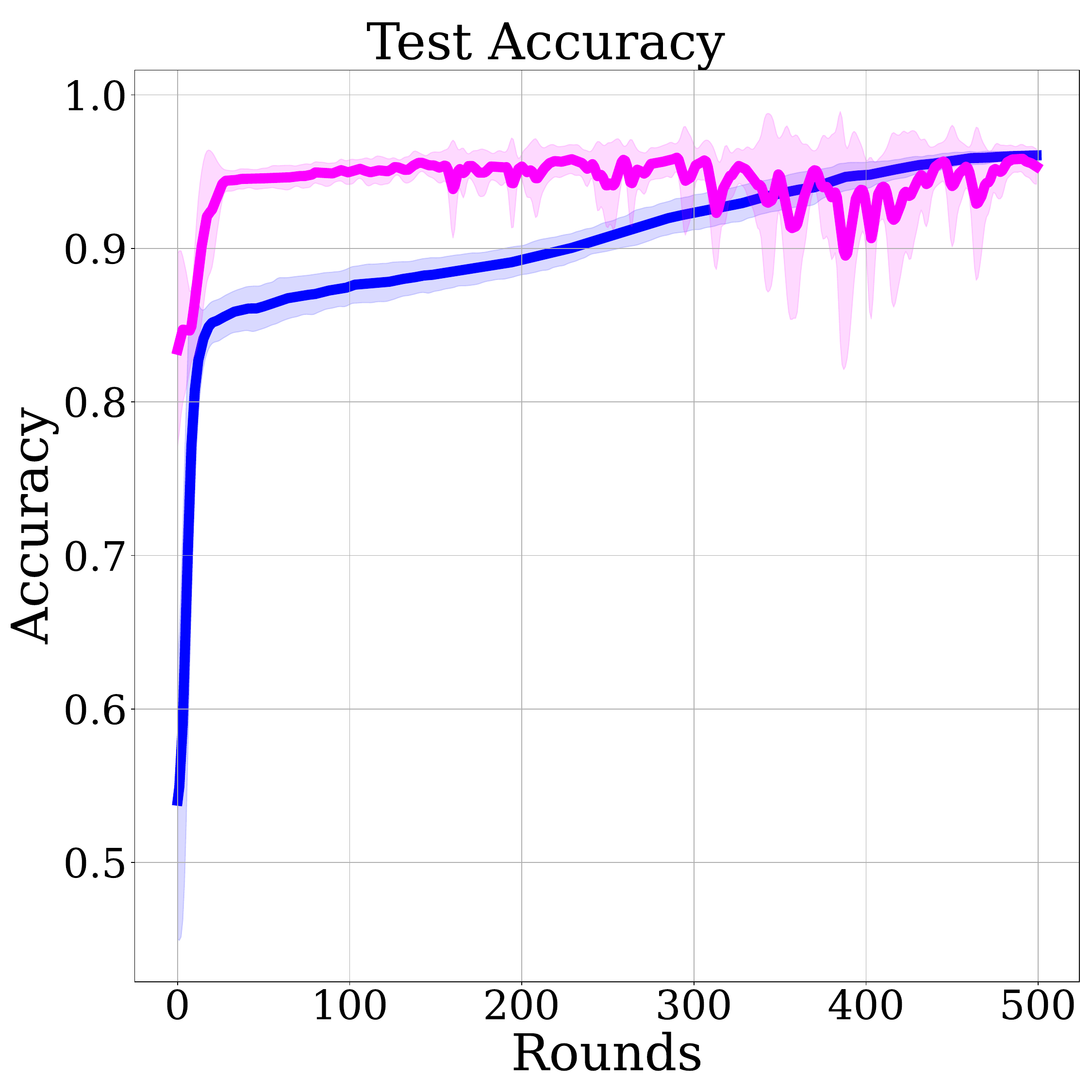}}
    \caption{Clean.}
    
    \end{subfigure}%
    \begin{subfigure}[t]{0.25\textwidth}\centering{\includegraphics[width=1\linewidth,trim=0 0 0 0,clip]{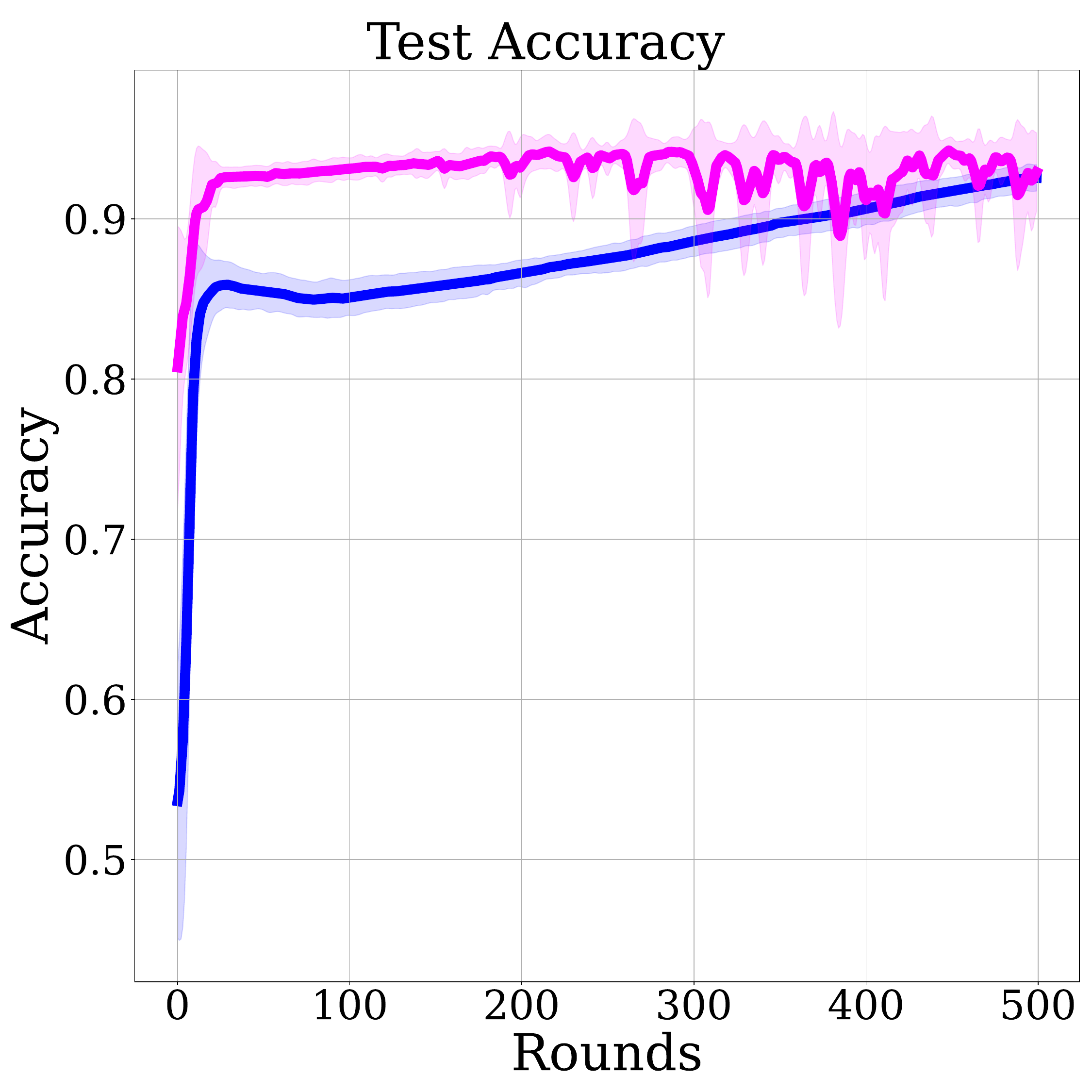}}
    \caption{$D^{\text{adv}} = 5\%$.}
    
    \end{subfigure}%
    \begin{subfigure}[t]{0.25\textwidth}\centering{\includegraphics[width=1\linewidth,trim=0 0 0 0,clip]{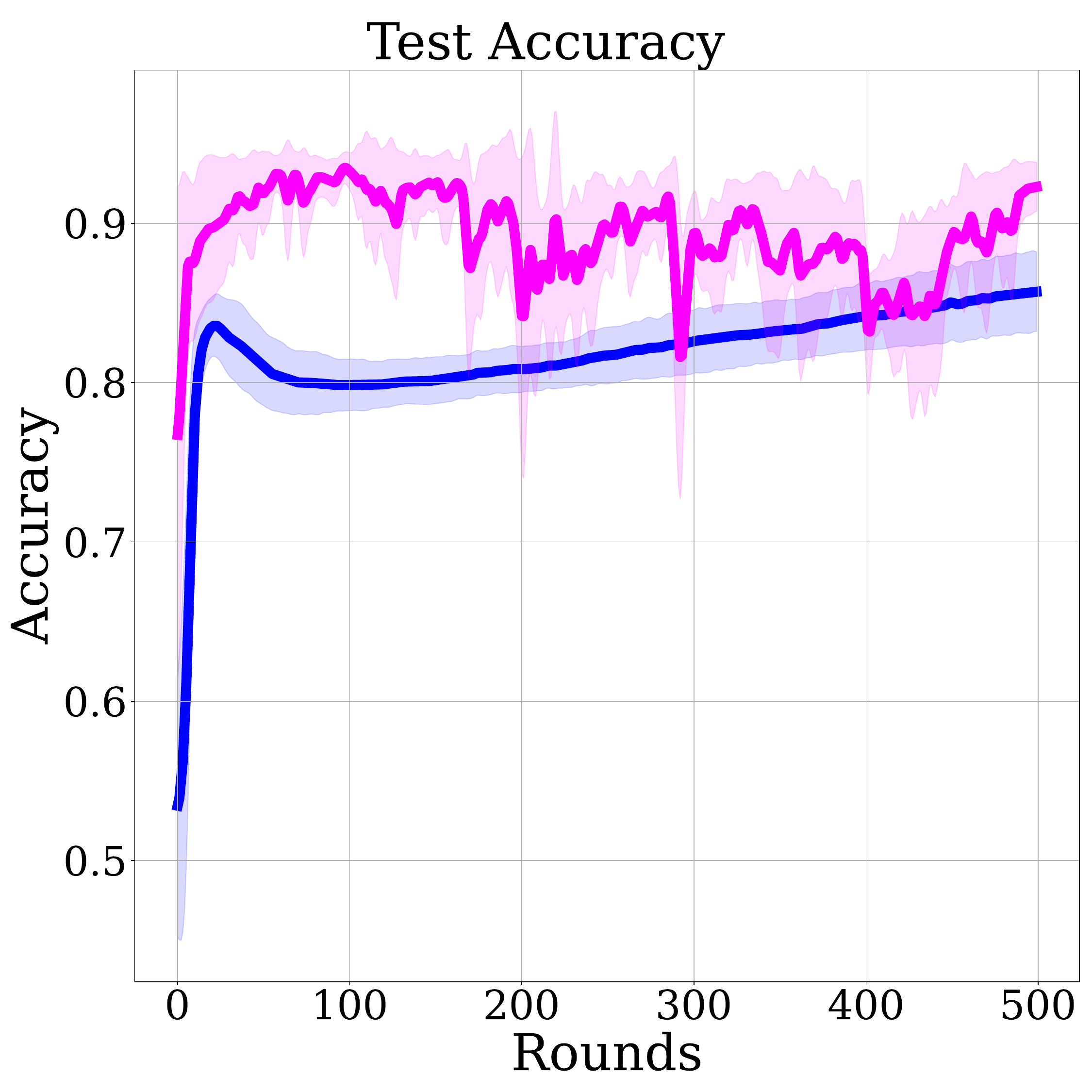}}
    \caption{$D^{\text{adv}} = 10\%$.}
    
    \end{subfigure}%
    \begin{subfigure}[t]{0.25\textwidth}\centering{\includegraphics[width=1\linewidth,trim=0 0 0 0,clip]{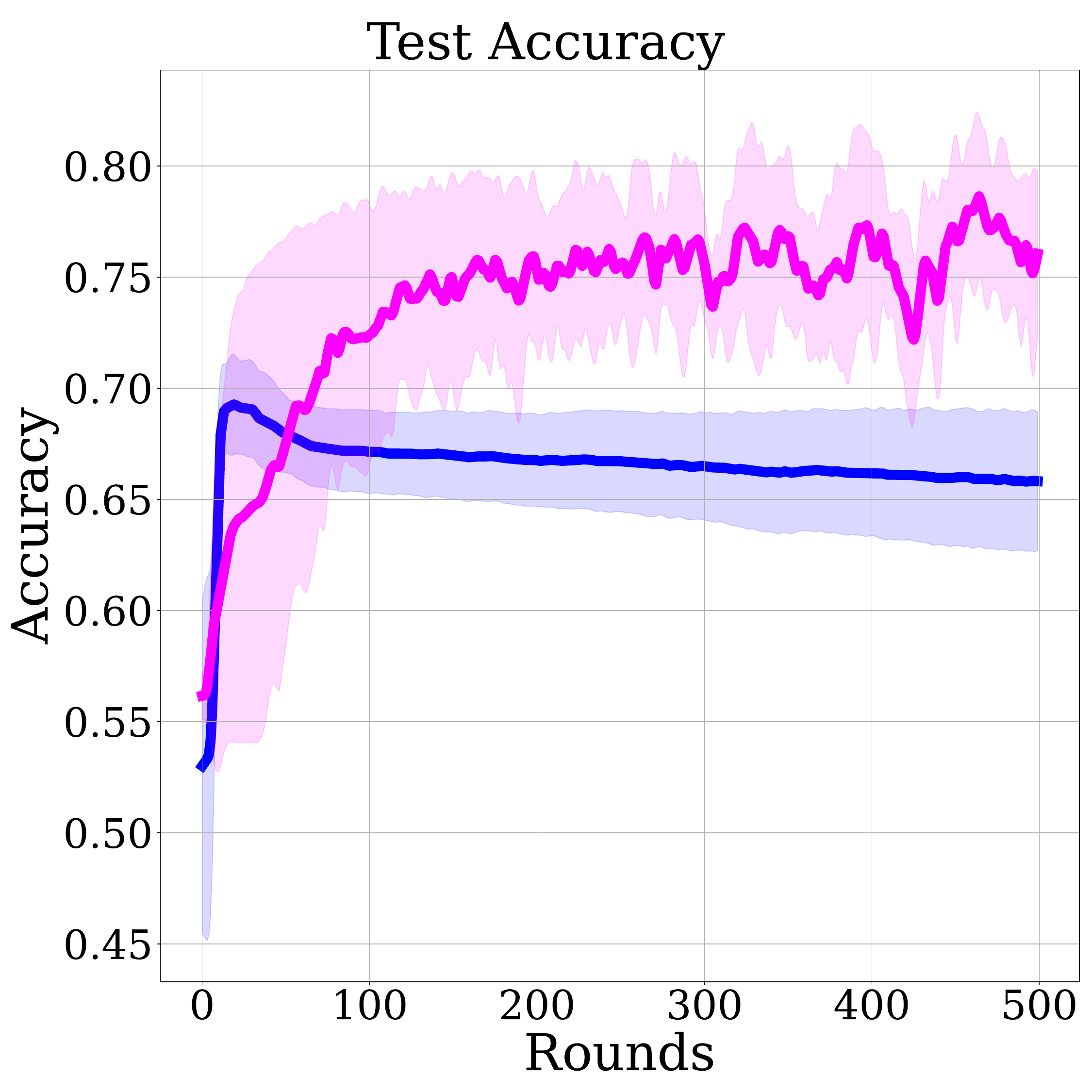}}
    \caption{$D^{\text{adv}} = 25\%$.}
    \end{subfigure}%
    \\
        \begin{subfigure}[t]
    {0.25\textwidth}\centering{\includegraphics[width=1\linewidth,trim=0 0 0 0,clip]{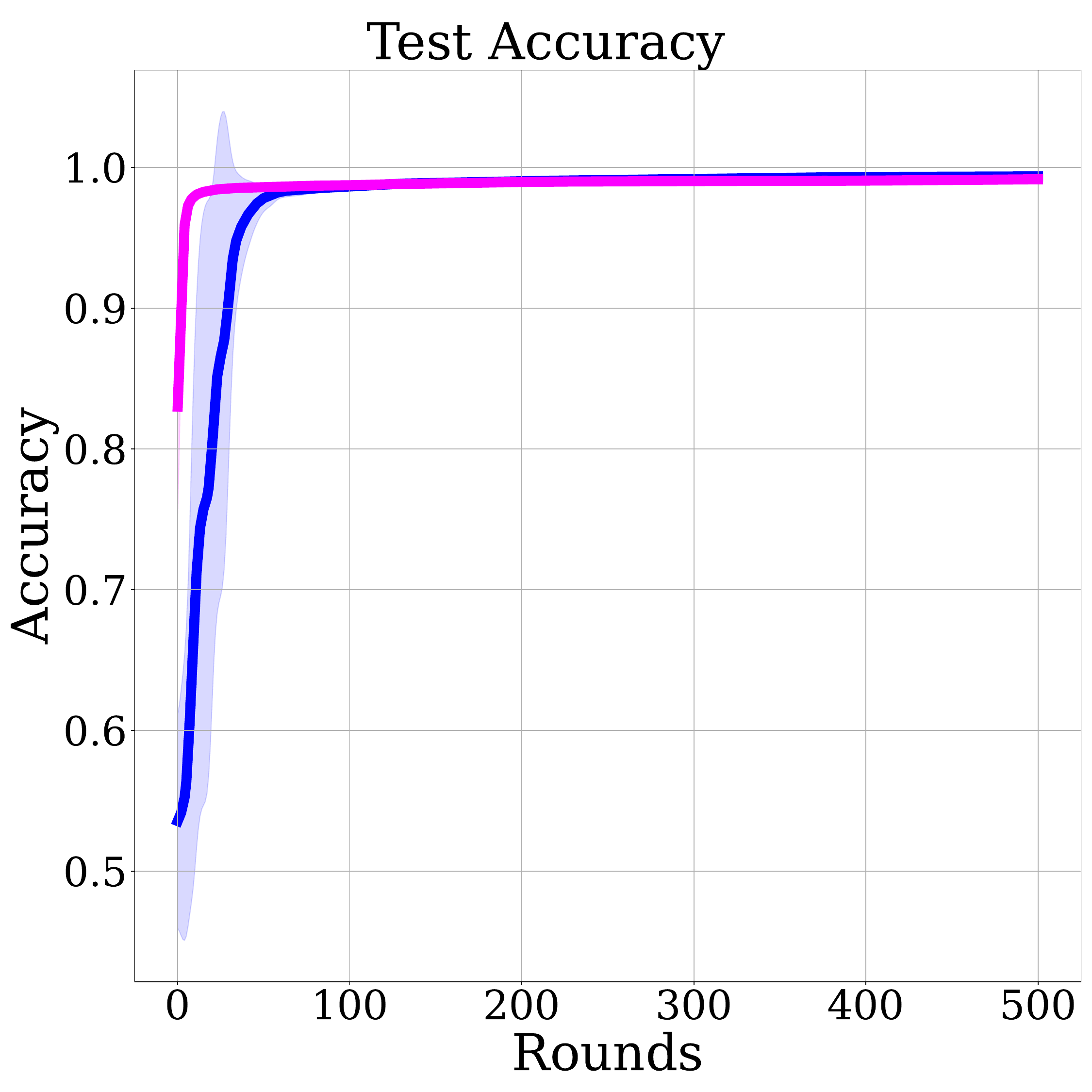}}
    \caption{Clean.}
    
    \end{subfigure}%
    \begin{subfigure}[t]{0.25\textwidth}\centering{\includegraphics[width=1\linewidth,trim=0 0 0 0,clip]{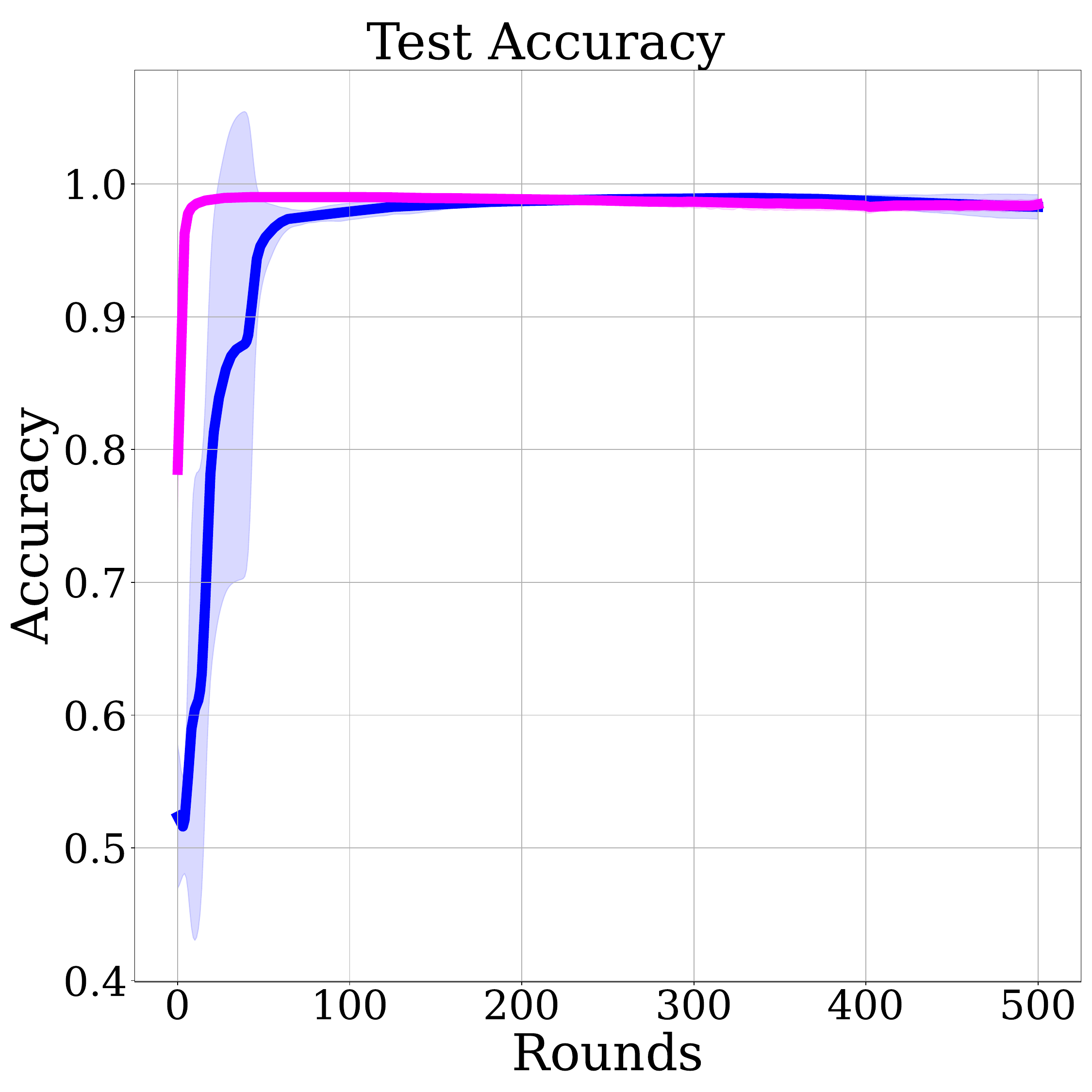}}
    \caption{$D^{\text{adv}} = 5\%$.}
    
    \end{subfigure}%
    \begin{subfigure}[t]{0.25\textwidth}\centering{\includegraphics[width=1\linewidth,trim=0 0 0 0,clip]{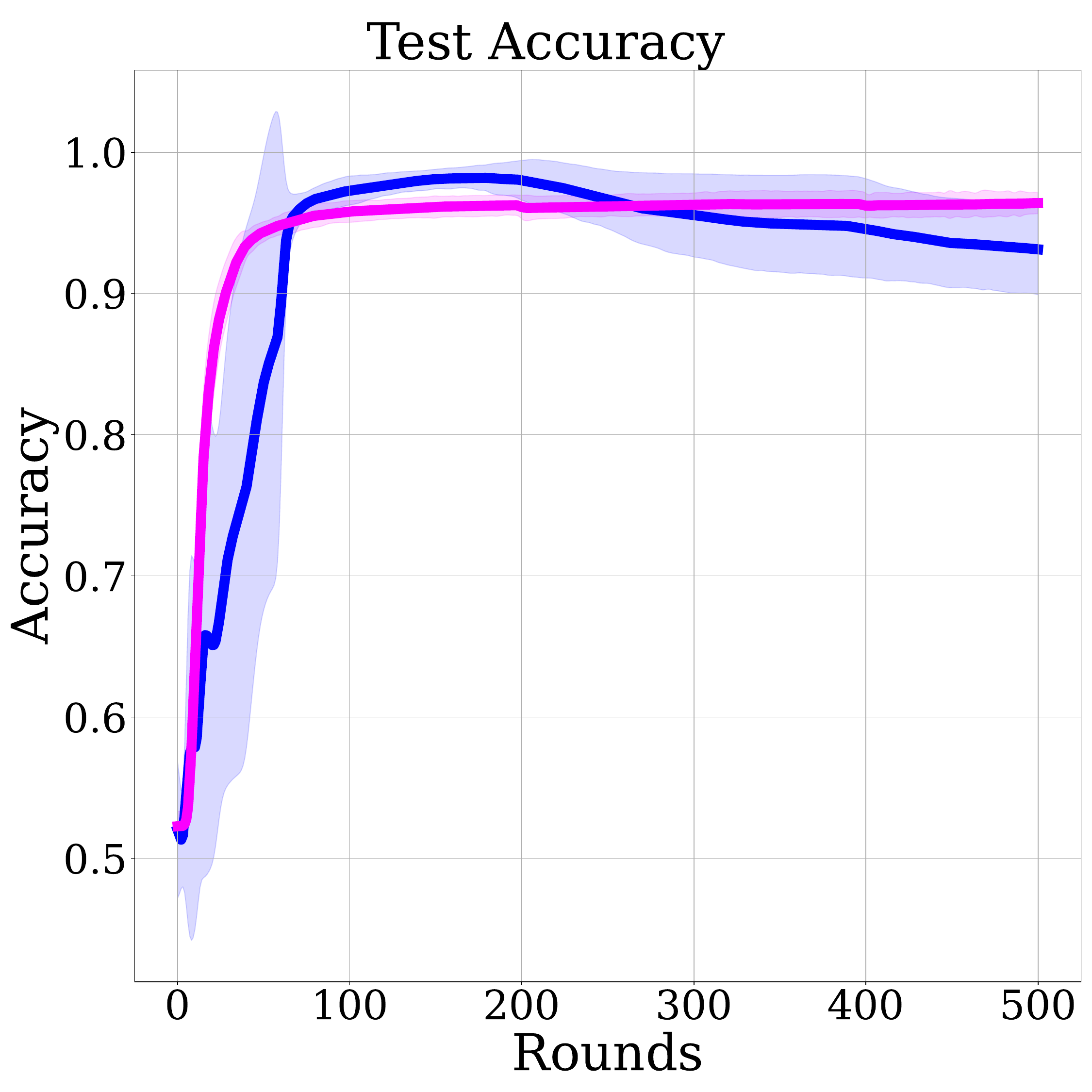}}
    \caption{$D^{\text{adv}} = 10\%$.}
    
    \end{subfigure}%
    \begin{subfigure}[t]{0.25\textwidth}\centering{\includegraphics[width=1\linewidth,trim=0 0 0 0,clip]{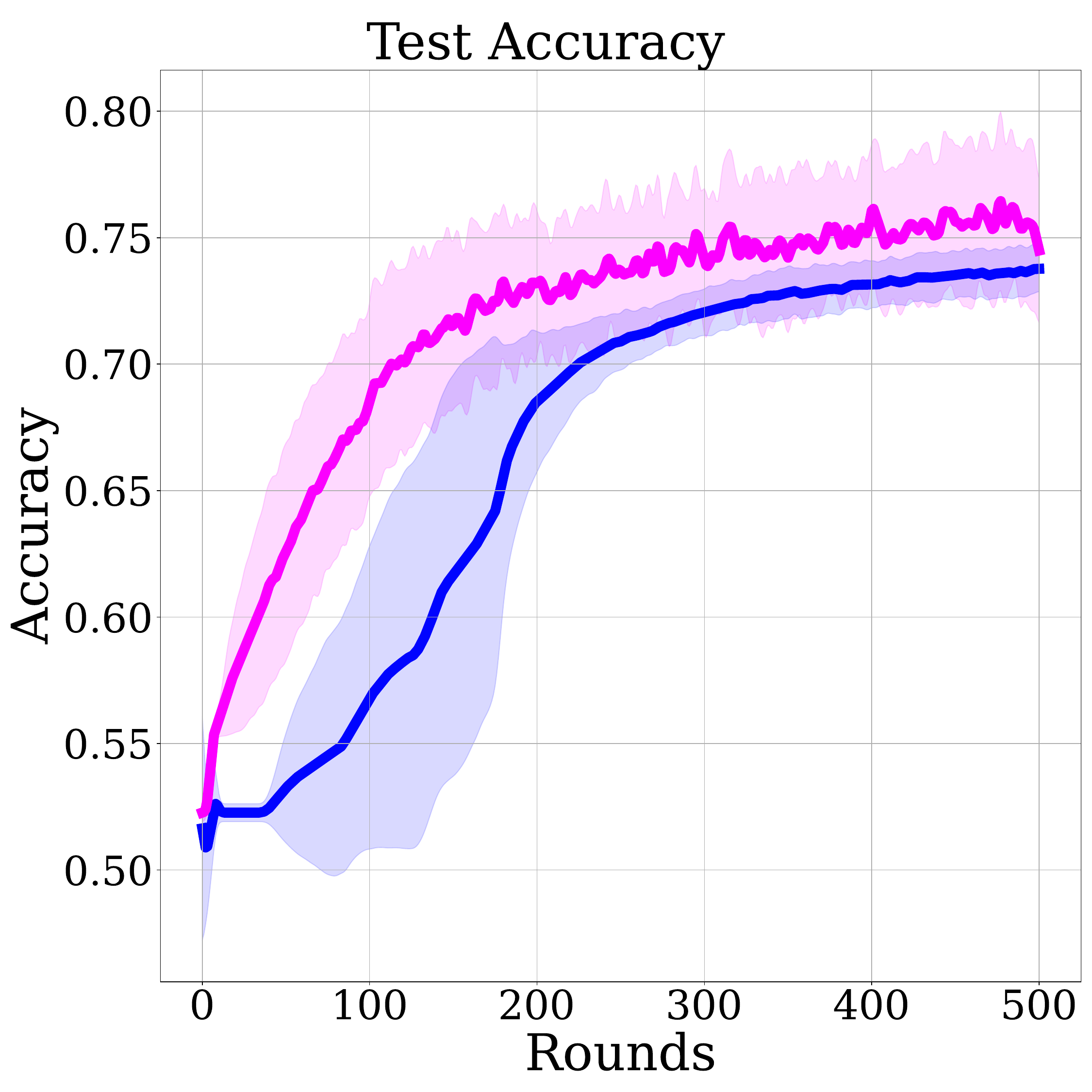}}
    \caption{$D^{\text{adv}} = 25\%$.}
    \end{subfigure}%
    \\
    \begin{subfigure}
    {0.15\textwidth}\centering{\includegraphics[width=1\linewidth,trim=40 10 920 10,clip]{figures/moon_legend_8cols_new.pdf}}
    \end{subfigure}%
    \begin{subfigure}
    {0.15\textwidth}\centering{\includegraphics[width=1\linewidth,trim=320 10 640 10,clip]{figures/moon_legend_8cols_new.pdf}}
    \end{subfigure}%
     \caption{Clean and robust test accuracy performance of neural network vs \textsc{Floral}-integrated neural network trained on the Moon (\textbf{the first row}) and MNIST-$1$vs$7$ (\textbf{the second row}) datasets. The results demonstrate that \textsc{Floral} integration helps to achieve a higher robust accuracy level. \looseness -1}
\label{fig:integration-nn-floral}
}
\end{figure}

 \begin{figure*}[ht]
    \centering  
    \begin{subfigure}
    {0.25\textwidth}\centering{\includegraphics[width=1\linewidth,trim=0 20 0 55,clip]{figures/synthetic-moon-svm-robust-db-round500-test_clean_C10_gamma05_seed1.png}}
    \caption{\textsc{Floral} (Clean).}
    
    \end{subfigure}%
    \begin{subfigure}
    {0.25\textwidth}\centering{\includegraphics[width=1\linewidth,trim=0 20 0 55,clip]{figures/synthetic-moon-svm-robust-db-round500-test_adv_5_C10_gamma05_seed1.png}}
    \caption{$D^{\text{adv}}= 5\%$.}
    
    \end{subfigure}%
    \begin{subfigure}
    {0.25\textwidth}\centering{\includegraphics[width=1\linewidth,trim=0 20 0 55,clip]{figures/synthetic-moon-svm-robust-db-round500-test_adv_10_C10_gamma05_seed1_2.png}}
    \caption{$D^{\text{adv}}= 10\%$.}
    
    \end{subfigure}%
    \begin{subfigure}{0.25\textwidth}\centering{\includegraphics[width=1\linewidth,trim=0 20 0 55,clip]{figures/synthetic-moon-svm-robust-db-round500-test_adv_25_C10_gamma05_seed1234.png}}
    \caption*{$D^{\text{adv}}= 25\%$.}
    \end{subfigure}%
    \\
    \begin{subfigure}{0.25\textwidth}\centering{\includegraphics[width=1\linewidth,trim=0 20 0 55,clip]{figures/synthetic-moon-svm-db-round500-test_clean_C10_gamma05_seed1.png}}
    \caption{SVM (Clean).}
    
    \end{subfigure}%
    \begin{subfigure}{0.25\textwidth}\centering{\includegraphics[width=1\linewidth,trim=0 20 0 55,clip]{figures/moon_svm_motivation_db_5_C10_gamma05_seed1.png}}
    \caption{$D^{\text{adv}}= 5\%$.}
    
    \end{subfigure}%
    \begin{subfigure}{0.25\textwidth}\centering{\includegraphics[width=1\linewidth,trim=0 20 0 55,clip]{figures/moon_svm_motivation_db_10_C10_gamma05_seed1.png}}
    \caption{$D^{\text{adv}}= 10\%$.}
    
    \end{subfigure}%
    \begin{subfigure}{0.25\textwidth}\centering{\includegraphics[width=1\linewidth,trim=0 20 0 55,clip]{figures/moon_svm_motivation_db_25_C10_gamma05_seed1.png}}
    \caption{$D^{\text{adv}}= 25\%$.}
    
    \end{subfigure}%
    \\
    \begin{subfigure}
    {0.25\textwidth}\centering{\includegraphics[width=1\linewidth,trim=0 20 0 55,clip]{figures/moon_nn_motivation_db_clean_seed1.png}}
    \caption{NN (Clean).}
    
    \end{subfigure}%
    \begin{subfigure}{0.25\textwidth}\centering{\includegraphics[width=1\linewidth,trim=0 20 0 55,clip]{figures/moon_nn_motivation_db_5_seed1.png}}
    \caption{$D^{\text{adv}}= 5\%$.}
    
    \end{subfigure}%
    \begin{subfigure}{0.25\textwidth}\centering{\includegraphics[width=1\linewidth,trim=0 20 0 55,clip]{figures/moon_nn_motivation_db_10_seed1.png}}
    \caption{$D^{\text{adv}}= 10\%$.}
    
    \end{subfigure}%
    \begin{subfigure}{0.25\textwidth}\centering{\includegraphics[width=1\linewidth,trim=0 20 0 55,clip]{figures/moon_nn_motivation_db_25_seed1.png}}
    \caption{$D^{\text{adv}}= 25\%$.}
    
    \end{subfigure}%
    \\
    \begin{subfigure}{0.25\textwidth}\centering{\includegraphics[width=1\linewidth,trim=0 20 0 55,clip]{figures/moon_nn_pgd_motivation_db_clean_seed1.png}}
    \caption{NN-PGD (Clean).}
    
    \end{subfigure}%
    \begin{subfigure}{0.25\textwidth}\centering{\includegraphics[width=1\linewidth,trim=0 20 0 55,clip]{figures/moon_nn_pgd_motivation_db_5_seed1.png}}
    \caption{$D^{\text{adv}}= 5\%$.}
    
    \end{subfigure}%
    \begin{subfigure}{0.25\textwidth}\centering{\includegraphics[width=1\linewidth,trim=0 20 0 55,clip]{figures/moon_nn_pgd_motivation_db_10_seed1.png}}
    \caption{$D^{\text{adv}}= 10\%$.}
    
    \end{subfigure}%
    \begin{subfigure}{0.25\textwidth}\centering{\includegraphics[width=1\linewidth,trim=0 20 0 55,clip]{figures/moon_nn_pgd_motivation_db_25_seed1.png}}
    \caption{$D^{\text{adv}}= 25\%$.}
    
    \end{subfigure}%
    \\
    \begin{subfigure}{0.25\textwidth}\centering{\includegraphics[width=1\linewidth,trim=0 20 0 55,clip]{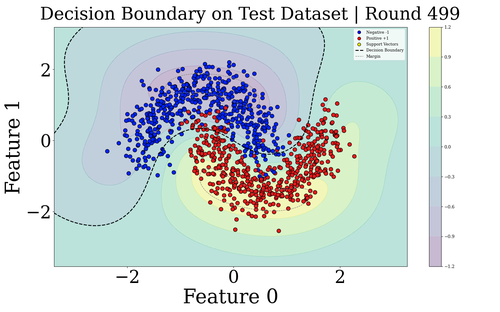}}
    \caption{LN-SVM (Clean).}
    
    \end{subfigure}%
    \begin{subfigure}{0.25\textwidth}\centering{\includegraphics[width=1\linewidth,trim=0 20 0 55,clip]{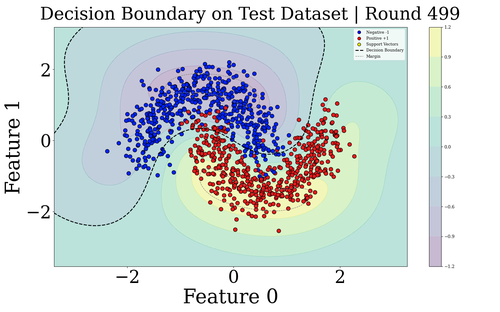}}
    \caption*{$D^{\text{adv}}= 5\%$.}
    
    \end{subfigure}%
    \begin{subfigure}{0.25\textwidth}\centering{\includegraphics[width=1\linewidth,trim=0 20 0 55,clip]{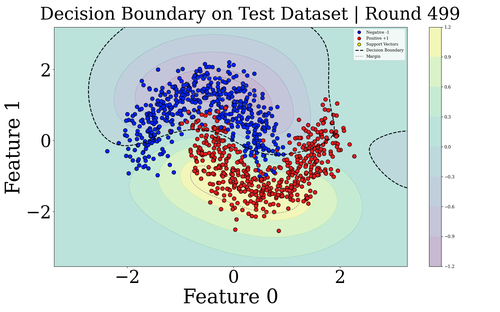}}
    \caption*{$D^{\text{adv}}= 10\%$.}
    
    \end{subfigure}%
    \begin{subfigure}{0.25\textwidth}\centering{\includegraphics[width=1\linewidth,trim=0 20 0 55,clip]{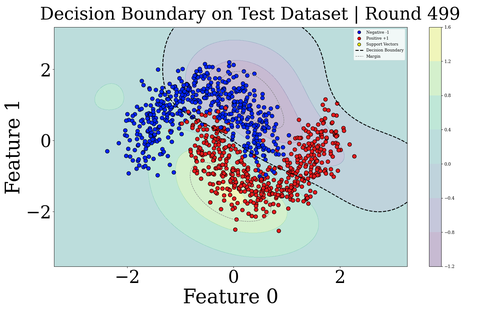}}
    \caption*{$D^{\text{adv}}= 25\%$.}
    
    \end{subfigure}%
    \\
    \begin{subfigure}{0.25\textwidth}\centering{\includegraphics[width=1\linewidth,trim=0 20 0 55,clip]{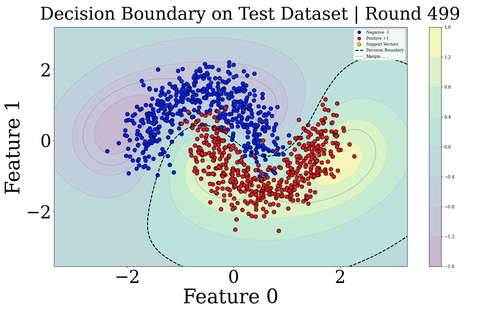}}
    \caption{Curie (Clean).}
    
    \end{subfigure}%
    \begin{subfigure}{0.25\textwidth}\centering{\includegraphics[width=1\linewidth,trim=0 20 0 55,clip]{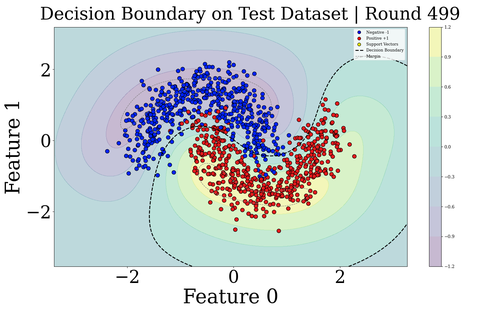}}
    \caption*{$D^{\text{adv}}= 5\%$.}
    
    \end{subfigure}%
    \begin{subfigure}{0.25\textwidth}\centering{\includegraphics[width=1\linewidth,trim=0 20 0 55,clip]{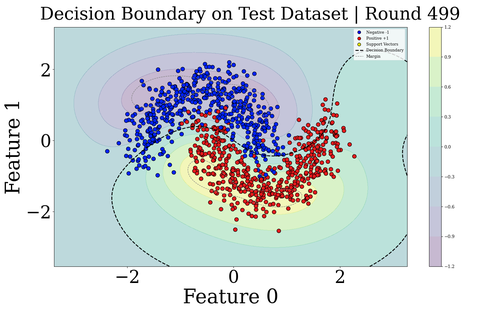}}
    \caption*{$D^{\text{adv}}= 10\%$.}
    
    \end{subfigure}%
    \begin{subfigure}{0.25\textwidth}\centering{\includegraphics[width=1\linewidth,trim=0 20 0 55,clip]{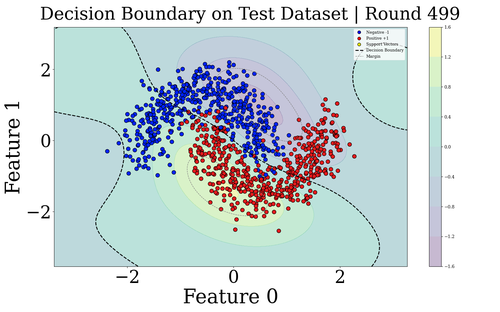}}
    \caption*{$D^{\text{adv}}= 25\%$.}
    
    \end{subfigure}%
    \\
    \begin{subfigure}{0.25\textwidth}\centering{\includegraphics[width=1\linewidth,trim=0 20 0 55,clip]{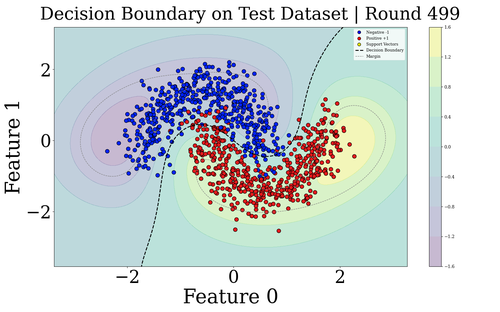}}
    \caption{LS-SVM (Clean).}
    
    \end{subfigure}%
    \begin{subfigure}{0.25\textwidth}\centering{\includegraphics[width=1\linewidth,trim=0 20 0 55,clip]{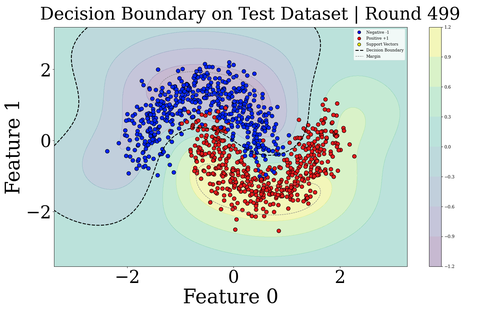}}
    \caption*{$D^{\text{adv}}= 5\%$.}
    
    \end{subfigure}%
    \begin{subfigure}{0.25\textwidth}\centering{\includegraphics[width=1\linewidth,trim=0 20 0 55,clip]{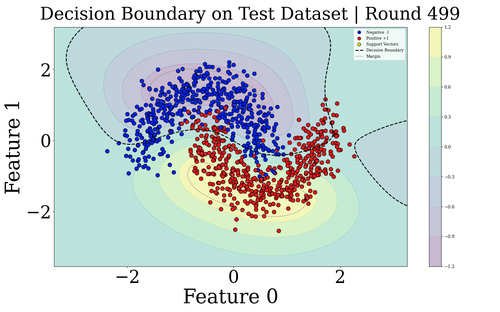}}
    \caption*{$D^{\text{adv}}= 10\%$.}
    
    \end{subfigure}%
    \begin{subfigure}{0.25\textwidth}\centering{\includegraphics[width=1\linewidth,trim=0 20 0 55,clip]{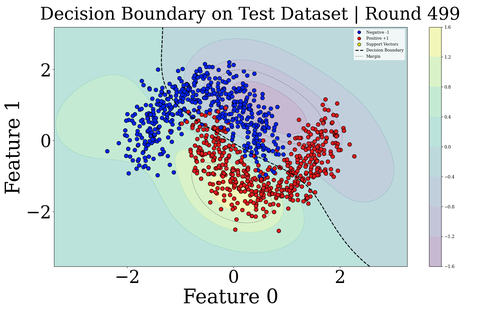}}
    \caption*{$D^{\text{adv}}= 25\%$.}
    
    \end{subfigure}%
    \\
    \begin{subfigure}{0.25\textwidth}\centering{\includegraphics[width=1\linewidth,trim=0 20 0 55,clip]{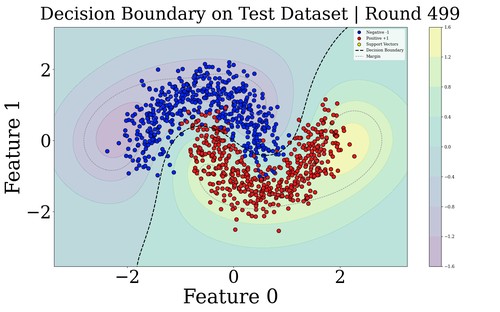}}
    \caption{K-LID (Clean).}
    
    \end{subfigure}%
    \begin{subfigure}{0.25\textwidth}\centering{\includegraphics[width=1\linewidth,trim=0 20 0 55,clip]{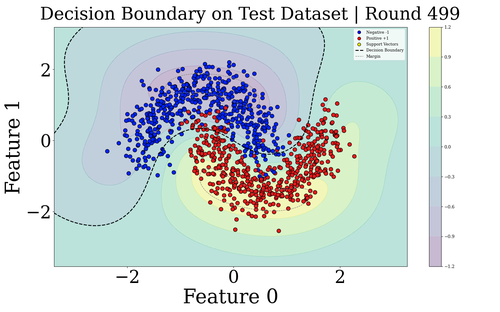}}
    \caption*{$D^{\text{adv}}= 5\%$.}
    
    \end{subfigure}%
    \begin{subfigure}{0.25\textwidth}\centering{\includegraphics[width=1\linewidth,trim=0 20 0 55,clip]{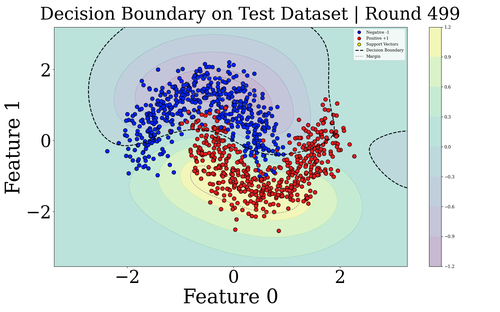}}
    \caption*{$D^{\text{adv}}= 10\%$.}
    
    \end{subfigure}%
    \begin{subfigure}{0.25\textwidth}\centering{\includegraphics[width=1\linewidth,trim=0 20 0 55,clip]{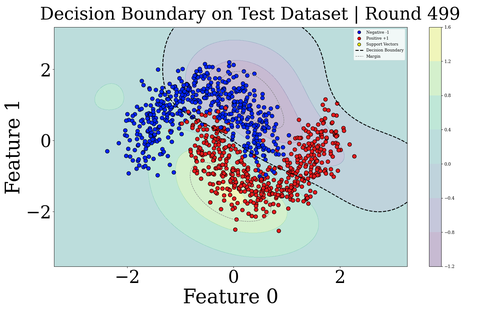}}
    \caption*{$D^{\text{adv}}= 25\%$.}
    
    \end{subfigure}%
    \caption{The decision boundaries on the Moon test dataset with various label poisoning levels. 
    SVM-related models use an RBF kernel with $C=10$ and $\gamma=0.5$. 
    \textsc{Floral} generates a relatively smooth decision boundary compared to baseline methods, particularly in $25\%$ adversarial setting, where baselines show drastic changes in their decision boundaries as a result of adversarial manipulations.
     }
\label{fig:moon-decision-boundaries-app-C10-gamma0.5}
\end{figure*}

\end{document}

%% file: math_commands.tex
\usepackage{amsmath,amsfonts,bm}

\def\eqref#1{equation~\ref{#1}}

\def\1{\bm{1}}

\DeclareMathAlphabet{\mathsfit}{\encodingdefault}{\sfdefault}{m}{sl}
\SetMathAlphabet{\mathsfit}{bold}{\encodingdefault}{\sfdefault}{bx}{n}

%% file: config/definitions.tex
\DeclareMathOperator*{\st}{subject\,\, to \quad}

\newtheorem{theorem}{Theorem}[section] 
\newtheorem*{theorem*}{Theorem}
\newtheorem*{proposition*}{Proposition}

\newtheorem{lemma}{Lemma}
\newtheorem{definition}{Definition}

\newtheorem{proposition}{Proposition}

\definecolor{darkgreen}{RGB}{104, 196, 84}

%% file: floral.bbl
\begin{thebibliography}{71}
\providecommand{\natexlab}[1]{#1}
\providecommand{\url}[1]{\texttt{#1}}
\expandafter\ifx\csname urlstyle\endcsname\relax
  \providecommand{\doi}[1]{doi: #1}\else
  \providecommand{\doi}{doi: \begingroup \urlstyle{rm}\Url}\fi

\bibitem[Bard(2013)]{bilevel-optimization}
Jonathan~F Bard.
\newblock \emph{Practical bilevel optimization: algorithms and applications}.
\newblock Springer Science \& Business Media, 2013.

\bibitem[Barreno et~al.(2010)Barreno, Nelson, Joseph, and Tygar]{security-ml}
Marco Barreno, Blaine Nelson, Anthony~D Joseph, and J~Doug Tygar.
\newblock The security of machine learning.
\newblock \emph{Machine Learning}, 81:\penalty0 121--148, 2010.

\bibitem[Biggio et~al.(2011)Biggio, Nelson, and Laskov]{svm-adversarial-noise}
Battista Biggio, Blaine Nelson, and Pavel Laskov.
\newblock Support vector machines under adversarial label noise.
\newblock \emph{Asian Conference on Machine Learning}, pp.\  97--112, 2011.

\bibitem[Biggio et~al.(2012)Biggio, Nelson, and Laskov]{biggio2012poisoning}
Battista Biggio, Blaine Nelson, and Pavel Laskov.
\newblock Poisoning attacks against support vector machines.
\newblock \emph{International Conference on Machine Learning}, pp.\  1467–1474, 2012.

\bibitem[Biggio et~al.(2013)Biggio, Fumera, and Roli]{biggio2013security}
Battista Biggio, Giorgio Fumera, and Fabio Roli.
\newblock Security evaluation of pattern classifiers under attack.
\newblock \emph{IEEE Transactions on Knowledge and Data Engineering}, 26\penalty0 (4):\penalty0 984--996, 2013.

\bibitem[Boser et~al.(1992)Boser, Guyon, and Vapnik]{boser1992training}
Bernhard~E Boser, Isabelle~M Guyon, and Vladimir~N Vapnik.
\newblock A training algorithm for optimal margin classifiers.
\newblock \emph{Proceedings of the 5th Annual Workshop on Computational Learning Theory}, pp.\  144--152, 1992.

\bibitem[Br{\"u}ckner \& Scheffer(2011)Br{\"u}ckner and Scheffer]{bruckner2011stackelberg}
Michael Br{\"u}ckner and Tobias Scheffer.
\newblock Stackelberg games for adversarial prediction problems.
\newblock \emph{Proceedings of the ACM SIGKDD International Conference on Knowledge Discovery and Data Mining}, pp.\  547--555, 2011.

\bibitem[Carlini \& Wagner(2017)Carlini and Wagner]{carlini2017towards}
Nicholas Carlini and David Wagner.
\newblock Towards evaluating the robustness of neural networks.
\newblock \emph{IEEE Symposium on Security and Privacy}, pp.\  39--57, 2017.

\bibitem[Chan et~al.(2018)Chan, He, Li, and Hsu]{chan2018data}
Patrick~PK Chan, Zhi-Min He, Hongjiang Li, and Chien-Chang Hsu.
\newblock Data sanitization against adversarial label contamination based on data complexity.
\newblock \emph{International Journal of Machine Learning and Cybernetics}, 9:\penalty0 1039--1052, 2018.

\bibitem[Chen et~al.(2022)Chen, Lou, Xu, Li, and Zhang]{chen2022clean}
Kangjie Chen, Xiaoxuan Lou, Guowen Xu, Jiwei Li, and Tianwei Zhang.
\newblock Clean-image backdoor: Attacking multi-label models with poisoned labels only.
\newblock \emph{International Conference on Learning Representations}, 2022.

\bibitem[Chen et~al.(2017)Chen, Zhang, Sharma, Yi, and Hsieh]{chen2017zoo}
Pin-Yu Chen, Huan Zhang, Yash Sharma, Jinfeng Yi, and Cho-Jui Hsieh.
\newblock Zoo: Zeroth order optimization based black-box attacks to deep neural networks without training substitute models.
\newblock \emph{Proceedings of the ACM Workshop on Artificial Intelligence and Security}, pp.\  15--26, 2017.

\bibitem[Cheng et~al.(2020)Cheng, Zhu, Li, Gong, Sun, and Liu]{cheng2020learning}
Hao Cheng, Zhaowei Zhu, Xingyu Li, Yifei Gong, Xing Sun, and Yang Liu.
\newblock Learning with instance-dependent label noise: A sample sieve approach.
\newblock \emph{ArXiv}, 2010.02347, 2020.

\bibitem[Chivukula et~al.(2020)Chivukula, Yang, Liu, Zhu, and Zhou]{chivukula2020game}
Aneesh~Sreevallabh Chivukula, Xinghao Yang, Wei Liu, Tianqing Zhu, and Wanlei Zhou.
\newblock Game theoretical adversarial deep learning with variational adversaries.
\newblock \emph{IEEE Transactions on Knowledge and Data Engineering}, 33\penalty0 (11):\penalty0 3568--3581, 2020.

\bibitem[Conitzer \& Sandholm(2006)Conitzer and Sandholm]{conitzer2006computing}
Vincent Conitzer and Tuomas Sandholm.
\newblock Computing the optimal strategy to commit to.
\newblock \emph{Proceedings of the ACM Conference on Electronic Commerce}, pp.\  82--90, 2006.

\bibitem[Cover \& Hart(1967)Cover and Hart]{cover1967nearest}
Thomas Cover and Peter Hart.
\newblock Nearest neighbor pattern classification.
\newblock \emph{IEEE Transactions on Information Theory}, 13\penalty0 (1):\penalty0 21--27, 1967.

\bibitem[Dalvi et~al.(2004)Dalvi, Domingos, Mausam, Sanghai, and Verma]{adversarial-classification}
Nilesh Dalvi, Pedro Domingos, Mausam, Sumit Sanghai, and Deepak Verma.
\newblock Adversarial classification.
\newblock \emph{Proceedings of the ACM SIGKDD International Conference on Knowledge Discovery and Data Mining}, pp.\  99–108, 2004.

\bibitem[Deng(2012)]{deng2012mnist}
Li~Deng.
\newblock The mnist database of handwritten digit images for machine learning research.
\newblock \emph{IEEE Signal Processing Magazine}, 29\penalty0 (6):\penalty0 141--142, 2012.

\bibitem[Deng et~al.(2020)Deng, Zheng, Zhang, Chen, Lou, and Kim]{adv-autonomous-driving}
Yao Deng, Xi~Zheng, Tianyi Zhang, Chen Chen, Guannan Lou, and Miryung Kim.
\newblock An analysis of adversarial attacks and defenses on autonomous driving models.
\newblock \emph{IEEE International Conference on Pervasive Computing and Communications}, pp.\  1--10, 2020.

\bibitem[Ester et~al.(1996)Ester, Kriegel, Sander, Xu, et~al.]{ester1996density}
Martin Ester, Hans-Peter Kriegel, J{\"o}rg Sander, Xiaowei Xu, et~al.
\newblock A density-based algorithm for discovering clusters in large spatial databases with noise.
\newblock \emph{Knowledge Discovery and Data Mining}, 96\penalty0 (34):\penalty0 226--231, 1996.

\bibitem[Finlayson et~al.(2019)Finlayson, Bowers, Ito, Zittrain, Beam, and Kohane]{adv-medical-ml}
Samuel~G. Finlayson, John~D. Bowers, Joichi Ito, Jonathan~L. Zittrain, Andrew~L. Beam, and Isaac~S. Kohane.
\newblock Adversarial attacks on medical machine learning.
\newblock \emph{Science}, 363\penalty0 (6433):\penalty0 1287--1289, 2019.

\bibitem[Fiore et~al.(2019)Fiore, {De Santis}, Perla, Zanetti, and Palmieri]{fraud-detection}
Ugo Fiore, Alfredo {De Santis}, Francesca Perla, Paolo Zanetti, and Francesco Palmieri.
\newblock Using generative adversarial networks for improving classification effectiveness in credit card fraud detection.
\newblock \emph{Information Sciences}, 479:\penalty0 448--455, 2019.

\bibitem[Fowl et~al.(2021)Fowl, Goldblum, Chiang, Geiping, Czaja, and Goldstein]{fowl2021adversarial}
Liam Fowl, Micah Goldblum, Ping-yeh Chiang, Jonas Geiping, Wojciech Czaja, and Tom Goldstein.
\newblock Adversarial examples make strong poisons.
\newblock \emph{Advances in Neural Information Processing Systems}, 34:\penalty0 30339--30351, 2021.

\bibitem[Fr{\'e}nay \& Verleysen(2013)Fr{\'e}nay and Verleysen]{frenay2013classification}
Beno{\^\i}t Fr{\'e}nay and Michel Verleysen.
\newblock Classification in the presence of label noise: a survey.
\newblock \emph{IEEE Transactions on Neural Networks and Learning Systems}, 25\penalty0 (5):\penalty0 845--869, 2013.

\bibitem[Geiping et~al.(2021)Geiping, Fowl, Somepalli, Goldblum, Moeller, and Goldstein]{geiping2021doesn}
Jonas Geiping, Liam Fowl, Gowthami Somepalli, Micah Goldblum, Michael Moeller, and Tom Goldstein.
\newblock What doesn't kill you makes you robust (er): How to adversarially train against data poisoning.
\newblock \emph{ArXiv}, 2102.13624, 2021.

\bibitem[Goodfellow et~al.(2015)Goodfellow, Shlens, and Szegedy]{goodfellow-2014}
Ian~J. Goodfellow, Jonathon Shlens, and Christian Szegedy.
\newblock Explaining and harnessing adversarial examples.
\newblock \emph{ArXiv}, 1412.6572, 2015.

\bibitem[Hallaji et~al.(2023)Hallaji, Razavi-Far, Saif, and Herrera-Viedma]{hallaji2023label}
Ehsan Hallaji, Roozbeh Razavi-Far, Mehrdad Saif, and Enrique Herrera-Viedma.
\newblock Label noise analysis meets adversarial training: A defense against label poisoning in federated learning.
\newblock \emph{Knowledge-Based Systems}, 266:\penalty0 110384, 2023.

\bibitem[Hampel(1974)]{hampel1974influence}
Frank~R Hampel.
\newblock The influence curve and its role in robust estimation.
\newblock \emph{Journal of the American Statistical Association}, 69\penalty0 (346):\penalty0 383--393, 1974.

\bibitem[Hearst et~al.(1998)Hearst, Dumais, Osuna, Platt, and Scholkopf]{hearst1998support}
Marti~A. Hearst, Susan~T Dumais, Edgar Osuna, John Platt, and Bernhard Scholkopf.
\newblock Support vector machines.
\newblock \emph{IEEE Intelligent Systems and Their Applications}, 13\penalty0 (4):\penalty0 18--28, 1998.

\bibitem[Hsieh et~al.(2019)Hsieh, Liu, and Cevher]{hsieh2019finding}
Ya-Ping Hsieh, Chen Liu, and Volkan Cevher.
\newblock Finding mixed nash equilibria of generative adversarial networks.
\newblock \emph{International Conference on Machine Learning}, pp.\  2810--2819, 2019.

\bibitem[Hsu \& Lin(2002)Hsu and Lin]{hsu2002comparison}
Chih-Wei Hsu and Chih-Jen Lin.
\newblock A comparison of methods for multiclass support vector machines.
\newblock \emph{IEEE Transactions on Neural Networks}, 13\penalty0 (2):\penalty0 415--425, 2002.

\bibitem[Huang et~al.(2015)Huang, Xu, Schuurmans, and Szepesv{\'a}ri]{huang2015learning}
Ruitong Huang, Bing Xu, Dale Schuurmans, and Csaba Szepesv{\'a}ri.
\newblock Learning with a strong adversary.
\newblock \emph{ArXiv}, 1511.03034, 2015.

\bibitem[Jha et~al.(2023)Jha, Hayase, and Oh]{label-poisoning}
Rishi~D. Jha, Jonathan Hayase, and Sewoong Oh.
\newblock Label poisoning is all you need.
\newblock \emph{Advances in Neural Information Processing Systems}, 36:\penalty0 71029--71052, 2023.

\bibitem[Koh \& Liang(2017)Koh and Liang]{understanding-influence}
Pang~Wei Koh and Percy Liang.
\newblock Understanding black-box predictions via influence functions.
\newblock \emph{International Conference on Machine Learning}, pp.\  1885--1894, 2017.

\bibitem[Kumar et~al.(2020)Kumar, Nystr{\"o}m, Lambert, Marshall, Goertzel, Comissoneru, Swann, and Xia]{kumar2020adversarial}
Ram Shankar~Siva Kumar, Magnus Nystr{\"o}m, John Lambert, Andrew Marshall, Mario Goertzel, Andi Comissoneru, Matt Swann, and Sharon Xia.
\newblock Adversarial machine learning-industry perspectives.
\newblock \emph{IEEE Security and Privacy Workshops}, pp.\  69--75, 2020.

\bibitem[Kurakin et~al.(2016)Kurakin, Goodfellow, and Bengio]{kurakin2018adversarial}
Alexey Kurakin, Ian~J Goodfellow, and Samy Bengio.
\newblock Adversarial examples in the physical world.
\newblock \emph{International Conference on Learning Representations}, 2016.

\bibitem[Laishram \& Phoha(2016)Laishram and Phoha]{curie}
Ricky Laishram and Vir~Virander Phoha.
\newblock Curie: A method for protecting svm classifier from poisoning attack.
\newblock \emph{ArXiv}, 1606.01584, 2016.

\bibitem[Liu et~al.(2019)Liu, Ott, Goyal, Du, Joshi, Chen, Levy, Lewis, Zettlemoyer, and Stoyanov]{roberta}
Yinhan Liu, Myle Ott, Naman Goyal, Jingfei Du, Mandar Joshi, Danqi Chen, Omer Levy, Mike Lewis, Luke Zettlemoyer, and Veselin Stoyanov.
\newblock Ro{BERT}a: A robustly optimized {BERT} pretraining approach.
\newblock \emph{ArXiv}, 1907.11692, 2019.

\bibitem[Ma et~al.(2018)Ma, Li, Wang, Erfani, Wijewickrema, Schoenebeck, Song, Houle, and Bailey]{ma2018characterizing}
Xingjun Ma, Bo~Li, Yisen Wang, Sarah~M Erfani, Sudanthi Wijewickrema, Grant Schoenebeck, Dawn Song, Michael~E Houle, and James Bailey.
\newblock Characterizing adversarial subspaces using local intrinsic dimensionality.
\newblock \emph{ArXiv}, 1801.02613, 2018.

\bibitem[Maas et~al.(2011)Maas, Daly, Pham, Huang, Ng, and Potts]{maas-EtAl:2011:ACL-HLT2011}
Andrew~L. Maas, Raymond~E. Daly, Peter~T. Pham, Dan Huang, Andrew~Y. Ng, and Christopher Potts.
\newblock Learning word vectors for sentiment analysis.
\newblock In \emph{Proceedings of the 49th Annual Meeting of the Association for Computational Linguistics: Human Language Technologies}, pp.\  142--150, 2011.

\bibitem[Madry et~al.(2017)Madry, Makelov, Schmidt, Tsipras, and Vladu]{madry2017towards}
Aleksander Madry, Aleksandar Makelov, Ludwig Schmidt, Dimitris Tsipras, and Adrian Vladu.
\newblock Towards deep learning models resistant to adversarial attacks.
\newblock \emph{ArXiv}, 1706.06083, 2017.

\bibitem[Moosavi-Dezfooli et~al.(2015)Moosavi-Dezfooli, Fawzi, and Frossard]{moosavi2016deepfool}
Seyed-Mohsen Moosavi-Dezfooli, Alhussein Fawzi, and Pascal Frossard.
\newblock Deepfool: A simple and accurate method to fool deep neural networks.
\newblock \emph{IEEE Conference on Computer Vision and Pattern Recognition}, pp.\  2574--2582, 2015.

\bibitem[Natarajan et~al.(2013)Natarajan, Dhillon, Ravikumar, and Tewari]{learning-w-noisy-labels}
Nagarajan Natarajan, Inderjit~S Dhillon, Pradeep~K Ravikumar, and Ambuj Tewari.
\newblock Learning with noisy labels.
\newblock \emph{Advances in Neural Information Processing Systems}, 26, 2013.

\bibitem[Pal \& Vidal(2020)Pal and Vidal]{pal2020game}
Ambar Pal and Ren{\'e} Vidal.
\newblock A game theoretic analysis of additive adversarial attacks and defenses.
\newblock \emph{Advances in Neural Information Processing Systems}, 33:\penalty0 1345--1355, 2020.

\bibitem[Papernot et~al.(2016)Papernot, McDaniel, and Goodfellow]{papernot2016transferability}
Nicolas Papernot, Patrick McDaniel, and Ian Goodfellow.
\newblock Transferability in machine learning: from phenomena to black-box attacks using adversarial samples.
\newblock \emph{ArXiv}, 1605.07277, 2016.

\bibitem[Paudice et~al.(2018)Paudice, Mu{\~n}oz-Gonz{\'a}lez, and Lupu]{label-sanitization}
Andrea Paudice, Luis Mu{\~n}oz-Gonz{\'a}lez, and Emil~C Lupu.
\newblock Label sanitization against label flipping poisoning attacks.
\newblock \emph{ArXiv}, 1803.00992, 2018.

\bibitem[Pedregosa et~al.(2011)Pedregosa, Varoquaux, Gramfort, Michel, Thirion, Grisel, Blondel, Prettenhofer, Weiss, Dubourg, et~al.]{pedregosa2011scikit}
Fabian Pedregosa, Ga{\"e}l Varoquaux, Alexandre Gramfort, Vincent Michel, Bertrand Thirion, Olivier Grisel, Mathieu Blondel, Peter Prettenhofer, Ron Weiss, Vincent Dubourg, et~al.
\newblock Scikit-learn: Machine learning in {P}ython.
\newblock \emph{Journal of Machine Learning Research}, 12:\penalty0 2825--2830, 2011.

\bibitem[Pinot et~al.(2020)Pinot, Ettedgui, Rizk, Chevaleyre, and Atif]{pinot2020randomization}
Rafael Pinot, Raphael Ettedgui, Geovani Rizk, Yann Chevaleyre, and Jamal Atif.
\newblock Randomization matters how to defend against strong adversarial attacks.
\newblock \emph{International Conference on Machine Learning}, pp.\  7717--7727, 2020.

\bibitem[Robey et~al.(2024)Robey, Latorre, Pappas, Hassani, and Cevher]{at-nonzero-game}
Alexander Robey, Fabian Latorre, George Pappas, Hamed Hassani, and Volkan Cevher.
\newblock Adversarial training should be cast as a non-zero-sum game.
\newblock \emph{International Conference on Learning Representations}, 2024.

\bibitem[Rosenfeld et~al.(2020)Rosenfeld, Winston, Ravikumar, and Kolter]{rosenfeld2020certified}
Elan Rosenfeld, Ezra Winston, Pradeep Ravikumar, and Zico Kolter.
\newblock Certified robustness to label-flipping attacks via randomized smoothing.
\newblock \emph{International Conference on Machine Learning}, pp.\  8230--8241, 2020.

\bibitem[Smola \& Sch{\"o}lkopf(1998)Smola and Sch{\"o}lkopf]{smola1998learning}
Alex~J Smola and Bernhard Sch{\"o}lkopf.
\newblock \emph{Learning with kernels}, volume~4.
\newblock Citeseer, 1998.

\bibitem[Steinhardt et~al.(2017)Steinhardt, Koh, and Liang]{steinhardt2017certified}
Jacob Steinhardt, Pang Wei~W Koh, and Percy~S Liang.
\newblock Certified defenses for data poisoning attacks.
\newblock \emph{Advances in Neural Information Processing Systems}, 30:\penalty0 3520–3532, 2017.

\bibitem[Szegedy et~al.(2013)Szegedy, Zaremba, Sutskever, Bruna, Erhan, Goodfellow, and Fergus]{szegedy2013intriguing}
Christian Szegedy, Wojciech Zaremba, Ilya Sutskever, Joan Bruna, Dumitru Erhan, Ian Goodfellow, and Rob Fergus.
\newblock Intriguing properties of neural networks.
\newblock \emph{ArXiv}, 1312.619, 2013.

\bibitem[Tavallali et~al.(2022)Tavallali, Behzadan, Alizadeh, Ranganath, and Singhal]{tavallali2022adversarial}
Pooya Tavallali, Vahid Behzadan, Azar Alizadeh, Aditya Ranganath, and Mukesh Singhal.
\newblock Adversarial label-poisoning attacks and defense for general multi-class models based on synthetic reduced nearest neighbor.
\newblock \emph{IEEE International Conference on Image Processing}, pp.\  3717--3722, 2022.

\bibitem[Von~Stackelberg(2010)]{von2010market}
Heinrich Von~Stackelberg.
\newblock \emph{Market structure and equilibrium}.
\newblock Springer Science \& Business Media, 2010.

\bibitem[Wang et~al.(2023)Wang, Tan, Guo, and Li]{wang-etal-2023-noise}
Song Wang, Zhen Tan, Ruocheng Guo, and Jundong Li.
\newblock Noise-robust fine-tuning of pretrained language models via external guidance.
\newblock \emph{Findings of the Association for Computational Linguistics}, pp.\  12528--12540, 2023.

\bibitem[Weerasinghe et~al.(2021)Weerasinghe, Alpcan, Erfani, and Leckie]{defending-svms}
Sandamal Weerasinghe, Tansu Alpcan, Sarah~M. Erfani, and Christopher Leckie.
\newblock Defending support vector machines against data poisoning attacks.
\newblock \emph{IEEE Transactions on Information Forensics and Security}, 16:\penalty0 2566--2578, 2021.

\bibitem[Wu et~al.(2023)Wu, Zhu, Liu, Liu, He, and Lyu]{survey-wu2023attacks}
Baoyuan Wu, Zihao Zhu, Li~Liu, Qingshan Liu, Zhaofeng He, and Siwei Lyu.
\newblock Attacks in adversarial machine learning: A systematic survey from the life-cycle perspective.
\newblock \emph{ArXiv}, 2302.09457, 2023.

\bibitem[Wu et~al.(2021)Wu, Hu, and Gu]{fast-scalable-adv-svm}
Huimin Wu, Zhengmian Hu, and Bin Gu.
\newblock Fast and scalable adversarial training of kernel svm via doubly stochastic gradients.
\newblock \emph{Proceedings of the AAAI Conference on Artificial Intelligence}, 35\penalty0 (12):\penalty0 10329--10337, 2021.

\bibitem[Xiao et~al.(2012)Xiao, Xiao, and Eckert]{adversarial-flip-svm}
Han Xiao, Huang Xiao, and Claudia Eckert.
\newblock Adversarial label flips attack on support vector machines.
\newblock \emph{European Conference on Artificial Intelligence}, pp.\  870--875, 2012.

\bibitem[Xiao et~al.(2015)Xiao, Biggio, Nelson, Xiao, Eckert, and Roli]{xiao2015support}
Huang Xiao, Battista Biggio, Blaine Nelson, Han Xiao, Claudia Eckert, and Fabio Roli.
\newblock Support vector machines under adversarial label contamination.
\newblock \emph{Neurocomputing}, 160:\penalty0 53--62, 2015.

\bibitem[Xu et~al.(2023)Xu, Sun, Goldblum, Goldstein, and Huang]{explore-exploit-db-dynamics}
Yuancheng Xu, Yanchao Sun, Micah Goldblum, Tom Goldstein, and Furong Huang.
\newblock Exploring and exploiting decision boundary dynamics for adversarial robustness.
\newblock \emph{International Conference on Learning Representations}, 2023.

\bibitem[Yang et~al.(2020)Yang, Rashtchian, Zhang, Salakhutdinov, and Chaudhuri]{yang2020closer}
Yao-Yuan Yang, Cyrus Rashtchian, Hongyang Zhang, Russ~R Salakhutdinov, and Kamalika Chaudhuri.
\newblock A closer look at accuracy vs. robustness.
\newblock \emph{Advances in Neural Information Processing Systems}, 33:\penalty0 8588--8601, 2020.

\bibitem[Yasodharan \& Loiseau(2019)Yasodharan and Loiseau]{yasodharan2019nonzero}
Sarath Yasodharan and Patrick Loiseau.
\newblock Nonzero-sum adversarial hypothesis testing games.
\newblock \emph{Advances in Neural Information Processing Systems}, 32:\penalty0 11, 2019.

\bibitem[Zhang et~al.(2017)Zhang, Bengio, Hardt, Recht, and Vinyals]{zhang2021understanding}
Chiyuan Zhang, Samy Bengio, Moritz Hardt, Benjamin Recht, and Oriol Vinyals.
\newblock Understanding deep learning requires rethinking generalization.
\newblock \emph{ArXiv}, 1611.03530, 2017.

\bibitem[Zhang et~al.(2021)Zhang, Zhu, Niu, Han, Sugiyama, and Kankanhalli]{gairat}
Jingfeng Zhang, Jianing Zhu, Gang Niu, Bo~Han, Masashi Sugiyama, and Mohan Kankanhalli.
\newblock Geometry-aware instance-reweighted adversarial training.
\newblock \emph{International Conference on Learning Representations}, 2021.

\bibitem[Zhang et~al.(2024)Zhang, Huang, Xu, and Bai]{zhang2024effective}
Peng-Fei Zhang, Zi~Huang, Xin-Shun Xu, and Guangdong Bai.
\newblock Effective and robust adversarial training against data and label corruptions.
\newblock \emph{IEEE Transactions on Multimedia}, 26:\penalty0 9477, 2024.

\bibitem[Zhao et~al.(2017)Zhao, An, Gao, and Zhang]{label-contamination-linear}
Mengchen Zhao, Bo~An, Wei Gao, and Teng Zhang.
\newblock Efficient label contamination attacks against black-box learning models.
\newblock \emph{International Joint Conference on Artificial Intelligence}, pp.\  3945--3951, 2017.

\bibitem[Zheng et~al.(2023)Zheng, Yan, Zhu, Chen, and Wu]{survey-zheng2023blackboxbench}
Meixi Zheng, Xuanchen Yan, Zihao Zhu, Hongrui Chen, and Baoyuan Wu.
\newblock Blackboxbench: A comprehensive benchmark of black-box adversarial attacks.
\newblock \emph{ArXiv}, 2312.16979, 2023.

\bibitem[Zhou et~al.(2012)Zhou, Kantarcioglu, Thuraisingham, and Xi]{zhou2012adversarial}
Yan Zhou, Murat Kantarcioglu, Bhavani~M. Thuraisingham, and Bowei Xi.
\newblock Adversarial support vector machine learning.
\newblock \emph{Knowledge Discovery and Data Mining}, pp.\  1059--1067, 2012.

\bibitem[Zhou et~al.(2019)Zhou, Kantarcioglu, and Xi]{zhou2019survey}
Yan Zhou, Murat Kantarcioglu, and Bowei Xi.
\newblock A survey of game theoretic approach for adversarial machine learning.
\newblock \emph{Wiley Interdisciplinary Reviews: Data Mining and Knowledge Discovery}, 9\penalty0 (3):\penalty0 e1259, 2019.

\bibitem[Zhu et~al.(2022)Zhu, Hedderich, Zhai, Adelani, and Klakow]{bert-noise}
D.~Zhu, Michael~A. Hedderich, Fangzhou Zhai, David~Ifeoluwa Adelani, and Dietrich Klakow.
\newblock Is {BERT} robust to label noise? {A} study on learning with noisy labels in text classification.
\newblock \emph{ArXiv}, 2204.09371, 2022.

\end{thebibliography}
